 \documentclass[10pt,onecolumn]{article}
 \usepackage{nccfoots}

 \author{
 Meltem Tatl{\i}\thanks{Meltem Tatl{\i} and Ali Tajer are with the Department of Electrical, Computer, and Systems Engineering, Rensselaer Polytechnic Institute.}  \quad Arpan Mukherjee\thanks{Arpan Mukherjee completed this work while he was with the Department of Electrical, Computer, and Systems Engineering, Rensselaer Polytechnic Institute. He is currently with the Department of Electrical and Electronic Engineering, Imperial College London.} \quad Prashanth L.A.\thanks{Prashanth L.A. is with the Department of Computer Science and Engineering, Indian Institute of Technology Madras.}\quad Karthikeyan Shanmugam\thanks{Karthikeyan Shanmugam is with Google Research India.}  \quad  Ali Tajer\footnotemark[1] }
\usepackage{ISG_style}
\usepackage{times}
\usepackage[T1]{fontenc}
\usepackage{url}
\usepackage{ifthen}
\usepackage{cite}
\usepackage[cmex10]{amsmath}
\usepackage{amssymb, amsfonts,dsfont,mathrsfs,bigints}
\usepackage{dsfont}
\usepackage{upgreek}
\usepackage{subcaption}
\usepackage{comment}
\usepackage{bigints}
\usepackage{times}
\usepackage{cite}
\usepackage{booktabs} 
\usepackage{tabularx}
\usepackage{nicefrac}       % compact symbols for 1/2, etc.
\usepackage{microtype}
\usepackage{minitoc}
\usepackage{footmisc} 
\usepackage{graphicx}
\usepackage{algorithm,algorithmic}
\usepackage{appendix}
\usepackage[table]{xcolor} 
\usepackage[colorlinks=true,allcolors=blue]{hyperref}
\usepackage{caption}
\usepackage{subcaption}
\usepackage{enumitem}

\newcommand{\holder}{H\"{o}lder}

\newcommand{\der}{{\rm d}}
\newcommand{\norm}[1]{\left\lVert#1\right\rVert}
\renewcommand{\F}{\mathbb{F}}
\newcommand{\G}{\mathbb{G}}
\renewcommand{\Q}{\mathbb{Q}}


\DeclareMathOperator*{\argmax}{arg\,max}
\DeclareMathOperator*{\argmin}{arg\,min}
\DeclareMathAlphabet\mathbfcal{OMS}{cmsy}{b}{n}

\title{\textbf{\huge Preference-centric Bandits:}\\ {\bf \Large Optimality of Mixtures and Regret-efficient Algorithms}}
\date{}

\usepackage{xspace}
\usepackage{todonotes}

\begin{document}
\allowdisplaybreaks
\maketitle

\begingroup
  \renewcommand\thefootnote{}%
  \footnotetext{{
   The results have been partially presented at the International Conference on Artificial Intelligence and Statistics (AISTATS) 2025.}}
\endgroup

\begin{abstract}

The objective of canonical multi-armed bandits is to identify and repeatedly select an arm with the largest reward, often in the form of the expected value of the arm's probability distribution. Such a utilitarian perspective and focus on the probability models' first moments, however, is agnostic to the distributions' tail behavior and their implications for variability and risks in decision-making.
For instance, in high-stakes applications like clinical trials or autonomous control, an arm with a slightly lower mean but significantly lower probability of catastrophic outcomes may be far preferable.
This paper introduces a principled framework for shifting from expectation-based evaluation to an alternative reward formulation, termed a {\em preference metric} (PM). The PMs can place the desired emphasis on different reward realization and can encode a richer modeling of preferences that incorporate risk aversion, robustness, or other desired attitudes toward uncertainty. A fundamentally distinct observation in such a PM-centric perspective is that designing bandit algorithms will have a significantly different principle: as opposed to the reward-based models in which the optimal sampling policy converges to repeatedly sampling from the single best arm, in the PM-centric framework the optimal policy converges to selecting a mix of arms based on a specific and well-defined mixing weights. Designing such \emph{mixture policies} departs from the principles for designing bandit algorithms in significant ways, primarily because of uncountable mixture possibilities. The paper formalizes the PM-centric bandit framework and offers two broad classes of algorithms for the \emph{horizon-dependent} and \emph{infinite-horizon} (anytime) settings. These algorithms are capable of learning and tracking mixtures in a regret-efficient fashion. These algorithms have two distinctions from their canonical counterparts: (i) they involve an estimation routine {\em} to form reliable estimates of optimal mixtures, and (ii) they are equipped with {\em tracking mechanisms} to navigate arm selection fractions to track the optimal mixtures. These algorithms' regret guarantees are investigated under various algebraic forms of the PMs. Finally, empirical experiments are conducted to demonstrate the regret performance of the proposed algorithms.

\end{abstract}

%\paragraph{Keywords.} Bandits, Choquet integral, mixture policy,  preference metric.

\section{Introduction} 
\label{sec:introduction}

The canonical objective of multi-armed bandits (MABs) is to design a sequence of experiments to identify or repeatedly select the arm with the highest expected reward. While this \emph{utilitarian} view provides a simple and meaningful metric to compare and rank the arms, relying only on the expected value abstracts away critical information about the underlying reward distributions. Specifically, summarizing each arm by its expected value --- computed as the integral of the complement of its cumulative distribution function (CDF) --- ignores aspects such as risk, variability, and tail behavior, which can be crucial in real-world decision-making. For instance, in high-stakes applications like clinical trials, finance, or autonomous control, an arm with a slightly lower mean but significantly lower probability of catastrophic outcomes may be far preferable. Motivated by this observation, we propose a principled shift from expectation-based evaluation to distorted expectation, where we apply a transformation (or distortion) to the complement CDF before integrating. Such distortion places the desired emphasis on different reward realization, and has the versatility to encode a richer modeling of preferences that incorporate risk aversion, robustness, or other desired attitudes toward uncertainty. By embedding distortion functions directly into the reward criterion, we open new avenues for adaptive learning algorithms that align more closely with complex real-world objectives. 

This approach naturally balances the tradeoff between utility and other metrics of interest --- such as risk sensitivity, robustness to outliers, and tail behavior --- by shaping the decision criterion through the distortion function. Unlike averaging, which treats all realizations uniformly, distorting the complement CDF allows us to amplify or attenuate the influence of specific parts of the distribution (e.g., penalizing low outcomes or overweighting high-confidence gains). By carefully selecting the distortion function, one can interpolate between purely utilitarian strategies and highly conservative or risk-seeking policies, making this framework highly adaptable to diverse application needs.

A natural and general way to formalize these distortions is through the \emph{Choquet} integral, which provides a framework for aggregating outcomes with respect to non-additive measures. Unlike the standard expectation, which is a linear operator over probabilities, the Choquet integral allows for weighting different regions of the outcome space non-linearly, making it particularly well-suited for modeling preferences, encoding pessimism, optimism, or ambiguity aversion directly into the evaluation of an arm's performance. This approach enables nuanced tradeoffs between utility and robustness, allowing the decision-maker to account for rare but impactful events or to safeguard against worst-case scenarios, all within a unified and interpretable mathematical framework.

\paragraph{Preference Metric.} To formalize these, we adopt a {\em distortion function} function  $h:[0,1]\to [0,1]$ that satisfies $h(0)=0$ to apply desired distortions to any probability distribution. Based on this, corresponding to any given CDF $\Q$, we define  
\emph{preference metrics} (PMs) of $\Q$, denoted by $U_h$ as the \emph{signed Choquet integral}, which was first introduced in the capacity-theory setting~\cite{Choquet1954}, later given a decision-theoretic integral representation~\cite{Schmeidler1986}, and more recently studied in risk-management concept~\cite{Wang2020}. The PM $U_h$ associated with $\Q$ is specified by
\begin{align}
\label{eq:wang_def}
    U_h(\Q) \triangleq  \int_{-\infty}^{0} & \Big(h\big(1- \Q(x)\big)  - h(1)\Big) \, \der x  + \int_{0}^{\infty} h\big(1-\Q(x)\big) \, \der x \ .
\end{align}
The distortion function $h$ enables high versatility in specifying and unifying various notions of preference. Given this notion of PM, it can replace the conventional expected value as the desired metric. Subsequently, the objective of a bandit algorithm becomes to identify the arm with the largest PM. We note that setting $h(u)=u$, the PM associated with distribution $\Q$ reduces to its expected value. Hence, the PM-centric bandit framework subsumes the conventional mean-based bandit framework. The choice of PM depends on the nature of the problem, and it can be designed to emphasize certain distributional characteristics, e.g., tail behavior, over others, e.g., the mean.  To illustrate the effect of distortion, we provide two examples.
\begin{example}[Inverted-$S$]
A procurement officer repeatedly selects among suppliers whose delivery times have heavy‑tailed delays due to strikes, weather, or geopolitical shocks. Adopting the inverted-$S$ distortion function
\begin{align}
h(u)=\exp \Big[-\big (- \ln u\big)^{\beta}\Big]\ ,
\qquad \mbox{where} \quad u\in(0,1]\  ,\;\beta>0 \ ,
\end{align}
balances a tradeoff between utility and delay and yields a smooth, prospect‑theoretic metric that accentuates precisely those parts of the loss–gain distribution the decision‑maker values. Specifically, \vspace{-.05 in}
 \begin{itemize}
  \item With $\beta<1$, the distortion function over‑weights the small probabilities of extreme delays, steering the algorithm toward suppliers with shallower tail risk, even at slightly higher mean cost
  \item If later contracts guarantee compensation for catastrophic delays, setting $\beta>1$ under‑weights those rare events, allowing the agent to prioritize lower expected cost instead.
\end{itemize} 
Figure~\ref{fig:inverted_S} showcases applying this distortion function on light-tailed (Gaussian and exponential) and heavy-tailed (lognormal) distributions.
\end{example}

\begin{figure}[t]
    \centering
\includegraphics[width=1\linewidth]{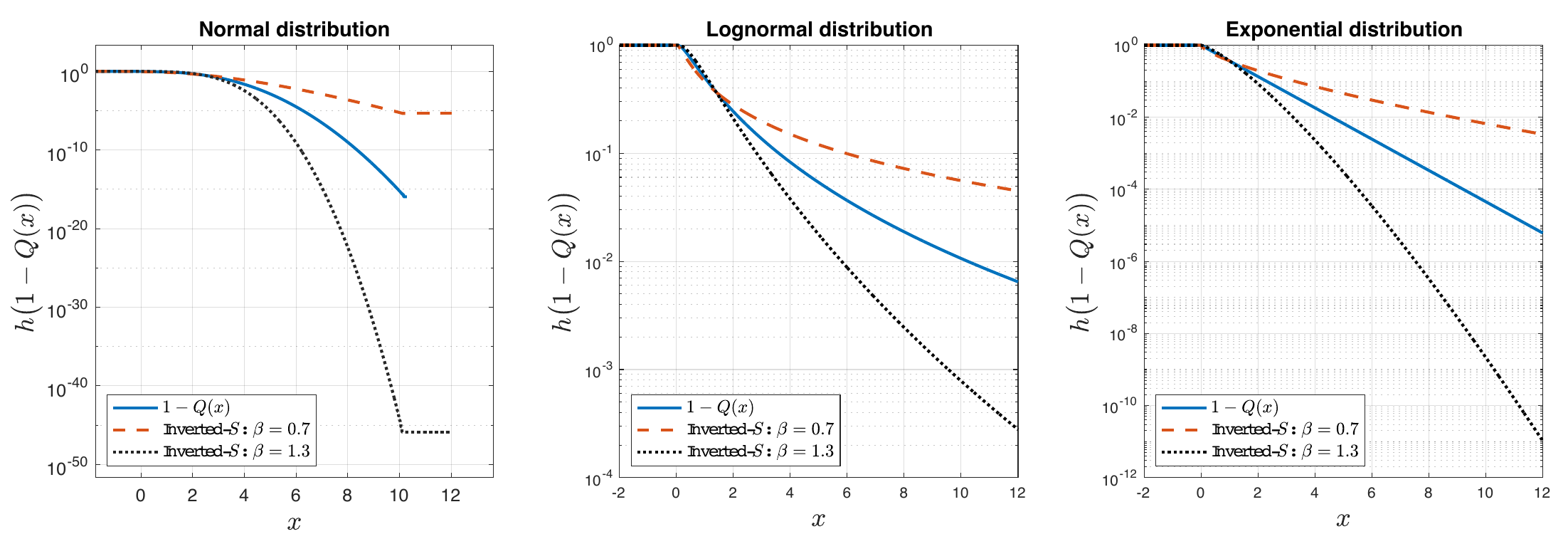}
    \caption{Distorted complement CDFs by Inverted-$S$ distortion.}
    \label{fig:inverted_S}
\end{figure}

There exists extensive literature on PMs.  
These include $L$-functionals in robust location estimation~\cite{huber2009robust}, Yaari’s dual utilities in decision theory~\cite{yaari1987dual}, distorted premium principles in insurance \cite{denneberg1994nonadditive}, and distortion risk measures in finance \cite{kusuoka2001law}. Some widely used examples of PMs are risk measures such value-at-risk (VaR), expected shortfall (ES) \cite{artzner99}, and conditional VaR (CVaR) \cite{rockafellar2000optimization}, Gini shortfall~\cite{furman2017gini}; measures of variability, such as deviation measures \cite{rockafellar2006generalized},  mean-median deviation, inter-quantile range, Wang's right-tail deviation, inter-expected shortfall, Gini deviation \cite{Gini1912}, cumulative Tsallis past entropy \cite{Cali2017CTE}.

\paragraph{PM-centric Optimal Policy.} The utilitarian MAB settings often adopt the tacit assumption that there exists a {\em single optimal} arm that maximizes the desired reward function. This premise, subsequently, guides the design of regret-efficient algorithms focused on identifying and repeatedly selecting the single best arm.  In a sharp contrast, in the PM-centric framework, we show that the assumption that a single arm achieves the optimal PM is not necessarily valid. Specifically, we show that the optimal policy can consist of carefully randomizing among all arms. Such a mixture policy yields a PM that strictly dominates those yielded by each of the arms individually. The following lemma  

\begin{lemma}[Gini Deviation]
\label{example utility}
Consider a two-arm Bernoulli bandit model. For a given $p\in[0,1]$, the arms'  distributions are ${\rm Bern}(p)$ and ${\rm Bern}(1-p)$, with CDFs $\F_1$ and $F_2$, respectively. For distortion function $h(u)=u(1-u)$, we have  
\begin{align}
    U_h \left(\frac{1}{2}\F_1 +\frac{1}{2} \F_2\right) > \max\big\{ U_h(\F_1)\; , \;  U_h(\F_2)\big\}\ .
\end{align} 
\end{lemma}
\begin{proof}
    See Appendix~\ref{proof:lemma_example}. 
\end{proof}
This result emphasizes the important observation that the optimal value of the desired PM is achieved when the samples are generated by the mixture distribution $\frac{1}{2}(\F_1 + \F_2)$. Operationally, this means that to achieve the optimal performance, an MAB algorithm should randomize sampling from the two arms. This work is the first to establish the need for randomized sampling to achieve a desired metric.

Throughout the rest of the paper, when the optimal performance is achieved by selecting a mixture of arms, we refer to that as a {\em mixture} policy. Otherwise, when it follows the convention of one arm dominating others, we refer to that as a {\em solitary} arm policy. Whether the optimal policy is a mixture or solitary, critically hinges on the choice of the distortion function $h$, and whether a PM might involve mixtures depends on the {\bf distortion function}. Categorically, we establish that under \emph{convex} distortion functions, the optimal solution is always a solitary arm. On the other hand, for any \emph{strictly concave} distortion function, there is at least one bandit model for which the optimal solution is a mixture policy.

\newpage
\paragraph{Technical Challenge of Learning Optimal Mixtures.}  Designing an MAB algorithm that can randomize sampling from different arms according to the optimal randomization probabilities is fundamentally from the algorithms with solitary arm policies. Specifically, searching for a solitary arm is guided by sequentially identifying and selecting an arm that optimizes a desired metric (e.g., upper confidence bound). 
In contrast, a mixture policy introduces two new challenges to the arm-selection policy. 
\begin{itemize}
    \item {\em Estimating mixing coefficients:} In a $K$-arm MAB, the optimal mixture will be specified by $K$ continuous-valued mixing coefficients. Determining these coefficients inevitably necessitates incorporating routines to reliably estimate them to the desired levels of accuracy. Such estimation routines generally have no precedence in designing bandit algorithms, which generally either lack estimation routines or, even when they do, have auxiliary natures in support of the terminal decision (isolating one arm). Since the optimal set of mixing coefficients lies in a continuous space with infinite possibilities, the nature of mixture policies becomes fundamentally distinct from the solitary arm policies with a decision space with finite possibilities (arms).

    \item {\em Tracking the optimal mixture:} The second dimension is the need for an efficient algorithm to track the optimal mixture distribution. This necessitates balancing arm selections over time so that their mixture, in aggregate, conforms to the optimal one.  
      
\end{itemize}

\subsection{Literature Review}

There are two main lines of research that have considered bandit objectives beyond optimizing the first moment. 
\paragraph{Empirical Distribution Performance Measures Framework. } In the MAB literature, the work most closely related to ours is the empirical distribution performance measure (EDPM) framework~\cite{cassel2018general}. EDPM shifts focus from mean‑based bandits to a richer family of path‑dependent objectives which depend on the entire sequence of observed rewards rather than just their average. This framework is primarily motivated by risk criteria such as the Sharpe ratio, value‑at‑risk (VaR), and conditional value‑at‑risk (CVaR). These metrics highlight distributional features like variance and tail losses. Optimizing such measures results in having oracles that may change arms over time, unlike the fixed‑arm oracle in the conventional setting. This important observation made by~\cite{cassel2018general} precisely aligns with our key observation that PM-centric optimality policies do not necessarily commit to a single arm and might have mixture structures. Recognizing such a shift, the EDPM framework designs an algorithm and specifies systematic conditions under which a solitary arm policy can \emph{approximate} a mixture policy. When the path-dependent objective is (quasi-)convex in the space of arms' distributions, expectedly, the approximate policy becomes exact. 

Our PM-centric framework has the following major distinctions from the EDPM framework. {\em Objective:} We shift away from having \emph{approximate} solitary policies to characterizing the exact policies. To this end, the emphasis will be on \emph{concave} settings, for which the optimal policies are unknown --- and we show that they do not necessarily admit solitary policies. {\em Assumptions \& Scope:} There are subtle differences in assumptions on the continuity model of the performance/preference metrics that have important implications on the scope of the bandit settings. Specifically, the EDPM assumes Lipschitz continuity for its performance metric, which excludes important classes of objectives, such as distortion risk measures (DRMs). In contrast, the PM-centric framework assumes only \holder~continuity, accommodating a broader class of objectives, including DRMs.
{\em Mixture algorithm:} We offer algorithms that implement mixture policies, shifting away from the conventional solitary policies. The PM-centric algorithms are designed to determine and track the optimal mixture of arm selections.
{\em Regret guarantees:} For the measures such as mean and CVaR, we match their regret bounds using the ETC-based algorithm under the condition that the sub-optimality gap used in the ETC-based algorithm can be expressed as a function of the discretization level.   While they have the assumption of boundedness, we consider sub-Gaussian bandit models. However, the ETC-based algorithm assumes the gap knowledge, while their algorithm is UCB-based.

\paragraph{Risk-sensitive Bandits.} There is a rich line of research on risk sensitivity in MABs. Unlike the canonical objective of maximizing expected reward, risk-sensitive approaches consider finer distributional properties that capture deviation and tail behavior. Widely used risk measures include variance, VaR, CVaR, and quantile-based criteria. Some recent developments on various notations of risk-sensitivity include~\cite{ chang2022,baudry2021optimal, tamkin2019cvar,tan2022cvar,kolla2023bandit,liang2023distribution }. In Section~\ref{sec: risk_application}, we discuss the connection of our PM-centric framework to the existing literature on risk-sensitive bandits. Furthermore, there are general frameworks for viewing and analyzing risk-sensitivity in decision-making, such as the cumulative prospect theory (CPT)~\cite{tversky1992advances} and the framework based on DRMs~\cite{WirchHardy2003Distortion}.  The study in~\cite{prashanth2022wasserstein} introduces a framework that unifies various MAB risk measures, CPT-inspired bandits, and DRM-based bandits. It specified a Wasserstein distance-based approach to delineate concentration bounds for estimating empirical risk and incorporating them in decision-making. 

Our PM-centric approach has the following major differences from the existing studies on risk-sensitive bandits. {\em Generality:} Our approach recovers the existing literature on risk-sensitive bandits as special cases. In particular, when the distortion function is chosen appropriately, the PM-based framework subsumes the DRMs. {\em Objective:} The risk-sensitive bandits literature focuses primarily on \emph{minimizing} different notions of decision variation. Although such an approach penalizes dispersion to avoid tail risk, in some domains we explicitly require emphasizing and valuing unpredictability or convex payoffs in volatility. Designing bandit objectives that reward high variance, enabling favoring variance without forcing an unnatural mean‑maximization surrogate. Some common examples include volatility trading, in which a trader buys an at‑the‑money straddle on an index and earns profit proportional to the absolute price movement. In this case, large swings in either direction are desirable. Another example is wildlife protection or airport security patrols, for which predictability is a weakness. In these examples, an adversary exploits regular patterns. As a counter-measure, the patrols want to monitor routes whose occurrence probabilities are spread out, i.e., have high entropy or variance over actions, to keep the adversary uncertain. {\em Mixture algorithm:} The existing literature including the works ~\cite{ chang2022,baudry2021optimal, tamkin2019cvar,tan2022cvar,kolla2023bandit} %,liang2023distribution,prashanth2022wasserstein} 
 assumes the existence of an optimal solitary arm. We show that for PMs with convex distortion functions, the optimal solution is a solitary arm. In contrast, when the PMs have strictly concave functions, the optimal solution might be a mixture. Our analysis does not rely on any assumptions about the nature of the optimal solution. As a result, our framework not only recovers settings in which the optimal solution is a solitary arm, but also generalizes settings where there is mixture optimality. {\em Regret guarantees:} Our general regret guarantees either recover some of the existing regret guarantees (e.g., CVaR) or improve the known guarantees (e.g., proportional hazard transform).

\vspace{-.1 in}
\subsection{Contributions} 
\vspace{-.05 in}
We propose a PM-centric framework the objective of which is designing a sequence of arm selections such that a notion of \emph{preference} --- beyond the convention of the expected values of arms' distributions --- is maximized. Notably, in this PM-centric framework, an arm's utility may no longer be characterized by its mean value. Rather, the quality of an arm is measured by functionals that depend on the entire arm distribution. This framework is highly versatile where by properly selecting the PM and the associated distortion functions, as its special cases, it subsumes the conventional mean-based bandits, the EDPM framework, and the existing approaches to risk-sensitive bandits. In this paper, we address the following overarching questions: What are the key challenges introduced by the PM-centric framework for stochastic bandits? Can we design regret-efficient algorithms for stochastic bandits for optimizing PMs?

In pursuing the above question, we uncover a novel and important property of PMs: for many PMs, the optimal utility is achieved not by one arm's distribution, but a distribution that is the mixture of multiple arms' distributions. The important implication of this observation is that an algorithm for achieving the optimal regret will not converge to identifying one arm as the best one, and instead converges to a policy that selects a group of arms and selects each overtime according to a very specific frequency.  This is a significant departure from the canonical bandit framework, which inherently assumes the existence of a solitary optimal arm. This shift is consequential --- most existing algorithms for regret minimization in stochastic bandits are fundamentally designed around the tacit premise that the optimal performance achieved by an oracle is a solitary arm policy. As a result, these algorithms fail to adapt to the PM-centric setting and incur {\bf linear regret} when applied to settings with optimal mixture policies. 

In what follows, we discuss the fundamental challenges posed by the PM-centric framework, the limitations of adapting state-of-the-art approaches, and the key technical obstacles in designing and analyzing regret-efficient algorithms under this generalized framework.

\vspace{.1 in}
\noindent\emph{PM-centric Bandit Algorithms.}
We begin by highlighting the new design mechanisms that need to be incorporated to transcend from selecting solitary arms and accommodate for learning mixtures over arms. Subsequently, we provide brief overviews of the proposed algorithms for the finite-horizon and infinite-horizon (anytime) settings. Finally, we discuss the technical challenges associated with regret analysis that do not have precedence in the bandit literature. 

\vspace{.1 in}
\noindent\emph{PM-centric Horizon-dependent Algorithms.} We first consider the class of finite-horizon algorithms, in which the horizon is known to the learner. This equips the learner with a fine-grained control over constructing confidence sequences to measure the estimation fidelity, such that undesirable error events can be controlled as a decreasing function of the horizon. We provide adaptations of the explore-then-commit (ETC) and upper-confidence bound (UCB) algorithms in the PM-centric framework. These algorithms have major departures from their canonical counterparts --- the most significant distinction pertaining to estimating and tracking mixture distributions instead of individual arm distributions.

% \end{itemize}
\vspace{.1 in}
\noindent\emph{PM-centric Anytime Algorithm.} In a wide range of applications, decisions are not bound by fixed horizons (e.g., recommendation systems).  
Based on UCB principles, we propose an anytime algorithm in Section~\ref{sec: anytime}, in which removing the dependence on budget 
creates the following technical challenge.
 A critical step in PM-centric algorithms is a routine for estimating the coefficients needed for mixing the arms' distributions to generate the optimal mixture. In the horizon-dependent algorithms, the learner can choose the accuracy level for mixing coefficients estimates as a function of the horizon, ensuring tight regret bounds.  Intuitively, such accuracy levels are inversely related to the horizon. In the anytime algorithm, specifying the desired estimation accuracy faces the challenge that the horizon is unknown.  
    In our anytime algorithm design in Section~\ref{sec: anytime}, we let the learner set the needed accuracy for mixing coefficient estimates is refined successively as more data is collected. 

%\todop{Add a paper organization para here.}
The remainder of this paper is organized as follows. Section \ref{sec:setting} introduces the PM-centric bandit framework, along with the necessary definitions, notation, and assumptions. In Section \ref{sec:algorithm}, we present horizon-dependent algorithms and discuss how they differ from their canonical counterparts. Section \ref{sec:analysis} characterizes regret bounds for these horizon-aware methods. Section \ref{sec: anytime} then introduces an anytime variant and establishes its regret guarantees. We present empirical evaluations of all proposed algorithms in Section \ref{sec: experiments}. Finally, Section \ref{sec: risk_application} discusses the application of our framework to risk-sensitive bandits and the relevance of our results to the existing literature.

\section{PM-centric Bandit Framework}
\label{sec:setting}

\paragraph{Bandit Model.} Consider a $K$-armed unstructured stochastic bandit.  Each arm $i\in [K]\triangleq\{1,\cdots,K\}$ is endowed with a probability space $(\Omega,\mcF,\F_i)$, where $\mcF$ is the $\sigma$-algebra on $\Omega\subseteq\R_+$ and $\F_i$ is an {\em unknown} probability measure.
Accordingly, define 
$\F \triangleq  \{\F_i:i\in[K]\}$ as the set of all probability measures.
At time $t\in\N$, a policy $\pi$ selects an arm $A_t\in[K]$ and the arm generates a stochastic sample $X_t$ distributed according to $\F_{A_t}$. Denote the sequence of actions, observations, and the $\sigma$-algebra that policy $\pi$ generates up to time $t\in\N$ by 
\begin{align}
    \mcX_t\triangleq\left ( X_1,\cdots, X_t\right )\ , \qquad \mcA^\pi_t \triangleq\left( A_1,\cdots,A_t\right )\ , \qquad \mbox{and} \quad \mcH^\pi_t\triangleq\sigma\left (A_1,X_1,\cdots,A_{t-1},X_{t-1}\right )\ , 
\end{align}
respectively. Corresponding to any bandit instance $\bnu$ and policy $\pi$, $\P_{\bnu}^{\pi}$ denotes the push-forward measure on $\mcH^\pi_t$, and $\E_{\bnu}^{\pi}$ denotes the associated expectation. The sequence of independent and identically distributed (i.i.d.) rewards generated by arm $i\in[K]$ up to time $t\in\N$ is denoted by  $\mcX_t(i)\triangleq \{X_t : A_t= i\}$. We define $\tau_t^\pi(i)\triangleq |\mcX_t(i)|$ as the number of times that policy $\pi$ selects arm $i\in[K]$ up to time $t$.

\paragraph{Preference Metric.} We focus on evaluating a measure of preference associated with the probability measure associated with each arm. 
%The notion of a PM has been introduced in~\eqref{eq:wang_def}, and we define a 
The performance metrics are tailored to the specific application domains of interest. Specifically, based on the definitions of distortion function $h$ and PM $U_h$ defined in~\eqref{eq:wang_def}, corresponding to any given policy $\pi$ at time $t$, the overall preference associated with the sequence of arm selections $\mcA_t^\pi$ by a policy $\pi$ is specified by 
\begin{align}    
\label{eq:U} U_h\left(\sum_{s=1}^t\sum_{i\in[K]}\frac{1}{t} \mathds{1}\{A_s=i\} \; \F_i\right) = U_h\left(\sum_{i\in[K]}\frac{1}{t}\tau_t^\pi(i) \; \F_i\right).
\end{align}
This is a strict generalization of the notion of rewards in the canonical MAB framework, which can be recovered from the PM induced by setting the distortion function to $h(p) = p$.

\paragraph{Oracle Policy.} Specifying an oracle that serves as a benchmark for policy performance. Such an oracle accurately identifies the optimal sequence of arm selections $\{A_t:t\in\N\}$. Given the structure in~\eqref{eq:U}, designing an oracle policy is equivalent to determining the optimal mixing of the CDFs, where the mixing coefficients up to time $t$ are $\big\{\frac{1}{t}\tau_t^\pi(i):i\in[K]\big\}$ . For a bandit instance $\bnu\triangleq(\F_1,\cdots,\F_K)$, the vector of optimal mixture coefficients is denoted by
\begin{align}
\label{eq:optimal alpha}
\balpha^{\star}_{\bnu}\;\in\;\argmax\limits_{\balpha\in\Delta^{K-1}}\;U_h\Big ( \sum\limits_{i\in[K]}\alpha(i) \; \F_i\Big )\ ,
\end{align}
where  $\Delta^{K-1}$ denotes a $K$-dimensional simplex. When clear from the context, we use the shorthand  $\balpha^{\star}$ for $\balpha^{\star}_{\bnu}$.
We note that the oracle considered in~\eqref{eq:optimal alpha} is omniscient, i.e., it is aware of all the distributions $\F =  \{\F_i:i\in[K]\}$ and the mixing policy $\balpha^{\star}_{\bnu}$ that optimizes the PM. When the mixture coefficients vector $\balpha^{\star}$ is 1-sparse, the optimal policy becomes a solitary (single-arm) policy, and otherwise it is a mixture policy.

Throughout the rest of the paper, for any given vector $\balpha$ and set of CDFs $\F =  \{\F_i:i\in[K]\}$, we use the definition
\begin{align}
    \label{eq:V}
    V(\balpha, \F) \triangleq  U_h\Big ( \sum\limits_{i\in[K]}\alpha(i) \; \F_i\Big )\ .
\end{align}\vspace{-.1 in}
\paragraph{Mixture-centric Regret.} 
Motivated by the observation that for a given distortion function $h$, the optimal sampling rule may be achieved by a \emph{mixture} of arm distributions, we provide a general mixture-centric framework in which bandit policies can take mixture forms, subsuming the solitary policies as a special case when the mixing mass is placed on one arm. By the oracle policy's definition $\balpha^{\star}_{\bnu}$ in~\eqref{eq:optimal alpha}, for a bandit instance $\bnu$, 
our objective is to design a bandit algorithm with minimal regret with respect to the PM achieved by the oracle policy. Hence, for a given policy~$\pi$ and horizon $T$, we define the  regret as the gap between the PM achieved by the oracle policy and the \emph{average} PM achieved by $\pi$, i.e., 
\begin{align}
\label{eq:regret}
\mathfrak{R}_{\bnu}^\pi(T)\;\triangleq\;  V(\balpha^{\star}, \F) 
    -      \E_{\bnu}^\pi\Big[V\Big(\frac{1}{T}\btau_T^\pi, \F \Big)\Big]\ .
\end{align}

\paragraph{Technical Assumptions.} 
For each arm $i\in[K]$, let $\mcP(\Omega)$ denote the set of all probability measures with finite first moment on $\Omega$. 
We focus on distributions $\{\F_i:i\in[K]\}$ that are $1$-sub-Gaussian and belong to the metric space $(\mcP(\Omega),\norm{\cdot}_{\rm W})$, where $\norm{\G-\mathbb{S}}_{\rm W}$ denotes the $1$-Wasserstein distance between distributions $\G$ and $\mathbb{S}$. We define $W$ as the maximum ratio between Wasserstein and total variation distances between any two mixture of arms, i.e.,
\begin{align}\label{eq:W}
    W \triangleq \max \limits_{\balpha \neq \bbeta \in \Delta^K} \frac{1}{\lVert \balpha - \bbeta \rVert_1}\Big\|\sum_{i=1}^K \alpha_i \F_i-\sum_{i=1}^K \beta_i \F_i\Big\|_{\rm W}\ .
\end{align}  

\begin{theorem}[Upper bound on $W$] \label{theorem:W}
Let $\{\mathbb{F}_i : i\in[K]\}$ be probability measures on $\Omega\subseteq \R$ that are $1$-sub-Gaussian and define $W$ as in~\eqref{eq:W}.  We have $W\;\leq\;\sqrt{2\pi}$.
\end{theorem}
\begin{proof}
    See Appendix \ref{Appendix:W_finitess}.
\end{proof}
Next, define $\Xi$ as the convex hull of the set $\{\F_i: i\in[K] \}$ and assume $U_h$ satisfies the following notions of continuity. 

\begin{definition}[H\"{o}lder continuity]
\label{assumption:Holder} The PM $U_h$ is said to be H\"{o}lder continuous with exponent $q\in(0,1]$, if for all distributions $\G_1,\G_2\in(\Xi,W_1)$, there exists a finite $\mcL \in\R_+$ such that
    \begin{align}
    \label{eq:Holder}
        U_h(\G_1) - U_h(\G_2) \; \leq \; \mcL\norm{\G_1-\G_2}_{\rm W}^q.
    \end{align}
\end{definition}
Many of the widely-used PMs are \holder-continuous. 
In Table~\ref{table:beta_values}, we present a list of PMs, and specify their associated \holder~exponents $q$.

\begin{lemma}[Finiteness of $U_h$]
\label{lemma:U}
Suppose $h$ is \holder-continuous. Then, $U_h(\G)$ is  finite for any $\G\in(\Xi,W_1)$.
\end{lemma}
\begin{proof}
    See Appendix \ref{Appendix:lemma_U_proof}. 
\end{proof}
For a given {\holder}-continuous function $h$, we denote the maximum value of the associated PM by $B_h\in \R_{+}$.

\section{PM-Centric horizon-dependent Algorithm Design}
\label{sec:algorithm}

We begin by analyzing the horizon-dependent setting in which the learner knows the fixed horizon and leverages that for designing algorithms. Two fundamental bottlenecks make canonical algorithms inapplicable in the PM-centric setting. %

\begin{itemize}
    \item {\bf Estimating mixing coefficients versus arm means:} Canonical algorithms, such as ETC- and UCB-based algorithms, form estimates of the arm means. Based on the confidence sets constructed around these estimates, these algorithms have various arm selection routines for balancing exploration and exploitation. Since the PM depends on the entire description of the arms' CDFs (not just their moments), the learner needs to estimate these CDFs up to a desired accuracy to be able to balance exploration and exploitation.     Subsequently, we face a functional estimation to determine the mixing coefficients as a function of the estimated arms' CDFs. Both of these steps have clear distinctions that make the canonical algorithms inapplicable in the mixture-centric setting.

    \item {\bf Tracking mixtures:} Once the optimal mixtures are estimated, there needs to be a mechanism that translates the mixing coefficient estimates into arm selection decisions. These decisions have to balance two concurrent objectives. First, there is a need to improve the estimates of the mixing coefficients. This may be achieved by exploring the arms to ensure accurate estimates of the CDFs, albeit at the cost of increased regret. Secondly, the arm selection rules need to ensure that, in aggregate, their selections track the mixing estimates. This necessitates an explicit mixture tracking rule that is regret-efficient. 
    \end{itemize}
In this section, we design algorithms in the PM-centric framework that are regret-efficient. We observe that adaptations of the canonical ETC and UCB algorithms can lead to sublinear regret in the PM-centric setting, each having its own performance or viability (model information) advantages. 
To start, we provide a few definitions and notations that will be used to describe and analyze the algorithms. We use the shorthand $\pi\in\{{\rm E,U,C}\}$ to refer to the ETC-type, UCB-type, and a computationally efficient UCB-type policies, respectively. An important routine in these algorithms is estimating the arms' CDFs accurately.
 We denote the empirical estimate of $\F_i$ at time $t$ generated by policy $\pi$ by
\begin{align}
\label{eq:empirical_CDF}
    \F_{i,t}^\pi(x)\triangleq \frac{1}{\tau_t^\pi(i)}\sum\limits_{s\in[t] : A_s = i} \mathds{1}\left \{X_s \leq x \right\}\ . 
\end{align}
Knowing the horizon equips the learner with the following key advantage. 
Analyzing the accuracy of these estimates and their impacts on the arm-selection decisions involves constructing confidence sets that measure the uncertainty of the CDF (and PM) estimates. In designing horizon-dependent confidence sets, the learner has a fine-grained control over undesirable error events in CDF estimation measured in terms of the horizon $T$. This leads to stronger performance guarantees compared to the settings in which the horizon information is unavailable. To analyze and control the undesirable error events in CDF estimates, we leverage the following lemma, which provides a concentration bound on these empirical CDFs in the $1$-Wasserstein metric. 
\begin{lemma}
\label{lemma: concentration}
    For any policy $\pi$ and any $y\in\R_+$ we have
    \begin{align}
\label{eq:meta_concentration}
    \P_{\bnu}^\pi & \Big( \norm{\F_{i,t}^\pi  -\F_i}_{\rm W}  >   y \Big) \leq 
    2\exp\Bigg( -\; \frac{\tau^{\pi}_t(i)}{256 {\rm e}} \Big(y- \frac{512}{\sqrt{\tau^{\pi}_t(i)}}\Big)^2\Bigg)\ .
\end{align}
\end{lemma}
\begin{proof}
    See~\cite[Lemma 8]{prashanth2022wasserstein}.
\end{proof}
Using~\eqref{eq:meta_concentration}, we can construct distribution confidence sets to have a horizon-dependent control over error events. Specifically, choosing $y$ such that the right-hand side of~\eqref{eq:meta_concentration} falls below $\frac{1}{T^2}$ ensures that aggregate error events accumulate in a regret which is sublinear in $T$. 

\paragraph{Discretizing Mixing Coefficients.} Given a finite number of samples, the fidelity of the estimates of $\balpha^{\star}$ cannot be made arbitrarily high (information-theoretic limits). Furthermore, even if the exact $\balpha^{\star}$ is known, they can generally take arbitrary irrational values while sampling frequencies are always rational values. Hence, estimating $\balpha^{\star}$ beyond a certain accuracy level is not only infeasible, but also is not beneficial. We define $\varepsilon$ to specify the desired level of fidelity for each mixing coefficient. Accordingly, we specify $\Delta^{K-1}_{\varepsilon}$ by uniformly discretizing each coordinate of $\Delta^{K-1}$ into intervals of length $\varepsilon$ (when $\frac{1}{\varepsilon}$ is not an integer, we make up for the deficit/excess of the weights in the last coordinate). The precise choices of discretization will be slightly different in the ETC and UBC algorithms and will be specified in Sections~\ref{Section:RS_ETC_M}~and~\ref{RS_UCB_M}, respectively. For a given discretized simplex $\Delta^{K-1}_{\varepsilon}$, we define the top three discrete mixing coefficients chosen from $\Delta^{K-1}_{\varepsilon}$ as
\begin{align}
\label{eq:discrete optimal mixture}
    \ba^{(1)} \triangleq \argmax\limits_{\ba\in\Delta^{K-1}_{\varepsilon}}\; V(\ba, \F)\ ,\quad \ba^{(2)} \triangleq \argmax\limits_{\ba\in\Delta^{K-1}_{\varepsilon}: \ba^{(2)} \neq \ba^{(1)}}\; V(\ba, \F) \ . 
\end{align}
We also adopt the convention $\ba^{(0)}\triangleq \balpha^\star$.  Furthermore, we introduce three  {\em suboptimality gap} terms within the space of the discrete simplex $\Delta_{\varepsilon}^{K-1}$. The suboptimality gaps play important roles in designing the ETC-based algorithm and in analyzing both the ETC-based and UCB-based algorithms. Let us define
\begin{align}
\label{eq:suboptimality_gaps}
    \delta_{ij} (\varepsilon) &\triangleq  V\Big(\ba^{(i)}, \mathbb{F}\Big) - V\Big(\ba^{(j)}, \mathbb{F}\Big) \ ,  \qquad \forall i,j\in\{0,1,2\}\ .
\end{align}
These gap terms appear in different forms in the performance (regret) guarantees.

\subsection{PM-centric ETC for Mixtures}
\label{Section:RS_ETC_M}
We propose the {\bf PM}-centric  {\bf ETC} for {\bf M}ixtures (PM-ETC-M) algorithm, following the ETC principles, albeit with important modifications needed to accommodate implementing mixture policies. %\footnote{For PM-ETC-M algorithm, in the discretization set we also consider the corner cases of $0$ and $1$ for each coordinate}. 

\paragraph{Conventional ETC.} The conventional ETC algorithm implements two phases within a known horizon $T$. In the first phase (exploration phase), ETC samples arms uniformly at random to estimate the arms' means. At the end of the exploration phase, ETC identifies the arm that is expected to have the largest mean and commits to this arm for the remainder of the horizon. An adaptation of ETC to the PM-centric framework leads to two challenges. First, at the end of the exploration phase, we face an infinite-dimensional problem (owing to the continuous space $\Delta^{K-1}$), as opposed to a finite-dimensional problem of identifying the arm with the largest mean. Secondly, at the end of the exploration phase, there is no single arm that we can commit to. Rather, we need to balance arm selection proportions to ensure a regret-efficient allotment of resources based on the estimated optimal mixing coefficient from the exploration phase.

\paragraph{Adapting ETC for PMs.} The PM-ETC-M algorithm consists of an initial \emph{exploration} phase during which the arms are sampled uniformly for a fixed interval to form high-fidelity estimates of the arms' CDFs (in contrast to canonical ETC, which estimates arms' mean values). Similarly to ETC, the duration of this phase depends on the minimum sub-optimality gap. Subsequently, in the next phase, the ETC algorithm \emph{commits} to a fixed policy, which is a mixture of arms with pre-fixed mixing coefficients. To alleviate the issue of infinite dimensionality of the optimization space, we adopt the discrete simplex $\Delta_{\varepsilon}^{K-1}$ instead of $\Delta^{K-1}$, which preserves the finite-dimensional setting of ETC at the end of the exploration phase. For the ETC algorithm, we set the discrete simplex as follows.
\begin{align}
\label{eq:discrete_set_ETC}
     \Delta^{K-1}_{\varepsilon} \triangleq  \{\varepsilon\bn: \bn\in\N^K \quad \mbox{and} \quad \bone^\top \cdot \varepsilon \bn = 1  \}\ ,
 \end{align}
where $\bone$ denotes the $K$-dimensional all-one vector. 
The ETC algorithm is designed to select $\ba^{(1)}$ as the mixing coefficient with high probability. The main processes of this algorithm are explained next, and its pseudocode is presented in Algorithm \ref{algorithm:PM-ETC-M-Alg}.

\begin{algorithm}
       % \small
%		\algsetup{linenosize=\small}
%		\setstretch{0.85}
		\caption{PM-ETC-M}
		\label{algorithm:PM-ETC-M-Alg}
 		%\small
 		\begin{algorithmic}[1]
            \STATE \textbf{Input:} Suboptimality gap $\delta_{12}(\varepsilon)$, horizon $T$
            \STATE Sample each arm $\lceil N(\varepsilon)/K \rceil$ times and obtain observation sequences $\mcX_{\lceil N(\varepsilon)/K \rceil}(1),\cdots,\mcX_{\lceil N(\varepsilon)/K \rceil}(K)$
            \STATE \textbf{Initialize:} $\tau_K^{\rm E}(i) = \lceil N(\varepsilon)/K \rceil ;\forall\;i\in[K]$, empirical arm CDFs $\F^{\rm E}_{1, \lceil N(\varepsilon)/K \rceil},\cdots,\F^{\rm E}_{K, \lceil N(\varepsilon)/K \rceil}$
			\FOR{$t= K\lceil N(\varepsilon)/K \rceil + 1,\cdots,T$}
			    \STATE Select an arm $A_{t}$ via~\eqref{eq: ETC_samp} and obtain reward $X_t$\\
                \STATE Update the empirical CDF $\F^{\rm E}_{A_t,t}$ according to~\eqref{eq:empirical_CDF}
			\ENDFOR
 		\end{algorithmic}
	\end{algorithm}

\paragraph{Explore (Estimate Mixing coefficients).} The purpose of this phase is to form high-confidence estimates of the empirical arm CDFs. We specify a time instant $N(\varepsilon)$ that determines the duration of the exploration phase, and that is the instance by which we have confident enough CDF estimates. The arms are selected uniformly, each $\lceil \frac{1}{K}N(\varepsilon)\rceil$ times. For a PM with a distortion function that has \holder~continuity exponent $q$ and constant $\mcL$, the time instant $N(\varepsilon)$ becomes a function of $q$, $\mcL$, the horizon $T$, and the minimum sub-optimality gap $\delta_{12}(\varepsilon)$ as specified below.
\begin{align} 
\label{eq:number of samples main paper}
 & N(\varepsilon) \triangleq 256K  \e \left(\frac{2K \mcL}{\delta_{12}(\varepsilon)}\right)^{\frac{2}{q}} \bigg[\frac{32}{\sqrt{\rm e}} + \log^{\frac{1}{2}} \Big( 2K T^2   \big(\varepsilon^{-(K-1)} + 1\big) \Big) \bigg]^2\ .
\end{align}
Note that $N(\varepsilon)$ scales with $T$ according to $O(\log T)$. Accordingly we define
\begin{align}
\label{eq:M}
    M(\varepsilon) \triangleq \frac{N(\varepsilon)}{\log T}\ ,
\end{align}
which based on~\eqref{eq:number of samples main paper} scales as $O(1)$. 
Next, using the collected samples during exploration, the PM-ETC-M algorithm constructs the empirical estimates of arms' CDFs, as specified in~\eqref{eq:empirical_CDF}.

\paragraph{Choosing the Mixing Coefficient.} At the time instant $N(\varepsilon)$, the PM-ETC-M algorithm identifies the discrete mixture coefficients that maximize the PM using the empirical CDFs. These coefficients are denoted by $\ba^{\rm E}_{N(\varepsilon)}$, where we have defined
\begin{align}
    \label{eq:ETC_alpha}
    \ba^{\rm E}_{t}\;\in\; \argmax\limits_{\ba\in\Delta^{K-1}_{\varepsilon}} U_h\Bigg( \sum\limits_{i\in[K]} a(i)\; \F_{i,t}^{\rm E}\Bigg)\ . 
\end{align}
\paragraph{Tracking Mixtures (Commit).} For the remaining sampling instants $t\in[N(\varepsilon),T]$, the PM-ETC-M algorithm commits to selecting the arms such that their selection frequencies are as close as possible to $\ba^{\rm E}_{N(\varepsilon)}$. Until the time instant $N(\varepsilon)$, the algorithm samples all arms uniformly. Uniform sampling results in some arms being sampled more than what the policy $\ba^{\rm E}_{N(\varepsilon)}$ dictates and hence, these arms will not be sampled again. Having over-sampled arms implies that some arms are under-sampled. In that case, these arms would be sampled such that the number of times they are chosen converges to the mixing coefficient $\ba^{\rm E}_{N(\varepsilon)}$. To formalize how to track the mixture, let $S$ be the set of the first $(K-1)$ arms (or any desired set of arms). For each arm $i \in S$, the algorithm calculates the required number of samples, which is $T a_t(i)$ for $t={N(\varepsilon)}$ and horizon~$T$. The sampling procedure proceeds as follows: 
\begin{enumerate}
\label{eq: ETC_samp}
    \item if $T a_t(i) > \lceil\frac{ t}{K} \rceil$ at $t=N(\varepsilon)$, we interpret that arm $i$ is not explored sufficiently, and we proceed with sampling arm $i$ to $T a_{N(\varepsilon)}(i)$ before moving on to the next arm.
    \item if $T a_t(i) \leq  \lceil\frac{ t}{K} \rceil$ at $t=N(\varepsilon)$, we interpret that arm $i$ is explored sufficiently and we skip selecting it.
\end{enumerate}
After we apply these conditions on the first $(K-1)$ arms, we allocate the remaining sampling budget to arm $K$.

\paragraph{Discussion.} 
An important advantage of the PM-ETC-M algorithm is its computational simplicity. The algorithm involves uniform exploration for a finite interval, followed by committing to a mixture estimate based on the data in the exploration phase. UCB-type algorithms, on the other hand, require substantially larger computation due to optimizing over confidence sets. Hence, in scenarios with a computational complexity bottleneck, PM-ETC-M serves as a regret-efficient algorithm to optimize PMs in an online setting. Furthermore, even if we have access to sufficient computational resources, PM-ETC-M may potentially have better regret guarantees depending on the choice of the PM. In Section~\ref{RS_UCB_M}, we will discuss that the relative performances of PM-ETC-M versus that of the UCB-based counterpart depends on the choice of PM, and neither has a uniform regret advantage over the other. It is also noteworthy that despite its computational simplicity and its better regret guarantee for some PMs, PM-ETC-M inherits the crucial drawback of ETC-type algorithms --- the knowledge of instance-dependent gap information. We observe that the PM-ETC-M algorithm relies on the instance-specific gap information (through $N(\varepsilon)$), which may not always be available --- an issue that is addressed by the UCB-type algorithms.

\subsection{PM-centric UCB for Mixtures}
\label{RS_UCB_M}
A critical bottleneck of the PM-ETC-M algorithm is that its decisions hinge on knowing the instance-dependent gap through the exploration parameter $N(\varepsilon)$. As in the conventional setting, we adopt a confidence-based approach to address the drawback of PM-ETC-M. We begin by providing a brief overview of the UCB algorithm, detailing the bottlenecks in the direct adaptation of the UCB algorithm in the PM-centric setting. 

\paragraph{Overview of and Comparison with Canonical UCB.} The UCB algorithm forms estimates of the arm means and constructs confidence sets around these estimates, quantifying the estimation fidelity. Contrasted to the ETC algorithm, which balances exploration and exploitation through the instance-dependent $N(\varepsilon)$, UCB uses the confidence sequences to balance exploration and exploitation in a regret-efficient fashion. Specifically, at each round, UCB computes the upper confidence bound for each arm and samples the one with the largest upper confidence bound. Sampling an arm reduces the uncertainty in its estimate, and it can be shown that no suboptimal arm is sampled more than $O(\log T)$ times, resulting in sublinear regret. In the PM-centric setting, we have significant distinctions from the canonical setting. First, as in PM-ETC-M, we estimate the entire CDFs of the arms instead of their first moments. Secondly, upper confidence bounds for each arm do not directly apply in the PM-centric setting. Instead, we need to translate distribution confidence sets to confidence bounds for discrete mixing coefficients in the (discrete) simplex for estimating mixtures. Third, we need an additional routine to translate the estimated mixing coefficients to arm selection decisions.

\paragraph{Adapting UCB for PMs.} We now present the {\bf PM}-centric {\bf UCB} for {\bf M}ixtures (PM-UCB-M) algorithm, which does not require the information on the sub-optimality gap $N(\varepsilon)$. The salient features of this algorithm are (i) a distribution estimation routine for forming high-confidence estimates for arms' CDFs and subsequently mixture coefficients; and (ii) a sampling rule based on an {\em under-sampling} criteria to ensure that arm selections track the mixture coefficients. For the UCB algorithm, we set the discrete simplex as follows.
\begin{align}
\label{eq:discrete_set_UCB}
     \Delta^{K-1}_{\varepsilon} \triangleq   \left\{\varepsilon\bn: \bn\in\N^K \quad \mbox{and} \quad \bone^\top \cdot \varepsilon \Big(\bn +\frac{1}{2}\Big)= 1  \right\}\ .
 \end{align}  
 The pseudocode of PM-UCB-M is presented in Algorithm~\ref{algorithm:B-UCB-M}. 

\begin{algorithm}
       {\small 
		\caption{PM-UCB-M}
		\label{algorithm:B-UCB-M}
		
 		\begin{algorithmic}[1]
            \STATE \textbf{Input:} Exploration rate $\rho$, horizon $T$
            \STATE Sample each arm {$N(\rho, \varepsilon) \triangleq \lceil \rho T \varepsilon/ 4 \rceil$} times and obtain observation sequences $\mcX_{K N(\rho, \varepsilon)}(1),\cdots,\mcX_{K N(\rho, \varepsilon)}(K)$
            \STATE \textbf{Initialize:} {$\tau_{KN(\rho, \varepsilon)}^{\rm U}(i) = N(\rho, \varepsilon)$} $\;\forall i\in[K]$, emprical arm CDFs $\F^{\rm U}_{1, KN(\rho, \varepsilon)},\cdots,\F^{\rm U}_{K, KN(\rho, \varepsilon)}$, confidence sets $\mcC_{KN(\rho, \varepsilon)}(1),\cdots\mcC_{KN(\rho, \varepsilon)}(K)$ according to~\eqref{eq:UCB_confidence_sets}
			\FOR{$t=KN(\rho, \varepsilon)+1,\cdots,T$}
			    \STATE Select an arm {$A_t$} {via}~(\ref{eq:UCB_sampling_rule}) and obtain reward $X_t$\\
                \STATE Update the empirical CDF $\F^{\rm U}_{A_t,t}$ according to~\eqref{eq:empirical_CDF}
                \STATE Update the confidence set $\mcC_t(A_t)$ according to~\eqref{eq:UCB_confidence_sets} 
                \STATE Compute the optimistic estimate $\ba_t^{\rm U}$ according to~\eqref{eq:UCB_alpha}
			\ENDFOR
 		\end{algorithmic}
        }
	\end{algorithm} 
\vspace{-.1 in}

\paragraph{Choosing the Mixing Coefficient.} Based on the concentration of CDF estimates \eqref{eq:meta_concentration}, at time $t$ and given the empirical CDFs $\{\F_{i,t}^{\rm U}:i\in[K]\}$, for arm $i\in[K]$ defined in \eqref{eq:empirical_CDF}, we define the {\em distribution confidence space} as the collection of all the distributions that are within a bounded $1$-Wasserstein distance of $\F_{i,t}^{\rm U}$. Specifically, for each $i\in[K]$, we define
\begin{align}
\label{eq:UCB_confidence_sets}
\mcC_t(i)\triangleq\bigg\{\eta \in  \Omega & :\norm{\F_{i,t}^{\rm U}-\eta}_{\rm W} 
\leq 16\ \frac{\sqrt{2 {\rm e} \log T } + 32}{\sqrt{\tau^{\rm U}_t(i)}} \bigg\}\  .
\end{align}
Next, we also need to estimate the optimal mixing coefficients. For this purpose, 
we apply the UCB principle. This, in turn, requires that we compute the following  {\em optimistic} estimates for the mixing coefficients:
\begin{align}
\label{eq:UCB_alpha}
    \ba^{\rm U}_t\in \argmax\limits_{\ba\in\Delta_{\varepsilon}^{K-1}}\;\max\limits_{\eta_i\in\mcC_t(i) , \forall i\in[K]}\; U_h\Big ( \sum\limits_{i\in[K]} a(i) \; \eta_i\Big )\ .
\end{align}

\paragraph{Tracking Mixtures.} Once we have estimates of the optimal mixing coefficients, i.e., $\ba_t^{\rm U}$, we design an arm selection rule that translates these coefficients and CDF estimates to arm selection choices. Designing such a rule requires addressing the following technical challenges. First, the optimal mixing coefficients might not be unique. 
Let us denote them by $\{\psi_\ell:\ell\in[L]\}$. It is critical to ensure that we \emph{consistently} track only one of these optimal choices over time. The reason is that if we track multiple mixtures, in aggregate, we will be tracking a mixture of $\{\psi_\ell:\ell\in[L]\}$, which is not necessarily optimal. Secondly, in the initial sampling rounds, the estimates $\balpha_t^{\rm U}$ are relatively inaccurate, and tracking them leads to highly suboptimal decisions. Finally, we need a rule for translating the estimated mixtures to arm selections. Next, we discuss how we address these issues.
\begin{itemize}
    \item \emph{Multiple optimal mixtures:} Switching among multiple discrete mixtures hinders the tracking mechanism, since we may fail to track {\em any} of these optimal coefficients, rather a combination of these which is likely suboptimal. To ensure tracking only one optimal mixture, at time $t$, the PM-UCB-M algorithm checks if the coefficient from the previous step, i.e., $\ba^{\rm U}_{t-1}$ also maximizes~\eqref{eq:UCB_alpha}. If it does, then $\ba^{\rm U}_{t-1}$ is chosen as a candidate optimistic estimate~$\ba^{\rm U}_t$. Otherwise, any random candidate solving~\eqref{eq:UCB_alpha} is chosen. 
    \item \emph{Initial rounds:} To circumvent estimation inaccuracies in the initial rounds, we introduce an explicit exploration phase for each arm.  Specifically, for $\rho \in (0,1)$, for the first $K\lceil\frac{\rho T \varepsilon}{4} \rceil$ rounds, we explore the arms in a round-robin fashion to initiate the algorithm with sufficiently accurate estimates of the arm CDFs. 
    \item \emph{Decision rules:} 
    When all the arms are sufficiently explored, motivated by the effective approaches to best arm identification~\cite{pmlr-v49-garivier16a,jourdan2022,pmlr-v117-agrawal20a,mukherjee2023best}, we 
    sample the most {\em under-sampled} arm. 
    An arm is considered under-sampled if it has been sampled less frequently than the rate indicated by the estimate $\ba^{\rm U}_t$, and has the largest gap between its current fraction and its estimated fraction $a^{\rm U}_t(i)$. At time  $t\geq K\lceil\frac{\rho T \varepsilon}{4} \rceil$, the arm selection rule is specified by
    \begin{align}
    \label{eq:UCB_sampling_rule}
A_{t+1}\;\triangleq\;\argmax\limits_{i\in[K]}\; \left\{ta^{\rm U}_t(i) - \tau_t^{\rm U}(i)\right\}\ .
    \end{align}
\end{itemize}

\subsection{Computationally Efficient CE-UCB-M }
In the PM-UCB-M algorithm, determining the mixing coefficients $ \ba^{\rm U}_t$ via~\eqref{eq:UCB_alpha} involves extremization over a class of distribution functions, and it is computationally expensive. To circumvent this, we present the {\bf C}omputationally-{\bf E}fficient PM-centric {\bf UCB} for {\bf M}ixtures (CE-UCB-M) algorithm as a computationally tractable modification of PM-UCB-M. In CE-UCB-M, instead of solving ~\eqref{eq:UCB_alpha}, for any given set of CDF estimates $\{\F_{i,t}^{\rm C}:i\in[K]\}$ defined in \eqref{eq:empirical_CDF} and mixing vector $\ba$, we define
\begin{align}
{\rm UCB}_t(\ba) & \;\triangleq\;  U_h\left( \sum\limits_{i\in[K]} a(i)\; \F_{i, t}^{\rm C}\right)   + \mcL \sum\limits_{i\in[K]} \bigg( a(i) \cdot 16\; \frac{\sqrt{2 {\rm e} \log T } + 32}{\sqrt{\tau^{\rm U}_t(i)}} \bigg)^q \ ,
\end{align}
where $q$ and $\mcL$ are the \holder~parameters of the underlying distortion function $h$. We specify estimates of the  mixing coefficients as
\begin{align}
\label{eq:UCB_alpha2}
    \ba_t^{\rm C}\;\in\; \argmax\limits_{\ba\in\Delta_{\varepsilon}^{K-1}}\; {\rm UCB}_t(\ba)\ .
\end{align}   
The remainder of the algorithm (the tracking block) follows the same steps as in PM-UCB-M. The CE-UCB-M procedure is summarized in Algorithm~\ref{algorithm:UCB-M}.

		\begin{algorithm}[h]
        \small
		\caption{CE-UCB-M}
		\label{algorithm:UCB-M}
		
 		%\small
 		\begin{algorithmic}[1]
            \STATE \textbf{Input:} Exploration rate $\rho$, horizon $T$
            \STATE Sample each arm $\lceil \rho T \varepsilon/ 4 \rceil$ times and obtain observation sequences $\mcX_{\lceil \rho T \varepsilon/ 4\rceil}(1),\cdots,\mcX_{\lceil \rho T \varepsilon/ 4 \rceil}(K)$
            \STATE \textbf{Initialize:} $\tau_{K\lceil \rho T \varepsilon/ 4 \rceil}^{\rm C}(i) = \lceil \rho T \varepsilon/ 4 \rceil\;\forall\;i\in[K]$, arm CDFs $\F^{\rm C}_{1, \lceil K \rho T \varepsilon/ 4 \rceil}(1),\cdots,\F^{\rm C}_{K, K\lceil \rho T \varepsilon/ 4\rceil}$, confidence sets $\mcC_{K \lceil \rho T \varepsilon/ 4 \rceil}(1),\cdots\mcC_{K\lceil \rho T \varepsilon/ 4\rceil}(K)$ according to~\eqref{eq:UCB_confidence_sets}
			\FOR{$t= K \lceil \rho T \varepsilon/ 4\rceil+1,\cdots,T$}
			    \STATE Select an arm $A_{t}$ specified by~(\ref{eq:UCB_sampling_rule}) and obtain reward $X_t$\\
                \STATE Update the empirical CDF $\F^{\rm C}_{A_t,t}$ according to~\eqref{eq:empirical_CDF}
                \STATE Compute the optimistic estimate $\ba^{\rm C}_t$ according to~\eqref{eq:UCB_alpha2}
			\ENDFOR
 		\end{algorithmic}
	\end{algorithm}

\paragraph{Discussion.} Unlike PM-ETC-M,  the PM-UCB-M and the CE-UCB-M algorithms are independent of instance-dependent parameters. The explicit exploration phase involves a hyperparameter $\rho$, which must be bounded away from $0$. The performance of PM-ETC-M and PM-UCB-M depends on the choice of PM, with neither uniformly dominating the other. In Section~\ref{sec:analysis}, we show that under some conditions, PM-ETC-M has better regret guarantees for various PMs. However, there are also PMs, such as PHT, for which PM-UCB-M outperforms PM-ETC-M.

\section{Regret Analysis for Horizon-dependent Algorithms}
\label{sec:analysis}

In this section, we establish the performance guarantees for the algorithms provided in Section~\ref{sec:algorithm}. We begin by stating some properties of PMs and then provide upper bounds on the average regret for the three algorithms in Section~\ref{sec:algorithm}.

\paragraph{Properties of PMs.} In Lemma~\ref{example utility}, we observed that choosing the distortion function $h(p)=p(1-p)$ (corresponding to Gini deviation) leads to an {\em optimal mixture} in a two-armed Bernoulli bandit. To make this observation more concrete, we now explore broader classes of bandit instances -- identifying conditions under which the optimal solution is a {\em solitary} arm versus when it is a {\em mixture}. Our first key finding is that when the distortion function is {\em convex}, the optimal solution is {\em always solitary}. This observation is formalized in the following theorem. 
\begin{theorem}[Convex Distortion Functions]
\label{lemma:convex_solitary}
    Suppose the distortion function $h$ is convex and the PMs $\{U_h(\F_i):i\in[K]\}$ are finite. Then, $U_h$ is maximized by a solitary arm, i.e., for any $\ba \in \Delta^{K-1}$, 
    \begin{align}
        U_h\Bigg(\sum_{i \in [K]}a(i)\F_i\Bigg) \leq \max_{i \in [K]} \; U_h(\F_i)\ .
    \end{align}
\end{theorem}
\begin{proof}
    See Appendix~\ref{proof:lemma:convex_solitary}.
\end{proof}
We note that a similar property has also been shown for {\em quasiconvex functionals} in~\cite{cassel2018general}. Convexity holds for various risk measures studied in~\cite{cassel2018general}, such as variance, entropic risk, and mean-variance. Under convexity, as a consequence of Theorem~\ref{lemma:convex_solitary}, we have the additional information that the optimal arm is solitary. While the algorithms presented in Section~\ref{sec:algorithm} work for solitary arms, as we will see shortly, their regret guarantees may be weaker compared to the algorithm presented in~\cite{cassel2018general} for certain PMs. This is the expense paid for the generality to accommodate mixture policies, too. 
Next, we show that when the distortion function $h$ is concave, the optimal policy is not necessarily solitary. Specifically, we show that for any given strictly concave $h$, there always exists a bandit instance for which the optimal policy is strictly mixture, i.e., the PM associated with the mixture of the arms' distributions strictly dominates that of each individual arm.

\begin{theorem}[Strictly Concave Distortion Functions]
\label{lemma:mixture_lemma_exponential}
    Suppose the distortion function $h$ is strictly concave. Then there exists a $K$-armed bandit instance for which a mixture of arms maximizes the induced PM $U_h$.
\end{theorem}
\begin{proof}
    See Appendix~\ref{proof:lemma:mixture_lemma_exponential}.
\end{proof}
This observation generalizes the anecdotal example of Gini deviation and Bernoulli bandits discussed in Lemma~\ref{example utility}.  It highlights the importance of designing algorithms that can adapt to mixture-optimal scenarios. The specific choice of the bandit model that results in mixture-optimal policies depends on the choice of the PM. For instance,  Table~\ref{table:table_regrets_bernoulli} shows that for 
Bernoulli bandits, several performance measures with concave distortion functions, such as mean–median deviation, inter‑ES range, Wang’s right‑tail deviation, and Gini deviation—still exhibit mixture optimality.

\paragraph{Regret Analysis.} 
The performance (regret) of the PM-centric algorithms inevitably depends on the quality of the estimates we find for the mixing coefficients. With a finite number of samples, the estimation errors are guaranteed to be bounded away from zero, rendering the implementation of the mixing coefficients different from their optimal values. Hence, the regret terms can be decoupled into two parts: (i) {\em arm selection regret} term, which accounts for the regret incurred due to the suboptimal selection of arms over time. This regret term is non-zero even when the mixing coefficients are known perfectly, and is the counterpart of the regret in the conventional bandit settings; and (ii) \emph{estimation regret} term that accounts for the noise in estimating mixing coefficients. Owing to the discretization process we have introduced, this term can also be considered the discretization error. 
Using this decomposition, we first present regret bounds as a function of the discretization granularity $\varepsilon$, which controls how finely the action space is discretized. Subsequently, we also provide an $\varepsilon$-independent regret in which $\varepsilon$ is chosen carefully to achieve the (near-)best performance subject to algorithmic constraints. For a given discretization level $\varepsilon$ and a bandit instance $\bnu\triangleq (\F_1,\cdots,\F_K)$,  for policy $\pi$, we decompose the regret defined in~\eqref{eq:regret} into an estimation regret component and an arm selection regret component as follows.
\begin{align}
\label{eq:regret_decomposition}
    \mathfrak{R}^\pi_{\bnu}(T)\; & =\;   \bar{\mathfrak{R}}_{\bnu, i}^\pi(T)+\delta_{0i}(\varepsilon)\ , \qquad i \in \{1,2\}\ ,
\end{align}
where we have defined $\delta_{0i}(\varepsilon)$ as the estimation regret, which is specified as follows based on the definition of $\ba^{(i)}$ in~\eqref{eq:discrete optimal mixture}.
\begin{align}
    \label{eq:regret_decomposition1_1} \delta_{0i}(\varepsilon) & \triangleq V\Big(\ba^{(0)},\F\Big) - V\Big(\ba^{(i)},\F\Big)  \ , \qquad i \in \{1,2\}\ ,
\end{align}
where $\ba^{(0)} \triangleq \balpha^\star$.
We have also defined $\bar{\mathfrak{R}}_{\bnu, i}^\pi(T)$ as the arm selection regret term, which is specified as follows using the definitions in~\eqref{eq:regret}~and~\eqref{eq:regret_decomposition1_1}.
\begin{align}
    \label{eq:regret_decomposition2_1}
\bar{\mathfrak{R}}_{\bnu, i}^\pi(T) & 
    \triangleq \mathfrak{R}^\pi_{\bnu}(T) - \delta_{0i}(\varepsilon) = V\Big(\ba^{(i)},\F\Big) - \E_{\bnu}^\pi\Big[V\Big(\frac{1}{T}\btau_T^\pi, \F \Big)\Big] \ , \qquad i \in \{1,2\}\ .
\end{align}
Throughout the rest of this section, we provide regret guarantees for all the algorithms. For each algorithm, we provide two guarantees, one as a function of $\varepsilon$ and one independent of it.
\begin{theorem}[PM-ETC-M -- $\varepsilon$-dependent]
\label{theorem: ETC upper bound}
For any $\varepsilon\in\R_+$ and distortion function $h$ with H\"older exponent $q$, for all $T>N(\varepsilon)$, 
PM-ETC-M's regret
is upper bounded as
    \begin{align}
 %   \label{eq: ETC regret}
\mathfrak{R}_{\bnu}^{{{\rm E}}}(T)\ &
   \leq   (\mcL K + W^{-q}) \bigg(3 WM(\varepsilon)\;  \frac{\log T}{T} \bigg)^q + \delta_{01}(\varepsilon)\ ,
    \end{align}
where $M(\varepsilon)$, defined in~\eqref{eq:M}, scales as $O(1)$.
\end{theorem}
\begin{proof}
    See Appendix~\ref{Appendix:ETC_theorem_proof}.
\end{proof}

\paragraph{Proof sketch.} The proof follows the general principles of ETC with the following differences. (i) Designing the exploration horizon necessitates appropriately bounding error events due to CDF estimation, as opposed to mean estimation in the canonical setting, ultimately ensuring that we estimate the {\em best} or the {\em second-best} discrete mixing coefficient with a sufficiently high fidelity. (ii) For the commitment phase, sampling a solitary arm significantly simplifies the analysis for canonical ETC, and its regret bound is guided by the probability of identifying a suboptimal arm at the end of exploration. On the contrary, we have the additional task of ensuring that our sampling fractions appropriately match the chosen discrete mixture at the end of exploration. (iii) For some arms, the exploration phase might already be longer than necessary. In this case, the worst-case sampling error scales as $O(\frac{N(\varepsilon)}{T})$. For the others, the PM-ETC-M commitment strategy ensures that we incur minimal regret. In Section~\ref{sec:ETC_sampling_error}, we show that on aggregate, the proposed sampling rule incurs a sublinear regret.

\paragraph{Choosing $\varepsilon$.} Theorem~\ref{theorem: ETC upper bound} is valid for any $\varepsilon\in\R_+$, allowing freedom to appropriately choose the desired accuracy for the mixing coefficients. Observe that {$\delta_{01}(\varepsilon)$} is proportional to $\varepsilon$ as shown in Lemma \ref{lemma:Delta_error}, and $N(\varepsilon)$ is inversely proportional to $\varepsilon$. Hence, while arbitrarily diminishing $\varepsilon$ decreases the estimation regret, it may violate the condition that $T>N(\varepsilon)$. We chose $\varepsilon$ to be small enough (small estimation regret), while conforming to the condition $T>N(\varepsilon)$ in Theorem~\ref{theorem: ETC upper bound} (feasibility). The minimum feasible $\varepsilon$ depends on $q$ and its connection to the sub-optimality gap as follows. Let us define
\begin{align}
\label{eq:beta_def}
    \beta\triangleq \lim_{\varepsilon \to 0}\frac{\log \delta_{12}(\varepsilon)}{\log \varepsilon}\  ,
\end{align}
which quantifies how fast the minimum sub-optimality gap $\delta_{12}(\varepsilon$) diminishes as $\varepsilon$ tends to $0$.
\begin{lemma}[$\beta$]
\label{lemma:beta12}
    When $h$ is continuous and strictly concave, $\beta$ is specified as follows.
    \begin{enumerate}
        \item For bandit models with continuous sub-Gaussian distributions, $\beta\in\{1,2\}$.
        \item For bandit models with discrete distributions supported on $\Omega$, if $
 |\Omega|\geq K$, we have $\beta\in\{1,2\}$.
    \end{enumerate}
\end{lemma}
\begin{proof}
    See Appendix~\ref{Appendix:proof of lemma beta12}. 
\end{proof}
Based on the definition of $\beta$, we next provide an $\varepsilon$-independent regret bound by setting the discretization interval as an appropriate function of $q$ and $\beta$. Specifically, we set 
\begin{align}
\label{eq:ETC_var_gamma}
    \varepsilon = \Theta\left(\left(K^{2+\frac{2}{q}} \cdot \frac{\log T}{T^{\gamma}}\right)^{\frac{q}{2\beta}}\right)\ , \qquad \mbox{where} \quad \gamma\triangleq \frac{2\beta}{2\beta + q}\ .
\end{align}
This choice of $\varepsilon$ leads to the following regret bound.
\begin{theorem}[PM-ETC-M -- $\varepsilon$-independent]
\label{theorem:PM-ETC-M}
Under the conditions of Theorem \ref{theorem: ETC upper bound}, when $\beta$ exists, the minimum feasible regret that PM-ETC-M incurs is upper-bounded as
\begin{align}
\mathfrak{R}_{\bnu}^{{\textnormal{\rm E}}}(T)\ & \leq O\Bigg(\Bigg[K^{c_{\rm E}}\cdot  \frac{\log T}{T^{\gamma}}\Bigg]^{\frac{q^2}{2\beta}}\Bigg)\ ,
\end{align}
where $c_{\rm E}\triangleq {2(1+\frac{\beta+1}{q})}$ and $\gamma = \frac{2\beta}{2\beta+q}$.
\end{theorem}
\begin{proof}
    See Appendix~\ref{Appendix:ETC_theorem_epsilon_proof}.
\end{proof}
From Theorem~\ref{theorem:PM-ETC-M}, we observe that the order of convergence in the horizon scales as $O(T^{-\frac{\gamma}{2\beta}q^2})$. This regret term is evaluated and reported in Table~\ref{table:table_regrets} for various notions of PMs, which have their specific $q$. Next, we present the regret guarantees for the two UCB-based algorithms, i.e., PM-UCB-M and CE-UCB-M. An important observation is that these algorithms yield weaker \emph{arm selection} regret guarantees compared to PM-ETC-M. The regret degradations are the expense of not knowing the instance-dependent gaps. To formalize the regret terms, we define an instance-dependent finite time instant $T(\varepsilon)$ as 
\begin{align}
    \textstyle T(\varepsilon)\;\triangleq\; \frac{2}{\varepsilon}\Big(T_0(\varepsilon) -1\Big) \ ,
\end{align}
where
\begin{align}
  T_0(\varepsilon)\; \triangleq \;\inf \left\{ t\in\N : \frac{ 2\sqrt{2\e\log s} + 32}{\sqrt{ \rho s \varepsilon }}
     \leq \frac{1}{16} \left( \frac{\delta_{12}(\varepsilon)}{2K\mcL}\right)^{\frac{1}{q}}\forall s \geq t\ \right\}\ .
\end{align}
The next theorem presents an $\varepsilon$-dependent regret for the PM-UCB-M and CE-UCB-M algorithms. 
\begin{theorem}[PF/CE-UCB-M -- $\varepsilon$-dependent]
\label{theorem:UCB upper bound}
For any $\varepsilon\in\R_+$, and distortion function $h$ with \holder~exponent $q$, for all $T>\max\{\e^K,T(\varepsilon)\}$ for $\pi\{{\rm U,C}\}$ we have
\begin{align}
\mathfrak{R}_{\bnu}^{\pi }(T) & \leq [B_h + 2\mcL K (W^q+1)] \Bigg[\frac{64}{\sqrt{\varepsilon \rho T}}\Big( \sqrt{2\e\log T} + 32\Big)\Bigg]^q 
          + \delta_{01}(\varepsilon) \ .
    \end{align}
\end{theorem}
\begin{proof}
    See Appendix~\ref{Appendix:UCB_theorem_proof}.
\end{proof}
\textbf{Proof sketch.} The distinction of the proof with that of the vanilla UCB, is that in the vanilla UCB, the general proof uses the fact that selecting any suboptimal arm more than $O(\log T)$ times is unlikely, and hence, the overall regret is bounded by $O(\frac {1}{T}K\log T)$. However, in our analyses, we have the new dimension of estimating the mixing coefficients and need these estimates to converge to an optimal choice. Such convergence requires that all arms be sampled at a rate linear in $T$  (unless an arm's mixing coefficient is 0). Hence, selecting any arm $O(\log T)$ times is insufficient. The other key difference pertains to finding a bound on the mixing coefficient errors (estimation and convergence). Characterizing this bound hinges on two key steps: (i) the convergence of the UCB estimates ($\ba_t^{\rm U}$ for PM-UCB-M and $\ba_t^{\rm C}$ for CE-UCB-M) to the discrete optimal solution $\ba^{(1)}$ in probability, and (ii) a sublinear regret incurred in the process of tracking the mixing coefficient estimates using under-sampling. The first step is analyzed in Appendix~\ref{appendix: UCB sampling estimation error}, where we show that the probability of error for the PM-UCB-M and CE-UCB-M algorithms in identifying the discrete optimal mixture is upper bounded by $T((\frac{1}{T^2} + 1)^K-1)$. The second step is analyzed in Lemma~\ref{lemma:undersampling} in Appendix~\ref{appendix: UCB sampling estimation error}, in which we show that the regret incurred by the tracking block of the PM-UCB-M and CE-UCB-M algorithms is of the order $O(\frac{K}{T})$. The regret upper bound is a combination of the regrets in these results.

\paragraph{Choosing $\varepsilon$.} Similarly to PM-ETC-M, we choose an $\varepsilon$ that ensures $T > T(\varepsilon)$ for the bound in Theorem \ref{theorem:UCB upper bound}. We set 
\begin{align}
\label{eq:vareps_kappa_UCB}
    \varepsilon = \Theta \left(\left(K^{\frac{2}{q}}\; \frac{\log T}{T}\right)^{\kappa}\right)\ , \quad \mbox{where} \;\; \kappa \triangleq \Big(\frac{2\beta}{q}+2\Big)^{-1}\ .
\end{align}
Based on these choices, we specify the $\varepsilon$-independent regret in the next theorem. We note that, similar to the ETC-based algorithm, characterizing $\beta$ might not be feasible.

\begin{theorem}[PF/CE-UCB-M -- $\varepsilon$-independent]
\label{corollary:PM-UCB-M}
Under the conditions of Theorem \ref{theorem:UCB upper bound}, we have the following regret guarantees.
     When $\beta$ exists, the regret upper bound for $\pi\in\{{\rm U,C}\}$ is
\begin{align}
\mathfrak{R}_{\bnu}^{\pi}(T) \leq O \left (K^{c_{\rm U}}  \Big[\frac{\log T}{T}\Big]^{q \kappa} \right) \ , \qquad \mbox{where} \quad c_{\rm U}\triangleq \max\left\{ q\left(1+ \frac{2\kappa}{q}\right),{1-\kappa} \right\}.
\end{align}
\end{theorem}
\begin{proof}
    See Appendix~\ref{Appendix:UCB_theorem_epsilon_proof}.
\end{proof}
    An important observation is the following contrast between the regrets characterized for the mixture algorithms and their canonical ETC and UCB  counterparts: the canonical ETC and UCB algorithms generally exhibit the same regret, even though ETC requires access to instance-dependent parameters. In the mixtures setting, however, access to instance-dependent parameters yields better arm selection regret guarantees for the ETC-type algorithm (i.e., PM-ETC-M). After setting an appropriate discretization level $\varepsilon$, the relationship becomes significantly more nuanced, and the performance of PM-ETC-M and PM-UCB-M depends on the choice of the utility. Note that when we set $\varepsilon$, we assume to know $\beta$ -- and hence, PM-ETC-M and PM-UCB-M have the same amount of information in this setting.

\paragraph{Regret Bounds for Important Cases.} 

We specialize our general results to several widely-used PMs with differentiable distortion functions and report their regret bounds in Table~\ref{table:table_regrets}. 
For the PMs listed in this table $\beta$ takes values in $\{1,2\}$ depending on whether the underlying distribution models. Except for PHT measure and Wang's Right-tail deviation, where we restrict ourselves to bounded distributions, these underlying distributions include continuous sub-Gaussian distributions and discrete distributions with $|\Omega| \geq K$. While reporting the regret bounds in Table~\ref{table:table_regrets}, we used the \holder~continuity exponents specified in Table~\ref{table:table_risks} (Appendix \ref{Appendix:Holder_exp}).
We remark that the $\varepsilon$-dependent regret bounds depend only on the \holder~continuity constant $q$ (Theorems~\ref{theorem: ETC upper bound} and \ref{theorem:UCB upper bound}). On the other hand, the $\varepsilon$-independent regret bounds will additionally depend on $q$ and $\beta$ (Theorems~\ref{theorem:PM-ETC-M} and \ref{corollary:PM-UCB-M}). From Table~\ref{table:table_regrets}, we observe that \emph{no algorithm uniformly dominates the other}. Under the assumption of solitary arm, i.e., when the optimal solution lies on the boundaries where we have $\beta=1$, PM-ETC-M performs better than PM-UCB-M for several measures, including mean and Gini deviation, as in Table \ref{table:table_regrets}. However, when the optimal solution does not lie on the boundaries of the $\Delta^{K-1}$, where we have $\beta=2$, PM‑UCB‑M actually outperforms PM‑ETC‑M on measures such as PHT (see Table \ref{table:table_regrets}). Intuitively, PM‑ETC‑M commits in advance to a single mixture coefficient using the minimum sub-optimality gap, whereas in PM‑UCB‑M analysis, this gap is only used to choose the discretization level $\varepsilon$ that conforms to $T(\varepsilon)$. We do something smaller while choosing a feasible $\varepsilon$ for PM-ETC-M. As a result, PM‑UCB‑M’s less reliance on gap information results in worse regret bounds. Finally, we comment on Gini deviation, which served as a motivation example in Section~\ref{sec:introduction}. In the case of Gini deviation, we observe that for arm selection regret, the PM-ETC-M algorithm achieves \(O(\frac{K}{T})\) regret while PM-UCB-M achieves $O(\frac{K}{\sqrt{T}})$ regret, ignoring polylogarithmic factors. However, when we optimize the algorithms for the discretization level \(\varepsilon\), we observe an order-wise improvement in the regret of the algorithms. 

\setlength{\textfloatsep}{5pt}
\begin{table*}[h!]
  \centering
  \begin{minipage}{\textwidth}
    \rowcolors{2}{gray!15}{white}
    \caption{Regret bounds of horizon-dependent algorithms, where $\varpi(T)\!\triangleq\!\sqrt{\tfrac{\log T}{T}}$, $\kappa = (\frac{2\beta}{q}+2)^{-1}$ and $\gamma = \frac{2\beta}{2\beta+q}$.}    
    \label{table:table_regrets}
    \vspace{3pt}
    \centering
    \begin{tabular}{|l|c|c|}
      \hline
      \textbf{Preference Metric}
      & UCB-type
      & ETC-type \\
      \hline\hline
      Mean Value
        & $O \bigl(\varpi^{2\kappa}(T)\bigr)$
        & $O \bigl(T^{\tfrac{1}{\beta}}\varpi(T^\gamma)\bigr)$ \\

      Dual Power ($s\ge2$)
        & $O \bigl(\varpi^{2\kappa}(T) \bigr)$
        & $O \bigl(T^{\tfrac{1}{\beta}}\varpi(T^\gamma)\bigr)$ \\

      Quadratic ($s\in[0,1]$)
        & $O\bigl(\varpi^{2\kappa}(T)\bigr)$
        & $O\bigl(T^{\tfrac{1}{\beta}}\varpi(T^\gamma)\bigr)$ \\

      PHT ($s=\tfrac12$)
        & $O \bigl(\varpi^{\kappa}(T)\bigr)$
        & $O \bigl(\varpi^{\tfrac{1}{\beta}}(T^\gamma)\bigr)$ \\

      Wang’s Right–Tail Deviation 
        & $O\bigl(\varpi^{\kappa}(T)\bigr)$
        & $O \bigl(\varpi^{\tfrac{1}{\beta}}(T^\gamma)\bigr)$ \\

      Gini Deviation
        & $O \bigl(\varpi^{2\kappa}(T)\bigr)$
        & $O \bigl(T^{\tfrac{1}{\beta}}\varpi(T^\gamma)\bigr)$ \\
      \hline
    \end{tabular}
  \end{minipage}
\end{table*}
\setlength{\textfloatsep}{5pt}

\paragraph{Degenerate cases of $\beta$.} For a all continuous and majority of discrete distributions, $\beta$ is guaranteed to exist and is characterized by Lemma \ref{lemma:beta12}. Nevertheless, for discrete models, when the cardinality of the sample space is less than $K$, i.e., $|\Omega|<K$, $\beta$ will be degenerate for the choice of the discretization strategy we have. This is because the minimum sub-optimality gap $\delta_{12}(\varepsilon)$ can be arbitrarily small for certain PMs and bandit models. For instance, in the $K$-arm Bernoulli bandit where some arms have mean values strictly larger than $\frac{1}{2}$ and the others strictly smaller than $\frac{1}{2}$, under Gini deviation and the uniform discretization approach we have adopted, the minimum gap $\delta_{12}(\varepsilon)$ can shrink arbitrarily fast, rendering its rate of decay with shrinking $\varepsilon$ undefined. One solution to circumvent such cases, is modifying the PM-ETC-M algorithm by replacing  $\delta_{12}$ with $\delta_{13}$ defined as
\begin{align}
\label{eq:discrete optimal mixture2}
     \ba^{(3)} \triangleq \argmax\limits_{\ba\in\Delta^{K-1}_{\varepsilon}: \ba^{(3)} \neq \ba^{(1)},\ba^{(2)}}\; V(\ba, \F) \ , \qquad \mbox{and} \qquad \delta_{13} (\varepsilon) &\triangleq  V\Big(\ba^{(1)}, \mathbb{F}\Big) - V\Big(\ba^{(3)}, \mathbb{F}\Big)\ .
\end{align}
This modification results in a change in the exploration parameter. We denote this new exploration parameter as $N_{13}(\varepsilon)$ and formally define it as
\begin{align} 
\label{eq:number of samples main paper beta13}
 & N_{13}(\varepsilon) \triangleq 256K  \e \left(\frac{2K \mcL}{\delta_{13}(\varepsilon)}\right)^{\frac{2}{q}} \bigg[\frac{32}{\sqrt{\rm e}} + \log^{\frac{1}{2}} \Big( 2K T^2   \big(\varepsilon^{-(K-1)} + 1\big) \Big) \bigg]^2\ .
\end{align}
In order to quantify the scaling of $\delta_{13}(\varepsilon)$ with shrinking $\varepsilon$, we define
\begin{align}
\label{eq:beta_new_def}
    {\bar\beta}\triangleq \lim_{\varepsilon \to 0}\frac{\log \delta_{13}(\varepsilon)}{\log \varepsilon}\  .
\end{align}
Next, similarly to Theorem~\ref{theorem:PM-ETC-M}, we provide an $\varepsilon$-independent regret bound under the existence of ${\bar\beta}$, which is obtained by setting the discretization interval as an appropriate function of $q$ and ${\bar\beta}$. Specifically, we set
\begin{align}
\label{eq:ETC_var_gamma_13}
    \varepsilon = \Theta\left(\left(K^{2+\frac{2}{q}} \cdot \frac{\log T}{T^{{\bar\gamma}}}\right)^{\frac{q}{2{\bar\beta}}}\right)\ , \qquad \mbox{where} \quad {\bar\gamma}\triangleq \frac{2{\bar\beta}}{2{\bar\beta} + q}\ .
\end{align}
This choice of $\varepsilon$ leads to the following $\varepsilon$-independent regret bound under the existence of ${\bar\beta}$.
\begin{theorem}[PM-ETC-M -- $\varepsilon$-independent]
\label{theorem:PM-ETC-M_beta13}
For all $T > N_{13}(\varepsilon)$ defined in \eqref{eq:number of samples main paper beta13} and when ${\bar\beta}$ exists, the minimum feasible regret that PM-ETC-M incurs is upper-bounded as
\begin{align}
\mathfrak{R}_{\bnu}^{{\textnormal{\rm E}}}(T)\ & \leq O\Bigg(\Bigg[K^{c_{\rm E}}\cdot  \frac{\log T}{T^{\gamma}}\Bigg]^{\frac{q^2}{2{\bar\beta}}}\Bigg)\ ,
\end{align}
where $c_{\rm E}\triangleq {2(1+\frac{{\bar\beta}+1}{q})}$ and $\gamma = \frac{2{\bar\beta}}{2{\bar\beta}+q}$.
\end{theorem}
\begin{proof}
    See Appendix~\ref{Appendix:ETC_theorem_epsilon_proof_13}.
\end{proof}

\paragraph{Regret Bounds for $K$-arm Bernoulli Bandits.} A $K$-arm Bernoulli bandit where $K >2$ constitutes a corner case in which $\beta$ does not exist, as the conditions of Lemma~\ref{lemma:beta12} are violated. Consequently, we analytically derive the ${\bar\beta}$ values for $K$-arm Bernoulli bandit models and report their values in Table \ref{table:beta_values}. In Table \ref{table:beta_values}, for the first $5$ rows, both $\beta$ and $\bar\beta$ was identified. However, for the last $4$ rows, when the distortion function is non-monotone, $\beta$ could not be identified. For some PMs, we have a range of $\beta$ values. The corresponding regret bounds for the PM-ETC-M algorithm is provided in Table~\ref{table:table_regrets_bernoulli}, as obtained from Theorem~\ref{theorem:PM-ETC-M_beta13}. The first column of regret is when the suboptimality gap considered is $\delta_{13}(\varepsilon)$, while the second column of regret is when the suboptimality gap considered is $\delta_{12}(\varepsilon)$. For PMs associated with monotonically increasing distortion functions $h$, Lemma~\ref{lemma:beta_lemma_monotone} establishes that such PMs are maximized by solitary arms. In these cases, an alternative discretization strategy for the PM-ETC-M algorithm, distinct from~\eqref{eq:discrete_set_ETC}, can be employed by setting $\varepsilon = 1$. This choice implies that the estimation regret $\delta_{01} = 0$, and only the arm selection regret remains, as reported in the last column of Table~\ref{table:table_regrets_bernoulli}. While under the sub-optimality gap $\delta_{12}(\varepsilon)$, we report better regret guarantees,  for the sub-optimality gap $\delta_{13}(\varepsilon)$, we can characterize all $\bar\beta$ values for the PMs in Table \ref{table:beta_values}.

In Table \ref{table:beta_values}, we specify $\beta$, $\bar\beta$ for various PMs for different PMs in the $K$-arm Bernoulli bandit setting. For some PMs, we obtain exact values of $\beta$ and ${\bar\beta}$, and for other PMs, we determine valid ranges. These values are characterized in Appendix~\ref{Appendix:Beta_Appendix}. 
 Analyzing these values involves two key steps. First, we characterize an upper-bound on $\bar\beta$  by bounding $\delta_{13}(\varepsilon)$ from below. Subsequently, we establish a lower-bound on $\bar\beta$ by bounding $\delta_{13}(\varepsilon)$ from above. If the two bounds match, we have obtained an exact value of $\bar\beta$. Alternatively, we report a range for it. In Table \ref{table:beta_values}, we report ${\bar\beta}$ for several PMs, where for some of the PMs, $\beta$ does not exist.

\setlength{\textfloatsep}{5pt}
\begin{table*}[t!]
  \centering
  \begin{minipage}{\textwidth}
    %\footnotesize
    \caption{Regret bound of the ETC‐type algorithm for the $K$‐arm Bernoulli bandits, where $\varpi(T)\!\triangleq\!\sqrt{\tfrac{\log T}{T}}$.}
    \label{table:table_regrets_bernoulli}
    \vspace{3pt}
    \centering
    \begin{tabular}{|l|c|c|}
      \hline
      \textbf{Preference Metric%
        \footnote{In the shaded rows, mixtures of arms are optimal. In non‐shaded rows, solitary arms are optimal and the discretization constant is $\varepsilon=1$. }
      }
      & \textbf{$\delta_{13}(\varepsilon)$}
      & \textbf{ $\delta_{12}(\varepsilon)$} \\
      \hline\hline
      Mean Value            & $O\!\bigl(T^{\tfrac16}\varpi(T)\bigr)$           & $O\!\bigl(\varpi^{2}(T)\bigr)$            \\
      Dual Power ($s\ge2$)  & $O\!\bigl(T^{\tfrac16}\varpi(T)\bigr)$           & $O\!\bigl(\varpi^{2}(T)\bigr)$            \\
      Quadratic ($s\in[0,1]$)& $O\!\bigl(T^{\tfrac16}\varpi(T)\bigr)$           & $O\!\bigl(\varpi^{2}(T)\bigr)$            \\
      CVaR$-c$              & $O\!\bigl(T^{\tfrac16}\varpi(T)\bigr)$           & $O\!\bigl(\varpi^{2}(T)\bigr)$            \\
      PHT ($s=\tfrac12$)    & $O\!\bigl(T^{\tfrac{1}{40}}\varpi^{\tfrac14}(T)\bigr)$ & $O\!\bigl(\varpi(T)\bigr)$     \\
      \rowcolor{gray!30}
      Mean–Median Deviation  & $O\!\bigl(T^{\tfrac16}\varpi(T)\bigr)$           & $\cdots$         \\
      \rowcolor{gray!30}
      Inter–ES Range         & $O\!\bigl(T^{\tfrac16}\varpi(T)\bigr)$           & $\cdots$            \\
      \rowcolor{gray!30}
      Wang’s Right–Tail Dev. & $O\!\bigl((\log T)^{\tfrac{5}{72}}\varpi^{\tfrac19}(T)\bigr)$ & $\cdots$ \\
      \rowcolor{gray!30}
      Gini Deviation         & $O\!\bigl((\log T)^{\tfrac{1}{20}}\varpi^{\tfrac25}(T)\bigr)$ & $\cdots$ \\
      \hline
    \end{tabular}
  \end{minipage}
\end{table*}
\setlength{\textfloatsep}{5pt}

\setlength{\textfloatsep}{5pt}
\begin{table*}[h]
\caption{Gap parameters $\beta$ and  $\bar\beta$ for Bernoulli Bandits.}
\label{table:beta_values}
\begin{minipage}[t]{\textwidth}
\centering
{
%\fontsize{8.5pt}{10pt}\selectfont
\begin{tabular}{|l|l|l|l|l|}
\hline 
\textbf{Preference Metrics}
& \(h(u)\)
& parameter
& \(\beta\)
& \(\bar\beta\)\\
\hline\hline 
Mean Value     
& \(u\) 
& 
&  $1$ 
&  $1$ \\ 
Dual Power                 
& \(1-(1-u)^s\) 
& \(s \in [2,+\infty)\)  
&  $1$ 
&  $1$  \\
Quadratic                   
& \((1+s)u - s u^2\)  
& \(s\in[0,1]\) 
&  $1$ 
&  $1$  \\
CVaR\(-c\)
& \(\min\!\bigl\{\tfrac{u}{1-c},\,1\bigr\}\) 
& 
&  $1$\footnote{When arm means are smaller than $<1-c$.} 
&  $1$  \\ 
PHT               
& \(u^s\)  
& \(s \in (0,1)\) 
&  $[s, 1]$
&  $[s, 1]$  \\
\rowcolor{gray!30}
Mean–Median Deviation             
& \(\min\{u,\,1-u\}\)  
& 
&  –
&  $1$  \\
\rowcolor{gray!30}
Inter-ES Range       
& \(\min\!\bigl\{\tfrac{u}{1-c},1\bigr\} + \min\!\bigl\{\tfrac{c-u}{1-c},0\bigr\}\)
& \(c=\tfrac12\)
&  –
&  $1$  \\
\rowcolor{gray!30}
Wang’s Right-Tail Deviation             
& \(\sqrt{u}-u\) 
& 
&  –
&  $[0.5,2]$  \\
\rowcolor{gray!30}
Gini Deviation        
& \(u(1-u)\) 
& 
&  –
&  $[1,2]$  \\
\hline 
\end{tabular}
}
\end{minipage}
\end{table*}

%\newpage 

\section{PM-centric Anytime Algorithm}
\label{sec: anytime}

A computational bottleneck for discretization is that the discrete space's dimension scales exponentially in the number of arms $K$. Hence, for a discretization interval length $\varepsilon$, we face an $O(\frac{1}{\varepsilon^K})$-dimensional discrete optimization problem for both the PM-ETC-M and the PM-UCB-M algorithms. Additionally, these algorithms crucially depend on the horizon information $T$. To circumvent these shortcomings, in this section, we propose the \textbf{C}onfidence-guided \textbf{I}nterval \textbf{R}efinement and \textbf{T}racking (CIRT) algorithm, denoted by $\pi = {\rm IR}$. This algorithm has two key features. (i) It resolves the computational bottleneck of the PM-ETC-M and PM-UCB-M algorithms. (ii) It is an {\em anytime algorithm}, i.e., it does not require the knowledge of the horizon $T$. Designing such an algorithm introduces the following two challenges.
\begin{itemize}
    \item  \textbf{Challenges in exploration.} Knowing the horizon helps the learner to calibrate the rate of exploration explicitly in terms of the horizon. For instance, the horizon-dependent UCB-based algorithms perform an initial horizon-dependent explicit exploration that ensures sufficiently accurate CDF estimates before starting the UCB-based sampling. This is no longer feasible for the anytime algorithm, which requires designing a horizon-adaptive explicit exploration routine to ensure high-fidelity CDF estimates. In CIRT, we take inspiration from the best arm identification literature~\cite{Garivier2016,mukherjee2023best} to devise an adaptive exploration scheme, which at round $t$, ensures that each arm is explored as much as an appropriately designed increasing function of~$t$.
   \item \textbf{Constructing confidence sets.} Distribution confidence sequences used in the horizon-dependent algorithms are carefully constructed so that the probability of error events decreases as a function of the horizon. As a result, these confidence sequences explicitly depend on the known horizon. In contrast, when the horizon is unknown, confidence sequences can only satisfy a fixed, horizon-independent guarantee on error probabilities, leading to weaker performance guarantees compared to their horizon-dependent counterparts.
\end{itemize}
To address these challenges as well as the exponential growth of the discrete space's cardinality, the central idea of CIRT is to start with a coarse discretization scheme and, over time, adaptively refine it based on the collected data up to a pre-specified granularity. This facilitates solving a low-dimensional optimization (cardinality independent of $K$) at each step. The algorithm proceeds in two steps. The first step, which we call the {\em interval refinement} step, focuses on eliminating intervals from the simplex that are likely sub-optimal and splitting the retained intervals into finer discretization lengths. The second step, referred to as the {\em tracking} step, focuses on minimizing the average regret while committing the arm selection to match an estimated discrete optimal coefficient from the previous step.
\subsection{Algorithm Description}
\label{sec:CIRT_description}

\subsubsection{Interval Refinement Process.} In the first step, the CIRT algorithm adopts a confidence-guided methodology to adaptively eliminate intervals from the simplex that are likely sub-optimal and split the retained intervals into finer discretization lengths. This step is broken down into multiple phases indexed by $\ell \in [L]$, such that at the end of the last phase $L$, the discretization length adheres to a pre-specified accuracy $\epsilon$. Here, the algorithm takes $\epsilon$ as an input, which is the minimum length of the discretization interval prescribed by the user. 
Furthermore, CIRT takes an input $A$, which specifies the order of the optimization problem that the learner can afford at each step. Specifically, $A$ denotes the number of discrete mixing coefficients along each coordinate for performing any optimization at each instant. In each phase, the algorithm shrinks the discretization interval by an order $1/A$. Hence, the number of phases $L$ for conforming to the accuracy $\epsilon$ can be readily verified to be $L=\lceil \log_{\frac{A}{2}} (1/\epsilon) \rceil$. Each phase $\ell\in[L]$ is characterized by three decision rules. 
\begin{enumerate}
    \item  A {\em sampling rule}, which specifies the arm to be sampled at the next instant within each phase.
    \item A {\em stopping time},  which governs the length of the phase.
    \item An {\em interval retention decision}, which specifies the intervals in the simplex $\Delta^{K-1}$  that are retained for the next phase.
\end{enumerate}
For specifying these rules, the CIRT algorithm creates a sequence of ``refined'' discrete sets $\{\Delta_{\ell}^{K-1} : \ell\in[L]\}$. For the first phase, we initialize $\Delta_1^{K-1} = \Delta_{1/A}^{K-1}$, where $\Delta_{\varepsilon}^{K-1}$ has been specified in~\eqref{eq:discrete_set_ETC} for any $\varepsilon\in(0,1)$. Subsequently, CIRT gradually refines the set of mixtures over the subsequent phases, and the exact refinement procedure, i.e., $\Delta_{\ell}^{K-1}$ for each phase $\ell > 1$ will be formally specified shortly (see~\eqref{eq:any_discrete_set}). Furthermore, based on the sets $\Delta_{\ell}^{K-1}$ in phase $\ell\in[L]$, we denote the best mixing coefficient computed according to the estimated CDFs up to the current time instant $t\in\N$ by
\begin{align}
\label{eq:any_alpha}
    \ba_{t}^{\rm IR}\;\triangleq\; \argmax_{\ba \in \Delta^{K-1}_\ell} U_h \left(\sum_{i \in [K]}a(i) \F_{i, t}^{\rm IR} \right) \ .
\end{align}

\paragraph{Stopping Time.} We begin by characterizing the stopping mechanism of the CIRT algorithm, which specifies the length of each phase. At any phase $\ell\in[L]$, CIRT decides to end the current phase when it is confident about the sub-optimal intervals in the phase. We design the stopping rule such that at the end of each phase, the correct interval, i.e, the one containing the optimal solution, will be chosen with a high probability. We denote this probability by $\delta$, which is provided as an input to the algorithm. Accordingly, the distribution confidence sets defined for the CIRT algorithm are given by 
\begin{align}
\label{eq:any_UCB_confidence_sets}
\mcC_t^{\rm IR}(i)\triangleq\bigg\{\eta \in \mcM_i & :\norm{\F_{i,t}^{\rm IR}-\eta}_{\rm W} %\nonumber\\
%&
\leq 16\ \frac{\sqrt{{\rm e} \log{2/\delta_K} } + 32}{\sqrt{\tau^{\rm IR}_t(i)}} \bigg\}\ ,
\end{align}
where we have defined
\begin{align}
\delta_K \triangleq \left( \left(1+\delta \right)^{1/K} - 1 \right)\ .    
\end{align}
Furthermore, similar to PM-UCB-M, for any discrete mixing coefficient $\ba\in\Delta^{K-1}_{\ell}$ at a given phase $\ell \in [L]$ and at any time instant $t\in \N$, we define upper and lower confidence bounds for the estimated utilities as
\begin{align}
\label{eq: CIRT_UCB_LCB}
    {\rm UCB}_t^{\rm IR}(\ba)\;&\triangleq\; \max_{\eta_i \in C_t(i), \forall i \in [K]} U_h\left( \sum\limits_{i\in[K]} a(i) \eta_i \right) ,\\
\text{and}\quad
    {\rm LCB}_t^{\rm IR}(\ba)\;&\triangleq\; \min_{\eta_i \in C_t(i), \forall i \in [K]} U_h\left( \sum\limits_{i\in[K]} a(i) \eta_i \right)  \ .
\end{align}
Using these confidence bounds, we define the stopping rule as the earliest time instant at which the lower confidence bound of the chosen mixing coefficient stops overlapping with the upper confidence bound of all other mixing coefficients. After this point, the learner pauses the sampling process and discretizes the space again for the next phase according to \eqref{eq:any_discrete_set}. Formally, for phase $\ell$ we define the stopping time as
\begin{align}
\label{eq:stopping_rule}
\tau^{(\ell)}\;\triangleq\;\inf\Big\{t\in\N  : {\rm LCB}_t^{\rm IR}\left(\ba_t\right) > \max_{\ba\in \Delta^{K-1}_{\ell}: \ba \neq \ba_t^{\rm IR}}{\rm UCB}_t^{\rm IR}(\ba) \Big\} \ .
\end{align}

\paragraph{Interval Retention Decision.} For any phase $\ell\in[L]$, at the stopping time $\tau^{(\ell)}$, CIRT forms an interval retention decision for performing a more fine-grained optimization in the next phase. To specify the retained interval, the algorithm first obtains the best discrete mixing coefficient at stopping, i.e.,
\begin{align}
\bb^{(\ell)}\;\triangleq\;\ba_{\tau_\ell}^{\rm IR}\ .
\end{align}
Based on this, the retained interval along each coordinate $i\in[K]$ is chosen as the interval of length $\left(\frac{2}{A}\right)^{\ell}$, which contains $b^{(\ell)}(i)$ as the midpoint. More specifically, we zoom into the interval around the empirical best discrete coefficient from the previous phase, i.e., the interval 
\begin{align}
\left[b^{(\ell)}(i) - \frac{2^{\ell -1}}{A^\ell} , b^{(\ell)}(i) + \frac{2^{\ell -1}}{A^\ell}\right]\ ,    
\end{align}
along each coordinate $i\in[K]$. These intervals are further discretized into $A$ levels, and the set of discrete mixing coefficients is constructed such that these coefficients conform to the probability constraint. More formally, we define
\begin{align}
\label{eq:any_discrete_set}
\nonumber
    \Delta^{K-1}_{\ell+1} \triangleq \bigg\{& \ba: %  \triangleq [\mathbf{x}; u]^\top \in [0, 1]^K : \\& 
     \;\;n \in \{\N \cup 0\}, \;  a(i) \triangleq  \frac{2^{\ell}n}{A^{\ell+1}}, \; a(i) \in \left[b^{(\ell)}(i) - \frac{2^{\ell-1}}{A^{\ell}}, b^{(\ell)}(i) + \frac{2^{\ell-1}}{A^{\ell}}\right]\;,\; \mathbf{1}^T\ba = 1\bigg\} \ .
\end{align}

\paragraph{Sampling Rule.} Having characterized the stopping and the interval shrinking mechanisms, we will next specify the sampling process that the algorithm adopts within each phase. The sampling mechanism resembles that of the PM-UCB-M algorithm, with the following two exceptions: 
\begin{enumerate}
    \item It uses {\em sublinear} explicit exploration scheme as opposed to a linear order of exploration of the PM-UCB-M algorithm.
    \item At each step, it chooses the discrete solution that maximizes the empirical utility instead of UCB and this maximizing step is performed on the lower dimensional simplex $\Delta_\ell^{K-1}$ as opposed to $\Delta^{K-1}_{\varepsilon}$, which was the a key element of the PM-UCB-M algorithm.
\end{enumerate}

\paragraph{Explicit Exploration:} The PM-UCB-M algorithm, which presumes the knowledge of the horizon $T$, samples each arm at least $\frac{1}{4}\rho T \varepsilon$ number of times which ensures that no arm is over-explored due to the explicit exploration phase. The UCB sampling rule, together with under-sampling, carefully balances the sampling process such that the arm selection fractions may converge to the (discrete) optimal coefficients, having high-fidelity estimates of the arm CDFs as a result of the initial horizon-dependent explicit exploration phase. However, in the premise of an anytime algorithm, the horizon is {\em unknown}. Consequently, the linear exploration scheme, which is suited for PM-UCB-M, does not work for CIRT. Hence, rather than resorting to an exploration phase followed by the UCB and under-sampling steps, the CIRT algorithm, at every instant $t\in\N$, has to decide whether any arm is under-explored, or the learner should resort to UCB-based under-sampling. Striking a balance between explicit exploration and under-sampling may not be achieved if the order of explicit exploration is {\em linear}, since some arms may be heavily over-sampled. Instead, we resort to a {\em sub-linear exploration} for each arm, and let the under-sampling procedure adjust the allocation fractions for the arms which require a linear order of exploration. Such a sub-linear exploration scheme has been investigated in the best arm identification literature, see, e.g.,~\cite{Garivier2016, mukherjee2023best}. CIRT requires an exploration exponent $\xi>0$ which specifies the order of explicit exploration of the arms for accurate estimation. Accordingly, we define the set of under-explored arms
\begin{align}
    E_t \triangleq \left\{i \in [K] : \tau_{t}^{\rm IR}(i) \leq \left\lceil \bigg(\frac{t}{K}\bigg)^{\frac{1}{1+\xi}} \right\rceil \right\} \ ,
\end{align}
as the set of arms which have been sampled fewer than $\lceil (\frac{t}{K})^{\frac{1}{1+\xi}}\rceil$ times up to $t\in\N$, where $\xi$ dictates the order of the sub-linear exploration.

\textit{Arm selection:} At each instant $t\in\N$, CIRT samples an under-explored arm $i\in E_t$, if the set $E_t\neq \emptyset$. Otherwise, if $E_t=\emptyset$, CIRT samples an arm guided by the current UCB-maximizing coefficient, following the under-sampling principle. The CIRT arm selection rule is specified by  
\begin{align}
\label{eq:any_sampling_rule}
A_{t+1} \triangleq
    \begin{cases}
    \argmin_{i \in E_t} \tau_{t}^{\rm IR}(i), \qquad \qquad \qquad &\text{if} \; E_t \neq \emptyset \\
    &\\
    \argmax_{i \in [K]} \{t a_t - \tau^{\rm IR}_t(i) \}, \; & \text{if} \; E_t = \emptyset
    \end{cases} \ .
\end{align}

\subsubsection{Tracking Process} 

The interval refinement distills the set of discrete mixing coefficients such that the optimal discrete coefficient is contained at the end of $L$ phases with a high probability. We use this property to {\em commit} to a mixing coefficient, up to a permissible estimation regret. Specifically, at the end of the interval refinement step, we perform an additional round of interval refinement by zooming around a small vicinity of the last chosen mixing coefficient $\ba^{(L)}$, and obtain $\Delta_{\epsilon}^{K-1}$ according to the following rule.\eqref{eq:any_discrete_set_tracking}. 
\begin{align}
\label{eq:any_discrete_set_tracking}
    \Delta^{K-1}_{\epsilon} \triangleq \bigg\{& \ba: %  \triangleq [\mathbf{x}; u]^\top \in [0, 1]^K : \\& 
     \;\;n \in \N, \;  a(i) \triangleq \left(n + \frac{1}{2}\right) \frac{2^{L}}{A^{L+1}}, \; a(i) \in \left[b^{(L)}(i) - \frac{2^{L-1}}{A^{L}}, b^{(L)}(i) + \frac{2^{L-1}}{A^{L}}\right]\;,\; \mathbf{1}^T\ba = 1\bigg\} \ .
\end{align}
Subsequently, we select the discrete mixing coefficient $\bb^{(L+1)}$ that maximizes the empirical utility such that
\begin{align}
    \bb^{(L+1)} = \argmax_{\ba \in \Delta_{\epsilon}^{K-1}} U_h\left(\sum_{i \in [K]} a(i)\F_i\right)\ ,
\end{align}
and perform under-sampling according to the chosen mixture similarly to~\eqref{eq:any_sampling_rule}. Formally, for any $t>\tau^{(L)}$, we have    
\begin{align}
A_{t+1} \triangleq
    \begin{cases}
    \argmin_{i \in E_t} \tau_{t}^{\rm IR}(i), \qquad \qquad \qquad &\text{if} \; E_t \neq \emptyset \\
    &\\
    \argmax_{i \in [K]} \{t b^{(L+1)} - \tau^{\rm IR}_t(i) \}, \; & \text{if} \; E_t = \emptyset
    \end{cases} \ .
\end{align}

\begin{algorithm}[h]
         
%		\algsetup{linenosize=\small}
%		\setstretch{0.85}
		\caption{\textbf{C}onfidence-guided \textbf{I}nterval \textbf{R}efinement and \textbf{T}racking (CIRT)}
		\label{algorithm:anytime}
		
 		%\small
 		\begin{algorithmic}[1]
            \STATE \textbf{Input:} $A$, the order of the optimization problem that the learner can afford at each step, $\epsilon$, the minimum length of the discretization interval prescribed by the user, $\delta$, the error probability of not identifying the correct interval at the end of a phase, $\xi$ the exploration exponent
            \STATE \textbf{Initialize:} $\ell = 1$ ($\ell$ is the phase index), $\text{REFINE} \triangleq \TRUE $.
             
        \FOR{$t = 1, \cdots, T$}
    \IF{$\text{REFINE} $}
    \STATE \textbf{Interval Refinement Step:}
        \IF{$(2/A)^{\ell} > \varepsilon$}
			
			    \STATE Select an arm $A_{t}$ specified by~(\ref{eq:any_sampling_rule}) and obtain reward $X_t$\\
                \STATE Update the empirical CDF $\F^{\rm IR}_{A_t,t}$ according to~\eqref{eq:empirical_CDF}
                \STATE Update the confidence set $\mcC_t(A_t)$ according to~\eqref{eq:any_UCB_confidence_sets} 
                \STATE Compute the optimistic estimate $\ba_t^{\rm IR}$ according to~\eqref{eq:any_alpha}
                %\STATE Compare lower confidence bound of $\ba_t^{\rm IR}$ to the upper confidence bounds of other mixing coefficients
                
                \IF{${\rm LCB}\left(\ba^{\rm IR}_{t}\right) > \max_{\ba\in \Delta^{K-1}_{\ell}: \ba \neq \ba_{t}}{\rm UCB(\ba)}$}
                \STATE Update $\Delta^{K-1}_{\ell+1}$ according to \eqref{eq:any_discrete_set}
            \STATE $\ell = \ell + 1$ 
            \ENDIF
            %\STATE $\tau^{(L)} \triangleq t$
                 
                %\STATE $t = t+1$
                \ELSE \STATE REFINE = \FALSE
                \ENDIF
                \STATE Set $\Delta^{K-1}_{\epsilon}$ according to \eqref{eq:any_discrete_set}
            \STATE Select the mixing coefficient $\ba^{\rm IR} =  \argmax_{\ba \in \Delta^{K-1}_\epsilon} U(\sum_{i \in [K]} a(i) \F^{\rm IR}_{i,t})$ 
			\ENDIF

            \IF{not REFINE}

            \STATE \textbf{Tracking Step:} 
                %\STATE Set mixing coefficient $\ba_t^{\rm U+} = \ba^{U+}$
			    \STATE Select an arm $A_{t}$ specified by~(\ref{eq:any_sampling_rule}) and obtain reward $X_t$\\
                \STATE Update the empirical CDF $\F^{\rm IR}_{A_t,t}$ according to~\eqref{eq:empirical_CDF}
            \ENDIF
            \ENDFOR
 		\end{algorithmic}
\end{algorithm}

\subsection{Algorithm Properties}
\label{any_alg_properties}
In this section, properties of various decisions involved in the CIRT algorithm. As discussed in Section~\ref{sec:CIRT_description}, one of the key features of the CIRT algorithm is its computational efficiency. Specifically, CIRT is adaptive to the learner's computational budget $A$. In order for CIRT to exhibit sublinear regret guarantees, it is imperative that an optimal mixing coefficient is retained at the end of every phase, or else, the tracking procedure fails to perform arm selections that mimic an optimal mixing coefficient. Furthermore, since the phases in the interval refinement step are adaptive to the data, we must certify that the stopping time in each phase is finite on average. These properties would collectively ensure a sublinear regret for the CIRT algorithm. We begin by establishing an upper bound on the misidentification error, i.e., the probability that an optimal solution is not contained in the retained interval at the end of the interval retention phase. Subsequently, we will analyze the average interval length of any phase $\ell\in[L]$ of the interval refinement step and show that it is guaranteed to be finite on average. Finally, leveraging the error guarantee and finiteness of the interval lengths, we will provide an upper bound on the regret due to the CIRT sampling rule.

\vspace{-.1 in}
\paragraph{Probability of Error.} 
We begin by providing an upper bound on the probability that at the end of the interval retention phase, CIRT zooms into an interval that does not contain an optimal mixing coefficient. We first provide a few notations. Let us define $\Sigma^{K-1}_{\ell+1}$ as the continuous counterpart of the discrete simplex $\Delta^{K-1}_{\ell+1}$ in any phase $\ell\in[L]$,~as
\begin{align}
\label{eq:any_continous_set}
\nonumber
    \Sigma^{K-1}_{\ell+1} \triangleq \bigg\{ \balpha \in [0,1]^K :  \alpha(i) \in \left[b^{(\ell)}(i) - \frac{2^{\ell-1}}{A^\ell}, b^{(\ell)}(i) + \frac{2^{\ell-1}}{A^\ell} \right] \;,\; \mathbf{1}^T\balpha = 1 \bigg\}\ . 
\end{align}
Furthermore, since $L$ depends on the prescribed accuracy $\epsilon$, let us denote $\Sigma^{K-1}_{\epsilon}\triangleq \Sigma^{K-1}_{L+1}$.

\begin{lemma}
\label{lemma:alpha_star_in}
Consider a $K$-arm Bernoulli bandit instance. For any PM, using any sampling rule paired with the stopping rule defined in \eqref{eq:stopping_rule} and the interval retention decision defined in~\eqref{eq:any_discrete_set}, the probability of retaining an optimal mixing coefficient at the end of the interval retention phase satisfies
\begin{align}
    \P_{\bnu}^{\pi} \Big( \balpha^{\star} \in \Sigma^{K-1}_\epsilon \Big) \geq 1-L \delta \ ,
\end{align}
\end{lemma}
where we recall that $L = \left\lceil \log_{\frac{A}{2}} (1/\epsilon) \right\rceil$.

\begin{proof}
    See Appendix~\ref{Appendix:proof_of_lemma_alpha_star_in}.
\end{proof}

\paragraph{Proof sketch.} The proof follows from upper-bounding the probability that an optimal mixing coefficient is not contained in the refined continuous set $\Sigma_{\ell}^{K-1}$ at any phase $\ell\in [L]$, followed by a subsequent union bound over the phases. To obtain the phase-specific error bound, we leverage the fact that with probability at least $\delta$, the true arm CDFs are contained in the respective distribution confidence sets (i.e., the confidence sets are sufficiently accurate). As a result, when the lower confidence bound of the current empirical best mixing coefficient exceeds the upper confidence bound for all other candidate coefficients, which is the premise of the stopping rule in~\eqref{eq:stopping_rule}, with high probability, we have identified the discrete optimizer of the current phase.

\paragraph{Finiteness of the Phase Durations.} Next, we show that the length of any phase $\ell\in[L]$ is finite on average. A crucial intermediate step to this result is to show that the explicit exploration scheme, as defined in~\eqref{eq:any_sampling_rule}, ensures that each arm is explored at a sublinear rate almost surely. The immediate consequence of this step is a guarantee on the estimation accuracy of the arm CDFs. Alternatively, if arms are under-explored, it might result in a poor estimation accuracy, resulting in sub-optimal choices of interval retention. The next result, which has been significantly leveraged in the best arm identification literature~\cite{Garivier2016, mukherjee2022, mukherjee2023best}, establishes an almost sure lower bound on the frequency of arm selection. 

\newpage 
\begin{lemma}
\label{lemma:any_explicit_exploration}
    Due to explicit exploration of CIRT's sampling rule in~\eqref{eq:any_sampling_rule}, at every instant $t\in\N$, we almost surely have 
    \begin{align}
        \tau_t(i) \geq \left(\frac{t}{K} \right)^{\frac{1}{1+\gamma}}-1\ .
    \end{align}
\end{lemma}
\begin{proof}
    The proof follows the same line of arguments as~\cite[Lemma 17]{Garivier2016}.
\end{proof}
Next, we show that for any phase $\ell\in[L]$, its length is finite on average. Note that the following result may be viewed as a sample complexity guarantee for estimation. Specifically, the lemma establishes that in order to achieve an estimation accuracy no smaller than $\epsilon$, the phased estimation routine of the CIRT algorithm uses a finite number of samples. Hence, for a sufficiently large horizon, the regret of the CIRT sampling routine is purely guided by the tracking procedure. While it is sufficient to show the finiteness of the phase lengths for the purpose of this paper, investigating an instance-dependent upper bound on the sample complexity is a promising direction for future research.

\begin{lemma}
\label{lemma:finite_stopping_time}
    For each phase $\ell\in[L]$, the CIRT algorithm comprising the sampling rule in~\eqref{eq:any_sampling_rule} and the stopping rule in~\eqref{eq:stopping_rule} satisfies
    \begin{align}
        \E_{\bnu}^{\rm IR}\Big[\tau^{(\ell)}\Big] < +\infty\ .
    \end{align}
\end{lemma}
\begin{proof}
    See Appendix~\ref{Appendix:proof_of_finite_stopping_time}.
\end{proof}

\paragraph{Proof sketch.} The core idea in establishing a finite phase length on average is the observation that the distribution confidence sets defined in~\eqref{eq:any_UCB_confidence_sets} are monotonically decreasing functions of time, which follows from Lemma~\ref{lemma:any_explicit_exploration}. Consequently, there exists a finite instant at which point these sets have shrunk enough to become smaller than the sub-optimality gap in utilities between the discrete optimal coefficient and any other mixing coefficient. Leveraging the high-probability event that the true distributions are contained in the distribution confidence sets, the claim follows by noticing that the CIRT algorithm stops the current phase when the confidence sets between the empirical best coefficient cease to overlap with all other coefficients. Specifically, when the diameter of distribution confidence sets falls below the sub-optimality gap, the CIRT algorithm stops the current phase since any overlap between the distribution confidence sets is no longer feasible. We show that the time instant at which this happens is finite on average. 

\subsection{Regret Analysis}

In this section, we provide an upper bound on the regret incurred by the CIRT algorithm. Unlike Section~\ref{sec:analysis}, which quantifies the {\em average} regrets of the PM-ETC-M and the PM-UCB-M algorithms, we report a {\em probabilistic} bound on the regret of the CIRT algorithm. Specifically, with slight abuse of notation, we redefine the regret for the CIRT algorithm as
\begin{align}
\label{eq:any_regret}
    \mathfrak{R}^{\rm IR}_{\bnu}(T) \triangleq U_h\left(\sum_{i \in [K]} \alpha^{\star}(i) \F_i \right) - U_h\left(\sum_{i \in [K]} \frac{\tau_T^{\rm IR}(i)}{T} \F_i \right)
\end{align}
Then, we can decompose the regret into \emph{estimation regret} and \emph{sampling error}.
\begin{align}
    \mathfrak{R}^{\rm IR}_{\bnu}(T) &= \underbrace{U_h\left(\sum_{i \in [K]} \alpha^{\star}(i) \F_i \right) - U_h\left(\sum_{i \in [K]} a^{\rm IR}_T(i) \F_i \right)}_{\triangleq\;\delta_\tau(T)\text{ (estimation regret)}}  + \underbrace{U_h\left(\sum_{i \in [K]} a^{\rm IR}_T(i) \F_i \right)  - U_h\left(\sum_{i \in [K]} \frac{\tau_T^{\rm IR}(i)}{T}\F_i \right)}_{\triangleq\;H_2(T)\text{ (sampling error)}} \ .
\end{align}
The only difference between the regrets defined in~\eqref{eq:regret} and~\eqref{eq:any_regret} is the absence of the expectation operator with respect to the policy. The regret defined in~\eqref{eq:any_regret} is stochastic, and in Theorem~\ref{theorem:any_regret}, we provide a high probability upper bound on it due to the CIRT algorithm. Before stating the theorem, we remark that the result holds only for Bernoulli bandits. This is an inherent consequence of adaptive discretization, and the fact the probability of error bound in Lemma~\ref{lemma:alpha_star_in} has been shown to hold for Bernoulli bandits. While the result may seem restrictive, we would like to remark that even a {\em refinement oracle}, that knows the underlying arms' distributions, and forms interval refinement decisions based on this information, is not guaranteed to recover the best discrete mixing coefficient at the end of the refinement phase. Such a refinement oracle has been shown to retain the optimal mixing coefficient for Bernoulli bandits in Lemma~\ref{lemma: best_alpha_Bernoulli_mixture}. While extending this result to other distributions is interesting, we conjecture that it may not hold for distributions that yield non-concave PMs. Next, we state an upper bound on the regret of the CIRT algorithm for Bernoulli bandits.
\begin{theorem}
\label{theorem:any_regret}
    Let $\bnu$ denote a $K$-arm Bernoulli bandit instance. For any $\epsilon \in \R^+$ and distortion function $h$ with \holder~exponent $q$, there exist a stochastic time $\tau^{\epsilon}$
    %, $\E[\tau^{\epsilon}] < \infty $, 
    such that with a probability at least $1-\delta \lceil \log_{\frac{A}{2}} (1/\epsilon) \rceil $, we have
    \begin{align}
        \mathfrak{R}^{\rm IR}_{\bnu}(T) \leq \mcL W^q K^{q+1}  \left(\epsilon^q + T^{-q}\right) \ , \qquad \forall T > \tau^\epsilon\ .
    \end{align}
\end{theorem}
\begin{proof}
    See Appendix~\ref{sec:proof_any_regret}. 
\end{proof}
\par Before providing the proof sketch of Theorem~\ref{theorem:any_regret}, we provide some insights on the result. Note that in Theorem~\ref{theorem:any_regret}, the arm selection regret scales as $O(T^{-q})$, which matches the regret bound of the PM-ETC-M algorithm in Theorem~\ref{theorem: ETC upper bound}, and improves upon the regret bound of the PM-UCB-M algorithm in Theorem~\ref{theorem:UCB upper bound} by an order of $O(T^{-\frac{q}{2}})$. This observation stems from the fact that the CIRT algorithm inherits the better features of the PM-ETC-M and the PM-UCB-M algorithms. On the one hand, as opposed to the PM-ETC-M algorithm, which assumes knowing instance-dependent parameters through the length of its exploration phase, CIRT's interval refinement phase is purely data-adaptive, requiring no instance-dependence. The overarching idea here is to use PM-UCB-M's adaptive confidence set construction to guide the estimation routine to form a reliable estimate of the optimal mixing coefficient. On the other hand, from Theorem~\ref{theorem:UCB upper bound}, we observe that PM-UCB-M's rationale of forming estimates in each round and subsequently tracking these estimates may hurt the arm selection regret, since, while more refined estimates of the mixing coefficients should be helpful to the algorithm, it may not be required after a certain point where we have reached the desired estimation accuracy. The CIRT algorithm identifies this drawback and commits to an estimate that it forms from the interval refinement phase. 

We note that the first term in the regret, i.e., the regret contribution from discretizing the mixing‐coefficient grid, is driven entirely by the user’s choice of accuracy level, $\epsilon$. A tighter grid (smaller $\epsilon$) directly reduces this error term but also weakens the regret guarantee. Because the user specifies $\epsilon$ upfront, CIRT lets practitioners calibrate exactly how much estimation regret they are willing to accept in exchange for faster runtime. In contrast, horizon‐aware schemes can tie their discretization levels to the time horizon $T$, often forcing finer discretization (and higher cost) than strictly needed. When limited compute resources or rapid, anytime deployment are priorities—and a modest, user‐controlled estimation regret is acceptable — CIRT presents a clear practical advantage. Additionally, note that the statement in Theorem~\ref{theorem:any_regret} holds for $T>\tau^\epsilon$, where $\tau^\epsilon$ is a stochastic quantity which is finite on average. It can be readily verified using Markov's inequality that the probabilistic guarantee in Theorem~\ref{theorem:any_regret} holds in the asymptotic limit of the horizon $T$. \vspace{-.1 in}

\paragraph{Proof sketch of Theorem~\ref{theorem:any_regret}.} We first decompose the regret into an estimation regret component and an arm selection regret component. A bound on the arm selection regret follows the same arguments as Theorems~\ref{theorem:PM-ETC-M} and~\ref{theorem:UCB upper bound}. The fundamental analytical crux appears in bounding the second term, i.e., quantifying the gap between the utility evaluated at the discrete optimal solution and CIRT's arm selection fractions. A sublinear regret crucially hinges on the premise that the optimal solution is contained within the interval obtained after the last phase at the end of the interval refinement step, which holds due to Lemma~\ref{lemma:alpha_star_in}. The regret bound in Theorem~\ref{theorem:any_regret} is conditioned on the event that this indeed holds, which is the source of the probabilistic guarantee on CIRT's regret bound. Subsequently, the analysis proceeds by quantifying the sample complexity required for the desired estimation accuracy $\epsilon$, followed by upper bounding the regret in the tracking phase incurred by the under-sampling routine. The finiteness of the average sample complexity for a sufficiently accurate estimation is an immediate consequence of Lemma~\ref{lemma:finite_stopping_time}. Finally, the analysis concludes by upper-bounding the sampling estimation error in the tracking phase. While the analysis is analogous to upper-bounding the sampling estimation error of the PM-UCB-M algorithm in Theorem~\ref{theorem:UCB upper bound}, there are subtle differences due to the order of explicit exploration being sublinear. For further details, we refer to Appendix~\ref{sec:proof_any_regret}.

\section{Empirical Evaluations}
\label{sec: experiments}

In this section, we provide empirical evaluations of the horizon-dependent algorithms and the anytime algorithm, namely PM-ETC-M, PM-UCB-M, CE-UCB-M, and CIRT algorithms. First, in Section~\ref{sec:emp_horizon_dependent}, we provide empirical evaluations for the horizon-dependent algorithms. We begin by empirically validating our claim (in Theorem~\ref{lemma:mixture_lemma_exponential}) that for strictly concave utilities, the PM may be maximized by an optimal mixture. Subsequently, we empirically investigate two distinct aspects: (1) the scaling of the (average) regret versus horizon, and (2) the impact of the discretization level $\varepsilon$ on the regret. Next, we digress to the anytime algorithm in Section~\ref{section:emp_eval_anytime}, where we show (1) the scaling of the (average) regret versus horizon after the tracking step, and (2) showcase the computational advantage of the anytime algorithm against the fixed-discretization counterpart.

\paragraph{Utilities and Bandit Models.} We empirically evaluate the PM-centric algorithms for different PMs, specifically, for the Gini deviation (GD) with distortion function $h(u)=u(1-u)$, Wang's right-tail deviation (WRTD) with distortion function $h(u)=\sqrt{u}-u$, and mean (expected value) with distortion function $h(u)=u$. For the empirical evaluations in Section \ref{sec:emp_horizon_dependent} (horizon-dependent algorithms) and in Section~\ref{section:emp_eval_anytime} (anytime algorithm), we focus on Bernoulli bandits with mean vectors $\bp\in[0,1]^K$, in which case, the PM has a closed-form expression for the case of mixtures, given by $V(\balpha,\F)\;=\; \langle \balpha , \bp\rangle (1-\langle \balpha, \bp\rangle )$. Additionally, for horizon-dependent algorithms, we consider Gaussian bandit models, for which we do not have a closed form.

\subsection{Horizon-dependent Algorithms}
\label{sec:emp_horizon_dependent}

We first show that when the optimal solution is a mixture, the resulting PM is strictly higher than those yielded by any single arm. Subsequently, we empirically investigate the regrets that PM-UCB-M and PM-ETC-M algorithms incur for the Bernoulli bandit for different horizons. In the next step, we move on to Gaussian bandits and investigate the regrets of CE-UCB-M and PM-ETC-M algorithms for different horizons. Finally, we perform an ablation study to assess the impact of the discretization level $\varepsilon$ on regret.

\paragraph{Mixture Optimality.} We investigate the observations established in Theorems~\ref{lemma:convex_solitary} and~\ref{lemma:mixture_lemma_exponential} for the optimality of solitary policies for convex distortion functions and the optimality of mixture policies for strictly concave ones. For strictly concave cases, we consider GD and WRTD PMs, and for the convex case, we consider the mean PM. To empirically showcase the optimality of mixtures, for a $3$-armed bandit instance, we evaluate the PM achieved by PM-UCB-M against the utility achieved by solitary arm baselines for each arm indexed by $\{0,1,2\}$. Each experiment is averaged over $100$ independent trials. In Figures~\ref{fig:vs_arms_gini} and~\ref{fig:vs_arms_wang}, we observe that after a sufficiently large horizon, for GD and WRTD, the PM achieved by PM-UCB-M is strictly higher than those of the solitary arm baselines, which corroborates the finding in Theorem~\ref{lemma:mixture_lemma_exponential}. Furthermore, for the mean value, in Figure~\ref{fig:vs_arms_mean} we observe that the utility is maximized by arm $1$, and PM-UCB-M almost yields a lower PM, which aligns with the property foreseen by Theorem~\ref{lemma:convex_solitary}. Furthermore, these observations establish the important property that the PM-UCB-M algorithm automatically tracks the optimal mixture or the optimal solitary arm, depending on the underlying setting. 

\begin{figure}[h]
  \centering
  \begin{subfigure}[b]{0.32\linewidth}
    \centering
    \includegraphics[width=\linewidth]{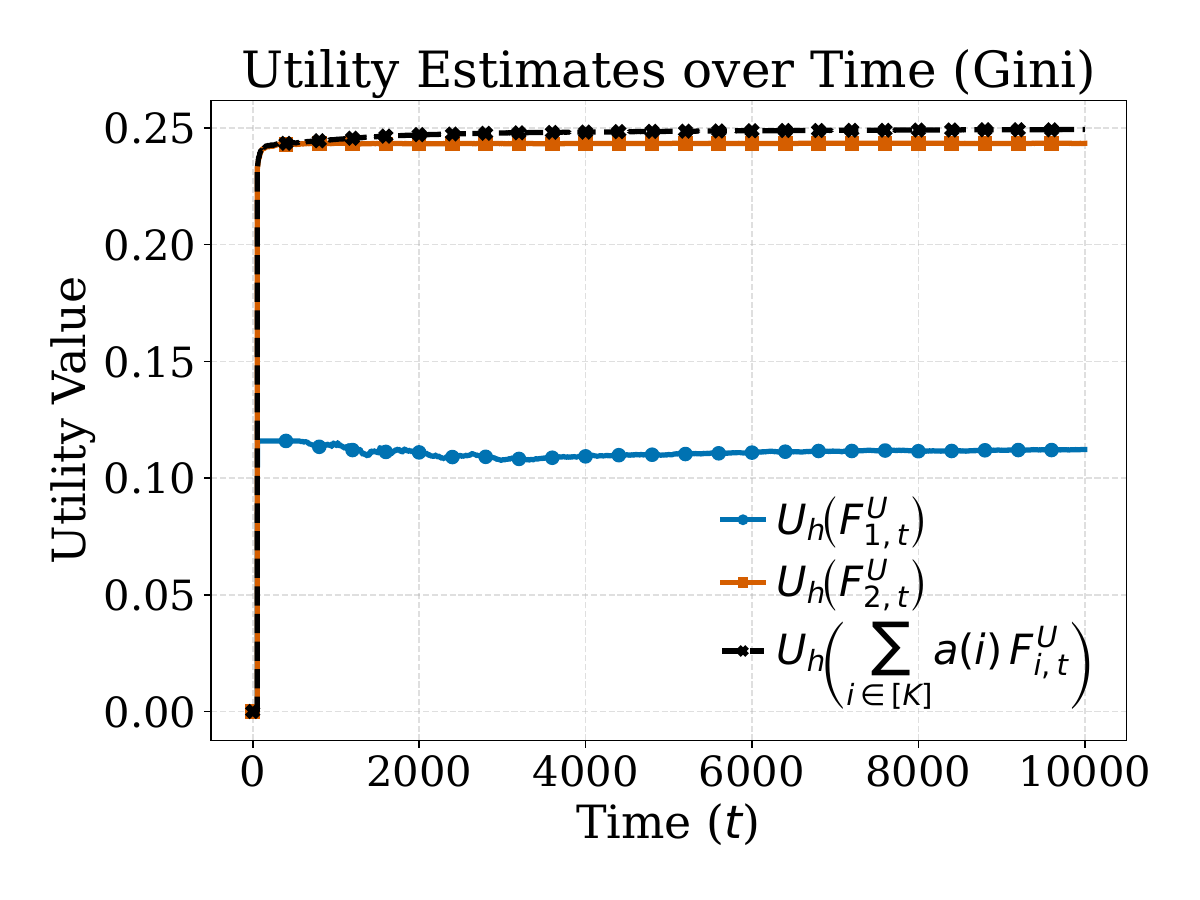}
    \caption{GD.}
    \label{fig:vs_arms_gini}
  \end{subfigure}\hfill
  \begin{subfigure}[b]{0.32\linewidth}
    \centering
    \includegraphics[width=\linewidth]{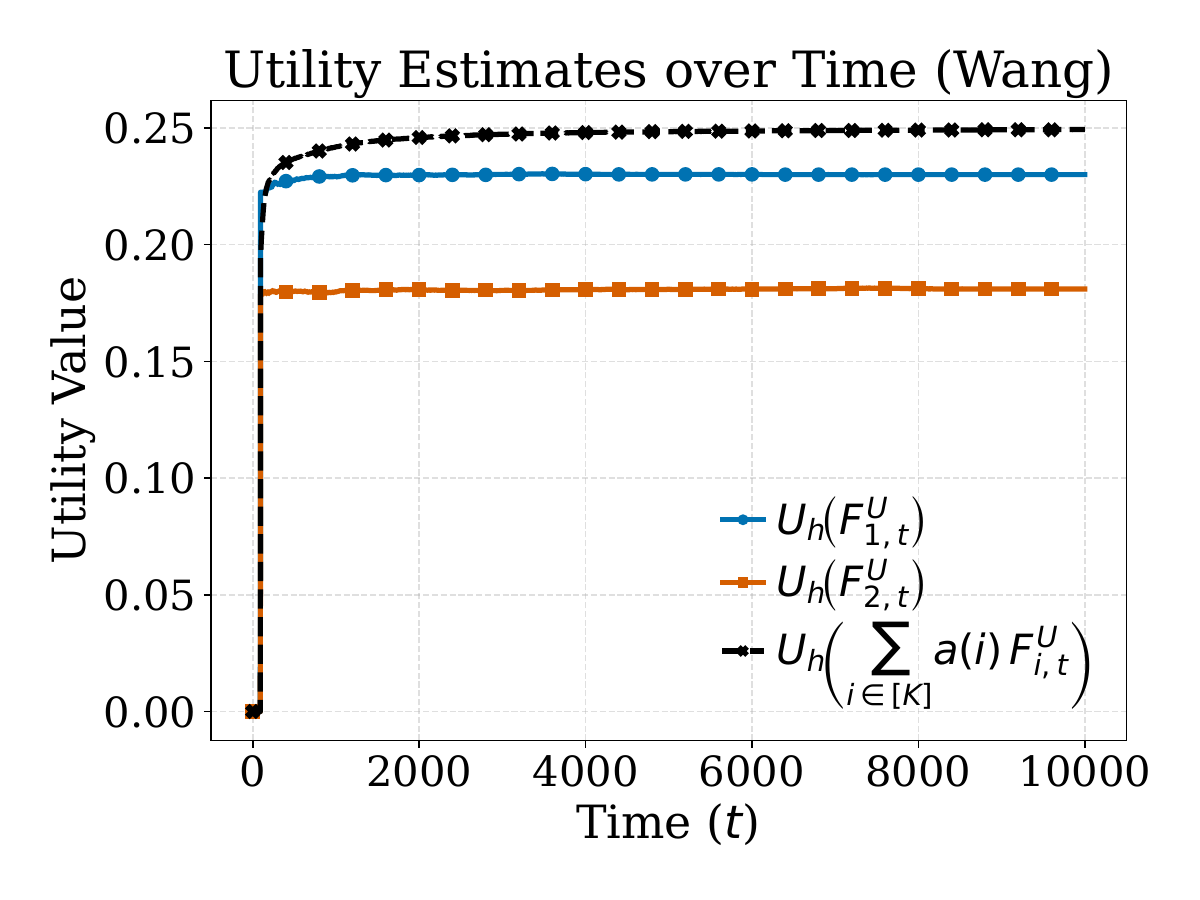}
    \caption{WRTD.}
    \label{fig:vs_arms_wang}
  \end{subfigure}\hfill
  \begin{subfigure}[b]{0.32\linewidth}
    \centering
    \includegraphics[width=\linewidth]{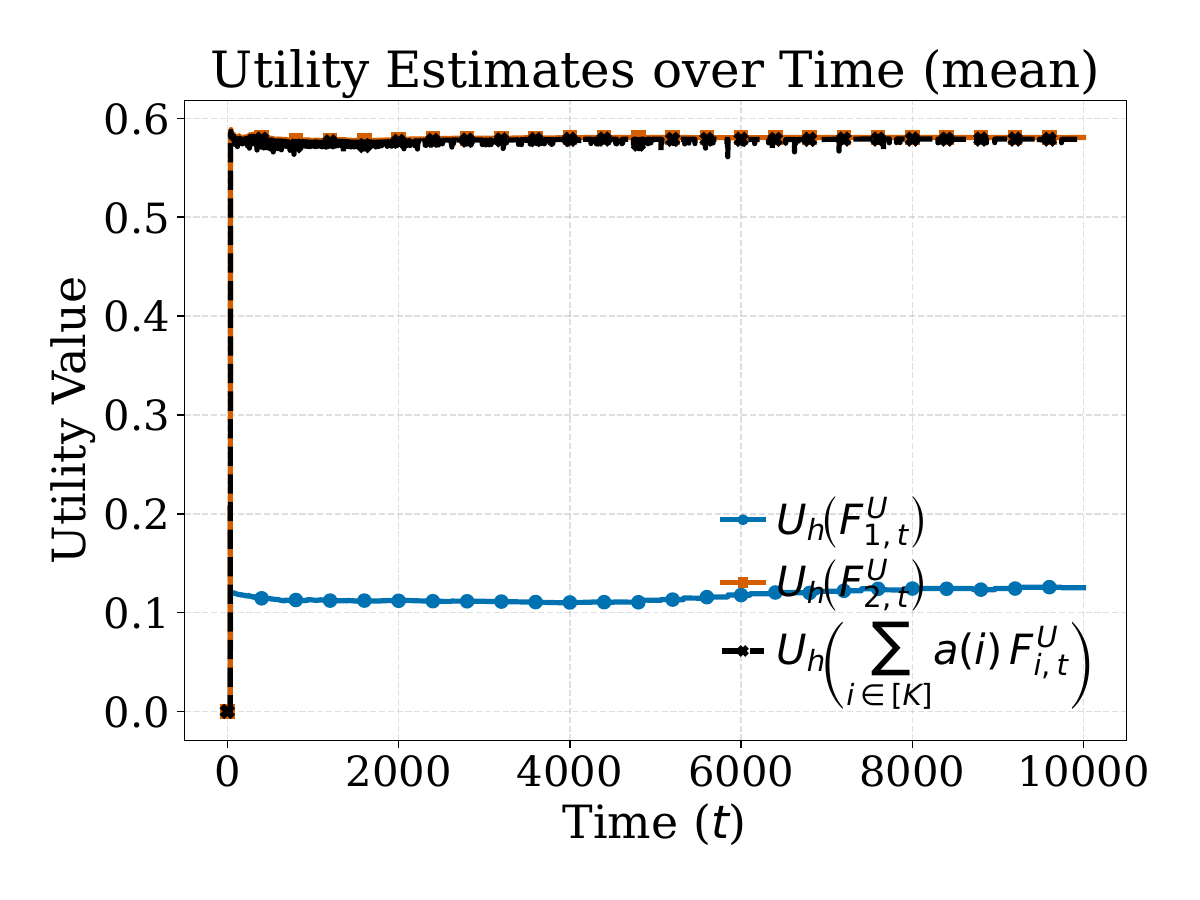}
    \caption{mean}
    \label{fig:vs_arms_mean}
  \end{subfigure}
  \caption{Empirical utility values for GD, WRTD, and mean for PM-UCB-M algorithm.}
  \label{fig:vs_arms_all}
\end{figure}

\paragraph{Regret versus Horizon (Bernoulli Bandits).}
%\label{section:emp_eval_bern}
%\paragraph{Regret versus horizon.} 
Next, we evaluate the regret incurred by the PM-ETC-M and PM-UCB-M algorithms for various horizon values $T$. For performance comparison, we adopt the \emph{uniform sampling} as the baseline. The uniform sampling strategy allocates the same allocation resource for each arm, i.e., each arm is sampled $\big\lfloor \frac{T}{K}\big\rfloor$ times, with the remaining instants uniformly randomly assigned to arms. 

For the PM-ETC-M algorithm, note that $T>N(\varepsilon)$ is a necessary condition. However, for the utilities and bandit instances in consideration, $N(\varepsilon)$ may potentially be of the order~$O(T)$, and hence, providing experiments with a larger horizon is computationally prohibitive. Instead, we fix the horizon $T$ within an implementable range. If $N(\varepsilon) > \frac{T}{2}$, we set $N(\varepsilon) = \frac{1}{2}K\varepsilon T$. Furthermore, we choose $\varepsilon$ according to values specified in Section \ref{sec:analysis}. All the experiments are averaged over $100$ independent trials.
As seen in Figures~\ref{fig:gini_2_uni}, ~\ref{fig:wang_2_uni}, the uniform sampling strategy incurs a large regret in all our experiments. 
We observe that the regret of PM-ETC-M has a larger variance compared to PM-UCB-M. This is expected, as it commits to one mixing coefficient while PM-UCB-M continuously uses the information acquired from the rewards to choose a mixing coefficient. Therefore, if PM-ETC-M commits to an optimal mixing coefficient in one run, the resulting regret will be low. However, if it commits a suboptimal mixing coefficient, the regret will be higher. As PM-UCB-M does not commit to a mixing coefficient and instead continuously uses the new information to have a better estimate for the mixing coefficient, the resulting regret varies less. 

\begin{figure}[h]
  \centering
  \begin{subfigure}[b]{0.45\linewidth}
    \centering
    \includegraphics[width=\linewidth]{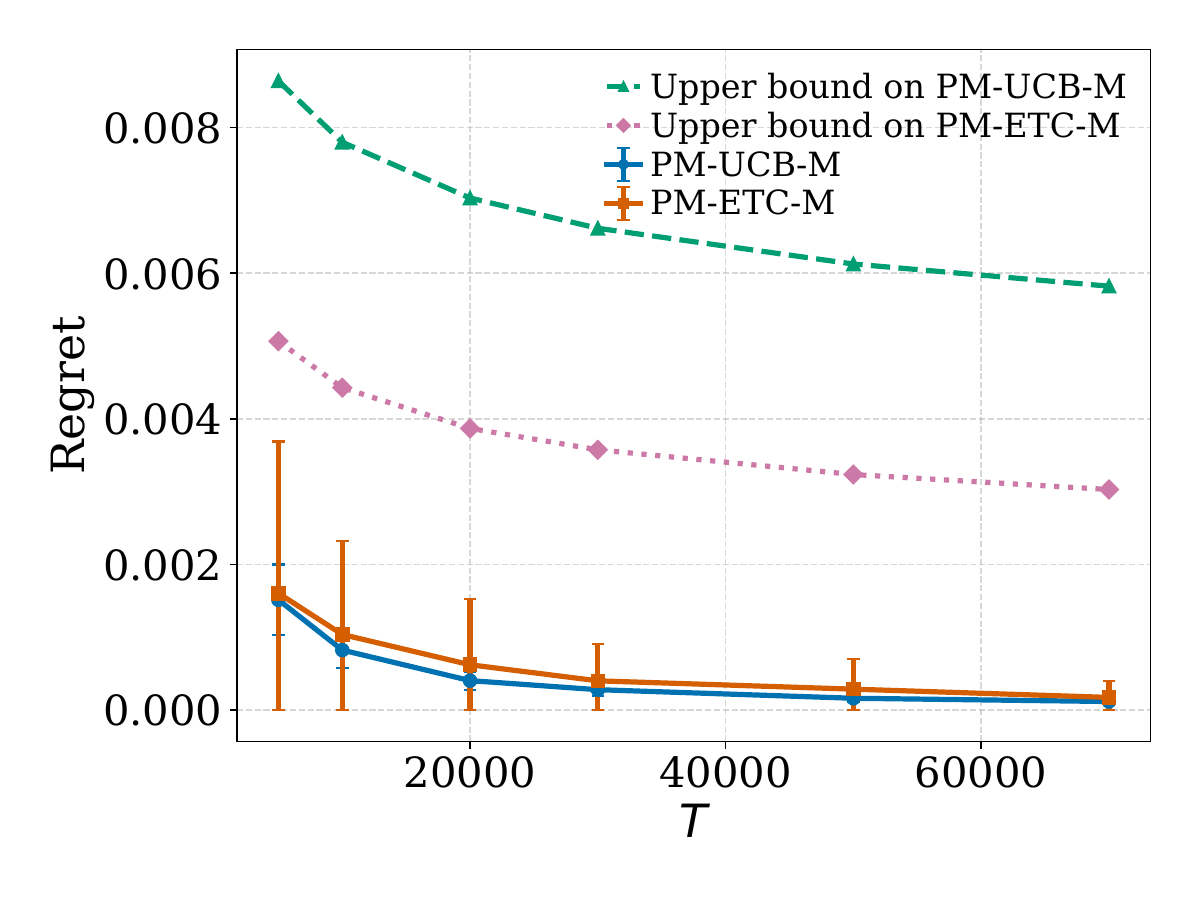}
    \caption{Regret of PM-ETC-M vs PM-UCB-M algorithms for varying horizons for GD.}
    \label{fig:gini_2_dev}
  \end{subfigure}\hfill
  \begin{subfigure}[b]{0.45\linewidth}
    \centering
    \includegraphics[width=\linewidth]{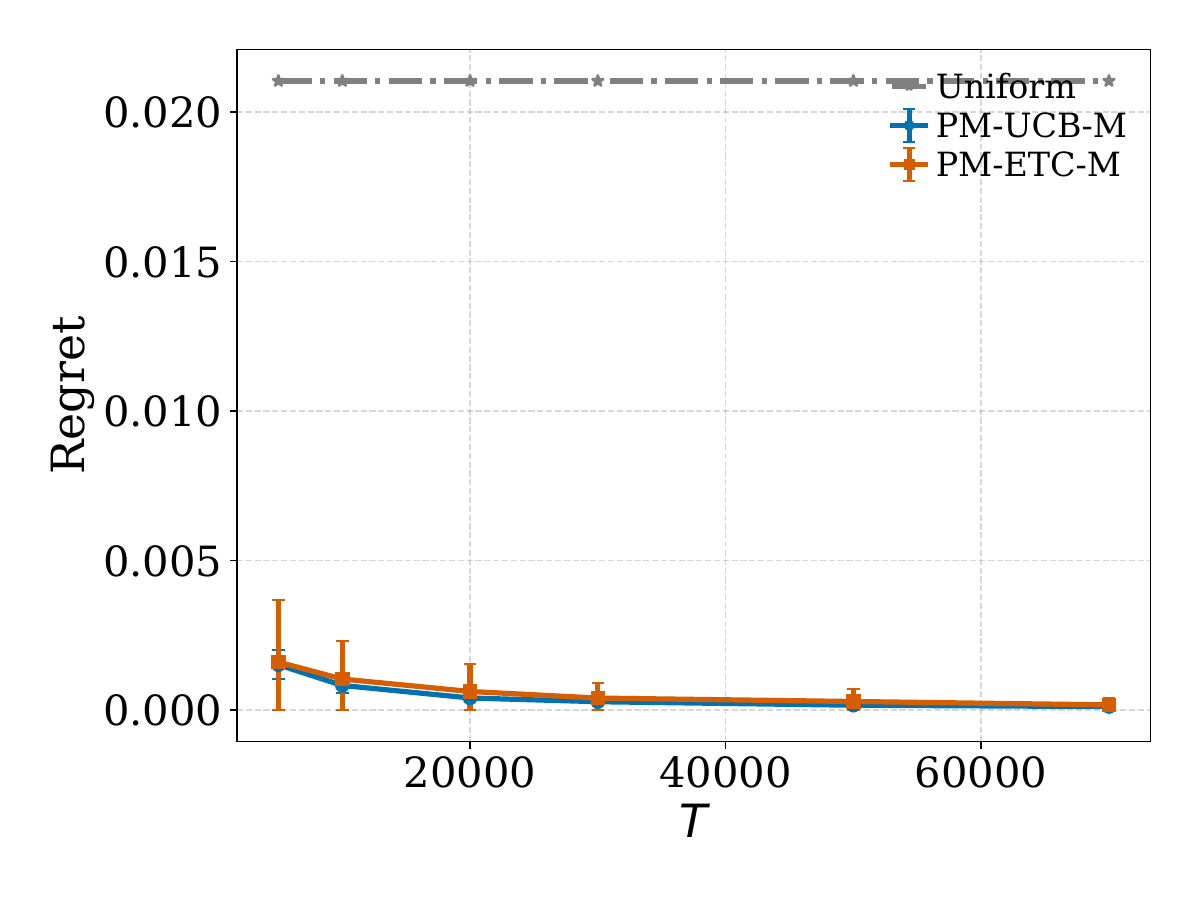}
    \caption{Horizon-dependent algorithms compared to Uniform sampling for GD}
    \label{fig:gini_2_uni}
  \end{subfigure}

  \vspace{1em}

  %–– Bottom row: Wang ––
  \begin{subfigure}[b]{0.45\linewidth}
    \centering
    \includegraphics[width=\linewidth]{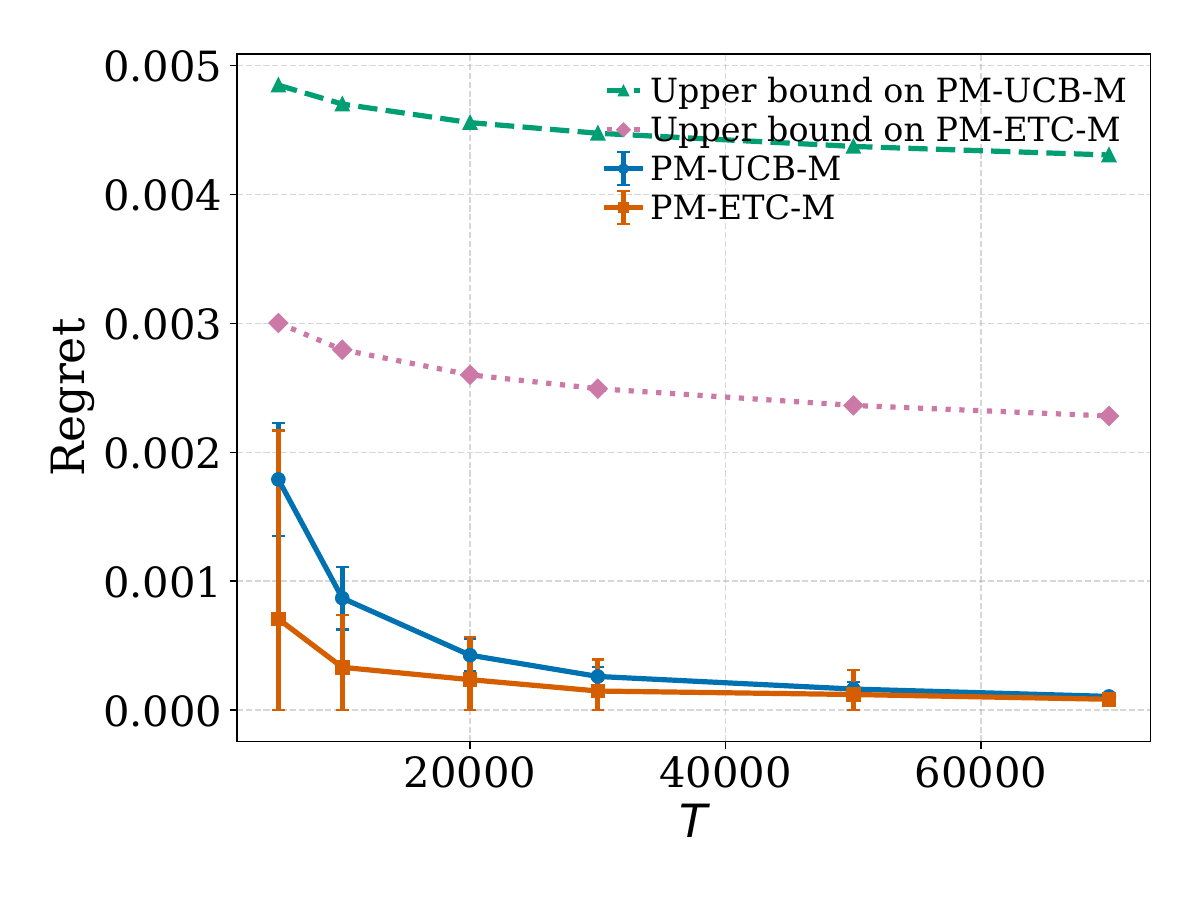}
    \caption{Regret of PM-ETC-M vs PM-UCB-M algorithms for varying horizons for WRTD.}
    \label{fig:wang_2_dev}
  \end{subfigure}\hfill
  \begin{subfigure}[b]{0.45\linewidth}
    \centering
    \includegraphics[width=\linewidth]{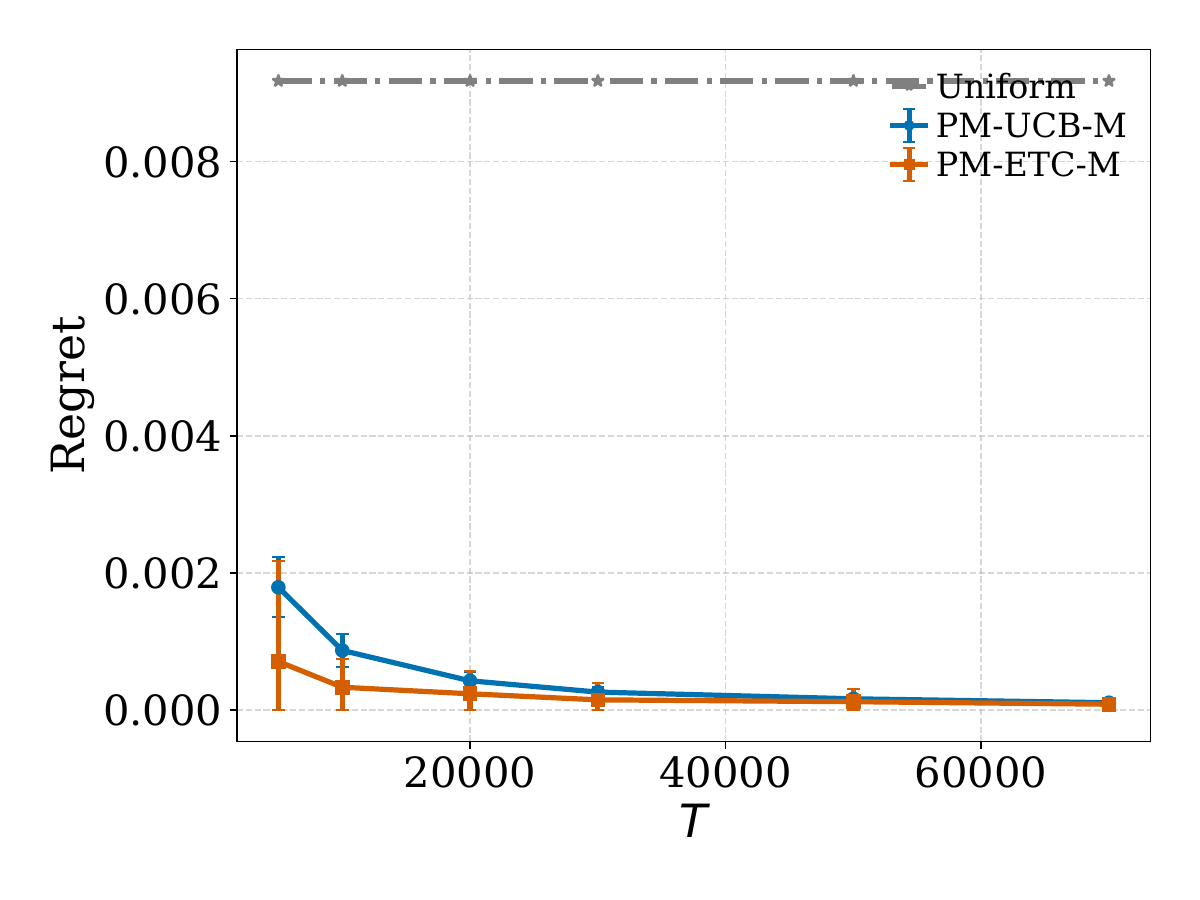}
    \caption{Horizon-dependent algorithm compared to uniform sampling for WRTD}
    \label{fig:wang_2_uni}
  \end{subfigure}

  \caption{Performance comparison of PM-ETC-M and PM-UCB-M algorithms on 2-arm Bernoulli bandits under  
    (top) GD and (bottom) WRTD PMs, with and without a uniform-sampling baseline.  
    Error bars denote one standard error over repeated trials.}
  \label{fig:combined_gini_wang}
\end{figure}

\paragraph{Regret versus Horizon (Gaussian Bandits)}

We also empirically evaluate CE-UCB-M and PM-ETC-M algorithms for the Gaussian Bandit model. We assess the regrets of two convex (mean) and concave (GD) models. We implement CE-UCB-M since it has computational advantages over PM-UCB-M, while achieving the same regret bounds. The Gaussian distributions have variance $1$ and mean values $1,3$ and $5$. Our choice of $\varepsilon$ depends on $\beta$. For the mean, the optimal solution is the arm that has the highest mean $5$, i.e., the solution lies on the boundaries of $\Delta^{K-1}_{\varepsilon}$. By Case 1 analyzed in \ref{Appendix:proof of lemma beta12}, which does not require strict concavity of the distortion function, it follows that $\beta=1$. For GD, the optimal solution might be a mixture of arm distributions. From Lemma \ref{lemma:beta12}, it is known that $\beta \in \{ 1,2\}$ for GD. In the absence of knowledge about the optimal policy, we set $\beta=2$, to obtain conservative scaling for $\delta_{12}(\varepsilon)$. 
As seen in Figures \ref{fig:gini_gaus}~and~\ref{fig:mean_gaus}, both algorithms achieve sublinear regret for the mean and GD. Furthermore, these figures show that for GD, PM-ETC-M incurs less regret, whereas for the mean, CE-UCB-M incurs less regret. 

\begin{figure}[t]
    \centering
    
    \begin{subfigure}[b]{0.45\linewidth}
        \centering
        \includegraphics[width=\linewidth]{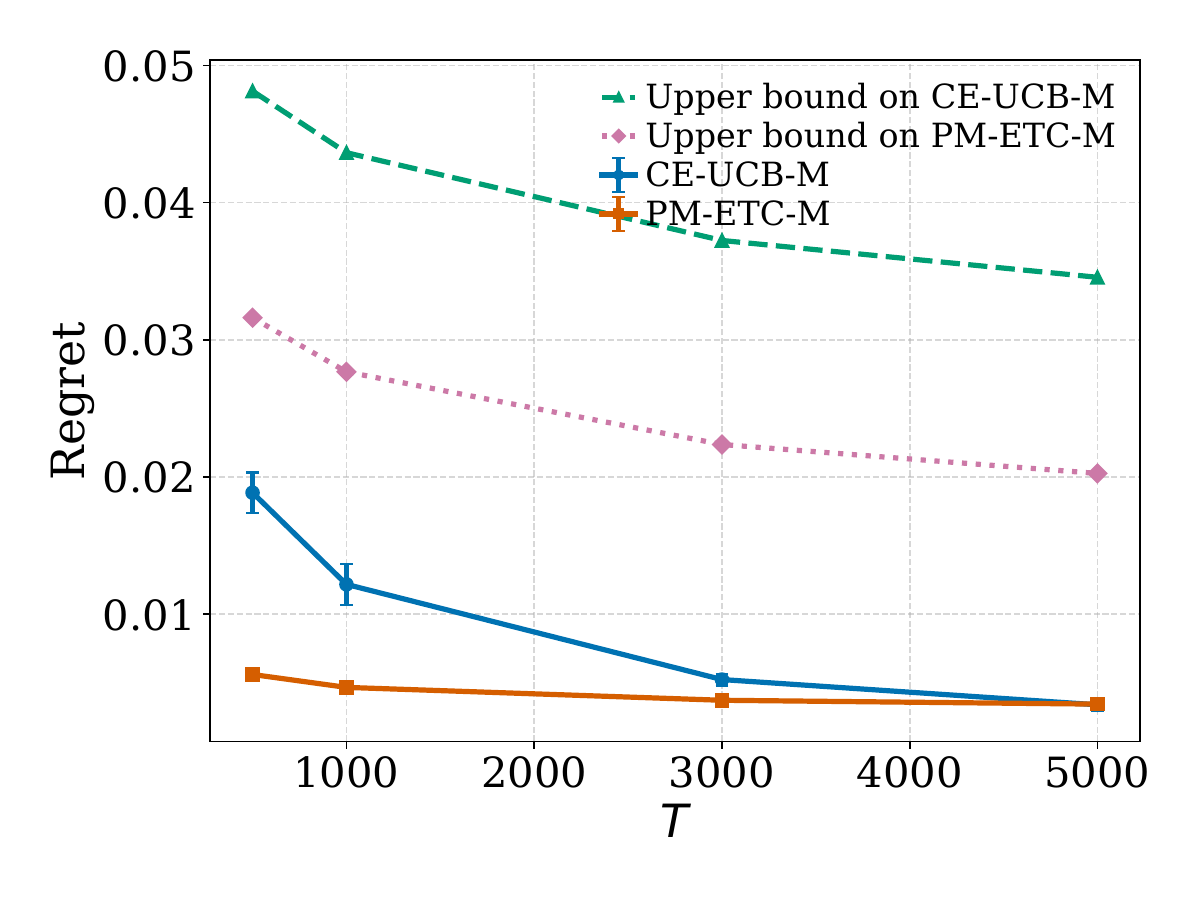}
        \caption{Scaling of regret for GD}
        \label{fig:gini_gaus}
    \end{subfigure}
    \hfill
    \begin{subfigure}[b]{0.45\linewidth}
        \centering
        \includegraphics[width=\linewidth]{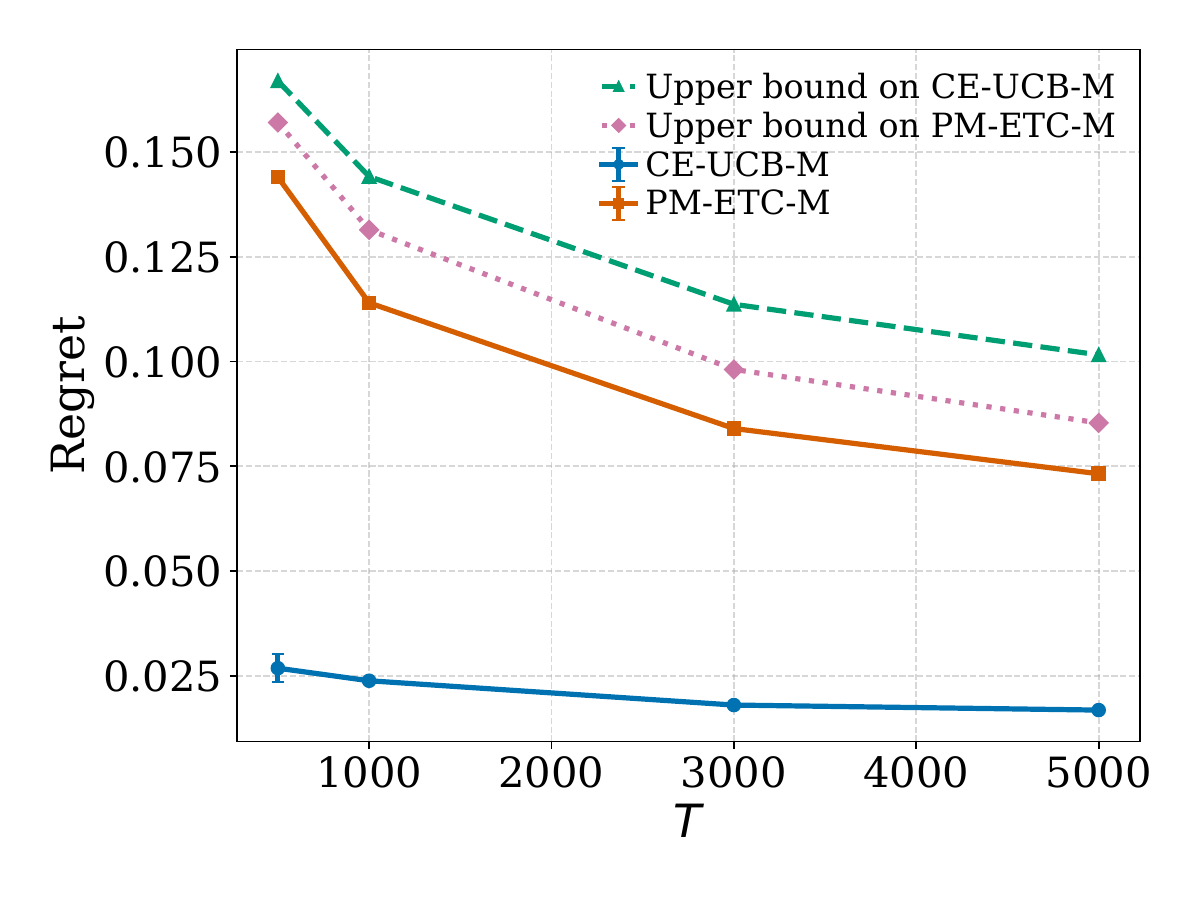}
        \caption{Scaling of regret for mean.}
        \label{fig:mean_gaus}
    \end{subfigure}

    \caption{Scaling of regrets of PM-ETC-M and CE-UCB-M algorithms for $3-$arm Gaussian bandit model for GD and mean.}
    \label{fig:gini_mean_gaus}
\end{figure}

\paragraph{${\bar\beta}$ evaluations}

We have run empirical evaluations to show how the ratio $\frac{\log(\delta_{13}(\varepsilon))}{\log(\varepsilon)}$ converges to $\bar\beta$ for the PMs in Table \ref{table:beta_values}. As seen in Figure \ref{fig:beta_all_pms_ucb}, for smaller $\varepsilon$ values, the ratio converges to the ${\bar\beta}$ values characterized in Table \ref{table:beta_values}. For the PMs we shared a range for ${\bar\beta}$, we observe that the ratio converges to a value in the given range.

\begin{figure}[h]
  \centering
  \begin{subfigure}[b]{0.48\textwidth}
    \centering
    \includegraphics[width=\linewidth]{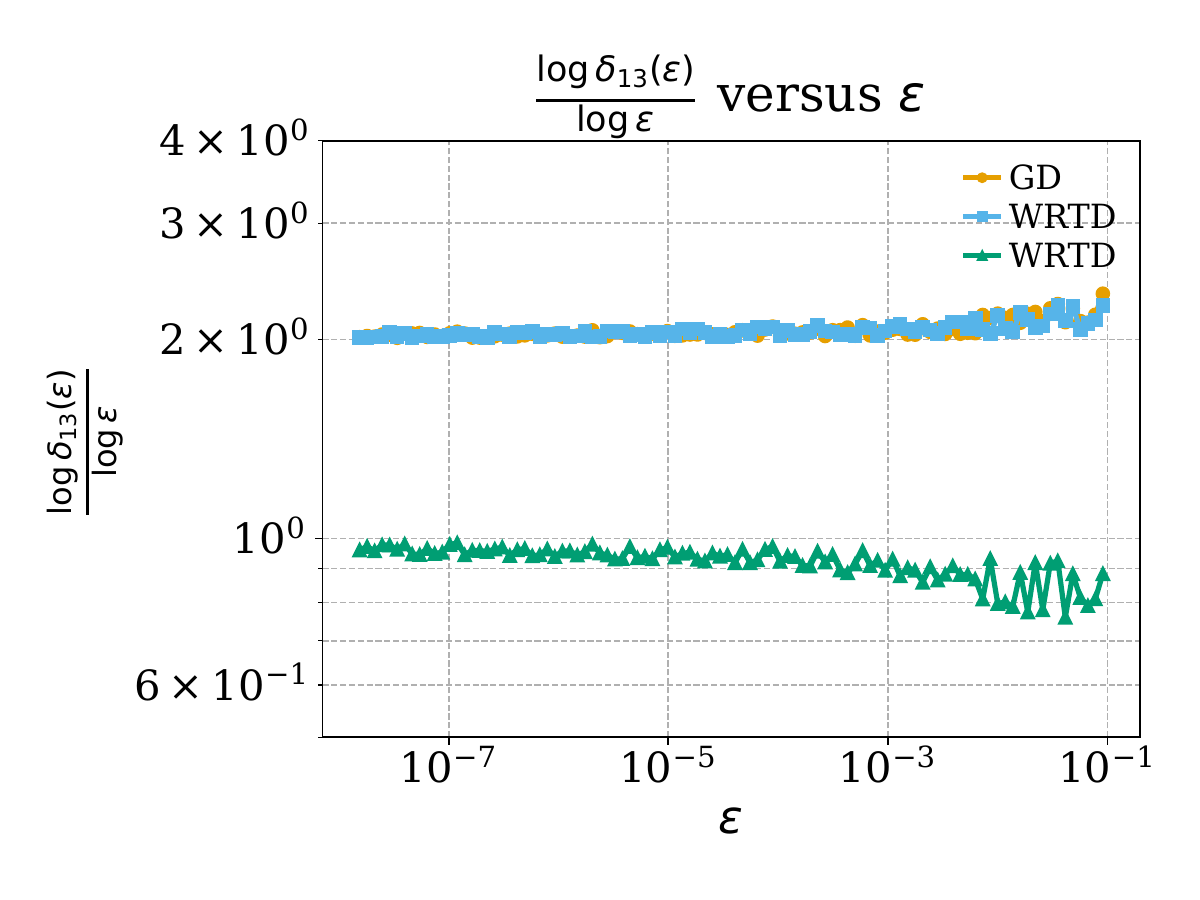}
    \caption{GD and WRTD}
    \label{fig:beta_gini_wang}
  \end{subfigure}\hfill
  \begin{subfigure}[b]{0.48\textwidth}
    \centering
    \includegraphics[width=\linewidth]{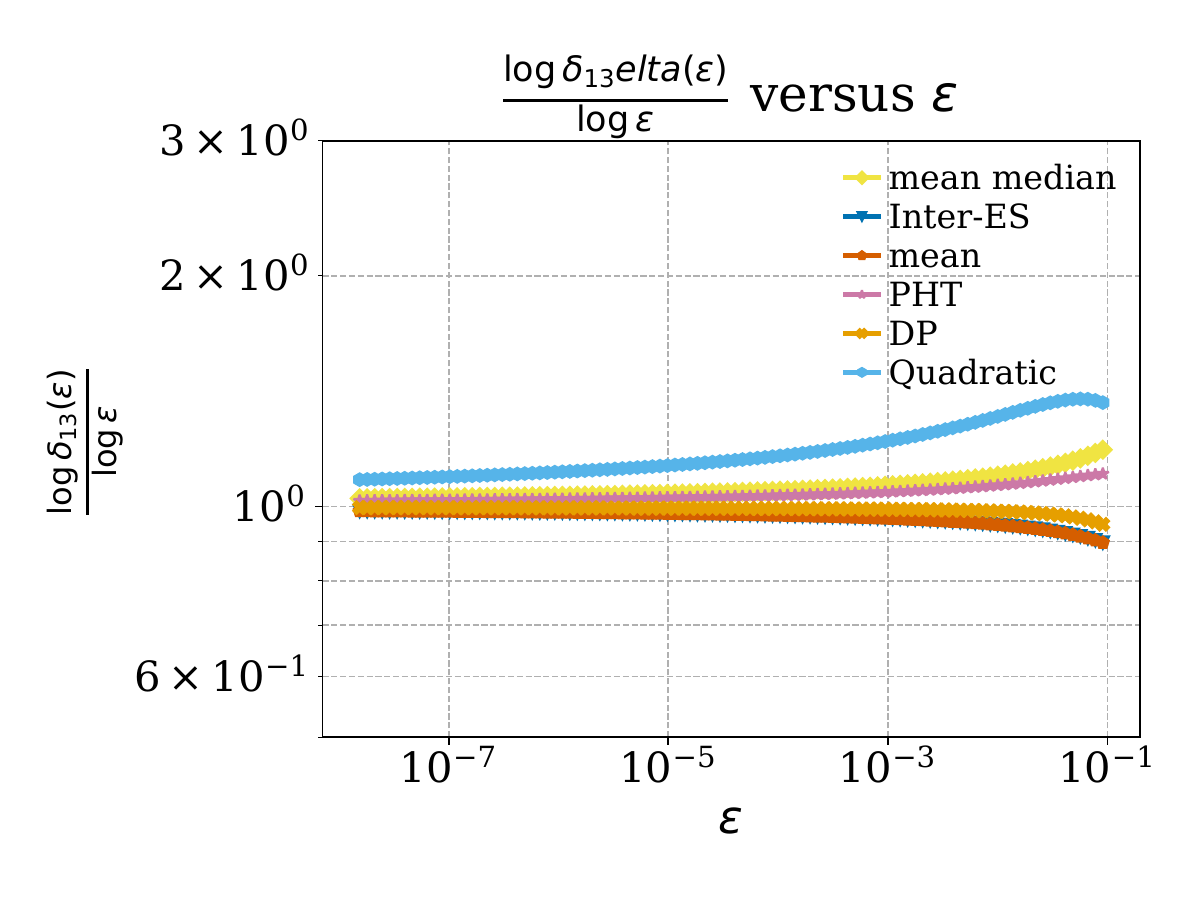}
    \caption{Mean-median, inter-ES($c=0.5$), mean, PHT measure, dual power(DP) and quadratic }
    \label{fig:beta_other_pms}
  \end{subfigure}
  \caption{Estimated $\bar\beta$ versus\ $\varepsilon$ for the PMs in Table \ref{table:beta_values}.}
  \label{fig:beta_all_pms_ucb}
\end{figure}

\paragraph{Impact of Discretization Resolution.}  For our empirical experiments, we have chosen the smallest discretization parameter $\varepsilon$ that still satisfies these guarantees required for the regret bounds in Section \ref{sec:analysis}. While a smaller $\varepsilon$ brings the solution closer to optimality, it also increases computational cost. By selecting the minimal $\varepsilon$ that preserves our regret bounds, we strike a balance: as the horizon $T$ grows, regret decreases without incurring unnecessary computation.  For GD, we empirically show how the regret upper bounds would scale for different $\varepsilon$ values in Figures \ref{fig:Gini_emp_epsilon} and \ref{fig:Gini_emp_epsilon_ETC}. We consider the distributions in which the optimal solution will not lie on the boundary points, hence, $\beta=2$. We plotted the regret bounds for $\varepsilon$ that comply with $T>T(\varepsilon)$ requirement for Theorem \ref{corollary:PM-UCB-M} in Figure \ref{fig:Gini_emp_epsilon} and $T>N(\varepsilon)$ requirement for Theorem \ref{theorem:PM-ETC-M} in Figure \ref{fig:Gini_emp_epsilon_ETC}. We observe that as $\varepsilon$ increases, the regret bounds worsen as the estimate for the mixing coefficient becomes less accurate.

\begin{figure}[h]
    \centering
    \begin{subfigure}[b]{0.45\linewidth}
        \centering
        \includegraphics[width=\linewidth]{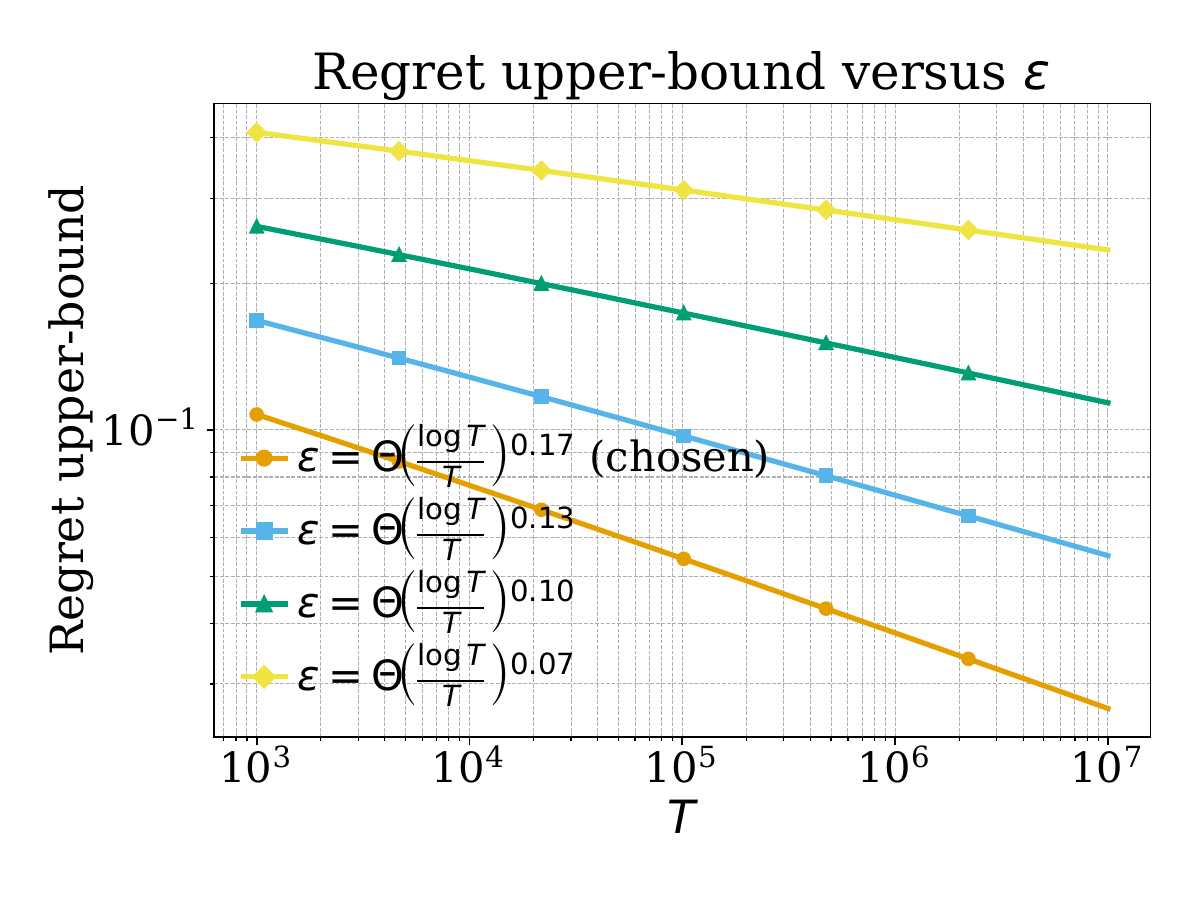}
        \caption{PM-UCB-M}
        \label{fig:Gini_emp_epsilon}
    \end{subfigure}
    \hfill
    \begin{subfigure}[b]{0.45\linewidth}
        \centering
        \includegraphics[width=\linewidth]{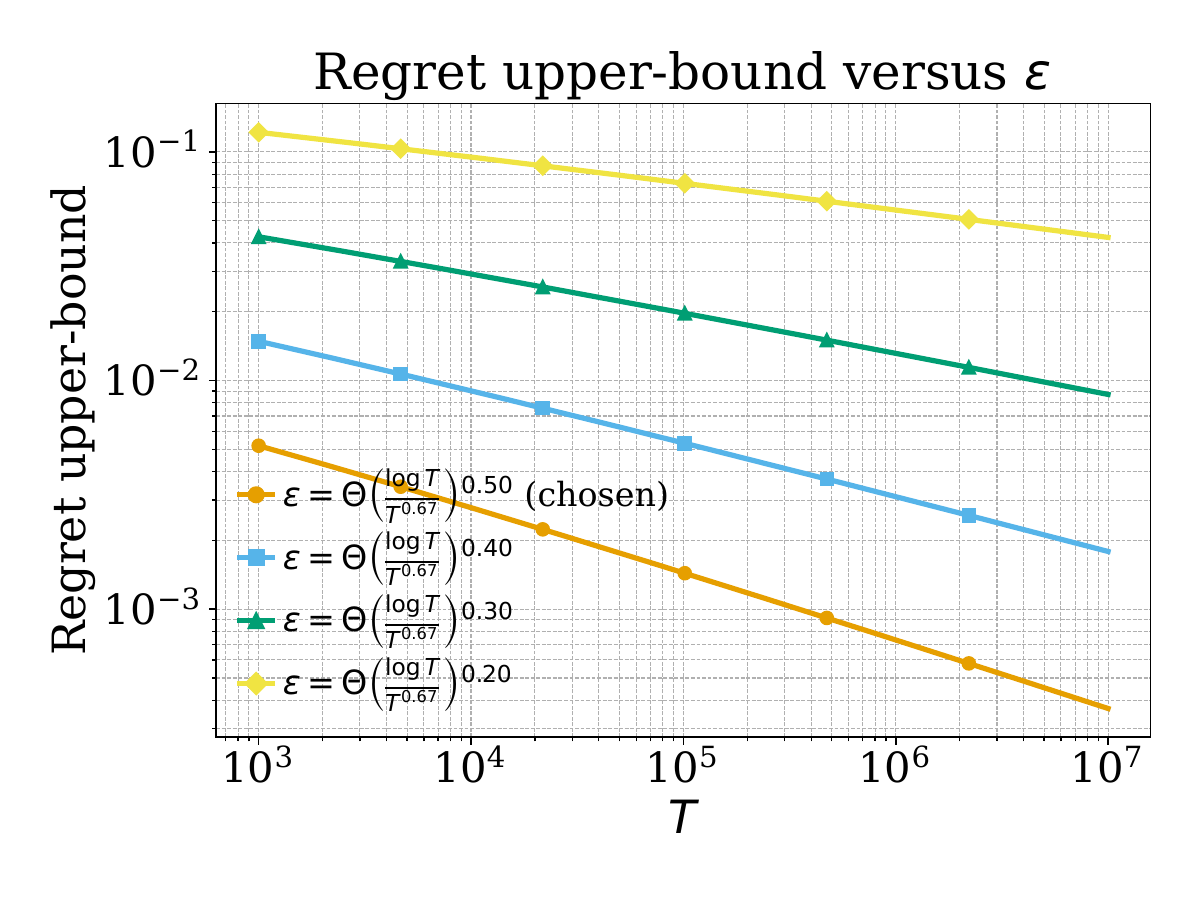}
        \caption{PM-ETC-M}
        \label{fig:Gini_emp_epsilon_ETC}
    \end{subfigure}
    \caption{Empirical evaluations of regret upper-bound of horizon-dependent algorithms for GD}
    \label{fig:gini_comparison}
\end{figure}

\subsection{Anytime Algorithm}
\label{section:emp_eval_anytime}

The anytime algorithm has two main phases: the interval refinement phase and the tracking phase. We focus on evaluating the algorithm's behavior once the tracking phase starts. This is due to the fact that before the refinement phase starts, the discretization is not on the level that the user specified. Due to this, our regret guarantee is also after the tracking phase of the algorithm starts. To this end, we first apply the anytime algorithm to the $2$-arm Bernoulli with horizon of $T=3\cdot10^6$, discretization interval length of $\epsilon \in \{0.12, 0.06, 0.03\}$, discretization number $A=4$, and target probability of error $\delta = 0.05$. We perform the experiment 100 times. We observe that for larger $\epsilon$, the tracking phase starts far earlier. However, the regret incurred is much larger, which is expected as the estimation regret for a given $\epsilon$ is a constant.

\begin{figure}[h]
  \centering
  \begin{subfigure}[b]{0.32\textwidth}
    \centering
    \includegraphics[width=\textwidth]{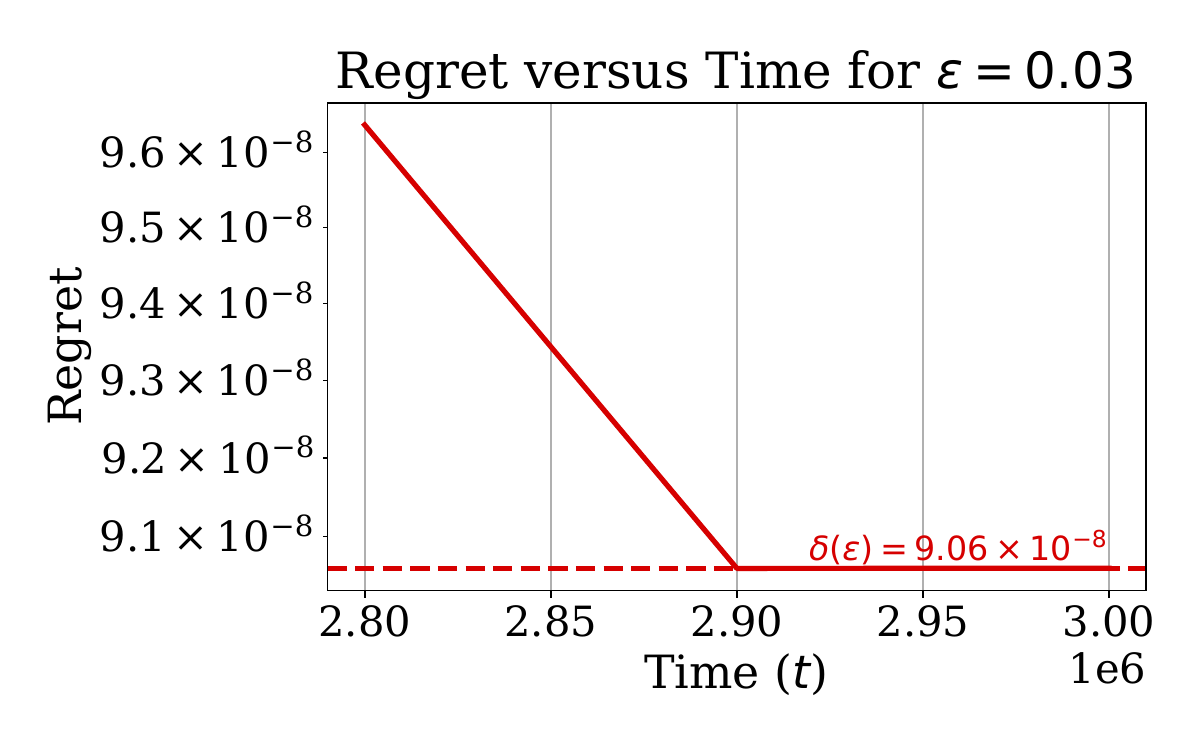}
    \caption{$\varepsilon = 0.03$}
    \label{fig:eps_003}
  \end{subfigure}\hfill
  \begin{subfigure}[b]{0.32\textwidth}
    \centering
    \includegraphics[width=\textwidth]{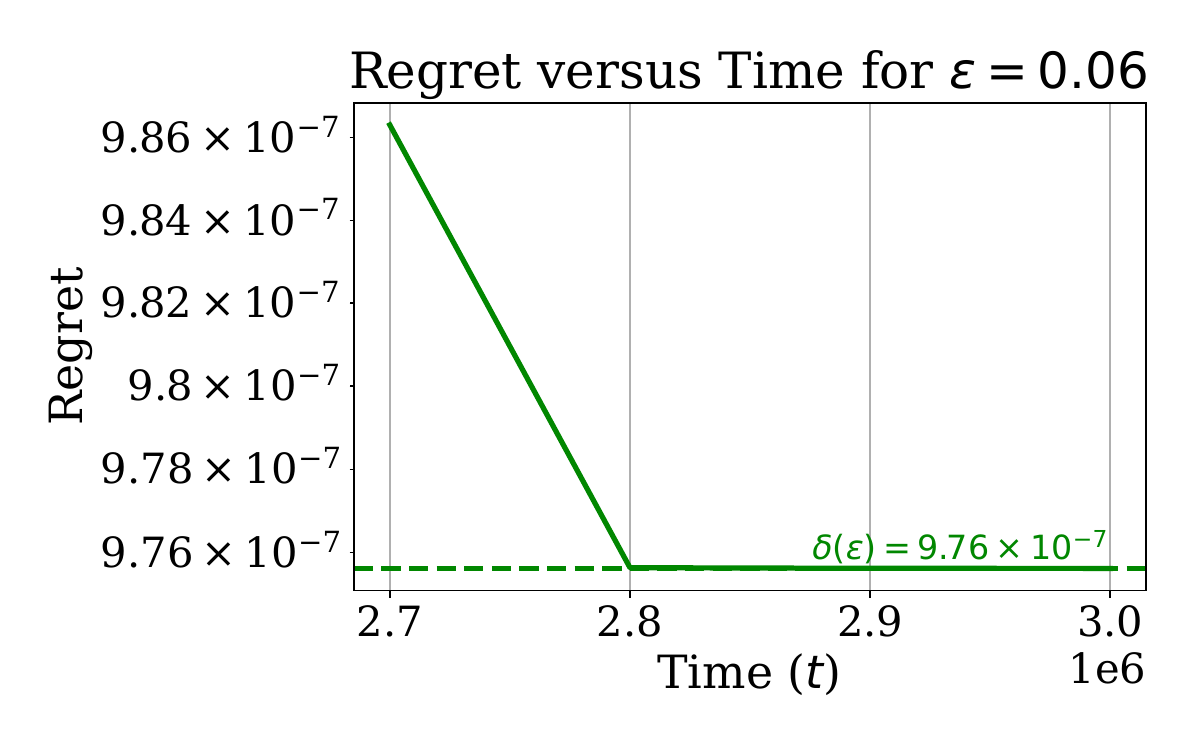}
    \caption{$\varepsilon = 0.06$}
    \label{fig:eps_006}
  \end{subfigure}\hfill
  \begin{subfigure}[b]{0.32\textwidth}
    \centering
    \includegraphics[width=\textwidth]{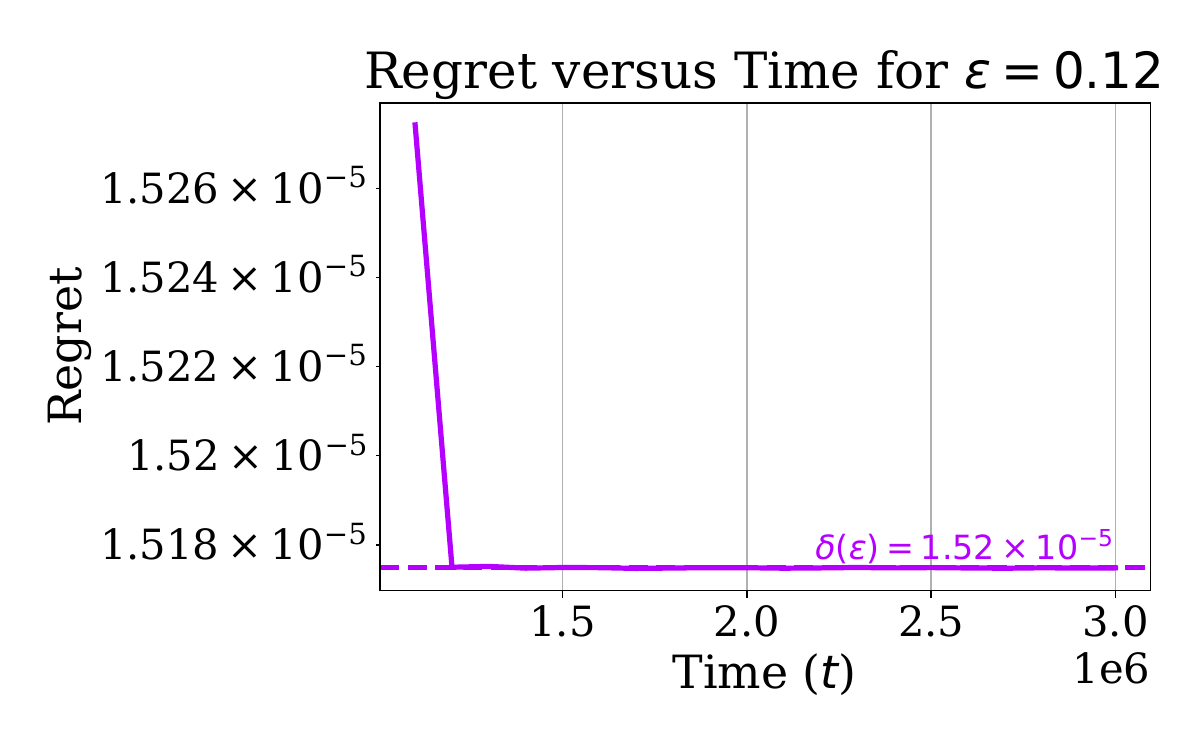}
    \caption{$\varepsilon = 0.12$}
    \label{fig:eps_012}
  \end{subfigure}
  \caption{Regret of the anytime algorithm versus horizon~$T$ for various~$\varepsilon$ values.}
  \label{fig:regret_vs_time_row}
\end{figure}

An important feature of the anytime algorithm is the computational time. Different from horizon-dependent algorithms, where we encountered a $O(\frac{1}{\varepsilon^K})$-dimensional discrete optimization problem, with the anytime algorithm, we only have a $O(\frac{1}{A^K})$-dimensional discrete optimization problem where $A$ is a user specified constant that does not scale with the horizon.
Therefore, to show the benefit of adaptive discretization, we propose the following experimental setup. First, for performance comparisons, we design a reasonable baseline algorithm. Since the main characteristic of our approach is the adaptive discretization, we construct a variation of the algorithm with fixed discretization for comparison. Specifically, we adapt PM-UCB-M into an anytime setting by setting $\varepsilon = \epsilon/A$ and modifying the sampling rule according to the logic used in the CIRT algorithm. The result is a fixed-discretization, anytime-compatible algorithm. To evaluate and compare computational efficiencies of both algorithms, we define a set of thresholds for regret, namely, $\{10^{-2}, 10^{-3}, 10^{-4}, 10^{-5}, 10^{-6}\}$, $\epsilon=0.06$ and $A=4$. Each algorithm is initialized run for increasing horizons until the observed regret falls below the specified threshold. We then record the computation time required to reach this point. For each threshold, we run the experiments $100$ times, and report the average of the computational time. As shown in Figure \ref{fig:anytime_fixed_comparison}, CIRT achieves substantially lower computation time than its fixed-discretization variation, especially for smaller regret thresholds. This observation is expected as computational complexity per time step in CIRT  scales with $A$, whereas in the fixed-discretization algorithm it scales with $A/\epsilon$, which is significantly larger.

\begin{figure}[h]
    \centering
    \includegraphics[width=0.45\linewidth]{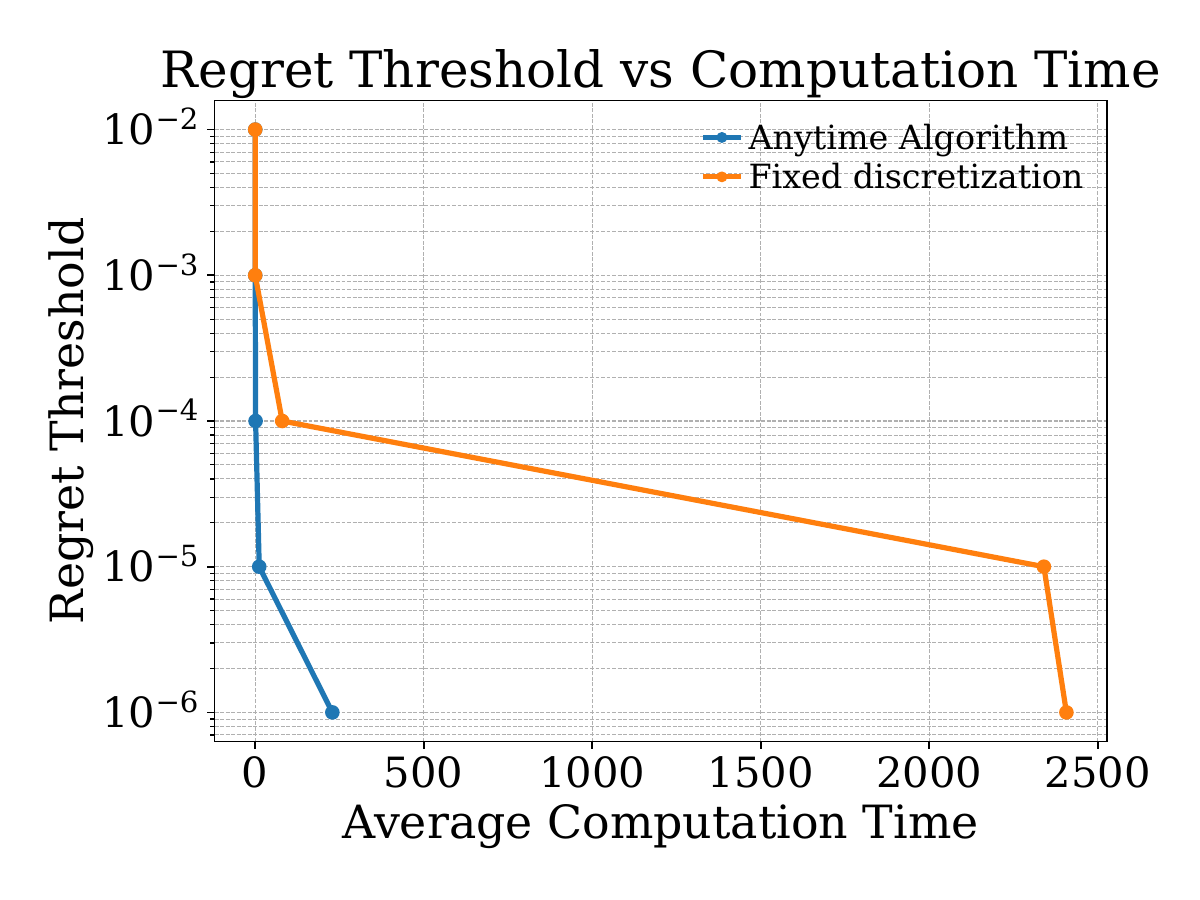}
    \caption{Computation times for the anytime algorithm and the fixed discretization algorithm.}
    \label{fig:anytime_fixed_comparison}
\end{figure}

\section{Application to Risk-sensitive Bandits}
\label{sec: risk_application}

Risk-sensitive bandits~\cite{tatli2025risk, prashanth2022wasserstein, chang2022, cassel2018general,baudry2021optimal, tamkin2019cvar,tan2022cvar,liang2023distribution,kagrecha2019distribution, prashanth2020concentration} has received notable attention in recent years. In such settings, instead of quantifying the quality of an arm through its mean value, a risk measure quantifies the quality of an arm by evaluating aspects of its reward distribution beyond the mean, often emphasizing tail behavior or variability to better capture application-specific preferences. Two main perspectives exist in the risk-sensitive bandit literature. One line of work focuses on {\em risk-averse} objectives, in which the regret is defined relative to the risk-minimizing arm~\cite{prashanth2022wasserstein, cassel2018general}, a setting typically adopted when both losses and gains are modeled through rewards. In contrast, in another line of investigation~\cite{baudry2020thompson, tamkin2019cvar, chang2022, kolla2023bandit} {\em risk-seeking} objectives are explored, in which the rewards primarily model gains, motivating the learner to favor highly volatile options over more stable choices.

While much of the literature has focused on specific risk measures such as CVaR~\cite{baudry2021optimal, tamkin2019cvar}, there have also been efforts to develop unified frameworks encompassing broader classes of risk measures~\cite{aditya2016weighted, cassel2018general, prashanth2022wasserstein}. For example, distortion risk measures (DRMs) -- which are equivalent to spectral risk measures -- form a rich family that includes CVaR, PHT, dual power, and quadratic risk measures. DRMs have been investigated in the context of risk-sensitive bandits in~\cite{kolla2023bandit, prashanth2022wasserstein}. Importantly, our definition of a PM in~\eqref{eq:wang_def} specializes to a DRM if the distortion function satisfies $h(1)=1$. In this section, we will discuss the application of the PM-centric framework in the context of risk-sensitivity, contrasting the regret guarantees of our proposed algorithms with the existing literature. We remark that the PM-centric framework specializes to the risk-seeking viewpoint, which has been investigated in earlier studies~\cite{baudry2020thompson, tamkin2019cvar, kolla2023bandit}. Hence, our performance comparisons have been provided with respect to these studies. The study in~\cite{tamkin2019cvar} investigates regret minimization in stochastic bandits, where regret of an arm is defined with respect to the arm having the maximal CVaR and provides a problem-independent bound of the order $O(1/\sqrt{T})$. The investigation in~\cite{baudry2020thompson} provides a Thompson sampling counterpart of the UCB-based algorithm proposed in~\cite{tamkin2019cvar} for CVaR bandits, showcasing an empirical improvement over the UCB-based strategy. Finally, \cite{kolla2023bandit} extends the class of risk measures from CVaR to DRMs with \holder-continuous utilities with exponent $\alpha$, providing an order of $\widetilde O(T^{-\frac{\alpha}{2}})$ problem-independent regret.

\paragraph{Performance Comparison.}

Next, we provide detailed comparisons of the settings and reported regrets for CVaR bandits and DRMs~\cite{tamkin2019cvar,baudry2020thompson,kolla2023bandit}. Furthermore, for CVaR bandits, we also compare against the UCB-based algorithm proposed for the EDPM framework in~\cite{cassel2018general}. Even though the focus of \cite{cassel2018general} is on risk aversion, for CVaR, a reward-based definition is considered, which translates to the risk-seeking objective. Before providing performance comparisons, we remark that the existing investigations on risk-sensitive bandits focus on solitary arm optimality. To recover the counterpart regrets of the PM-centric algorithms, we evaluate them when they generate solitary arm decisions. This can be controlled by properly selecting our simplex discretization scheme. Specifically, by setting $\varepsilon=1$, we can report solitary-arm policies and regrets for PM-ETC-M. Consequently, we have $\delta_{01}=0$, and hence we only have the arm selection regret $\bar{\mathfrak{R}}_{\bnu,1}^{\rm E}(T)$. For the PM-UCB-M algorithm, it can be readily verified that $\beta=1$, since the PM-maximizing arm is solitary, and we report regrets with $\beta=1$.

\begin{itemize}
    \item \textbf{CVaR:} The existing regret bounds for CVaR are restricted to distributions with bounded supports~\cite{tamkin2019cvar,baudry2020thompson,cassel2018general}. With this assumption, the best known regret bound is $O(\frac{1}{T}\log T)$. Furthermore, for a loss-based definition of CVaR,  \cite{prashanth2022wasserstein} reports a $O(\sqrt{\frac{1}{T}\log T})$ regret bound for sub-Gaussian distributions. Our PM-centric framework generalizes the regret bound for reward-based CVaR to sub-Gaussian distributions. PM-ETC-M achieves a regret bound of the order $O(\frac{1}{T}\log T)$, which is achieved by setting $\varepsilon=1$ and noting that $q=1$ for CVaR. The PM-UCB-M algorithm achieves a regret bound of the order $O((\frac{1}{T}\log T)^{\kappa})$, assuming that $\beta$ exists, where $\kappa$ and $\beta$ are defined in~\eqref{eq:vareps_kappa_UCB} and~\eqref{eq:beta_def}, respectively. Hence, for CVaR bandits, the PM-centric framework \emph{generalizes} to sub-Gaussian distributions (from bounded support) and \emph{matches} the best known regret reported in the literature. 
    \item \textbf{DRMs:} The existing studies of DRMs are reported in~\cite{kolla2023bandit}~and~\cite{prashanth2022wasserstein}, where the former adopts the risk-seeking objective while the latter adopts the risk-averse perspective, and reports a regret bound of $O(T^{-\frac{q}{2}}(\log T)^{\frac{q}{2}})$.
    The study in~\cite{kolla2023bandit} reports a problem-independent regret bound of the order $O(T^{-\frac{q}{2}}(\log T)^{\frac{q}{2}})$ for distributions with bounded supports. In contrast, the PM-centric framework generalizes to sub-Gaussian distributions. Furthermore, the PM-ETC-M algorithm outperforms the existing baseline in~\cite{kolla2023bandit}, achieving a regret of the order $O((\frac{1}{T}\log T)^q)$ when we set $\varepsilon=1$. On the other hand, the PM-UCB-M admits a regret bound of order $O((\frac{1}{T}\log T)^{q\kappa})$, assuming that $\beta$ exists.
    It is noteworthy that $\beta$ has been shown to exist for strictly concave distortion functions in Lemma~\ref{lemma:beta12}, which includes a broad range of DRMs such as PHT, quadratic, and dual power risk measures. To summarize, for DRMs, the PM-centric framework exhibits two advantages: (1) it \emph{generalizes} to sub-Gaussian distributions (from bounded support) for measures such as quadratic and dual power, and (2) it \emph{improves} upon the best known regret guarantees for DRMs for investigations into the risk-seeking objective. 
\end{itemize}

\section{Concluding Remarks}

In this paper, we have designed a novel preference-centric framework for viewing and analyzing bandit problems. It shifts bandit decisions from purely utilitarian and their rewards being defined as expected arm returns to a more general notion of reward that is sensitive not only to distributions' first moments, but also recognizes the events encoded in the distribution tails. For this purpose, we have designed the notion of preference metric (PM) as a generalization of the canonical objective, which encapsulates a wide variety of finer distributional properties through distorted functionals of the arms' cumulative distribution functions (CDFs). Adopting such a notion, besides facilitating decisions that are more sensitive to the entire probability model, also leads to a different way of thinking and designing bandit algorithms. Notably, an important observation in the PM-centric setting is that if the distortion function is strictly concave, there always exist bandit instances for which the PM will be optimized only by an arm selection policy that carefully selects a mixture of arms over time according to a well-defined mixture. This is in sharp contrast to the conventional approaches in bandits, which have the objective of identifying and pursuing a single best arm.  Motivated by this observation, we have designed both horizon-dependent (PM-ETC-M, PM-UCB-M, and CE-UCB-M) algorithms and an anytime algorithm (CIRT). The key components in these algorithms that facilitate generalizing to mixture optimality are: (i) an {\em estimation routine} that forms estimates of the mixing coefficients based on CDF estimates computed from the accrued rewards, and, (ii) a {\em tracking mechanism} that navigates the arm selection fractions to match the estimated mixtures. We have analyzed the regret of the proposed algorithms for a wide range of PMs based on utility-specific parameters, which guide in selecting an algorithm given a specific choice of PM and the availability of instance-dependent information. Furthermore, we have experimentally evaluated several facets of the proposed algorithms, including their regret performance, superiority compared to solitary arm policies, and computational complexity.   

This investigation lays the foundation for viewing bandit problems through the lens of PMs and opens several avenues for future research. First, designing minimax and instance-dependent lower bounds for PMs with mixture optimality is a promising direction. The canonical approach of applying the change-of-measures argument to two different bandit instances with distinct optimal arms does not work, since PMs may exhibit mixture optimality. Second, our results rely on the exponential convergence of the CDF estimates in the $1$-Wasserstein metric, which follows from the sub-Gaussianity of the class of bandits we investigate. Extending these results to heavy-tailed distributions is an open problem and would necessitate designing distribution confidence sequences for heavy-tailed distributions. Third, another important direction is to devise provably no-regret anytime algorithms beyond Bernoulli bandits. Finally, while we analyze PM-centric stochastic bandits, extending the PM-centric setting to correlated bandit classes, such as linear and causal bandits, has important applications.

\bibliographystyle{IEEEbib}
\bibliography{BAIRef, references}

\begin{thebibliography}{10}

\bibitem{Choquet1954}
Gustave Choquet,
\newblock ``Theory of capacities,''
\newblock {\em Annales de l'Institut Fourier}, vol. 5, pp. 131--295, 1954.

\bibitem{Schmeidler1986}
David Schmeidler,
\newblock ``Integral representation without additivity,''
\newblock {\em Proceedings of the American Mathematical Society}, vol. 97, no. 2, pp. 255--261, 1986.

\bibitem{Wang2020}
Qiuqi Wang, Ruodu Wang, and Yunran Wei,
\newblock ``Distortion riskmetrics on general spaces,''
\newblock {\em ASTIN Bulletin: The Journal of the International Actuarial Association}, vol. 50, no. 3, pp. 827--851, May 2020.

\bibitem{huber2009robust}
Peter~J Huber and Elvezio~M Ronchetti,
\newblock {\em Robust Statistics},
\newblock Wiley Series in Probability and Statistics. Wiley, New Jersey, 2nd edition, 2009.

\bibitem{yaari1987dual}
Menahem~E Yaari,
\newblock ``The dual theory of choice under risk,''
\newblock {\em Econometrica}, vol. 55, no. 1, pp. 95--115, 1987.

\bibitem{denneberg1994nonadditive}
Dieter Denneberg,
\newblock {\em Non-additive Measure and Integral},
\newblock Springer Science \& Business Media, 1994.

\bibitem{kusuoka2001law}
Shigeo Kusuoka,
\newblock ``On law invariant coherent risk measures,''
\newblock {\em Advances in Mathematical Economics}, vol. 3, pp. 83--95, 2001.

\bibitem{artzner99}
Philippe Artzner, Freddy Delbaen, Jean-Marc Eber, and David Heath,
\newblock ``Coherent measures of risk,''
\newblock {\em Mathematical Finance}, vol. 9, no. 3, pp. 203--228, 1999.

\bibitem{rockafellar2000optimization}
R.~Tyrrell Rockafellar and Stanislav Uryasev,
\newblock ``Optimization of conditional value-at-risk,''
\newblock {\em Journal of Risk}, vol. 2, no. 3, pp. 21--41, 2000.

\bibitem{furman2017gini}
Eliezer Furman, Ruodu Wang, and Ricardas Zitikis,
\newblock ``Gini-type measures of risk and variability: Gini shortfall, capital allocation and heavy-tailed risks,''
\newblock {\em Journal of Banking and Finance}, vol. 83, pp. 70--84, 2017.

\bibitem{rockafellar2006generalized}
R.~Tyrrell Rockafellar, Stanislav Uryasev, and Michael Zabarankin,
\newblock ``Generalized deviation in risk analysis,''
\newblock {\em Finance and Stochastics}, vol. 10, pp. 51--74, 2006.

\bibitem{Gini1912}
Corrado Gini,
\newblock {\em Variabilità e mutabilità: contributo allo studio delle distribuzioni e delle relazioni statistiche [Fasc. I.]},
\newblock Tipogr. di P. Cuppini, Bologna, Italy, 1912.

\bibitem{Cali2017CTE}
Camilla Cal\`i, Maria Longobardi, and Jafar Ahmadi,
\newblock ``Some properties of cumulative tsallis entropy,''
\newblock {\em Physica A}, vol. 486, pp. 1012--1021, 2017.

\bibitem{cassel2018general}
Asaf Cassel, Shie Mannor, and Assaf Zeevi,
\newblock ``A general approach to multi-armed bandits under risk criteria,''
\newblock in {\em Proc. Conference on Learning Theory}, Stockholm, Sweden, July 2018.

\bibitem{chang2022}
Joel Q.~L. Chang and Vincent Y.~F. Tan,
\newblock ``A unifying theory of thompson sampling for continuous risk-averse bandits,''
\newblock in {\em Proc. AAAI Conference on Artificial Intelligence}, virtual, February 2022.

\bibitem{baudry2021optimal}
Dorian Baudry, Romain Gautron, Emilie Kaufmann, and Odalric Maillard,
\newblock ``Optimal thompson sampling strategies for support-aware {CVaR} bandits,''
\newblock in {\em Proc. International Conference on Machine Learning}, virtual, July 2021.

\bibitem{tamkin2019cvar}
Alex Tamkin, Ramtin Keramati, Christoph Dann, and Emma Brunskill,
\newblock ``Distributionally-aware exploration for {CVaR} bandits,''
\newblock in {\em Proc. Advances in Neural Information Processing Systems}, Vancouver, Canada, December 2019.

\bibitem{tan2022cvar}
Chenmien Tan and Paul Weng,
\newblock ``{CVaR}-regret bounds for multi-armed bandits,''
\newblock in {\em Proc. Asian Conference on Machine Learning}, India, December 2022.

\bibitem{kolla2023bandit}
Ravi~Kumar Kolla, Prashanth~L. A., Aditya Gopalan, Krishna Jagannathan, Michael Fu, and Steve Marcus,
\newblock ``Bandit algorithms to emulate human decision making using probabilistic distortions,''
\newblock {\em arXiv:1611.10283}, 2023.

\bibitem{liang2023distribution}
Hao Liang and Zhi-Quan Luo,
\newblock ``A distribution optimization framework for confidence bounds of risk measures,''
\newblock in {\em Proc. Machine Learning Research}, Hawaii, July 2023.

\bibitem{tversky1992advances}
Amos Tversky and Daniel Kahneman,
\newblock ``Advances in prospect theory: Cumulative representation of uncertainty,''
\newblock {\em Journal of Risk and Uncertainty}, vol. 5, pp. 297--323, 1992.

\bibitem{WirchHardy2003Distortion}
Julia~L. Wirch and Mary~R. Hardy,
\newblock ``Distortion risk measures: Coherence and stochastic dominance,''
\newblock {\em Insurance: Mathematics and Economics}, vol. 32, no. 1, pp. 168--168, Feb 2003.

\bibitem{prashanth2022wasserstein}
Prashanth L.A. and Sanjay~P. Bhat,
\newblock ``{A Wasserstein distance approach for concentration of empirical risk estimates},''
\newblock {\em Journal of Machine Learning Research}, vol. 23, no. 238, pp. 1--61, 2022.

\bibitem{pmlr-v49-garivier16a}
Aur{\'e}lien Garivier and Emilie Kaufmann,
\newblock ``Optimal best arm identification with fixed confidence,''
\newblock in {\em Proc. Conference on Learning Theory}, New York, NY, June 2016.

\bibitem{jourdan2022}
Marc Jourdan, Rémy Degenne, Dorian Baudry, Rianne de~Heide, and Emilie Kaufmann,
\newblock ``Top two algorithms revisited,''
\newblock in {\em Proc. Advances in Neural Information Processing Systems}, New Orleans, LA, December 2022.

\bibitem{pmlr-v117-agrawal20a}
Shubhada Agrawal, Sandeep Juneja, and Peter Glynn,
\newblock ``Optimal $\delta$-correct best-arm selection for heavy-tailed distributions,''
\newblock in {\em Proc. International Conference on Algorithmic Learning Theory}, San Diego, CA, February 2020.

\bibitem{mukherjee2023best}
Arpan Mukherjee and Ali Tajer,
\newblock ``Best arm identification in stochastic bandits: Beyond $\beta-$ optimality,''
\newblock {\em IEEE Transactions on Information Theory}, vol. 71, no. 1, January 2025.

\bibitem{Garivier2016}
Aur{\'e}lien Garivier and Emilie Kaufmann,
\newblock ``Optimal best arm identification with fixed confidence,''
\newblock {\em Journal of Machine Learning Research}, vol. 49, pp. 1--30, 2016.

\bibitem{mukherjee2022}
A.~Mukherjee and A.~Tajer,
\newblock ``{SPRT}-based efficient best arm identification in stochastic bandits,''
\newblock {\em IEEE Journal on Selected Areas in Information Theory (accepted for publication)}, June 2023.

\bibitem{tatli2025risk}
Meltem Tatl{\i}, Arpan Mukherjee, Prashanth L.A., Karthikeyan Shanmugam, and Ali Tajer,
\newblock ``Risk-sensitive bandits: Arm mixture optimality and regret-efficient algorithms,''
\newblock in {\em Proc. International Conference on Artificial Intelligence and Statistics}, Thailand, May 2025.

\bibitem{kagrecha2019distribution}
Anmol Kagrecha, Jayakrishnan Nair, and Krishna Jagannathan,
\newblock ``Distribution oblivious, risk-aware algorithms for multi-armed bandits with unbounded rewards,''
\newblock in {\em Proc. Advances in Neural Information Processing Systems}, Vancouver, Canada, 2019.

\bibitem{prashanth2020concentration}
Prashanth L.~A., Krishna Jagannathan, and Ravi~Kumar Kolla,
\newblock ``{Concentration bounds for CVaR estimation: The cases of light-tailed and heavy-tailed distributions},''
\newblock in {\em {Proc. International Conference on Machine Learning}}, virtual, July 2020.

\bibitem{baudry2020thompson}
Dorian Baudry, Romain Gautron, Emilie Kaufmann, and Odalric-Ambryn Maillard,
\newblock ``Optimal thompson sampling strategies for support-aware {CVaR} bandits,''
\newblock in {\em Proc. International Conference on Machine Learning}, virtual, July 2021.

\bibitem{aditya2016weighted}
Aditya Gopalan, Prashanth L.~A., Michael Fu, and Steve Marcus,
\newblock ``Weighted bandits or: How bandits learn distorted values that are not expected,''
\newblock in {\em Proc. AAAI Conference on Artificial Intelligence}, San Francisco, CA, February 2017.

\bibitem{villani2009optimal}
C{\'e}dric Villani,
\newblock {\em Optimal Transport: Old and New}, vol. 338,
\newblock Springer, 2009.

\bibitem{vijayan2021policy}
Nithia Vijayan and Prashanth L.A.,
\newblock ``Policy gradient methods for distortion risk measures,''
\newblock {\em arXiv 2107.04422}, 2021.

\bibitem{dowd2006}
Kevin Dowd and David Blake,
\newblock ``After {VaR}: The theory, estimation and insurance applications of quantile-based risk measures,''
\newblock {\em The Journal of Risk and Insurance}, vol. 73, no. 2, May 2006.

\bibitem{jones03}
Bruce~L. Jones and Ricardas Zitikis,
\newblock ``Empirical estimation of risk measures and related quantities,''
\newblock {\em North American Actuarial Journal}, vol. 7, no. 4, pp. 44--54, October 2003.

\end{thebibliography}

%%%%%%%%%%%%%%%%%%%%%%%%%%%%%%%%%%%%%%%%%%%%%%%%%%%%%%%%%%%%

 \newpage
% \onecolumn 
% \appendix

\clearpage
\appendix
\onecolumn

\doparttoc % Tell to minitoc to generate a toc for the parts
\faketableofcontents

\newpage

\section*{Organization of the Appendix}

The supplementary material consists of nine parts grouped in Appendices \ref{proof:lemma_example}-\ref{Appendix:Anytime-alg}. 
 We begin by presenting the proof of Lemma~\ref{example utility} in Appendix \ref{proof:lemma_example}, Lemma~\ref{lemma:U} in Appendix \ref{Appendix:lemma_U_proof}, and the proof of the finiteness of the Wasserstein-related constant $W$ in Appendix~\ref{Appendix:W_finitess}. Then, in Appendix \ref{Appendix:convex_concave}, we provide the proofs for Theorem \ref{lemma:convex_solitary}, which shows solitary arm optimality for convex distortion functions, and Theorem \ref{lemma:mixture_lemma_exponential}, which shows possible mixture optimality for strict concave distortion functions. Subsequently, in Appendix \ref{Appendix:estimation_regret}, we provide an upper bound on the estimation regret in terms of the discretization parameter $\varepsilon$. Then, in Appendix \ref{Appendix:Holder_exp}, we obtain \holder~exponents and constants for PMs reported in Table \ref{table:beta_values}. We continue with proof of Lemma \ref{lemma:beta12} in Appendix \ref{Appendix:proof of lemma beta12}. Then, in Appendix \ref{Appendix:Beta_Appendix}, we obtain gap parameter values $\beta$ and ${\bar\beta}$ for $K-$arm Bernoulli bandit model. Lastly, we have Appendices \ref{Appendix:PM-ETC-M}-\ref{Appendix:Anytime-alg} to present the performance guarantees for the PM-ETC-M (Appendix~\ref{Appendix:PM-ETC-M}), PM-UCB-M (Appendix~\ref{proof:UCB upper bound}), CE-UCB-M and (Appendix~\ref{Appendix:Upper_last_UCB}). and CIRT (Appendix~\ref{Appendix:Anytime-alg}) algorithms.

\section{Proof of Lemma \ref{example utility}}
\label{proof:lemma_example}
For any values $p_1,p_2\in(0,1)$, consider two Bernoulli distributions ${\sf Bern}(p_1)$ and ${\sf Bern}(p_2)$ with CDFs $\F_1$ and $\F_2$, respectively. It can be readily verified that for any $i\in\{1,2\}$,
\begin{align} 
    U_h(\F_i) &= \int_{-\infty}^{0} \Big(h\big(1- \F_i(x)\big)  - h(1)\Big) \, \der x  + \int_{0}^{\infty} h\big(1-\F_i(x)\big) \, \der x \    = h(p_i)\ .
    \label{eq:h_1_p}
\end{align}
Owing to the strict concavity of the distortion function \(h(u) = u(1-u)\), the maximizer
\begin{align}
    p^\star \triangleq \argmax_{p \in [0,1]} h(p)\;=\;\frac{1}{2}
\end{align}
is unique. Furthermore, due to the function being non-monotone, \(p^\star\) cannot lie at the boundaries. Hence, \(p^\star \in (0,1)\). Let us choose the mean values of the arms such that $p_1 < \frac{1}{2}$ and $p_2>\frac{1}{2}$. With these choices, there exists $\lambda\in(0,1)$ such that
\begin{align}
\label{eq:p_max_def}
    p^\star = \lambda p_1 + (1-\lambda)p_2 \ .
\end{align}
For the mixture distribution 
$\lambda \F_1 + (1-\lambda)\F_2$, we have
\begin{align*}
    U_h\big(\lambda\F_1 + (1-\lambda)\F_2\big) = h(p^\star) > \max\{h(p_1), h(p_2)\} = \max\{U_h(\F_1),U_h(\F_2)\}\ .
\end{align*}
Hence proved.

\section{Proof of Lemma \ref{lemma:U}}
\label{Appendix:lemma_U_proof}
Owing to the \holder~continuity of the distortion function $h$, we have
\begin{align}
\label{eq:h_holder}
    |h(u)-h(v)| \leq L_h |u-v|^r \ ,
\end{align}
for some $L_h\in\R_+$ and $r > 0$. Furthermore, recalling that $\Xi$ denotes the space of $1$-sub-Gaussian distributions, for any $\G\in\Xi$ we have
\begin{align}
\label{eq:G_subgaussian}
    1-\G(t)\;\leq\;\e^{-\frac{t^2}{2}}\ ,\quad\text{and}\quad \G(-t)\;\leq\;\e^{-\frac{t^2}{2}}\ .
\end{align}
Now, we will provide an upper bound on $U_h(\G)$. We have
\begin{align} 
\label{eq:h_0_0_use}
    U_h(\G) &= \int_{-\infty}^{0} (h(1-\G(x)) -h(1)) \der x + \int_{0}^{\infty} (h(1-\G(x)) -h(0))\der x \\ 
    &\stackrel{\eqref{eq:h_holder}}{\leq} L_h\left(\int_{-\infty}^{0}  |G(x)|^r \der x + \int_{0}^{\infty} |1-\G(x)|^r \der x \right) \\ &\stackrel{\eqref{eq:G_subgaussian}}{\leq} L_h\left(\int_{-\infty}^{\infty} e^{\frac{-x^2r}{2}} \der x \right) 
     \label{eq:U_finite}\\
     &= L_h \sqrt{\frac{2 \pi}{r} }\ .
     \label{eq:U_finite_quant}
\end{align}
Furthermore, following the same steps as~\eqref{eq:h_0_0_use}-\eqref{eq:U_finite_quant}, noting that from the \holder~continuity condition stated in~\eqref{eq:h_holder} we have $h(u)-h(v)\geq - L_h|u-v|^r$, it can be readily verified that
\begin{align}
    \label{eq:U_finite_quant_neg}
    U_h(\G)\;\geq\; - L_h \sqrt{\frac{2 \pi}{r} }\ .
\end{align}
Lemma \ref{lemma:U} follows from~\eqref{eq:U_finite_quant} and~\eqref{eq:U_finite_quant_neg}.

\section{Proof of Theorem~\ref{theorem:W}}
\label{Appendix:W_finitess}

In order to show the finiteness of $W$ for $1$-sub-Gaussian distributions, we leverage the Kantorvich-Rubinstein duality of the $1$-Wasserstein measure, which is stated below.

\begin{theorem}[\cite{villani2009optimal}]
\label{theorem:Wasserstein_dual}
Let $\mcL(\Omega)$ denote the space of probability measures supported on $\Omega$ with finite first moment. For any $\P_1, \P_2\in \mcL(\Omega)$, we have
  \begin{align}
    \norm{\mathbb{P}_1 - \mathbb{P}_2}_{\rm W} = \sup \limits_{\lVert  f \rVert_L \leq 1} \left\{\displaystyle\int_{\Omega} f\der \P_1 -\displaystyle\int_{\Omega} f\der\P_2\right\}\ ,
  \end{align}
where $\lVert f \rVert_{L} \leq 1  $ denotes the space of all $1$- Lipschitz functions $f:\mathbb{R} \rightarrow \mathbb{R}$. 
\end{theorem}
Based on the characterization of the $1$-Wasserstein distance in Theorem~\ref{theorem:Wasserstein_dual}, we next provide an upper bound on $W$ for sub-Gaussian distributions. Let $\balpha,\bbeta\in\Delta^{K-1}$ denote two distinct probability mass functions on $[K]$. We have
\begin{align}
    \label{eq:W_1}
   \Big\|\sum_{i\in[K]} \Big(\alpha(i) - \beta(i)\Big) \mathbb{F}_i \Big\|_{\rm W} \;&=\;\sup \limits_{\lVert  f \rVert_L \leq 1} \sum_{i\in[K]} \Big(\alpha(i) - \beta(i)\Big) \E_{\mathbb{F}_i}\big[f(X)\big] \\
   \label{eq:W_2}
   \hfill & \leq \;\sup \limits_{\lVert  f \rVert_L \leq 1} \bigg\lvert \sum_{i\in[K]} \Big(\alpha(i) - \beta(i)\Big) \Big(\E_{\mathbb{F}_i}\big[f(X) - f(0)\big]  + f(0)\Big) \bigg\rvert \\ 
   \label{eq:W_3}
    \hfill & =\; \sup \limits_{\lVert  f \rVert_L \leq 1} \bigg\lvert \sum_{i\in[K]} \Big(\alpha(i) - \beta(i)\Big) \Big(\E_{\mathbb{F}_i}\big[f(X) - f(0)\big] \Big) \bigg\rvert \\ 
    \label{eq:W_4}
   \hfill & \leq\; \sup \limits_{\lVert  f \rVert_L \leq 1}
    \sum_{i\in[K]} \bigg\lvert  \Big(\alpha(i) - \beta(i)\Big) \E_{\mathbb{F}_i}\big[f(X)-f(0)\big] \bigg\rvert\\
    \label{eq:W_5}
    \hfill & \leq\; \sum_{i\in[K]} \Big\lvert \alpha(i) - \beta(i) \Big\rvert \cdot\sup \limits_{\lVert  f \rVert_L \leq 1} \Big\lvert \E_{\mathbb{F}_i}[f(X)-f(0)] \Big\rvert \\
    \label{eq:W_7}
    \hfill & \leq\; \sum_{i\in[K]} \Big\lvert \alpha(i) - \beta(i)\Big\rvert\cdot     \E_{\mathbb{F}_i}\big[|X|\big] \ ,
\end{align}
where,~the equality in \eqref{eq:W_1}~follows from Theorem~\ref{theorem:Wasserstein_dual},~the transition~\eqref{eq:W_1}-\eqref{eq:W_2} holds since we take the absolute value,~the transition~\eqref{eq:W_2}-\eqref{eq:W_3} follows from the fact that $\sum_{i\in[K]}\alpha(i) = \sum_{i\in[K]}\beta(i) = 1$,~the transition~\eqref{eq:W_3}-\eqref{eq:W_4} follows from triangle inequality,~the transition~\eqref{eq:W_4}-\eqref{eq:W_5} follows from the fact that for any two functions $f_1$ and $f_2$, we have $\sup_x \{f_1(x) + f_2(x) \}\leq \sup_x f_1(x) + \sup_x f_2(x)$, and,~the transition~\eqref{eq:W_5}-\eqref{eq:W_7} follows from $1$-Lipschitzness of $f$. For sub-Gaussian variables, $\E[|X|]$ is bounded in terms of the sub-Gaussian parameter. Since all distributions are $1$-sub- Gaussian, we have  
\begin{align}
    E_{\mathbb{F}_i}[|X|] &= \int_{0}^{+\infty} \P(|x| > u)\der u \;\leq\; \int_{0}^{+\infty} 2\e^{-\frac{t^2}{2}} \der t \;=\; \int_{-\infty}^{+\infty} \e^{-\frac{t^2}{2}} \der t \;=\; \sqrt{2\pi} \underbrace{\int_{-\infty}^{+\infty} \frac{1}{\sqrt{2\pi}} \e^{-\frac{t^2}{2}} \der t}_{= 1} \;=\;  \sqrt{2\pi}\ ,
\end{align}
 which implies that
\begin{align}
\label{eq:W_8}
     \Big\| \sum_{i\in[K]} \Big(\alpha(i) - \beta(i)\Big) \mathbb{F}_i \Big\|_{\rm W}\;\stackrel{\eqref{eq:W_7}}{\leq}\;\sqrt{2\pi}\cdot\norm{\balpha-\bbeta}_1\ .
\end{align}
Hence, from~\eqref{eq:W_8}, we find
\begin{align}
    W\;=\; \max \limits_{\balpha \neq \bbeta \in \Delta^K} \frac{1}{\lVert \balpha - \bbeta \rVert_1}\Big\|\sum_i \alpha_i \F_i-\sum_j \beta_j \F_j\Big\|_{\rm W}\;\leq\;\sqrt{2\pi}\ .
\end{align}

\section{Convex and Concave Preference Metrics}
\label{Appendix:convex_concave}

\subsection{Proof of Theorem~\ref{lemma:convex_solitary} (Convex Distortion Functions)}
\label{proof:lemma:convex_solitary}
For an arm $i \in [K]$, we have
\begin{align}
\label{eq:convex_def}
    U_h\left( \F_i\right) &= \int_{-\infty}^{\infty} \left(h\left(1- \F_i(x)\right) - h(1)u(-x) \right) \der x\ ,
\end{align}
where $u(\cdot)$ denotes the unit step function. Let us consider a convex combination of arm CDFs, $\sum_{i \in [K]} \alpha(i) \F_i$ where $\alpha(i)\in[0,1]$ and $\sum_{i \in [K]} \alpha(i) = 1$. For this mixture of CDFs, we have
    \begin{align}
        U_h\left(\sum_{i \in [K]} \alpha(i) \F_i \right) &= \int_{-\infty}^{\infty} \Big(h\Big(1-\sum_{i \in [K]} \alpha(i) \F_i(x)\Big) - h(1)u(-x) \Big) \der x \\
        &= \int_{-\infty}^{\infty} \Big(h\Big(\sum_{i \in [K]} \alpha(i) (1-\F_i(x))\Big) - h(1)u(-x) \Big) \der x \\
        \label{eq:convex_h_convex}
        &\leq \int_{-\infty}^{\infty} \sum_{i \in [K]} \alpha(i)\Big(h\Big( (1-\F_i(x))\Big) - h(1)u(-x) \Big) \der x \\
        \label{eq:convex_sum_int_change}
        &\leq \sum_{i \in [K]} \alpha(i) \int_{-\infty}^{\infty} \Big(h\Big( (1-\F_i(x))\Big) - h(1)u(-x) \Big) \der x \\
        \label{eq:convex_max}
        &\stackrel{\eqref{eq:convex_def}}{\leq} \max_{i \in [K]}  \int_{-\infty}^{\infty} \Big(h\Big( (1-\F_i(x))\Big) - h(1)u(-x) \Big) \der x \\
        \label{eq:convex_result}
        &= \max_{i \in [K]}  U_h(\F_i)
    \end{align}
where~\eqref{eq:convex_h_convex} follows from the convexity of the distortion function, and~\eqref{eq:convex_sum_int_change} follows from Fubini-Tonelli's theorem and $\sum_{i}\alpha(i) = 1$. From~\eqref{eq:convex_result}, we can conclude that for a PM with a convex distortion function, the optimal solution is a solitary arm.

\subsection{Proof of Theorem~\ref{lemma:mixture_lemma_exponential} (Concave Distortion Functions)}
\label{proof:lemma:mixture_lemma_exponential}
The proof proceeds in two steps. First, we construct an appropriate bandit instance such that each arm has the same utility for a strictly concave distortion function. Secondly, we will show that the constructed bandit instance always admits mixtures optimality for strictly concave distortion functions. 
\paragraph{Constructing bandit model:} Consider a $K$-arm shifted exponential bandit model. Let each arm have CDF denoted by $\F_i$. For $c_i, \lambda_i >0$ $\forall i \in [K]$, we have
\begin{align}
    \F_i(x) = \begin{cases}
        1-e^{-\lambda_i(x-c_i)}\ , \quad &\text{if} \quad x \geq c_i \\
        0\ , \qquad \qquad \qquad &\text{otherwise} %\quad x < c_i
    \end{cases} \ .
\end{align}
We start by expressing the utility function in terms of the parameters of distributions for arm $i \in [K]$. Denoting by $u(\cdot)$ the unit step function, we have
\begin{align}
    U_h(\F_i) &= \int_{-\infty}^{\infty} \left(h\left(1- \F_i(x)\right) - h(1)u(-x) \right) \der x \\
    &= \int_{0}^{\infty} h\left(1- \F_i(x)\right) \der x \\
    &= \int_{0}^{c_i} h\left(1\right) \der x + \int_{c_i}^{\infty} h\left(1- \F_i(x)\right) \der x \\
    &= c_i h\left(1\right) + \int_{c_i}^{\infty} h\left(e^{-\lambda_i (x-c_i)}\right) \der x \\
    &= c_i h\left(1\right) + \int_{0}^{\infty} h\left(e^{-\lambda_i x}\right) \der x 
\end{align}
In the case when $h(1)=0$, choosing different $c_i$ for each arm $i$ and same $\lambda_i$ results in $U_h(\F_i)=U_h(\F_j)$, $\forall i, j \in [K]$. In contrast, if $h(1)\neq0$, we can always construct $(c_i,\lambda_i)$ pairs such that $U_h(\F_i)=U_h(\F_j)$, $\forall i, j \in [K]$. Hence, we have designed bandit instances such that for all arms $i \in [K]$, $U_h(\F_i)=U_h(\F_j)$.

\paragraph{Mixtures optimality:} Let us choose $\balpha\in {\rm int}(\Delta^{K-1})$, where ${\rm int}(\mcS)$ denotes the interior of a compact set $\mcS$. We have,

\begin{align}
        U_h\left(\sum_{i \in [K]} \alpha(i)\F_i \right) &= \int_{-\infty}^{\infty} \left(h\left(1-\sum_{i \in [K]} \alpha(i) \F_i(x)\right) - h(1)u(-x) \right) \der x \\
        \label{eq:strict_concave}
        &> \int_{-\infty}^{\infty} \sum_{i \in [K]} \alpha(i) \left(h\left(1- \F_i(x)\right) - h(1)u(-x) \right) \der x \\
        \label{eq:fubini}
        &= \sum_{i \in [K]} \alpha(i) \int_{-\infty}^{\infty} \left(h\left(1- \F_i(x)\right) - h(1)u(-x) \right) \der x \\
        \label{eq:def_uh_f}
        &= \sum_{i \in [K]} \alpha(i) U_h(\F_i) \\
        \label{eq:noone_better_than_mixture}
        &= U_h(\F_i) \quad \forall i \in [K]
 \end{align}
 where~\eqref{eq:strict_concave} follows from {\em strict concavity} of the distortion function, together with the fact that $\sum_{i \in [K]} a(i) \F_i(x) \neq 0$, ~\eqref{eq:fubini} follows from Fubini-Tonelli's theorem,  ~\eqref{eq:def_uh_f} follows from the definition of a PM, and~\eqref{eq:noone_better_than_mixture} follows since $\balpha\in\Delta^{K-1}$. From \eqref{eq:noone_better_than_mixture}, it is observed that there is no solitary arm that is better than the optimal mixture $\balpha$.

\section{Upper Bound on Estimation Regret}
The following lemma upper bounds the estimation regret for all algorithms, PM-ETC-M, PM-UCB-M, CE-UCB-M, and CIRT. It will be used in the analyses of Theorems~\ref{theorem: ETC upper bound}, \ref{theorem:UCB upper bound}, and \ref{theorem:any_regret}.
\label{Appendix:estimation_regret}

\begin{lemma}
\label{lemma:Delta_error}
    The estimation regrets \(\delta_{01}({\varepsilon})\) and \(\delta_{02}({\varepsilon})\) defined in~\eqref{eq:regret_decomposition1_1} are upper bounded as
\begin{equation}
    \delta_{01}(\varepsilon) \leq \mcL  (KW\varepsilon)^q  \quad \textit{and} \quad
    \delta_{02}(\varepsilon) \leq 2\mcL  (KW\varepsilon)^q  \ .
\end{equation}

\end{lemma}

\begin{proof}
We start by decomposing the estimation regret into two parts as follows.
\begin{align}
    \delta_{02}(\varepsilon) &= \E_{\bnu}^{\pi} \left[U_h\left(\sum_{i\in[K]} \alpha^{\star}(i)\F_i\right) - U_h\left(\sum_{i\in[K]} a^{(2)}(i)\F_i\right)\right] \\
    &= U_h\left(\sum_{i\in[K]} \alpha^{\star}(i)\F_i\right) - U_h\left(\sum_{i\in[K]} a^{(2)}(i)\F_i\right) \\
    &= \underbrace{U_h\left(\sum_{i\in[K]} \alpha^{\star}(i)\F_i\right) - U_h\left(\sum_{i\in[K]} a^{(1)}(i)\F_i\right)}_{\triangleq\;\delta_{01}(\varepsilon)} \\
    &+ \underbrace{U_h\left(\sum_{i\in[K]} a^{(1)}(i)\F_i\right) - U_h\left(\sum_{i\in[K]} a^{(2)}(i)\F_i\right)}_{\triangleq\;\delta_{12}(\varepsilon)}\ .
\end{align}
To bound the first term $\delta_{01}(\varepsilon)$, let us define $\bar\ba$ as the discrete mixing coefficient that has the least $\ell_1$ distance to the optimal coefficient $\balpha^{\star}$, i.e.,
\begin{align}
    \bar\ba\;\in\; \argmin\limits_{\ba\in\Delta_{\varepsilon}^{K-1}} \norm{\balpha^{\star} - \ba}_1\ . 
\end{align}
Accordingly, we have
\begin{align}
    \delta_{01}(\varepsilon)& =  U_h\left(\sum_{i\in[K]} \alpha^{\star} (i)\F_i\right) - U_h\left(\sum_{i\in[K]} a^{(1)}(i)\F_i\right)  \\
    & \leq  U_h\left(\sum_{i\in[K]} \alpha^{\star} (i)\F_i\right) - U_h\left(\sum_{i\in[K]} \bar a(i)\F_i\right) \\
    \label{eq:delta_error_43}
    & \leq \mcL\left\|\sum_{i\in[K]}  \Big(\alpha^{\star}(i)- \bar a(i)\Big)  \F_i\right\|_{\rm W}^q \\
    \label{eq:delta_error_5}
    & \leq \mcL  \norm{\balpha^{\star}- \bar\ba}_1^q W^q \\
    & \leq \mcL K^{q} \varepsilon^q  W^q  ,
    \label{eq:Delta_error_2}
\end{align}
where,~\eqref{eq:delta_error_43} follows from Definition~\ref{assumption:Holder},~\eqref{eq:delta_error_5} follows from the definition of $W$ in~\eqref{eq:W}, and,~\eqref{eq:Delta_error_2} follows from the fact that $\bar\balpha$ may lie at most $\frac{\varepsilon}{2}$ away from the optimal coefficient $\balpha^{\star}$ along each coordinate. We use a similar approach to bound the second term $\delta_{12}(\varepsilon)$. Let us define $\ba^\ast \in \Delta^{K-1}_{\varepsilon}$ as the discrete mixing coefficient that is different from $\ba^{(1)}$ by at most $\varepsilon$ in any coordinate, such that $U_h\big(\sum_{i\in[K]} a^\ast (i)\F_i\big) < U_h\big(\sum_{i\in[K]} a^{(1)} (i)\F_i\big)$ (by definition). Specifically, we chose $\ba^\ast$ such that each coordinate $i\in[K]$ satisfies 
\begin{align}
\label{eq:a_ast}
    |a^{(1)}(i) - a^\ast(i)| \leq \varepsilon\ ,
\end{align}
and $\ba^{\ast}\neq \ba^{(1)}$. Accordingly, we have 
\begin{align}
    \delta_{12}(\varepsilon) &= U_h\left(\sum_{i\in[K]} a^{(1)}(i)\F_i\right) - U_h\left(\sum_{i\in[K]} a^{(2)}(i)\F_i\right) \\
    &\leq U_h\left(\sum_{i\in[K]} a^{(1)}(i)\F_i\right) - U_h\left(\sum_{i\in[K]} a^\ast(i)\F_i\right) \\ 
    \label{eq:delta_2_error_holder}
    &\leq \mcL \norm{\sum_{i \in [K]} (a^{(1)}(i) - a^\ast(i)) \F_i}^q_{\rm W} \\
    \label{eq:delta_2_W}
    &\leq \mcL \norm{\ba^{(1)} - \ba^\ast)}^q_{1}W^q \\
     \label{eq:delta_2_varepsilon}
    &\leq \mcL K^q \varepsilon^q W^q\ ,
\end{align}
where,~\eqref{eq:delta_2_error_holder} follows from Definition~\ref{assumption:Holder},~\eqref{eq:delta_2_W} follows from the definition of W in~\eqref{eq:W}, and~\eqref{eq:delta_2_varepsilon} follows from the fact that $\ba^\ast$ may lie at most $\varepsilon$ away from the optimal coefficient $\ba^{(1)}$ along each coordinate. Therefore, we have
\begin{align}
    \delta_{02}(\varepsilon) &= \delta_{01}(\varepsilon) + \delta_{12}(\varepsilon) \;\leq\; \mcL K^q\varepsilon^q  W^q + \mcL K^q \varepsilon^q W^q \;\leq\; 2\mcL (KW)^q \varepsilon^q  \ .
\end{align}

\end{proof}

\section{\holder~Exponents and Constants}
\label{Appendix:Holder_exp}

In this section, we will show the \holder~exponent and constant values for some PM functions. For the analyses of \holder~exponents and constants, we leverage the following lemma, the proof of which can be found in~\cite[Lemma 2]{prashanth2022wasserstein}.
\begin{lemma}
	\label{lemma:lipschitz-wasserstein}
	Consider random variables $X$ and $Y$ with CDFs $F_{X}$ and $F_{Y}$, respectively. Then,
	\begin{equation}
		\norm{F_X-F_Y}_{\rm{W}}=\sup_{\|f\|_L\leq 1} \Big|\E(f(X) - \E(f(Y))\Big|= \int_{-\infty}^{\infty}|F_{X}(s)-F_{Y}(s)|\mathrm{d}s=\int_{0}^{1}|F_{X}^{-1}(\beta)-F_{Y}^{-1}(\beta)|\mathrm{d}\beta \ . \label{lipwass2} 
	\end{equation}
\end{lemma}
The results provided in this section are summarized in Table~\ref{table:table_risks}. In order to unify some lemmas for presenting the \holder~constants and exponents, we define the following two categories of distortion functions.

\begin{definition}[Category $1$ PMs]
    We define the category $1$ of PMs as the one induced by a distortion function of the form $h(u)=C_1u^m - C_2u^{2m}$, where $C_1 \geq C_2 > 0$ and $m \in (0, 1]$.
\end{definition}
The PMs under Category $1$ include Gini deviation and Wang's Right-tail deviation.
\begin{definition}[Category $2$ PMs]
    We define the category $2$ of PMs as the one induced by a distortion function of the form $h(u)= C\min \{u, (1-u) \}$, where $C\in\R_+$.
\end{definition}
Some examples of Category $2$ PMs include mean-median deviation and the inter-ES range ($c=0.5$) (${\rm IER}_{0.5}$). 

\setlength{\textfloatsep}{5pt}
\begin{table}[h]
  \captionsetup{position=top}
  \caption{\holder~continuity parameters, distortion functions, and the supporting lemmas for the PMs reported in Table~\ref{table:table_regrets}.}
  \label{table:table_risks}
  \centering
  \begin{tabular}{@{}|l|c|c|c|@{}}
    \hline
    \textbf{Preference Metric}  
      & \textbf{$q$} 
      & \textbf{$\mathcal{L}$}
      & \textbf{Lemma} \\ 
    \hline\hline 
    Mean Value                      
      & $1$ 
      & $1$ 
      & - \\[0.5ex]
    Dual Power  \cite{vijayan2021policy}                  
      & $1$  
      & $s$ 
      & Lemma \ref{lemma:Dual_Power_general}\\[0.5ex]
    Quadratic    \cite{vijayan2021policy}                  
      & $1$  
      & $1+s$ 
      & Lemma \ref{lemma:Quadratic_general}\\[0.5ex]
    CVaR\footnote{\cite{prashanth2022wasserstein} reports $q=1$ for sub‐Gaussian distributions.}   \cite{dowd2006}  
      & $1$
      & $\tfrac{1}{1-c}$ 
      &  - \\[0.5ex]
    PHT Measure\protect\footnote{Under the assumption that arm distributions are bounded.} (\(s \in (0,1)\)) \cite{jones03}                   
      & $s$  
      & $\tau^{1-s}$ 
      & Lemma \ref{lemma:PHT}\\[0.5ex]
    \rowcolor{gray!30}
    Mean–median deviation  \cite{jones03,Wang2020}    
      & $1$   
      & $1$ 
      & Lemma \ref{lemma:mean_median_Holder_general}\\[0.5ex]
    \rowcolor{gray!30}
    Inter–ES Range \(\mathrm{IER}_{c=\tfrac12}\) \cite{Wang2020}
      & $1$ 
      & $2$ 
      & Lemma \ref{lemma:mean_median_Holder_general}\\[0.5ex]
    \rowcolor{gray!30}
    Wang’s Right–Tail Deviation\protect\footnote{Under the assumption that arm distributions are bounded between [V, Z].}  \cite{jones03}     
      & $\tfrac12$  
      & $\sqrt{Z-V}$ 
      & Lemma \ref{lemma:example_utility_Holder_general}\\[0.5ex]
    \rowcolor{gray!30}
    Gini Deviation         
      & $1$ 
      & $1$ 
      & Lemma \ref{lemma:example_utility_Holder_general}\\[0.5ex]
    \hline
  \end{tabular}
\end{table}
\setlength{\textfloatsep}{5pt}

\begin{lemma}[\holder~Constant for a Category $1$ PM]
\label{lemma:example_utility_Holder_general}
For any Category $1$ PM with the distortion function $h(u)=C_1u^m - C_2u^{2m}$, we have the following properties.
\begin{enumerate}
\item For sub-Gaussian distributions, if $m=1$, we have $q=1$ and $\mcL = C_1$.
\item For bounded distributions with support $\Omega=[V, Z]$, if $m<1$, we have $q=m$ and $\mcL = C_1(Z-V)^{1-m}$.
\end{enumerate}
\end{lemma}
\begin{proof}
Consider distinct CDFs $\F$ and $\G$. We have
\begin{align}
     U_h\left(\F  \right) &- U_h\left(\G  \right)\; \nonumber\\
     &= \int_{-\infty}^{\infty} \Big(C_1(1-\F(x))^m - C_2(1-\F(x))^{2m} -C_1(1-\G(x))^m + C_2(1-\G(x))^{2m} \Big)\der x \\
     &= \int_{-\infty}^{\infty} \Big(C_1(1-\F(x))^m - (1-\G(x))^m\Big)\left(1-\frac{C_2}{C_1}(1-\F(x))^{m}-\frac{C_2}{C_1}(1-\G(x))^{m}\right) \der x \\ 
     &\leq \int_{-\infty}^{\infty} C_1|(1-\F(x))^m - (1-\G(x))^m)(1-\frac{C_2}{C_1}(1-\F(x))^{m}-\frac{C_2}{C_1}(1-\G(x))^{m})| \der x \\ 
      &\leq \int_{-\infty}^{\infty} C_1 |(1-\F(x))^m - (1-\G(x))^m||1-\frac{C_2}{C_1}(1-\F(x))^{m}-\frac{C_2}{C_1}(1-\G(x))^{m}| \der x \\ 
      \label{eq:general_last_2}
      &\leq \int_{-\infty}^{\infty} C_1|(1-\F(x))^m - (1-\G(x))^m| \der x \\ 
       \label{eq:general_last_3}
      &\leq C_1 \int_{-\infty}^{\infty} |(1-\F(x) - 1+\G(x))|^m \der x \\ 
      \label{eq:general_last_in}
      &\leq C_1 \int_{-\infty}^{\infty} |\F(x) -\G(x)|^m \der x \ ,
\end{align}
where~\eqref{eq:general_last_2} follows from $0 \leq \F(x), \G(x) \leq 1$ and $C_1 \geq C_2$, and~\eqref{eq:general_last_3} follows from the inequality $|x^m-y^m| \leq |x-y|^m$ for $x, y \geq 0$. Furthermore, when $m=1$, we have
\begin{align}
    U_h\left(\F  \right) - U_h\left(\G  \right)\; &\leq C_1 \int_{-\infty}^{\infty} |\F(x) -\G(x)| \der x \;=\; C_1 \norm{\F - \G}_{\rm W} \ ,
\end{align}
which implies that the \holder~constant is $\mathcal{L}_{\rm H} = C_1$ and the exponent is $q=1$. Alternatively, when $m<1$, we have to restrict the distributions to be bounded with support $\Omega=[V, Z]$, in which case we have
\begin{align}
     U_h\left(\F  \right) - U_h\left(\G  \right)\; 
      &\leq C_1 \int_{-\infty}^{\infty} |\F(x) - \G(x)|^m \der x \\ 
      &\leq C_1 \int_{V}^{Z} |\F(x) - \G(x)|^m \der x \\ 
      \label{eq:general_last_4}
      &\leq C_1 \left[ \int_{V}^{Z} |\G(x)-\F(x) | \der x \right]^m\\
      \label{eq:general_last}
      &= C_1 (Z-V)^{1-m} \norm{\F - \G}_{\rm W}^m\ ,
\end{align}
where~\eqref{eq:general_last_4} follows from Jensen's inequality, and,~\eqref{eq:general_last} follows from Lemma \ref{lemma:lipschitz-wasserstein}. Hence, for Category $1$ PMs with $m<1$, the \holder~constant is $\mathcal{L}_{\rm H} = C_1(Z-V)^{1-m} $ and the exponent is $q=m$.
\end{proof}

\begin{lemma}[\holder~Constant for a Category $2$ PM]
\label{lemma:mean_median_Holder_general}
For any Category $2$ PM with the distortion function $h(u)= C\min \{u, (1-u) \}$, we have $q=1$ and $\mcL=C$.
\end{lemma}
\begin{proof}
Consider distinct CDFs $\F$ and $\G$. Let us define
\begin{align}
    x_\F = \inf\left\{x : \F(x) \geq \frac{1}{2}\right\}, \qquad\text{and}\qquad 
    x_\G = \inf\left\{x : \G(x) \geq \frac{1}{2}\right\} \ .
\end{align}
We have 
\begin{align}
     U_h\left(\F  \right) - U_h\left(\G  \right)\; &= \int_{-\infty}^{\infty} C\min\{1-\F(x), \F(x)\} - C\min\{1-\G(x), \G(x)\} \der x \\
     &= C\int_{-\infty}^{\min\{x_\F, x_\G \}} \min\{1-\F(x), \F(x)\} - \min\{1-\G(x), \G(x)\} \der x \\
     \nonumber
     &\quad + C\int_{\min\{x_\F, x_\G \}}^{\max\{x_\F, x_\G \}} \min\{1-\F(x), \F(x)\} - \min\{1-\G(x), \G(x)\} \der x \\
     \nonumber
     &\quad + C\int_{\max\{x_\F, x_\G \}}^{\infty}  \min\{1-\F(x), \F(x)\} - \min\{1-\G(x), \G(x)\} \der x \\
     \nonumber
     &= C\int_{-\infty}^{\min\{x_\F, x_\G \}} \F(x) - \G(x) \der x \\
     \nonumber
     &\quad + C\underbrace{\int_{\min\{x_\F, x_\G \}}^{\max\{x_\F, x_\G \}} \min\{1-\F(x), \F(x)\} - \min\{1-\G(x), \G(x)\} \der x}_{\triangleq\;M} \\
     \label{eq:lemma_m_decompost}
     &\quad + C\int_{\max\{x_\F, x_\G \}}^{\infty} \G(x) - \F(x) \der x 
\end{align}
Now let us consider two cases to provide an upper bound on $M$.
\begin{enumerate}
    \item \textbf{Case 1} ($x_\F < x_\G$) : 
    \begin{align}
        M &= \int_{\min\{x_\F, x_\G \}}^{\max\{x_\F, x_\G \}} \min\Big\{1-\F(x), \F(x)\Big\} - \min\Big\{1-\G(x), \G(x)\Big\} \der x \\
        &=\int_{x_\F}^{x_\G} 1-\F(x) - \G(x) \der x \\
        \label{eq:leamm_mm_general_1}
        &\leq \int_{x_\F}^{x_\G} \F(x) - \G(x) \der x\ ,
    \end{align}
    where \eqref{eq:leamm_mm_general_1} follows from $\F(x) \geq \frac{1}{2}$ for all $x \geq x_\F$.
    \item \textbf{Case 2} ($x_\F > x_\G$): In this case, we have
    \begin{align}
        M &= \int_{\min\{x_\F, x_\G \}}^{\max\{x_\F, x_\G \}} \min\Big\{1-\F(x), \F(x)\Big\} - \min\Big\{1-\G(x), \G(x)\Big\} \der x \\
        &=\int_{x_\G}^{x_\G}\Big(  \F(x) + \G(x) -1\Big)\der x \\
        \label{eq:leamm_mm_general_2}
        &\leq \int_{x_\G}^{x_\G} \G(x) - \F(x)\der x \ ,
    \end{align}
    where \eqref{eq:leamm_mm_general_2} follows from $\F(x) \leq \frac{1}{2}$ for all $x \leq x_\F$.  
\end{enumerate}
Finally combining \eqref{eq:lemma_m_decompost}, \eqref{eq:leamm_mm_general_1}, and \eqref{eq:leamm_mm_general_2}, we have
\begin{align}
U_h\left(\F  \right) - U_h\left(\G  \right)\; &= C\int_{-\infty}^{\infty} \min\Big\{1-\F(x), \F(x)\Big\} - \min\Big\{1-\G(x), \G(x)\Big\} \der x \\
    \label{eq:mm_general_last_2}
      &\leq C\int_{-\infty}^{\infty} |\F(x) - \G(x)| \der x \\
      \label{eq:mm_general_last}
      &= C \norm{\F - \G}_{\rm W}\ ,
\end{align}
where \eqref{eq:mm_general_last} follows from Lemma \ref{lemma:lipschitz-wasserstein}. This indicates that $q=1$ and $\mathcal{L}_{\rm H} =C$.
\end{proof}

\begin{lemma}[PHT Measure \holder~Constants]
\label{lemma:PHT}
    Consider distributions supported on $[0,\tau]$ for some $\tau\in\R_+$. For the PHT measure, i.e., $h(u) =  u^s$ for some $s\in (0,1)$, we have $q=s$ and $\mcL = \tau^{1-s}$.
\end{lemma}
\begin{proof}
Consider two distinct CDFs $\F$ and $\G$ supported on $[0,\tau]$. We have
\begin{align}
U_h(\F)-U_h(\G)&=  \int_{0}^{\tau}\Big((1-\F(x))^{s}-(1-\G(x))^{s}\mathrm{d}x \nonumber\\
&\leq   \int_{0}^{\tau}|\F(x)-\G(x)|^{s}\Big)\mathrm{d}x \\
&\leq   \left[\int_{0}^{\tau}|\F(x)-\G(x)|\mathrm{d}x\right]^{s} \\
& \leq \tau^{1-s} \norm{\F-\G}_{\rm W}^{s},
\end{align}
where we used Jensen's inequality for the penultimate inequality and Lemma \ref{lemma:lipschitz-wasserstein} for the final inequality. This indicates that $q=s$ and $\mcL=\tau^{1-s}$. 
\end{proof}

\begin{lemma}[Dual Power]
\label{lemma:Dual_Power_general}
Consider two distinct distributions with CDFs $\F$ and $\G$. 
For dual power with the parameter $s \geq 2$, i.e., $h(u)= 1-(1-u)^s$, we have $q=1$ and $\mcL=s$.
\end{lemma}
\begin{proof}
We have
\begin{align}
    U_h\left(\F\right) - U_h\left(\G\right) &= \int_{-\infty}^\infty (1-(\F(x))^s) - (1-(\G(x))^s) \der x \\ 
    &=\int_{-\infty}^\infty (\G(x))^s - (\F(x))^s   \der x \\ 
    &=\int_{\{x:\F(x)<\G(x)\}} (\G(x))^s - (\F(x))^s   \der x \\
    \label{eq:dp_beforelast}
    &\leq \int_{\{x:\F(x)<\G(x)\}} s (M(x))^{s-1}(\G(x)-\F(x))   \der x \\
    &\leq\int_{\{x:\F(x)<\G(x)\}} s (\G(x)-\F(x))   \der x \\
    &=s \int_{\{x:\F(x)<\G(x)\}}  |\F(x) - \G(x)|   \der x \\
    \label{eq:dp_last}
    &\leq s \int_{-\infty}^\infty  |\F(x) - \G(x)|   \der x \\
    &= s\norm{\F-\G}_{\rm{W}}\ ,
\end{align}
where the inequality in \eqref{eq:dp_beforelast} follows from the fact that the distortion function is continuous, and therefore, we can define $M(x) \in (\min\{\F(x),  \G(x)\}, \max\{\F(x),  \G(x)\})$ and use the mean-value theorem, and, the equality in \eqref{eq:dp_last} follows from $s\geq 2$ and $M(x) \leq 1$. Hence, \eqref{eq:dp_last} indicates that $q=1$ and $\mathcal{L}=s$. 
\end{proof}

\begin{lemma}[Quadratic]
\label{lemma:Quadratic_general}
For the quadratic PM with the parameter $s \in [0,1]$, i.e., the PM induced by the distortion function $h(u)= (1+s)u-su^2$, we have $q=1$ and $\mcL = 1+s$.
\end{lemma}
\begin{proof}
Consider two distinct CDFs $\F$ and $\G$. We have
\begin{align}
    U_h\left(\F\right) - U_h\left(\G\right) & = \int_{-\infty}^{\infty}(1+s)(1-\F(x))-s(1-\F(x))^2 - ((1+s)(1-\G(x))-s(1-\G(x))^2) \der x  \\
    &= \int_{-\infty}^{\infty} (1+s)(\G(x) - \F(x)) -s(\F(x) - \G(x))(2 - \F(x) -\G(x))\der x \\
    &= \int_{-\infty}^{\infty} (\G(x) - \F(x)) ((1+s) - s(2 - \F(x) -\G(x))\der x \\
    &= \int_{\{x: \G(x)>\F(x) \}} (\G(x) - \F(x)) ((1+s) - s(2 - \F(x) -\G(x))\der x \\
    &\quad +  \int_{\{x: \G(x)< \F(x) \}} (\G(x) - \F(x)) ((1+s) - s(2 - \F(x) -\G(x))\der x \\
    \label{eq:quad_first_term}
    &\leq \int_{\{x: \G(x)>\F(x) \}} (\G(x) - \F(x)) (1+s)\der x \\ 
    \label{eq:quad_second_term}
    &\quad +  \int_{\{x: \G(x)< \F(x) \}} (\G(x) - \F(x)) (1-s)\der x \\
    \label{eq:quad_second_term_2}
    &\leq \int_{\{x: \G(x)>\F(x) \}} (\G(x) - \F(x)) (1+s)\der x \\
    \label{eq:quad_first_term_2}
    &\quad +0 \\
    &\leq (1+s) \int_{\{x: \G(x)>\F(x) \}} |\G(x) - \F(x)| \der x \\
    &\leq (1+s) \int_{-\infty}^{\infty} |\F(x) - \G(x)| \der x \\
    \label{eq:quadratic_lastline}
    &= (1+s) \norm{\F-\G}_{\rm W}\ ,
\end{align}
where~\eqref{eq:quad_first_term} follows from $\F(x), \G(x) \leq 1$ $\forall x$, \eqref{eq:quad_second_term} follows from $\F(x), \G(x) \geq 0$ $\forall x$, \eqref{eq:quad_second_term_2} follows from $s \leq 1$, and, \eqref{eq:quadratic_lastline} follows from Lemma \ref{lemma:lipschitz-wasserstein}. Hence, \eqref{eq:quadratic_lastline} indicates that $q=1$ and $\mathcal{L} = 1+s$. .
\end{proof}

\section{Proof of Lemma \ref{lemma:beta12}}
\label{Appendix:proof of lemma beta12}

Define
\begin{align}
    U(\balpha) = \int_{-\infty}^{\infty} \Big(h\Big(1-\sum_{i \in [K] } \alpha(i)\F_i(x)\Big) -h(1)u(-x)\Big)\der x 
\end{align}
where $u(\cdot)$ is the unit step function.
The components of the gradient $\nabla U(\mathbf{a}^{(1)})$ and the Hessian matrix $H(\mathbf{a}^{(1)})=\nabla^{2}U(\mathbf{a}^{(1)})$ are given by
\begin{align}
[\nabla U(\mathbf{a}^{(1)})]_{i}&=\frac{\partial U}{\partial \alpha(i)}=-\int_{-\infty}^{\infty} h^\prime\Big(1-\sum_{i \in [K] } a(i)\F_i(x)\Big)\F_{i}(x)\der x \quad \forall\;i\in[K] \ ,\\
[H(\mathbf{a}^{(1)})]_{ij}&=\frac{\partial^{2}U}{\partial \alpha(i)\partial \alpha(j)}=\int_{-\infty}^{\infty}h^{\prime\prime}\Big(1-\sum_{i \in [K] } a(i)\F_i(x)\Big)\F_{i}(x)\F_{j}(x)\der x \quad\forall \;i,j\in[K]\times[K]\ .
\end{align}
Let us assume that $H(\balpha^\star)$ is negative definite. Shortly, we will show that this property holds for continuous sub-Gaussian distributions and discrete distributions with $|\Omega| \geq K$. By Taylor's theorem around $\mathbf{a}^{(1)}$, we have
\begin{align}
U(\mathbf{a}^{(2)}) = U(\mathbf{a}^{(1)}) + \nabla U(\mathbf{a}^{(1)}) \cdot (\mathbf{a}^{(2)} - \mathbf{a}^{(1)}) + \frac{1}{2}(\mathbf{a}^{(2)} - \mathbf{a}^{(1)})^T H(\mathbf{a}^{(1)}) (\mathbf{a}^{(2)} - \mathbf{a}^{(1)}) + O(||\mathbf{a}^{(2)} - \mathbf{a}^{(1)}||^3)\ .
\end{align}
We define $v_{\varepsilon} \triangleq \mathbf{a}^{(2)} - \mathbf{a}^{(1)}$.
Note that since $\ba^{(1)}$ and $\ba^{(2)}$ are discrete points, each coordinate of their difference vector is an integer multiplied by $\varepsilon$. Hence, we may express $v_{\varepsilon} = \varepsilon d_\varepsilon$, where $d_\varepsilon = O(1)$. Furthermore, each coordinate of $v_{\varepsilon}$ is of the order $O(\varepsilon)$. Next, we rewrite Taylor's expansion in terms of $v_\varepsilon$ as follows.
\begin{align}
\label{eq: Taylor_beta_1}
U(\mathbf{a}^{(1)}) - U(\mathbf{a}^{(2)}) = -\nabla U(\mathbf{a}^{(1)}) \cdot v_{\varepsilon} - \frac{1}{2}v_{\varepsilon}^T H(\mathbf{a}^{(1)}) v_{\varepsilon} + O(||v_{\varepsilon}||^3)\ .
\end{align}
From the definition of $\beta$ in~\eqref{eq:beta_new_def} combined with~\eqref{eq: Taylor_beta_1}, we have 
\begin{align}
    \beta\;=\;\lim_{\varepsilon \rightarrow 0} \frac{\log (U(\mathbf{a}^{(1)}) - U(\mathbf{a}^{(2)}))}{\log \varepsilon} &= \lim_{\varepsilon \rightarrow 0} \frac{-\nabla U(\mathbf{a}^{(1)}) \cdot v_{\varepsilon} - \frac{1}{2}v_{\varepsilon}^T H(\mathbf{a}^{(1)}) v_{\varepsilon} + O(||v_{\varepsilon}||^3)}{\log \varepsilon}
\end{align}
We have the following two cases.
\begin{itemize}
    \item\textbf{Case 1 ($\nabla U(\balpha^\star) \neq 0$):} When $\nabla U(\balpha^\star) \neq 0$, it means that the utility is maximized on the boundary of the set $\Delta^{K-1}$.  Due to the continuity of the gradient $\nabla U$, $\lim_{\varepsilon \rightarrow 0} \nabla U(\mathbf{a}^{(1)}) = \nabla U(\balpha^\star) = g^{*} \ne 0$. We have

\begin{align}
    \beta = \lim_{\varepsilon \rightarrow 0} \frac{\log (U(\mathbf{a}^{(1)}) - U(\mathbf{a}^{(2)}))}{\log \varepsilon} &= \lim_{\varepsilon \rightarrow 0} \frac{\log(- g^* \cdot \varepsilon d_{\varepsilon} - \frac{1}{2}\varepsilon d_{\varepsilon}^\top H(\mathbf{a}^{*}) \varepsilon d_{\varepsilon} + O(||\varepsilon d_{\varepsilon}||^3))}{\log \varepsilon}\\
    &= \lim_{\varepsilon \rightarrow 0} \frac{\log(- g^* \cdot \varepsilon d_{\varepsilon} + O(\varepsilon^2))}{\log \varepsilon}\\ 
    &= \lim_{\log\varepsilon \rightarrow 0} \frac{\log(\varepsilon (- g^* \cdot  d_{\varepsilon} + O(\varepsilon)))}{\log \varepsilon}\\ 
    &= \lim_{\varepsilon \rightarrow 0} \frac{\log \varepsilon}{\log \varepsilon} + \lim_{\varepsilon \rightarrow 0} \frac{\log(- g^* \cdot  d_{\varepsilon} + O(\varepsilon))}{\log \varepsilon}\\ 
    &= 1 + 0 = 1\ .
\end{align}

\item\textbf{Case 2 ($\nabla U(\balpha^\star) = 0$ and $H(\balpha^\star)$ is negative definite):} We have
\begin{align}
    \beta = \lim_{\varepsilon \rightarrow 0} \frac{\log (U(\mathbf{a}^{(1)}) - U(\mathbf{a}^{(2)}))}{\log \varepsilon} &= \lim_{\varepsilon \rightarrow 0} \frac{\log(- g^* \cdot \varepsilon d_{\varepsilon} - \frac{1}{2}\varepsilon d_{\varepsilon}^\top H(\mathbf{a}^{*}) \varepsilon d_{\varepsilon} + O(||\varepsilon d_{\varepsilon}||^3))}{\log \varepsilon}\\
    &= \lim_{\varepsilon \rightarrow 0} \frac{\log(0 + - \frac{1}{2}\varepsilon d_{\varepsilon}^\top H(\mathbf{a}^{*}) \varepsilon d_{\varepsilon} + O(||\varepsilon d_{\varepsilon}||^3))}{\log \varepsilon}\\ 
    &= \lim_{\varepsilon \rightarrow 0} \frac{\log(\varepsilon^2(- \frac{1}{2} d_{\varepsilon}^\top H(\mathbf{a}^{*}) d_{\varepsilon} + O(\varepsilon))}{\log \varepsilon}\\ 
    &= \lim_{\varepsilon \rightarrow 0} \frac{2 \log(\varepsilon)}{\log \varepsilon} + \lim_{\varepsilon \rightarrow 0} \frac{\log(- \frac{1}{2} d_{\varepsilon}^\top H(\mathbf{a}^{*}) d_{\varepsilon} + O(\varepsilon)}{\log \varepsilon}\\ 
    &= 2 + 0 = 2\ .
\end{align}
\end{itemize}
Next, we analyze $H(\balpha^\star)$ for continuous sub-Gaussian distributions and discrete distributions, identifying cases in which it is negative definite.

\paragraph{Continuous Distributions.}
For any non-zero vector $\bv \in \R^K$, let us define the set $S_v$ such that
\begin{align}
    S_v \triangleq \Big\{x : |\sum_{i \in [K]} v(i) \F_i(x) \neq 0 |\Big\} \ .
\end{align}
For continuous distributions, the set $S_v \subsetneq \emptyset$. Hence, we obtain

 \begin{align}
     \bv^TH(\balpha^\star)\bv &= 
     \sum_{i=1}^K\sum_{j=1}^K
  v(i)\,v(j)
\int_{-\infty}^{\infty}
    h^{\prime\prime}\Big(1-\sum_{i\in [K]}\alpha^\star(i)\F_i(x)\Big)
    \F_i(x)\,\F_j(x)
  \,\der x \\ 
  &=  
\int_{-\infty}^{\infty}
   h^{\prime\prime}\Big(1-\sum_{i\in [K]}\alpha^\star(i)\F_i(x)\Big)
    \sum_{i=1}^K\sum_{j=1}^K v(i)\,v(j) \F_i(x)\,\F_j(x)
  \,\der x \\ 
  &=  
\int_{-\infty}^{\infty}
   h^{\prime\prime}\Big(1-\sum_{i\in [K]}\alpha^\star(i)\F_i(x)\Big)
    \left(\sum_{i=1}^K v(i)\F_i(x)\right)^2
  \,\der x  \\
  &=  
\int_{x \in S_v}
   h^{\prime\prime}\Big(1-\sum_{i\in [K]}\alpha^\star(i)\F_i(x)\Big)
    \left(\sum_{i=1}^K v(i)\F_i(x)\right)^2
  \,\der x  \\ 
  &+  \int_{x \notin S_v}
   h^{\prime\prime}\Big(1-\sum_{i\in [K]}\alpha^\star(i)\F_i(x)\Big)
    \left(\sum_{i=1}^K v(i)\F_i(x)\right)^2
  \,\der x \\
  &=  
\int_{x \in S_v}
   h^{\prime\prime}\Big(1-\sum_{i\in [K]}\alpha^\star(i)\F_i(x)\Big)
    \left(\sum_{i=1}^K v(i)\F_i(x)\right)^2
  \,\der x  + 0 \\ &
  < 0\ .
 \end{align}
where due to the strict concavity of $h$ we have $h^{\prime\prime}(y) < 0$.
\paragraph{Discrete Distributions.} Let $\Omega = \{x_1,\cdots,x_M\}$ denote the support of the distributions. We have
\begin{align}
U(\balpha^\star)
=\int_0^\infty h\Bigl(1 - \sum_{i \in [K] } \alpha^\star(i)\F_i(x)\Bigr)\,\der x
=\sum_{m=1}^M (x_m - x_{m-1})\,
  h\Bigl(1 - \sum_{i \in [K] } \alpha^\star(i)\F_i(x_m)\Bigr).
\end{align}
Differentiating twice under the finite sum gives
\begin{align}
\frac{\partial^2 U_h}{\partial\alpha(j)\,\partial\alpha(k)}
(\balpha^\star)
=\sum_{m=1}^M
  (x_m - x_{m-1})\,
  h''\Bigl(1 - \sum_{i \in [K]}\alpha^\star(i)\F_i(x_m)\Bigr)\,
  \F_j(x_m)\,\F_k(x_m).
\end{align}
Hence, the Hessian is
\begin{align}
H_{jk}(\balpha^\star)
=\sum_{m=1}^M w_m\,\F_j(x_m)\,\F_k(x_m),
\quad
w_m=(x_m - x_{m-1})\,h''\Bigl(1 - \sum_{i \in [K]}\alpha^\star(i)\F_i(x_m)\Bigr)<0.
\end{align}
For any \(\bv\in\R^K\),
\begin{align}
\bv^T H(\balpha^\star)\,\bv
=\sum_{m=1}^M w_m\,
  \Bigl(\sum_{i \in [K]}v(i) \F_i(x_m)\Bigr)^{2}
\;\le\;0\ ,
\end{align}
and hence \(H(\balpha^\star)\) is always negative semi-definite. Furthermore, note that the finite sum $\sum_{i \in [K]}v(i) \F_i(x_m)$ is equal to $0$ if and only if the vector $[\F_1(x_m),\cdots,\F_K(x_m)]^\top$ is orthogonal to $\bv$ for every $m\in[M]$. In other words, if we denote $[\F(\bX)]_{m,i} = \F_i(x_m)$ for all $m \in [M]$ and $i \in [K]$, $\bv$ has to lie in the null space of the matrix $\F(\bX)\in\R^{M\times K}$, whose columns $m\in[M]$ are composed of the vectors $[\F_1(x_m),\cdots,\F_K(x_m)]^\top$. Noting that the dimension of the null-space of $\F(\bX)$ is equal to $K-M$, we require $M\geq K$ for multinomial distributions.
\paragraph{$K$-arm Bernoulli bandit.}
For Bernoulli bandits, we have $M=2$.  If the number of arms \(K>2\) then
\(\{\F(x_1),\F(x_2)\}\) cannot span \(\R^K\), hence \(\rank (H(\balpha^\star))\leq 2<K\) and
\(H\) is only negative semi-definite.  However, if the number of arms is exactly \(K=2\) %and \(\{\F(x_1),\F(x_2)\}\) are linearly independent, 
then \(\rank (H(\balpha^\star))=2\) and
\(H(\balpha^\star)\) is strictly negative–definite.

\section{${\bar\beta}$ Analysis for Bernoulli Bandits}

\label{Appendix:Beta_Appendix}
In this section, we analyze the parameter $\beta$ for various PMs and the $K$-arm Bernoulli bandit model with mean values $\bp = (p(1), \cdots, p(K))$. This parameter is significant for the analysis of fixed-horizon algorithms. We start this section by introducing some properties. In some of the proofs, we leverage a useful property of the gaps $\delta_{12}$ and $\delta_{23}$. Specifically, it can be readily verified that 
\begin{align}
\label{eq:13_greater_12}
    \delta_{13}(\varepsilon) > \max\{\delta_{12}(\varepsilon), \delta_{23}(\varepsilon) \} \ .
\end{align}
\subsection{Properties}

\begin{lemma}
\label{lemma:f_G_epsilon}
    Consider a $K$-arm Bernoulli bandit instance with mean values \(\bp = \left(p(1) \cdots p(K)\right)\). We consider two mixture distributions with the discrete mixing coefficients $\ba^{(r)}, \ba^{(s)} \in \Delta_{\varepsilon}^{K-1}$, where $r,s\in\{1,2,3\}$ such that $s=r+1$, and we define
\begin{align}
\label{eq:Bern_1_2_3}
    p^{(r)} \triangleq \sum_{i \in [K]} a^{(r)}(i)p(i) \qquad \text{and} \qquad p^{(s)} \triangleq \sum_{i \in [K]} a^{(s)}(i)p(i) \ .
\end{align} 
We have $|p^{(r)}-p^{(s)}| = \Theta(\varepsilon)\ .$

\end{lemma}
\begin{proof}
Under the Bernoulli bandit model, we have the following simplification for the utility function:
\begin{align}
    U_h\left(\sum_{i \in [K]}a(i) \F_i\right) = h\left(\sum_{i \in [K]}a(i) p(i)\right)\ .
\end{align}
Let us first define $m,n$ such that
\begin{align}
\label{eq:m_n_def}
  \{ m,n \}\;\in\; \argmin_{i,j \in [K], p(i) \neq p(j)} |p(i)-p(j)| \ ,
\end{align}
i.e., $m,n$ are two coordinates such that the means of these two coordinates are not equal to each other and the distance between them is minimum. Furthermore, let us define $p^\star \triangleq \sum_{i \in [K]} \alpha^\star(i)p(i)$. By concavity of $h$, the PM $U_h$ is increasing in the interval $[0,p^\star)$ and decreasing in the interval $[p^\star,1]$. We first consider the case when both $p^{(r)}$ and $p^{(s)}$ lie on the same side with respect to $p^\star$, i.e., either $p^{(r)}\in[0,p^\star)$ and $p^{(s)}\in[0,p^\star)$, or $p^{(r)}\in[p^\star,1]$ and $p^{(s)}\in[p^\star,1]$. Due to monotonicity, starting from $p^{(r)}$, we can only obtain $p^{(s)}$ by adjusting an $\varepsilon$ amount in the two coordinates $m$ and $n$, since this corresponds to the smallest increase (or decrease) in utility. Hence, we have
\begin{align}
    \big | p^{(r)} - p^{(s)}\big |\;=\; \Big| \sum\limits_{i\in[K]} \big( a^{(r)}(i) - a^{(s)}(i)\big)p(i)\Big |\;=\;\big | p(m) - p(n)\big |\cdot\varepsilon\; =\; \Theta(\varepsilon)\ .
\end{align}
Next, let us consider the case when $p^{(r)}$ and $p^{(s)}$ lie on opposite sides of $p^\star$. Assume that $p^{(2)} > p^{(1)}$, and the other direction follows the exact line of arguments.  First, let $p = \sum_{i\in[K]}a(i)p(i)$ denote the mixture corresponding to any {\em discrete} $\balpha\in\Delta_{\varepsilon}^{K-1}$ which is not $a^{(1)}$ or $a^{(2)}$. Let us assume that there is an $\ba$ such that $p^{(2)} > p > p^{(1)}$. For $\lambda \in (0,1)$ from concavity of the distortion function $h$, we have
\begin{align}
    h(p) > \lambda h(p^{(2)}) + (1-\lambda)h(p^{(2)}) > h(p^{(2)}) \ ,
\end{align}
which contradicts the definition of $\ba^{(2)}$. Therefore, we can conclude that there is no discrete mixing coefficient between $p^{(1)}$ and $p^{(2)}$. Now, let us consider $a^\dagger \in \Delta^{K-1}_\varepsilon$, such that for some $m, n \in [K]$ where $p(m) > p(n)$,
\begin{align}
 a^\dagger(i) \triangleq
\begin{cases}
   a^{(1)}(i) \quad \text{if } i \notin \{m, n\}, \\
   a^{(1)}(i)+\varepsilon \quad \text{if } i = m, \\
   a^{(1)}(i)-\varepsilon \quad \text{if } i = n .
\end{cases} 
\end{align}
We have
\begin{align}
    p^{(2)}-p^{(1)} \leq p^{\dagger}-p^{(1)} = \varepsilon|p(m)-p(n)| \ .
\end{align}
The minimum distance between two inner products $p^{(s)}$ and $p^{(r)}$ is
\begin{align}
    p^{(s)} - p^{(r)} \geq \sum_{i \in [K]} (a^{(s)} - a^{(q)})p(i) = \varepsilon \min_{i.j, p(i)\neq p(j)} |p(i)-p(j)| \ .
\end{align}
Therefore, we obtain $p^{(2)}-p^{(1)} = \Theta(\varepsilon)$. The other case $p^{(2)} > p^{(1)}$ follows a similar analysis. Therefore, we can conclude that $|p^{(2)}-p^{(1)}| = \Theta(\varepsilon)$. Next, we will show that similar arguments hold for $|p^{(2)} - p^{(3)}|$. First, let us assume that $p^{(2)} > p^{(3)}$ and the other direction follows the exact line of arguments. We know that when $p^{(2)} > p^\star$, we have $p^{(2)} > p^{(1)}$. Now, let us assume that $p^{(3)} > p^{(1)}$. For some $\lambda \in (0,1)$ we have
\begin{align}
    h(p^{(3)}) > \lambda h(p^{(2)}) + (1-\lambda)h(p^{(1)}) > h(p^{(2)}) \ ,
\end{align}
which contradicts the definition of $\ba^{(3)}$. Therefore, we have $p^{(2)} > p^{(1)} > p^{(3)}$. Now, let us assume $p^{(1)} < p^\star$. Since we know that the second best $p^{(2)}$ lies on the opposite side of $p^\star$, using exact line of arguments when $p^{(2)}$ and $p^{(1)}$ lies on the same side, we conclude that
$p^{(1)} - p^{(3)} = \Theta(\varepsilon)$,
and from the first case, we know that $p^{(2)} - p^{(1)} = \Theta(\varepsilon)$, and therefore, we have 
\begin{align}
    p^{(2)} - p^{(3)} = p^{(2)} - p^{(1)} + p^{(1)} - p^{(3)} = \Theta(\varepsilon) \ .
\end{align}
Now, assume $p^{(1)} > p^\star$. Since in this case the best and the second best lie on the same side of $p^\star$, we have $p^{(2)} - p^{(1)} = \Theta(\varepsilon)$. As $p^{(1)}$ and $p^{(2)}$ lies on the other side of the curve, following the exact line of arguments when $p^{(1)}$ and the $p^{(2)}$ lie on the opposite side of the curves for $p^{(1)}$ and $p^{(3)}$ we obtain $p^{(1)} - p^{(3)} = \Theta(\varepsilon)$. Therefore,
\begin{align}
    p^{(2)} - p^{(3)} = p^{(2)} - p^{(1)} + p^{(1)} - p^{(3)} = \Theta(\varepsilon) \ .
\end{align}
The case $p^{(3)} > p^{(2)}$ follows the exact line of arguments. Therefore, we conclude that $|p^{(2)} - p^{(3)}|=\Theta(\varepsilon)$

\end{proof}

\paragraph{Lower bound on $\beta$.}
Now, we will find a lower bound on ${\bar\beta}$.
From \holder~continuity, we have
\begin{align}
\label{eq:lower_on_delta13}
   \delta_{12} \leq \delta_{13} \leq \mathcal{L} |p_1-p_3|^q \leq \mathcal{L} (|p_1-p_2| + |p_2-p_3|)^q  = \Theta(\varepsilon^q) \ ,
\end{align}
where $p_1, p_2$ and $p_3$ is defined in \eqref{eq:Bern_1_2_3}, and this equation indicates that 
    $\delta_{13} = O(\varepsilon^q) \ $.
Therefore, a lower bound on is ${\bar\beta} \geq q$.

\subsection{$\beta$ and ${\bar\beta}$  values for $K-$arm Bernoulli Bandits}

\begin{lemma}
\label{lemma:beta_lemma_monotone}
    Consider a $K$-arm Bernoulli bandit instance with mean values \(\bp = \left(p(1) \cdots p(K)\right)\). For a PM with concave and monotonically increasing distortion function with \holder~exponent $q$, we have
    \begin{enumerate}
        \item $\beta={\bar\beta}=1$ for $q=1$,
        \item $\beta={\bar\beta} \in [q,1]$ for $q <1$.
    \end{enumerate}\
    % \begin{align}
    %     \beta  = 1 \ .
    % \end{align}
\end{lemma}
\begin{proof}
For a Bernoulli bandit instance, we have the following simplification. 
\begin{align}
    U_h \left(\sum_{i \in [K]}a(i)\F_i \right) = h\left(\sum_{i \in [K]}a(i) p_i\right) \ . 
\end{align}
Since the distortion function is strictly increasing, the optimal solution is a solitary arm, i.e,
\begin{align}
    \max_{\alpha \in \Delta^{K-1}} h\left(\sum_{i \in [K]} \alpha(i)p_i\right ) = \max_{i \in [K]} h(p(i)) \ .
\end{align} 
Let us denote this arm by $i^\star$,
\begin{align}
    i^\star \triangleq \argmax_{i \in [K]} h(p(i)) \ .
\end{align}
Furthermore, $\forall \ba \in \Delta^{K-1}_{\varepsilon}$ we have
\begin{align}
\label{eq:solitary_best}
    h\left(\sum_{i \in [K]} a(i)p_i\right) < h(p(i^\star)) \ .
\end{align}
Next, we have
\begin{align}
   \delta_{12}(\varepsilon) &= U_h \left(\sum_{i \in [K]}a^{(1)}(i)\F_i \right) - U_h \left(\sum_{i \in [K]}a^{(2)}(i)\F_i \right)
   \\ &= h\left(p^{(1)}\right) - h\left(p^{(2)}\right)\ \\
   \label{eq:concavity_h}
   &\geq h^\prime\left(p^{(1)} \right) \left(p^{(1)}-p^{(2)}\right) \\
   \label{eq:delta_min_conc_mon_analysis}
   &\geq h^\prime(p(i^\star)) (p^{(1)}-p^{(2)}) \ ,
\end{align}
where we defined $p_1,p_2$ and $p_3$ in \eqref{eq:Bern_1_2_3}, \eqref{eq:concavity_h} follows from concavity of $h$, and \eqref{eq:delta_min_conc_mon_analysis} follows from the fact that leveraging concavity, we have $h^\prime(a) \geq h^\prime(b)$ for $a \leq b$, together with~\eqref{eq:solitary_best}. For a strictly increasing function, the derivative $\forall p \in (0,1)$, $h^\prime(p) \neq 0$. From \eqref{eq:delta_min_conc_mon_analysis}, we can conclude that for PMs with strictly increasing and concave distortion functions \(\delta_{12}(\varepsilon) = \Omega(\varepsilon)\). Furthermore, note that 
\begin{align}
    \delta_{13}(\varepsilon) > \delta_{12}(\varepsilon) = \Omega(\varepsilon) \ .
\end{align}
Finally, from Lemma \ref{lemma:f_G_epsilon}, we have $\delta_{13}(\varepsilon) = O(\varepsilon^q)$. Considering, $\delta_{13}(\varepsilon) = \Omega(\varepsilon)$, we can conclude that for $q < 1$, we have $\beta \in [q, 1]$. For $q=1$, we have $\delta_{13}(\varepsilon) = O(\varepsilon)$ which results in $
    \delta_{13}(\varepsilon) = \Theta(\varepsilon)
$
which implies that ${\bar\beta} = 1$.

\end{proof}

\begin{lemma}
    Consider a $K$-arm Bernoulli bandit instance with the mean values $\bp = (p(1), \cdots, p(K))$. For PM with distortion function $h(u)=u^m - u^{2m}$ where $m >0$, we have $\beta \in [q,2]$.
\end{lemma}

\begin{proof}
    For a $K$-arm Bernoulli, we have
    \begin{align}
    \label{eq:gini_bernouli_beta_equal}
        U_h\bigg(\sum_{i \in [K]} a(i)\F_i \bigg) = h\bigg(\sum_{i \in [K]} a(i)p(i) \bigg)
    \end{align}
We first start by finding an upper-bound on ${\bar\beta}$.
Let us consider $r, s \in [3]$ and $r<s$. For the sub-optimality gaps defined in \eqref{eq:suboptimality_gaps}, we have
\begin{align}
    \delta_{ij} &= U_h\bigg(\sum_{i \in [K]} a^{(r)}(i)\F_{i} \bigg) - U_h\bigg(\sum_{{i} \in [K]} a^{(s)}({i})\F_{i} \bigg) \\ 
    &= h(p^{(r)}) - h(p^{(s)}) \\
    &= \big((p^{(r)})^m-(p^{(s)})^m\big) - \big((p^{(r)})^{2m}-(p^{(s)})^{2m}\big)\ . 
\end{align} 
We will show that for at least one $(r,s)$, we have $\Delta_{rs} \geq |(p^{(r)})^m-p^{(s)})^m|^2$. Let us define 
\begin{align}
    p_{\max} \triangleq \argmax_{p \in [0,1]} h(p) =  \argmax_{p \in [0,1]} p^m(1-p^m) \ ,
\end{align}
which implies that $p_{\max}^m = \frac{1}{2}$. From concavity, as shown in Lemma \ref{lemma:f_G_epsilon}, we have that if $p^{(r)}, p^{(s)} > p_{\max}$, then, $p^{(s)}>p^{(r)}$, and if $p^{(r)}, p^{(s)} < p_{\max}$, then, $p^{(s)}<p^{(r)}$. We will consider two cases,
\begin{enumerate}
    \item \textbf{$p^{(s)} < p^{(r)} < p_{\max}$:} In this case, we have 
    \begin{align}
        \delta_{ij} &= ((p^{(r)})^m-(p^{(s)})^m) (1 - (p^{(r)})^{m}-(p^{(s)})^{m}) \geq ((p^{(r)})^m-(p^{(s)})^m)^2\ .
    \end{align}
    \item \textbf{$p^{(s)} > p^{(r)} > p_{\max}$:} In this case, we have
    \begin{align}
        \delta_{ij} &= ((p^{(s)})^m-(p^{(r)})^m) ((p^{(r)})^{m}+(p^{(s)})^{m}-1) \geq ((p^{(s)})^m-(p^{(r)})^m)^2\ .
    \end{align}
\end{enumerate}
Since we consider three convex combinations, by the pigeon-hole principle, at least two of them should be less than $p^m_{\max}$ or more than $p^m_{\max}$. This implies that for at least one sub-optimality gap $\Delta_{rs}$, the following holds.
\begin{align}
\label{eq:gini_mvt}
    \delta_{ij} &\geq ((p^{(r)})^m-(p^{(s)})^m)^2 \geq m(p^{(r)}-p^{(s)})^2\ ,
\end{align}
where \eqref{eq:gini_mvt} follows from mean-value theorem. Furthermore, from Lemma \ref{lemma:f_G_epsilon}, we have
\begin{align}
\label{eq:bernoulli_wass_epsilon_ij}
    |p^{(r)} - p^{(s)}| = \Theta(\varepsilon)\ .
\end{align}
Combining \eqref{eq:gini_mvt} and \eqref{eq:bernoulli_wass_epsilon_ij} we have
\begin{align}
\label{eq:delta_ij_lower}
    \delta_{ij} = \Omega(\varepsilon^2)  \Longrightarrow \delta_{13} \stackrel{\eqref{eq:13_greater_12}}{\geq} \delta_{ij} = \Omega(\varepsilon^2) \ .
\end{align}
Therefore, an upper-bound is ${\bar\beta} \leq 2$. A lower-bound is ${\bar\beta} \geq q$ from \eqref{eq:lower_on_delta13}. Hence, we can conclude that for PMs with the distortion function $h(u)=u^m - u^{2m}$, we have the range ${\bar\beta} \in [q,2]$.
\end{proof}

\begin{lemma}
    Consider a $K$-arm Bernoulli bandit instance with the mean values $\bp = (p(1), \cdots, p(K))$. For a PM with distortion function $h(u)=C\min\{u,1-u\}$ where $C >0$, we have ${\bar\beta} = 1$.
\end{lemma}

\begin{proof}
For a $K$-arm Bernoulli, we have
    \begin{align}
    \label{eq:gini_bernouli_beta_equal_2}
        U_h\bigg(\sum_{i \in [K]} a(i)\F_i \bigg) = h\bigg(\sum_{i \in [K]} a(i)p(i) \bigg)
    \end{align}
We first start by finding an upper-bound on $\beta$.
Let us consider $r, s \in [3]$ and $r<s$. For the sub-optimality gaps defined in \eqref{eq:suboptimality_gaps},we have
\begin{align}
    \delta_{ij} &= U_h\bigg(\sum_{i \in [K]} a^{(r)}(i)\F_{i} \bigg) - U_h\bigg(\sum_{{i} \in [K]} a^{(s)}({i})\F_{i} \bigg) \\ &= h(p^{(r)}) - h(p^{(s)}) \\
    &= C(\min\{p^{(r)}, 1-p^{(r)} \} - \min\{p^{(s)}, 1-p^{(s)}\}),
\end{align}
and from \eqref{eq:13_greater_12}, we know that $\delta_{13}(\varepsilon)$ is the gap with greatest value, i.e., 
\begin{align}
\label{eq:13_better_ij}
    \delta_{13}(\varepsilon) \geq \delta_{ij}(\varepsilon) \ .
\end{align}
The maximum value of this PM is \begin{align}
    \max_{p \in [0,1]} h(p) = \frac{C}{2} \ .
\end{align}
The difference between the maximum value $\frac{C}{2}$ and any mixing coefficient $\balpha \in \Delta^{K-1}_{\varepsilon}$ with $p \triangleq \sum_{i \in [K]} a(i) p(i)$
\begin{align}
    \frac{C}{2} - U_h\bigg(\sum_{i \in [K]} a(i)\F_i \bigg) &= \frac{C}{2} - C\min\{p, 1-p \} = C\left|\frac{1}{2} - p \right| \ ,
\end{align}
which means the best, second best, and third best mixing coefficients are identified based on their distance to the fraction $\frac{1}{2}$, i.e.,
\begin{align}
\label{eq:genel_mean_best_second_third}
    \left|p^{(1)} - \frac{1}{2}\right| < \left|p^{(2)} - \frac{1}{2}\right| < \left|p^{(3)} - \frac{1}{2}\right| \ .
\end{align}
From \eqref{eq:genel_mean_best_second_third}, we know that when $p^{(r)}, p^{(s)} \leq \frac{1}{2}$, then, $p^{(s)}< p^{(r)}$. When $p^{(r)}, p^{(s)} \geq \frac{1}{2}$, then, $p^{(s)} > p^{(r)}$.
Since we consider three convex combinations, by the pigeon-hole principle, at least two of them should be less than $\frac{1}{2}$ or more than $\frac{1}{2}$, which would mean that for at least one sub-optimality gap, the following holds.
\begin{align}
\label{eq:ij_lower_mean}
    \delta_{ij} &= C(\min\{p^{(r)}, 1-p^{(r)} \} - \min\{p^{(s)}, 1-p^{(s)}\}) \geq C|p^{(r)}-p^{(s)}| = \Theta(\varepsilon)
\end{align}
where the last equation follows from Lemma \ref{lemma:f_G_epsilon}.
Therefore, we have 
\begin{align}
\label{eq:delta_ij_lower_mean}
    \delta_{ij} = \Omega(\varepsilon) \Longrightarrow \ \delta_{13} \geq \delta_{ij} = \Omega(\varepsilon).
\end{align}
Therefore, an upper-bound on is ${\bar\beta} \leq 1$. From Lemma \ref{lemma:mean_median_Holder_general} and \eqref{eq:lower_on_delta13}, a lower-bound on ${\bar\beta} \geq 1$ as $q=1$. Therefore, we can conclude that for PMs with this distortion function, we have ${\bar\beta} =1$.

\end{proof}

\begin{lemma}
    Consider a Bernoulli bandit instance with the following mean values $\bp = (p(1), \cdots, p(K))$. For CVaR with parameter $c$, we have 
    \begin{enumerate}
        \item $\beta = 1$ under the assumption for all $i \in [K]$, $p(i) < 1-c$,
        \item ${\bar\beta} = 1$.
    \end{enumerate}
\end{lemma}

\begin{proof}
For a $K$-arm Bernoulli, we have
    \begin{align}
    \label{eq:CVAR_bernouli_beta_equal}
        U_h\bigg(\sum_{i \in [K]} a(i)\F_i \bigg) = h\bigg(\sum_{i \in [K]} a(i)p(i) \bigg) \ .
    \end{align}
We defined $p^{(1)}, p^{(2)}$ and $p^{(3)}$ in \ref{eq:Bern_1_2_3}. The maximum value this PM can take for the Bernoulli bandit model is $1$. 
If $\bp < 1-c$, then, we cannot have $p^{(1)} = 1$. Therefore, for $m,n$ defined in \eqref{eq:m_n_def}, we obtain
\begin{align}
    \delta_{12}(\varepsilon) = h(p^{(1)})-h(p^{(2)})
    = \frac{p^{(1)}}{1-c} - \frac{p^{(2)}}{1-c}
    = \frac{p^{(1)}-p^{(2)}}{1-c}
    = \frac{\varepsilon (p(m)-p(n))}{1-c} \ ,
\end{align}
since we know that for CVaR $q=1$, we can conclude that under the condition that we have $\beta=1$.
Now, we will analyze ${\bar\beta}$; here, we do not require any condition on the means.
We know that $h(p^{(2)}) < 1$, from the definition of second-best discrete solution, as we might have $h(p^{(1)}) = 1$. Therefore,
we have $h(p^{(2)})= \frac{p^{(2)}}{1-c}$ and also, $h(p^{(3)})= \frac{p^{(3)}}{1-c}$.
\begin{align}
\label{eq:p_s_to_epsilon}
    \delta_{13}(\varepsilon ) &> \delta_{23}(\varepsilon ) = h(p^{(2)}) - h(p^{(3)})= \frac{p^{(2)}}{1-c} - \frac{p^{(3)}}{1-c} = \frac{1}{1-c} (p^{(2)}-p^{(3)}) =\Theta(\varepsilon) \ .
\end{align}
Therefore, from \eqref{eq:p_s_to_epsilon}, a upper-bound on ${\bar\beta} \leq 1$ and as $q=1$ for CVaR, from \eqref{eq:lower_on_delta13} as a lower-bound we have ${\bar\beta} \geq 1$. Hence, we can conclude that ${\bar\beta} = 1$ for CVaR.

\end{proof}

\section{PM-ETC-M Algorithm}
\label{Appendix:PM-ETC-M}

In this section, we provide 
the proofs of Theorems \ref{theorem: ETC upper bound} and \ref{theorem:PM-ETC-M}.

\subsection{Regret Decomposition}
\label{proof:ETC upper bound}
To prove Theorem~\ref{theorem: ETC upper bound}, we use the decomposition~\eqref{eq:regret_decomposition}  to provide a regret bound on the arm selection regret \(\bar{\mathfrak{R}}_{\bnu,1}^{\rm E}(T)\). Let us define the {\em set} of discrete optimal mixtures of the PM $U_h$ as
\begin{align}
    {\rm OPT}_{\varepsilon}\;\triangleq\; \argmax\limits_{\ba\in\Delta_{\varepsilon}^{K-1}}\; U_h\left (\sum\limits_{i\in[K]} a(i)\F_i\right)\ ,
\end{align}
and the {\em set} of optimistic mixtures computed from the estimated CDFs at the time instant $t$ as
\begin{align}
    \widehat{\rm OPT}_{\varepsilon,t}\;\triangleq\; \argmax\limits_{\ba \in\Delta_{\varepsilon}^{K-1}} U_h\left ( \sum\limits_{i\in[K]}a(i)\F_{i,t}^{\rm E}\right )\ .
\end{align}
In the regret decomposition we have provided in~\eqref{eq:regret_decomposition}, the estimation regret \(\delta_{01}(\varepsilon)\) is upper bounded in Lemma~\ref{lemma:Delta_error}. Hence, for regret analysis, we focus on the arm selection regret. 
We can decompose the arm selection regret into three main parts as follows.
\begin{align}
\label{eq: Regret_DISCRET_ETC}
    \bar{\mathfrak{R}}_{\bnu, 1}^{\rm E}(T) &= \underbrace{\E_{\bnu}^{\rm E} \left[ \left( U_h\left(\sum_{i\in[K]} a^{(1)}(i)\F_i\right) - U_h\left(\sum_{i\in[K]} a^{\rm E}_{N(\varepsilon)}(i)\F_i\right)\right)\mathds{1}\{\ba^{\rm E}_{N(\varepsilon)} \in{\rm OPT}_{\varepsilon} \}\right]}_{\triangleq A_1(T)}\nonumber \\
    &\qquad + \underbrace{\E_{\bnu}^{\rm E} \left[  \left(U_h\left(\sum_{i\in[K]} a^{(1)}(i)\F_i\right) - U_h\left(\sum_{i\in[K]} a^{\rm E}_{N(\varepsilon)}(i)\F_i\right)\right)\mathds{1}\{\ba^{\rm E}_{N(\varepsilon)} \notin{\rm OPT}_{\varepsilon}\}\right]}_{\triangleq A_2(T)}\nonumber \\
    &\qquad + \underbrace{\E_{\bnu}^{\rm E} \left[ U_h\left(\sum_{i\in[K]} a^{\rm E}_{N(\varepsilon)}(i)\F_i\right)  - U_h\left(\sum_{i\in[K]} \frac{\tau_t^{\rm E}(i)}{T}\F_i\right) \right]}_{\triangleq A_3(T)}.
\end{align}
It can be readily verified that \(A_1(T)=0\). The term $A_2(T)$ captures the {\em mixing coefficient estimation error} when the PM-ETC-M algorithm generates an incorrect mixing coefficient at the end of its exploration phase. Finally, the term $A_3(T)$ captures the {\em sampling estimation error}, i.e., the error in matching the arm selection fractions to the estimated mixing coefficient at the end of the exploration phase. Next, we provide an upper bound for \(A_2(T)\) and \(A_3(T)\).

\subsection{Upper Bound on $A_2(T)$}
\label{appendix:PM-ETC-M mixing error}

Expanding the mixing coefficient estimation error term $A_2(T)$, we obtain
\begin{align}
    A_2(T) &= {\E_{\bnu}^{\rm E} \left[ \left( U_h\left(\sum_{i\in[K]} a^{(1)}(i)\F_i\right) - U_h\left(\sum_{i\in[K]}a^{\rm E}_{N(\varepsilon)}(i)\F_i\right) \right)\mathds{1}\{\ba^{\rm E}_{N(\varepsilon)} \notin{\rm OPT}_{\varepsilon}\}\right]} \\
    &= \E_{\bnu}^{\rm E} \left[ U_h\left(\sum_{i\in[K]} a^{(1)}(i)\F_i\right) - U_h\left(\sum_{i\in[K]} a^{\rm E}_{N(\varepsilon)}(i)\F_i\right)\;\bigg\lvert\;\ba^{\rm E}_{N(\varepsilon)} \notin{\rm OPT}_{\varepsilon}\right]\times \P_{\bnu}^{\rm E}\left(\ba^{\rm E}_{N(\varepsilon)} \notin{\rm OPT}_{\varepsilon}\right)\ .
\end{align}
Furthermore, because the PM is bounded above by $B_h$, we have 
\begin{align}
\label{eq: A2first}
    A_2(T) \leq B_h \cdot \P_{\bnu}^{\rm E}\left(\ba^{\rm E}_{N(\varepsilon)} \notin{\rm OPT}_{\varepsilon}\right)\ .
\end{align}
Next, we bound the probability of forming an incorrect estimate of the mixing coefficients. For this, for any $t$, we define the following events.
\begin{align}
    \mcE_{1,t}(x)\;&\triangleq\; \Bigg\{ \bigg\lvert U_h\bigg(\sum_{i\in[K]} a^{(1)}(i)\F_{i,t}^{\rm E}\bigg) - U_h\bigg(\sum_{i\in[K]} a^{(1)} (i)\F_i\bigg)\bigg \rvert\leq x\Bigg\}\ ,\\
    \mcE_{2,t}(x)\;&\triangleq\; \Bigg\{ \bigg\lvert U_h\bigg(\sum_{i\in[K]} a^{\rm E}_{N(\varepsilon)}(i)\F_{i,t}^{\rm E}\bigg) - U_h\bigg(\sum_{i\in[K]} a^{\rm E}_{N(\varepsilon)}(i)\F_i\bigg)\bigg \rvert\leq x\Bigg\}\ ,\\
    \text{and,}\quad\mcE_t(x)\;&\triangleq\;\mcE_1(x,t)\bigcap\mcE_2(x,t)\ .
    \label{eq:E_event}
\end{align}
Note that for any $\ba^{(1)}\in{\rm OPT}_{\varepsilon}$, we have
\begin{align}
    & \P_{\bnu}^{\rm E}\Big(\ba^{\rm E}_{N(\varepsilon)} \notin{\rm OPT}_{\varepsilon}\Big)
    \\ &\leq\;\P_{\bnu}^{\rm E}\left ( U_h\bigg(\sum_{i\in[K]} a^{(1)} (i)\F_i\bigg) \geq U_h\bigg(\sum_{i\in[K]} a^{\rm E}_{N(\varepsilon)}(i)\F_i\bigg) + \delta_{12}(\varepsilon)\right)\\
    \label{eq:ETC_A21}
    &=\;\P_{\bnu}^{\rm E}\left ( U_h\bigg(\sum_{i\in[K]} a^{(1)}(i)\F_i\bigg) \geq U_h\bigg(\sum_{i\in[K]} a^{\rm E}_{N(\varepsilon)}(i)\F_i\bigg) + \delta_{12}(\varepsilon)\;\bigg\lvert\;\mcE_{N(\varepsilon)}\bigg(\frac{1}{2}\delta_{12}(\varepsilon)\bigg)\right)\nonumber\\
    &\qquad\qquad\qquad\times\P_{\bnu}^{\rm E} \left(\mcE_{N(\varepsilon)}\bigg(\frac{1}{2}\delta_{12}(\varepsilon)\bigg)\right)\\
    &\;\;+\P_{\bnu}^{\rm E} \left ( U_h\bigg(\sum_{i\in[K]} a^{(1)}(i)\F_i\bigg) \geq U_h\bigg(\sum_{i\in[K]} a^{\rm E}_{N(\varepsilon)}(i)\F_i\bigg) + \delta_{12}(\varepsilon)\;\bigg\lvert\; \overline{\mcE}_{N(\varepsilon)}\bigg(\frac{1}{2}\delta_{12}(\varepsilon)\bigg)\right)\nonumber\\
    &\qquad\qquad\qquad\times\P_{\bnu}^{\rm E} \bigg( \overline{\mcE}_{N(\varepsilon)}\bigg(\frac{1}{2}\delta_{12}(\varepsilon)\bigg)\bigg)\ .
    \label{eq:ETC_A22}
\end{align}
Note that the first term, i.e., ~\eqref{eq:ETC_A21} is upper bounded by
\begin{align}
    &\P_{\bnu}^{\rm E} \left ( U_h\bigg(\sum_{i\in[K]} a^{(1)}(i)\F_i\bigg) \geq U_h\bigg(\sum_{i\in[K]} a^{\rm E}_{N(\varepsilon)}(i)\F_i\bigg) + \delta_{12}(\varepsilon)\;\bigg\lvert\; \mcE_{N(\varepsilon)}\bigg(\frac{1}{2}\delta_{12}(\varepsilon)\bigg)\right)\\
    \label{eq:ETC_1}
    &\leq \P_{\bnu}^{\rm E} \Bigg ( U_h\bigg(\sum_{i\in[K]} a^{(1)}(i)\F_{i,N(\varepsilon)}^{\rm E}\bigg) +  \frac{1}{2}\delta_{12}(\varepsilon)\geq U_h\bigg(\sum_{i\in[K]} a^{\rm E}_{N(\varepsilon)}(i)\F_{i,N(\varepsilon)}^{\rm E}\bigg)\nonumber\\
    &\qquad\qquad + \frac{1}{2}\delta_{12}(\varepsilon)\;\bigg\lvert\;\mcE_{N(\varepsilon)}\bigg(\frac{1}{2}\delta_{12}(\varepsilon)\bigg)\Bigg)\\
    &=0\ ,
    \label{eq:ETC_2}
\end{align}
where ~\eqref{eq:ETC_1} holds due to the conditioning on the event $\mcE_{N(\varepsilon)}(\frac{1}{2}\delta_{12}(\varepsilon))$, and ~\eqref{eq:ETC_2} holds since $\ba^{\rm E}_{N(\varepsilon)}\in\widehat{\rm OPT}_{\varepsilon, N}$. The second term, i.e., the term in~\eqref{eq:ETC_A22} is upper bounded by
\begin{align}
    \P_{\bnu}^{\rm E} \bigg( \overline{\mcE}_{N(\varepsilon)}\Big(\frac{1}{2}\delta_{12}(\varepsilon)\Big )\bigg)\;\leq\;\P_{\bnu}^\pi\bigg( \overline{\mcE}_{1,N}\Big (\frac{1}{2}\delta_{12}(\varepsilon)\Big)\bigg) + \P_{\bnu}^{\rm E} \bigg( \overline{\mcE}_{2,N}\Big (\frac{1}{2}\delta_{12}(\varepsilon)\Big)\bigg)\ .
\end{align}
Expanding each term, we obtain
\begin{align}
    \label{eq:ETC_bern_cue_start}
    \P_{\bnu}^{\rm E} \bigg( \overline{\mcE}_{1,N}\Big(\frac{1}{2}\delta_{12}(\varepsilon)\Big)\bigg)&= \P\left( \left\lvert U_h\bigg(\sum_{i\in[K]} a^{(1)}(i)\F_{i, N(\varepsilon)}^{\rm E}\bigg) - U_h\bigg(\sum_{i\in[K]} a^{(1)}(i)\F_i\bigg)\right\rvert > \frac{1}{2}\delta_{12}(\varepsilon)\right )\\
    &\stackrel{\eqref{eq:Holder}} {\leq} \P_{\bnu}^{\rm E} \left(\mcL \sum_{i\in[K]} \norm{a^{(1)} (i)^\star(\F_{i,N(\varepsilon)}^{\rm E} - \F_i)}_{\rm W}^q > \frac{1}{2}\delta_{12}(\varepsilon) \right)\\
    & {\leq}  \sum_{i\in[K]}  \P_{\bnu}^{\rm E} \left(\mcL\underbrace{\big( a^{(1)}(i)\big)^q}_{\leq 1}\norm{\F_{i,N(\varepsilon)}^{\rm E} - \F_i }_{\rm W}^q > \frac{1}{2K}\delta_{12}(\varepsilon) \right)\\
    \label{eq:alpha_to_1}
    &\leq \sum_{i\in[K]} \P_{\bnu}^{\rm E} \left ( \norm{\F_{i,N(\varepsilon)}^{\rm E}-\F_i}_{\rm W} > \left( \frac{1}{2K\mcL}\delta_{12}(\varepsilon)\right)^{\frac{1}{q}}  \right )\\
    &\stackrel{\eqref{eq:meta_concentration}}{\leq} \sum_{i\in[K]} 2\exp\left ( -\frac{\tau^{\rm E} _{N(\varepsilon)}(i)}{256e}\Bigg(\left(\frac{1}{2K\mcL( a^{(1)}(i))^q}\delta_{12}(\varepsilon)\right)^{1/q} - \frac{512}{\sqrt{\tau^{\rm E} _{N(\varepsilon)}}}\Bigg)^2 \right) \\
     &{=} \sum_{i\in[K]} 2\exp\left ( -\frac{\frac{N(\varepsilon)}{K}}{256e}\Bigg(\left(\frac{1}{2K\mcL( a^{(1)}(i))^q}\delta_{12}(\varepsilon)\right)^{1/q} - \frac{512}{\sqrt{\frac{N(\varepsilon)}{K}}}\Bigg)^2 \right) \\
    &= 2K \exp\left ( -\frac{\frac{N(\varepsilon)}{K}}{256e}\Bigg(\left(\frac{1}{2K\mcL( a^{(1)}(i))^q}\delta_{12}(\varepsilon)\right)^{1/q} - \frac{512}{\sqrt{\frac{N(\varepsilon)}{K}}}\Bigg)^2 \right) \ .
    \label{eq:E1C}
\end{align}
Furthermore, we have
\begin{align}
&\P_{\bnu}^{\rm E} \bigg( \overline{\mcE}_{2,N(\varepsilon)}\Big(\frac{1}{2}\delta_{12}(\varepsilon)\Big)\bigg)\nonumber\\
&\quad= \P_{\bnu}^{\rm E} \left( \bigg\lvert U_h\bigg(\sum_{i\in[K]}a^{\rm E}_{N(\varepsilon)}(i)\F_{i,N(\varepsilon)}^{\rm E}\bigg) - U_h\bigg(\sum_{i\in[K]}a^{\rm E}_{N(\varepsilon)}(i)\F_i \bigg)\bigg\rvert > \frac{1}{2}\delta_{12}(\varepsilon)\right )\\
    &\quad\leq\P_{\bnu}^{\rm E} \left(\bigcup_{\ba^{\rm E}_{N(\varepsilon)}\in\widehat{\rm OPT}_{\varepsilon, N(\varepsilon)}}  \bigg\lvert U_h\bigg(\sum_{i\in[K]} a^{\rm E}_{N(\varepsilon)}(i)\F_{i,N(\varepsilon)}^{\rm E}\bigg) - U_h\bigg(\sum_{i\in[K]} a^{\rm E}_{N(\varepsilon)}(i)\F_i\bigg) \bigg \rvert > \frac{1}{2}\delta_{12}(\varepsilon)\right)\\
    &\quad\leq\sum_{\ba^{\rm E}_{N(\varepsilon)}\in\widehat{\rm OPT}_{\varepsilon, N(\varepsilon)}}\P_{\bnu}^{\rm E} \left(\bigg\lvert U_h\bigg(\sum_{i\in[K]} a^{\rm E}_{N(\varepsilon)}(i)\F_{i,N(\varepsilon)}^{\rm E}\bigg) - U_h\bigg(\sum_{i\in[K]} a^{\rm E}_{N(\varepsilon)}(i)\F_i \bigg) \bigg \rvert > \frac{1}{2}\delta_{12}(\varepsilon)\right)\\
    &\quad\stackrel{\eqref{eq:Holder}}{\leq} \sum_{\ba^{\rm E}_{N(\varepsilon)}\in\widehat{\rm OPT}_{\varepsilon, N(\varepsilon)}} \sum_{i\in[K]} \P_{\bnu}^{\rm E} \left ( \norm{\F_{i,N(\varepsilon)}^{\rm E}-\F_i}_{\rm W} > \left( \frac{1}{2K\mcL\underbrace{\big(a^{\rm E}_{N(\varepsilon)}(i)\big)^q}_{\leq 1}}\delta_{12}(\varepsilon)\right)^{\frac{1}{q}}  \right )\\
    &\quad {\leq} \sum_{\ba^{\rm E}_{N(\varepsilon)}\in\widehat{\rm OPT}_{\varepsilon, N(\varepsilon)}} \sum_{i\in[K]} \P_{\bnu}^{\rm E} \left ( \norm{\F_{i,N(\varepsilon)}^{\rm E}-\F_i}_{\rm W} > \left( \frac{1}{2K\mcL}\delta_{12}(\varepsilon)\right)^{\frac{1}{q}}  \right )\\
    &\quad\stackrel{\eqref{eq:meta_concentration}}{\leq}\sum_{\ba^{\rm E}_{N(\varepsilon)}\in\widehat{\rm OPT}_{\varepsilon, N(\varepsilon)}} \sum_{i\in[K]} \exp\left ( -\frac{\tau^{\rm E} _{N(\varepsilon)}(i)}{256e}\Bigg(\left(\frac{1}{2K\mcL( a^{(1)}(i))^q}\delta_{12}(\varepsilon)\right)^{1/q} - \frac{512}{\sqrt{\tau^{\rm E} _{N(\varepsilon)}}}\Bigg)^2 \right)\\
    & \quad\leq \sum_{\ba^{\rm E}_{N(\varepsilon)}\in\widehat{\rm OPT}_{\varepsilon, N(\varepsilon)}} 2K \exp\left ( -\frac{\frac{N(\varepsilon)}{K}}{256e}\Bigg(\left(\frac{1}{2K\mcL( a^{(1)}(i))^q}\delta_{12}(\varepsilon)\right)^{1/q} - \frac{512}{\sqrt{\frac{N(\varepsilon)}{K}}}\Bigg)^2 \right) \\
    &\quad\leq 2K\left(\frac{1}{\varepsilon} \right)^{K-1} \exp\left ( -\frac{\frac{N(\varepsilon)}{K}}{256e}\Bigg(\left(\frac{1}{2K\mcL( a^{(1)}(i))^q}\delta_{12}(\varepsilon)\right)^{1/q} - \frac{512}{\sqrt{\frac{N(\varepsilon)}{K}}}\Bigg)^2 \right), 
    \label{eq:E2C}
\end{align}
where~\eqref{eq:E2C} follows from the fact that the total number of discrete $\ba$ values may not exceed \(\left(\frac{1}{\varepsilon} \right)^{K-1}\). Combining~\eqref{eq:E1C} and~\eqref{eq:E2C}, we obtain
\begin{align}
    &\P_{\bnu}^{\rm E} \bigg( \overline{\mcE}_{1,N(\varepsilon)}\Big(\frac{1}{2}\delta_{12}(\varepsilon)\Big)\bigg) + \P_{\bnu}^{\rm E} \bigg( \overline{\mcE}_{2,N(\varepsilon)}\Big(\frac{1}{2}\delta_{12}(\varepsilon)\Big)\bigg) \nonumber \\
    & \quad\leq 2K \exp\left ( -\frac{\frac{N(\varepsilon)}{K}}{256e}\Bigg(\left(\frac{1}{2K\mcL( a^{(1)}(i))^q}\delta_{12}(\varepsilon)\right)^{1/q} - \frac{512}{\sqrt{\frac{N(\varepsilon)}{K}}}\Bigg)^2 \right) \nonumber \\
    & \qquad+ \left(\frac{1}{\varepsilon} \right)^{K-1} 2K \exp\left ( -\frac{\frac{N(\varepsilon)}{K}}{256e}\Bigg(\left(\frac{1}{2K\mcL( a^{(1)}(i))^q}\delta_{12}(\varepsilon)\right)^{1/q} - \frac{512}{\sqrt{\frac{N(\varepsilon)}{K}}}\Bigg)^2 \right)  \\
    &\quad= 2K \exp\left ( -\frac{\frac{N(\varepsilon)}{K}}{256e}\Bigg(\left(\frac{1}{2K\mcL( a^{(1)}(i))^q}\delta_{12}(\varepsilon)\right)^{1/q} - \frac{512}{\sqrt{\frac{N(\varepsilon)}{K}}}\Bigg)^2 \right) \cdot \left( \left(\frac{1}{\varepsilon} \right)^{K-1} +1 \right)\ .
    \label{eq:E3C}
\end{align}
At this step, for some $\delta\in(0,1)$, choosing \( N(\varepsilon) \) as
\begin{align}
\label{eq:number of samples}
    N(\varepsilon) \triangleq 256K  \e \left(\frac{2K \mcL}{\delta_{12}(\varepsilon)}\right)^{\frac{2}{q}} \bigg[\frac{32}{\sqrt{\rm e}} + \log^{\frac{1}{2}} \Big( 2K\delta   \big(\varepsilon^{-(K-1)} + 1\big) \Big) \bigg]^2\ , 
\end{align}
 and leveraging~\eqref{eq:E3C}, it can be readily verified that 
\begin{align}
    \P_{\bnu}^{\rm E} \bigg( \overline{\mcE}_{1,N(\varepsilon)}\Big(\frac{1}{2}\delta_{12}(\varepsilon)\Big)\bigg) + \P_{\bnu}^{\rm E} \bigg( \overline{\mcE}_{2,N(\varepsilon)}\Big(\frac{1}{2}\delta_{12}(\varepsilon)\Big)\bigg)\;\leq\;\delta\ ,
\end{align}
which implies that $\P_{\bnu}^{\rm E}\left(\ba^{\rm E}_{N(\varepsilon)}\notin{\rm OPT}_{\varepsilon}\right)  \leq \delta$. Hence, we have
\begin{align}
\label{eq: ETC A2T final}
    A_2(T) \stackrel{\eqref{eq: A2first}}{\leq} B_h \cdot \delta \ ,
\end{align}
where $B_h \in \R^{+}$ denotes the upper bound PM can take.

\subsection{Upper Bound on the PM-ETC-M Sampling Estimation Error $A_3(T)$}
\label{sec:ETC_sampling_error}

Next, we analyze \(A_3(T)\), which captures the sampling estimation error. We have
\begin{align}
    A_3(T) &= \E_{\bnu}^{\rm E} \left[ U_h\left(\sum_{i\in[K]} a^{\rm E}_{N(\varepsilon)}(i)\F_i\right)  - U_h\left(\sum_{i\in[K]} \frac{\tau_T^{\rm E}(i)}{T}\F_i\right) \right] \\ 
    \label{eq:lemma_8_2}
&\leq  \mcL\E_{\bnu}^{\rm E} \left[ \norm{\ba^{\rm E}_{N(\varepsilon)} - \frac{\btau_T^{\rm E}}{T} }_1^q W^q \right] \\
&\leq \mcL K W^q  \E_{\bnu}^{\rm E} \bigg[ \max_{i \in [K]} \left|a^{\rm E}_{N(\varepsilon)}(i) - \frac{\tau_t^{\rm E}(i)}{T}\right|^q  \bigg] ,
\end{align}
where~\eqref{eq:lemma_8_2} follows from the \holder~continuity of the utility stated in~\eqref{eq:Holder} and recalling the definition of $W$ stated in~\eqref{eq:W}. We need to find an upper bound on the estimation error due to the policy's sampling proportions. Note that as a result of the ``commitment'' process of the ETC algorithm stated in~\eqref{eq: ETC_samp}, for the first $K-1$ arms, the PM-ETC-M commits its arms selection such that it is pulled \( \max\{0,\lfloor Ta_{N(\varepsilon)}^{\rm E}(i)\rfloor - N(\varepsilon)/K\}\) times post exploration, and it allocates the remaining rounds to arm $K$. Let us define the set
\begin{align}
\label{eq:ETC set}
    \mcS^\prime\;\triangleq\; \Big\{i\in[K] : a^{\rm E}_{N(\varepsilon)}(i)< N(\varepsilon)/KT\Big\}\ .
\end{align}
Accordingly, for any arm $i\notin\mcS^\prime$, if $i\neq K$, we have the following bound.
\begin{align}
    \bigg \lvert \frac{\tau_t^{\rm E}(i)}{T} - a^{\rm E}_{N(\varepsilon)}(i)\bigg\rvert\;&\stackrel{\eqref{eq:ETC set}}{=}\; a^{\rm E}_{N(\varepsilon)}(i) - \frac{\tau_t^{\rm E}(i)}{T}\\
    &\stackrel{\eqref{eq: ETC_samp}}{=}\; a^{\rm E}_{N(\varepsilon)}(i) - \frac{\lfloor Ta^{\rm E}_{N(\varepsilon)}(i)\rfloor}{T}\\
    &\leq\; a^{\rm E}_{N(\varepsilon)}(i) - \frac{Ta^{\rm E}_{N(\varepsilon)}(i) - 1}{T}\\
    &<\;\frac{1}{T}\ .
    \label{eq: ETC commit 3}
\end{align}
Alternatively, if $K\notin\mcS^\prime$, we have
\begin{align}
    \bigg \lvert \frac{\tau_T^{\rm E}(K)}{T} - a^{\rm E}_{N(\varepsilon)}(K)\bigg\rvert\;&\stackrel{\eqref{eq:ETC set}}{=}\; a^{\rm E}_{N(\varepsilon)}(K) - \frac{\tau_T^{\rm E}(K)}{T}\\
    &\stackrel{\eqref{eq: ETC_samp}}{=}\; a^{\rm E}_{N(\varepsilon)}(K) - \frac{T - \sum\limits_{i\neq K}\tau_t^{\rm E}(i)}{T}\\
    &=\; a^{\rm E}_{N(\varepsilon)}(K) - 1 - \frac{1}{T} \left (\sum\limits_{i\in\mcS^\prime:i\neq K} \tau_t^{\rm E}(i) + \sum\limits_{i\notin\mcS^\prime: i \neq K} \tau_t^{\rm E}(i)\right)\\
    \label{eq: ETC commit 1}
    &=\; a^{\rm E}_{N(\varepsilon)}(K) - 1 - \frac{1}{T}\left( \frac{N(\varepsilon)}{K}|\mcS^\prime| + \sum\limits_{i\notin\mcS^\prime:i\neq K} \lfloor Ta^{\rm E}_{N(\varepsilon)}(i)\rfloor\right)\\
    &\leq\;a^{\rm E}_{N(\varepsilon)}(K) - 1 - \frac{1}{T}\left( \frac{N(\varepsilon)}{K}|\mcS^\prime| + \sum\limits_{i\notin\mcS^\prime:i\neq K}  Ta^{\rm E}_{N(\varepsilon)}(i)\right)\\
    &=\;\sum\limits_{i\notin\mcS^\prime} a^{\rm E}_{N(\varepsilon)}(i) - 1 + \frac{N(\varepsilon)|\mcS^\prime|}{KT}\\
    &= \;\frac{N(\varepsilon)|\mcS^\prime|}{KT} - \sum\limits_{i\in\mcS^\prime}a^{\rm E}_{N(\varepsilon)}(i) \\
    \label{eq: ETC commit 2}
    &\leq\; \sum\limits_{i\in\mcS^\prime}\frac{N(\varepsilon)}{KT} + \frac{N(\varepsilon)|\mcS^\prime|}{KT}\\
    &= \frac{2N(\varepsilon)|\mcS^\prime|}{KT}\\
    &\leq\; \frac{2N(\varepsilon)}{T}\ ,
    \label{eq: ETC commit 4}
\end{align}
where \eqref{eq: ETC commit 1} follows from the fact that for every $i\in\mcS^\prime$, since these arms have already been over-explored, they are not going to get sampled further by the PM-ETC-M algorithm, and~\eqref{eq: ETC commit 2} follows from the definition of the set $\mcS$, which dictates that for every arm $i\in\mcS^\prime$, we have \(a^{\rm E}_{N(\varepsilon)}(i)< N(\varepsilon)/KT\). Finally, for every $i\in\mcS^\prime$, we have
\begin{align}
     \bigg \lvert \frac{\tau_T^{\rm E}(i)}{T} - a^{\rm E}_{N(\varepsilon)}(i)\bigg\rvert\;&\stackrel{\eqref{eq:ETC set},\eqref{eq: ETC_samp}}{=}\; \bigg\lvert \frac{N(\varepsilon)}{KT} - a^{\rm E}_{N(\varepsilon)}(i)\bigg\rvert\\
     &\stackrel{\eqref{eq:ETC set}}{=}\; \frac{N(\varepsilon)}{KT} - a^{\rm E}_{N(\varepsilon)}(i)\\
     &\;< \frac{N(\varepsilon)}{KT}\ .
     \label{eq: ETC commit 5}
\end{align}
Hence, combining~\eqref{eq: ETC commit 3},~\eqref{eq: ETC commit 4}, and~\eqref{eq: ETC commit 5}, we conclude that for any $i\in[K]$, we have
\begin{align}
     \bigg \lvert \frac{\tau_T^{\rm E}(i)}{T} - a^{\rm E}_{N(\varepsilon)}(i)\bigg\rvert\;\leq\; \frac{2N(\varepsilon)}{T}\ .
     \label{eq: ETC commit 6}
\end{align}
Leveraging~\eqref{eq: ETC commit 6}, we can now upper bounded \(A_3(T)\) as follows.
\begin{align}
    A_3(T) &=\; \E_{\bnu}^{\rm E} \left[ U_h\left(\sum_{i\in[K]} a^{\rm E}_{N(\varepsilon)}(i)\F_i\right)  - U_h\left(\sum_{i\in[K]} \frac{\tau_t^{\rm E}(i)}{T}\F_i\right) \right] \\ 
&\leq\; \mcL KW^q\E_{\bnu}^{\rm E} \left[ \max_{i \in [K]} \left|a^{\rm E}_{N(\varepsilon)}(i) - \frac{\tau_t^{\rm E}(i)}{T}\right|^q  \right]\\
&\stackrel{\eqref{eq: ETC commit 6}}{\leq}\; \mcL KW^q \left(\frac{2N(\varepsilon)}{T}\right)^q \ .
\label{eq: ETC A3T final}
\end{align}

\subsection{Proof of Theorem \ref{theorem: ETC upper bound}}
\label{Appendix:ETC_theorem_proof}
We can upper bound the discretized regret leveraging the upper bounds on $A_1(T)$, $A_2(T)$ in~\eqref{eq: ETC A2T final}, and $A_3(T)$ in~\eqref{eq: ETC A3T final} and considering \(\delta = 1/T^2 \). Specifically,
\begin{align}
    \Bar{\mathfrak{R}}_{\bnu,1}^{\rm E}(T) &\stackrel{\eqref{eq: Regret_DISCRET_ETC}}{=}  A_1(T) + A_2(T) + A_3(T) \\
    & \leq  \frac{B_h}{T^2} + \mcL K \left(3W\frac{N(\varepsilon)}{T}\right)^q \\
    & \leq  (\mcL K + W^{-q}) \left(3W\frac{N(\varepsilon)}{T}\right)^q \ ,
    \label{eq:UP_ETC_eqn1}
\end{align}
where~\eqref{eq:UP_ETC_eqn1} follows from the fact that $\frac{B_h}{T^2}$ is upper bounded by $(\frac{N(\varepsilon)}{T})^q$. Finally, combining all the terms $A_1(T)$, $A_2(T)$ in~\eqref{eq: ETC A2T final}, $A_3(T)$ in~\eqref{eq: ETC A3T final}, and the estimation regret $\delta_{01}(\varepsilon)$ in~\eqref{lemma:Delta_error}, we obtain
\begin{align}
\label{eq:disc_up_ETC}
    \mathfrak{R}_{\bnu}^{\rm E}(T) &\stackrel{\eqref{eq:regret_decomposition}}{=} \delta_{01}(\varepsilon) + \Bar{\mathfrak{R}}_{\bnu,1}^{\rm E}(T) \\
    & \leq \delta_{01} + (\mcL K +W^{-q}) \left(\frac{2N(\varepsilon)}{T}\right)^q  \\
    \label{eq:ETC_M_eps}
    & = \delta_{01} +  (\mcL K +W^{-q}) \Bigg( 3WM(\varepsilon)\frac{\log T}{T} \Bigg)^q\ ,
\end{align}
where~\eqref{eq:ETC_M_eps} we define \(M(\varepsilon) \triangleq \frac{N(\varepsilon)}{\log T}\).

\subsection{Proof of Theorem~\ref{theorem:PM-ETC-M}}
\label{Appendix:ETC_theorem_epsilon_proof}
Let us set $\varepsilon = \Theta((K^{2+\frac{2}{q}}\frac{\log T}{T^{\gamma}})^{\frac{q}{2\beta}})$ for \(\gamma \in (0,1)\). Leveraging \(\delta_{12}(\varepsilon) = \Omega(\varepsilon^\beta)\ \), from~\eqref{eq:suboptimality_gaps}, the exploration horizon is $N(\varepsilon)=\Theta(T^\gamma)$. Hence, for arm selection regret, from~\eqref{eq:ETC_M_eps} we have
\begin{align}
\label{eq:Discrete_order}
    \Bar{\mathfrak{R}}_{\bnu,1}^{\rm E}(T) = O(KT^{(\gamma-1)q})\ .
\end{align}
Also, from Lemma~\eqref{lemma:Delta_error}, we have
\begin{align}
    \delta_{01}(\varepsilon) &\leq 2\mcL (KW\varepsilon)^q  \ .
\end{align}
As a result, \(\delta_{01}(\varepsilon) = O(\varepsilon^q)\). Hence, by our choice of $\varepsilon$, we obtain 
\begin{align}
\label{eq:delta_order}
    \delta_{02}(\varepsilon) = O\left(K^{q + \frac{q(q+1)}{\beta}}T^{-\frac{q^2\gamma}{2\beta}}(\log T)^{\frac{q^2}{2\beta}}\right).
\end{align}
From~\eqref{eq:disc_up_ETC},~\eqref{eq:Discrete_order}, and~\eqref{eq:delta_order} the regret of the PM-ETC-M algorithm is upper bounded by
    \begin{align}
    \label{eq:ETC upper bound gamma appendix}
        \mathfrak{R}_{\bnu}^{\rm E}(T) = O \left (\max\Big\{ K^{q + \frac{q(q+1)}{\beta}}T^{-\frac{rq\gamma}{2\beta}}(\log T)^{\frac{rq}{2\beta}}\;,\;K^qT^{(\gamma-1)q}\Big\} \right)\ .
    \end{align}
Finally, setting $\gamma = \frac{2\beta}{2\beta + q}$, the regret of the PM-ETC-M algorithm is upper bounded by 
\begin{align}
     \mathfrak{R}_{\bnu}^{{\textnormal{\rm E}}}(T)\ & \leq O\Big(\Big[K^{2(q+\beta+1)}\; T^{-\frac{q}{1+q/2\beta}}\; (\log T)^{q}\Big]^{\frac{q}{2\beta}}\Big)\ .
\end{align}

\subsection{Proof of Theorem~\ref{theorem:PM-ETC-M_beta13}}
\label{Appendix:ETC_theorem_epsilon_proof_13}
In this section, we prove Theorem~\ref{theorem:PM-ETC-M_beta13}. We leverage the regret decomposition in~\eqref{eq:regret_decomposition} to provide a bound on the arm selection regret \(\bar{\mathfrak{R}}_{\bnu, 2}^{\rm E}(T)\). The analysis closely follows PM-ETC-M, and we highlight the important distinctions. Let us define the sets of discrete optimal mixtures and second-best mixtures as
\begin{align}
    {\rm OPT}_{1, \varepsilon}\;\triangleq\; \argmax\limits_{\ba\in\Delta_{\varepsilon}^{K-1}}\; U_h\left (\sum\limits_{i\in[K]} a(i)\F_i\right)\ , \quad\text{and}\quad {\rm OPT}_{2, \varepsilon}\;\triangleq\; \argmax\limits_{\ba \neq \ba^{(1)},\: \ba\in\Delta_{\varepsilon}^{K-1}}\; U_h\left (\sum\limits_{i\in[K]} a(i)\F_i\right)\ .
\end{align}
Furthermore, let us define the combined set of discrete optimal and second-best mixing coefficients as
\begin{align}
   {\rm OPT}_{\varepsilon} \triangleq {\rm OPT}_{1, \varepsilon} \cup{\rm OPT}_{2, \varepsilon}\ .
\end{align}
and the {\em set} of optimistic mixtures computed from the estimated CDFs at the time instant $t$ as
\begin{align}
    \widehat{\rm OPT}_{\varepsilon,t}\;\triangleq\; \argmax\limits_{\ba \in\Delta_{\varepsilon}^{K-1}} U_h\left ( \sum\limits_{i\in[K]}a(i)\F_{i,t}^{\rm E}\right )\ .
\end{align}
In the regret decomposition we have provided in~\eqref{eq:regret_decomposition}, the estimation regret \(\delta_{02}(\varepsilon)\) is upper bounded in Lemma~\ref{lemma:Delta_error}. Hence, for regret analysis, we focus only on the arm selection regret.  We can decompose the arm selection regret into three main parts as follows.
\begin{align}
\label{eq: Regret_DISCRET_ETC_bern}
    \bar{\mathfrak{R}}_{\bnu,2}^{\rm E}(T) &= \underbrace{\E_{\bnu}^{\rm E} \left[ \left( U_h\left(\sum_{i\in[K]} a^{(2)}(i)\F_i\right) - U_h\left(\sum_{i\in[K]} a^{\rm E}_{N(\varepsilon)}(i)\F_i\right)\right)\mathds{1}\{\ba^{\rm E}_{N(\varepsilon)} \in{\rm OPT}_{\varepsilon} \}\right]}_{\triangleq A_1(T)}\nonumber \\
    &\qquad + \underbrace{\E_{\bnu}^{\rm E} \left[  \left(U_h\left(\sum_{i\in[K]} a^{(2)}(i)\F_i\right) - U_h\left(\sum_{i\in[K]} a^{\rm E}_{N(\varepsilon)}(i)\F_i\right)\right)\mathds{1}\{\ba^{\rm E}_{N(\varepsilon)} \notin{\rm OPT}_{\varepsilon}\}\right]}_{\triangleq A_2(T)}\nonumber \\
    &\quad\qquad + \underbrace{\E_{\bnu}^{\rm E} \left[ U_h\left(\sum_{i\in[K]} a^{\rm E}_{N(\varepsilon)}(i)\F_i\right)  - U_h\left(\sum_{i\in[K]} \frac{\tau_t^{\rm E}(i)}{T}\F_i\right) \right]}_{\triangleq A_3(T)}.
\end{align}
It can be readily verified that \(A_1(T)\leq 0\). The term $A_2(T)$ captures the {\em mixing coefficient estimation error} when the PM-ETC-M algorithm generates an incorrect mixing coefficient at the end of its exploration phase. Finally, the term $A_3(T)$ captures the {\em sampling estimation error}, i.e., the error in matching the arm selection fractions to the estimated mixing coefficient at the end of the exploration phase. An upper bound on $A_3(T)$ readily follows from Section~\ref{sec:ETC_sampling_error}. Hence, we provide an upper bound for \(A_2(T)\). Expanding $A_2(T)$, we obtain
\begin{align}
    A_2(T) &= {\E_{\bnu}^{\rm E} \left[ \left( U_h\left(\sum_{i\in[K]} a^{(2)}(i)\F_i\right) - U_h\left(\sum_{i\in[K]}a^{\rm E}_{N(\varepsilon)}(i)\F_i\right) \right)\mathds{1}\{\ba^{\rm E}_{N(\varepsilon)} \notin{\rm OPT}_{\varepsilon}\}\right]} \\
    &= \E_{\bnu}^{\rm E} \left[ U_h\left(\sum_{i\in[K]} a^{(2)}(i)\F_i\right) - U_h\left(\sum_{i\in[K]} a^{\rm E}_{N(\varepsilon)}(i)\F_i\right)\;\bigg\lvert\;\ba^{\rm E}_{N(\varepsilon)} \notin{\rm OPT}_{\varepsilon}\right] \nonumber \\ & \qquad \qquad \times \P_{\bnu}^{\rm E}\left(\ba^{\rm E}_{N(\varepsilon)} \notin{\rm OPT}_{\varepsilon}\right)\ .
\end{align}
Furthermore, owing to the fact that the PM is bounded above by $B_h$, we have 
\begin{align}
\label{eq: A2first beta 13}
    A_2(T) \leq B_h \cdot \P_{\bnu}^{\rm E}\left(\ba^{\rm E}_{N(\varepsilon)} \notin{\rm OPT}_{\varepsilon}\right)\ .
\end{align}
Next, we bound the probability of forming an incorrect estimate of the mixing coefficients. Let us recall the event $\mcE_t(x)$ for any $t\in\N$ and $x\in\R_+$, as defined in~\eqref{eq:E_event}.
Note that for any $\ba^{(1)}\in{\rm OPT}_{1,\varepsilon}$ and $\ba^{(2)}\in{\rm OPT}_{2,\varepsilon}$, we have
\begin{align}
    & \P_{\bnu}^{\rm E}\Big(\ba^{\rm E}_{N(\varepsilon)} \notin{\rm OPT}_{\varepsilon}\Big)
    \\ &=\;\P_{\bnu}^{\rm E}\Big ( U_h\bigg(\sum_{i\in[K]} a^{(1)} (i)\F_i\bigg) \geq U_h\bigg(\sum_{i\in[K]} a^{\rm E}_{N(\varepsilon)}(i)\F_i\bigg) + \delta_{13}(\varepsilon),   \\ & \quad \quad \quad \quad U_h\bigg(\sum_{i\in[K]} a^{(2)} (i)\F_i\bigg) \geq U_h\bigg(\sum_{i\in[K]} a^{\rm E}_{N(\varepsilon)}(i)\F_i\bigg) + \delta_{23}(\varepsilon) \Big)\\
    &\leq\;\P_{\bnu}^{\rm E}\left ( U_h\bigg(\sum_{i\in[K]} a^{(1)} (i)\F_i\bigg) \geq U_h\bigg(\sum_{i\in[K]} a^{\rm E}_{N(\varepsilon)}(i)\F_i\bigg) + \delta_{13}(\varepsilon) \right)\\
    \label{eq:ETC_A21_bern}
    &=\;\P_{\bnu}^{\rm E}\left ( U_h\bigg(\sum_{i\in[K]} a^{(1)}(i)\F_i\bigg) \geq U_h\bigg(\sum_{i\in[K]} a^{\rm E}_{N(\varepsilon)}(i)\F_i\bigg) + \delta_{13}(\varepsilon)\;\bigg\lvert\;\mcE_{N(\varepsilon)}\bigg(\frac{1}{2}\delta_{13}(\varepsilon)\bigg)\right)\nonumber\\
    &\qquad\qquad\qquad\times\P_{\bnu}^{\rm E} \left(\mcE_{N(\varepsilon)}\bigg(\frac{1}{2}\delta_{13}(\varepsilon)\bigg)\right)\\
    &\;\;+\P_{\bnu}^{\rm E} \left ( U_h\bigg(\sum_{i\in[K]} a^{(1)}(i)\F_i\bigg) \geq U_h\bigg(\sum_{i\in[K]} a^{\rm E}_{N(\varepsilon)}(i)\F_i\bigg) + \delta_{13}(\varepsilon)\;\bigg\lvert\; \overline{\mcE}_{N(\varepsilon)}\bigg(\frac{1}{2}\delta_{13}(\varepsilon)\bigg)\right)\nonumber\\
    &\qquad\qquad\qquad\times\P_{\bnu}^{\rm E} \bigg( \overline{\mcE}_{N(\varepsilon)}\bigg(\frac{1}{2}\delta_{13}(\varepsilon)\bigg)\bigg)\ .
    \label{eq:ETC_A22_bern}
\end{align}
Furthermore, the first term, i.e., ~\eqref{eq:ETC_A21_bern} is upper bounded by
\begin{align}
    &\P_{\bnu}^{\rm E} \left ( U_h\bigg(\sum_{i\in[K]} a^{(1)}(i)\F_i\bigg) \geq U_h\bigg(\sum_{i\in[K]} a^{\rm E}_{N(\varepsilon)}(i)\F_i\bigg) + \delta_{13}(\varepsilon)\;\bigg\lvert\; \mcE_{N(\varepsilon)}\bigg(\frac{1}{2}\delta_{13}(\varepsilon)\bigg)\right)\\
    \label{eq:ETC_1_bern}
    &\leq \P_{\bnu}^{\rm E} \Bigg ( U_h\bigg(\sum_{i\in[K]} a^{(1)}(i)\F_{i,N(\varepsilon)}^{\rm E}\bigg) +  \frac{1}{2}\delta_{13}(\varepsilon)\geq U_h\bigg(\sum_{i\in[K]} a^{\rm E}_{N(\varepsilon)}(i)\F_{i,N(\varepsilon)}^{\rm E}\bigg)\nonumber\\
    & \qquad\qquad + \frac{1}{2}\delta_{13}(\varepsilon)\;\bigg\lvert\;\mcE_{N(\varepsilon)}\bigg(\frac{1}{2}\delta_{13}(\varepsilon)\bigg)\Bigg)\\
    &=0\ ,
    \label{eq:ETC_10_bern}
\end{align}
where~\eqref{eq:ETC_1_bern} holds due to the conditioning on the event $\mcE_{N(\varepsilon)}(\delta_{13}(\varepsilon)/2)$, and ~\eqref{eq:ETC_10_bern} holds since $\ba^{\rm E}_{N(\varepsilon)}\in\widehat{\rm OPT}_{\varepsilon, N}$. Furthermore, the term in~\eqref{eq:ETC_A22_bern} is upper bounded by
\begin{align}
    \P_{\bnu}^{\rm E} \bigg( \overline{\mcE}_{N(\varepsilon)}\Big(\frac{1}{2}\delta_{13}(\varepsilon)\Big )\bigg)\;\leq\;\P_{\bnu}^\pi\bigg( \overline{\mcE}_{1,N}\Big (\frac{1}{2}\delta_{13}(\varepsilon)\Big)\bigg) + \P_{\bnu}^{\rm E} \bigg( \overline{\mcE}_{2,N}\Big (\frac{1}{2}\delta_{13}(\varepsilon)\Big)\bigg)\ ,
\end{align}
which can be analyzed following the exact steps as~\eqref{eq:ETC_bern_cue_start}-\eqref{eq:E3C}, with the only distinctions that $N_{12}(\varepsilon)$ is replaced by $N_{13}(\varepsilon)$, and $\delta_{12}(\varepsilon)$ is replaced by $\delta_{13}(\varepsilon)$. The remainder of the proof follows similar arguments as Section~\ref{Appendix:ETC_theorem_epsilon_proof}, except we set $\varepsilon = \Theta((K^{2+\frac{2}{q}}T^{-{\bar\gamma}}\log T)^{\frac{q}{2\beta}})$ for \(\bar\gamma \in (0,1)\). Since \(\delta_{13}(\varepsilon) = \Omega(\varepsilon^{\bar\beta})\ \), from~\eqref{eq:number of samples main paper beta13}, the exploration horizon is $N_{13}(\varepsilon)=\Theta(T^{\bar\gamma})$. Hence, for arm selection regret, from~\eqref{eq:ETC_M_eps} we have
\begin{align}
\label{eq:Discrete_order_3}
    \Bar{\mathfrak{R}}_{\bnu,2}^{\rm E}(T) = O(KT^{(\bar\gamma-1)q})\ .
\end{align}
Furthermore, from Lemma~\eqref{lemma:Delta_error}, we have
\begin{align}
    \delta_{02}(\varepsilon) &\leq 2\mcL (KW\varepsilon)^q  \ .
\end{align}
Hence, \(\delta_{02}(\varepsilon) = O(\varepsilon^q)\). Consequently, by our choice of $\varepsilon$, we obtain 
\begin{align}
\label{eq:delta_order_beta13}
    \delta_{02}(\varepsilon) = O\left(K^{q + \frac{q(q+1)}{{\bar\beta}}}T^{-\frac{q^2{\bar\gamma}}{2{\bar\beta}}}(\log T)^{\frac{q^2}{2{\bar\beta}}}\right).
\end{align}
From~\eqref{eq:disc_up_ETC},~\eqref{eq:Discrete_order}, and~\eqref{eq:delta_order_beta13} the regret of the PM-ETC-M algorithm is upper bounded by
    \begin{align}
    \label{eq:ETC upper bound gamma appendix beta13}
        \mathfrak{R}_{\bnu}^{\rm E}(T) = O \left (\max\Big\{ K^{q + \frac{q(q+1)}{{\bar\beta}}}T^{-\frac{rq{\bar\gamma}}{2{\bar\beta}}}(\log T)^{\frac{rq}{2{\bar\beta}}}\;,\;K^qT^{(\bar\gamma-1)q}\Big\} \right)\ .
    \end{align}
Finally, setting $\bar\gamma \triangleq \frac{2{\bar\beta}}{2{\bar\beta} + q}$, the regret of the PM-ETC-M algorithm is upper bounded by 
\begin{align}
     \mathfrak{R}_{\bnu}^{{\textnormal{\rm E}}}(T)\ & \leq O\Big(\Big[K^{2(q+{\bar\beta}+1)}\; T^{-\frac{q}{1+q/2{\bar\beta}}}\; (\log T)^{q}\Big]^{\frac{q}{2{\bar\beta}}}\Big)\ .
\end{align}

\section{PM-UCB-M Algorithm}
\label{proof:UCB upper bound}
In this section, we provide the proofs of Theorems \ref{theorem:UCB upper bound} and \ref{corollary:PM-UCB-M} for PM-UCB-M.

\subsection{Regret Decomposition}
In~\eqref{eq:regret_decomposition}, the regret is decomposed into a estimation regret component, and the arm selection regret. Lemma~\ref{lemma:Delta_error} shows an upper bound for the estimation regret. In this section, we provide an upper bound on arm selection regret $\bar{\mathfrak{R}}_{\bnu,1}^{\rm U}(T)$ of Algorithm~\ref{algorithm:B-UCB-M} as follows. Recall that we have defined \(\tau_t^{\rm U}(i)\) as the number of times PM-UCB-M selects arm \(i \in [K]\) up to time $t$. We have
\begin{align}
    \Bar{\mathfrak{R}}_{\bnu,1}^{\rm U}(T) &=   U_h\left(\sum_{i\in[K]} a^{(1)}(i)\F_i\right)  - \E_{\bnu}^{\rm U} \left[U_h\left(\sum_{i\in[K]} \frac{\tau_T^{\rm U}(i)}{T}\F_i\right)\right] \\
    &\leq \underbrace{\sum_{\mcS \subseteq [K]: \mcS \neq \emptyset}
    \E_{\bnu}^{\rm U} \left[\Bigg( U_h\left(\sum_{i\in[K]} a^{(1)}(i)\F_i\right)  -  U_h\left(\sum_{i\in[K]} \frac{\tau_T^{\rm U}(i)}{T}\F_i\right) \Bigg) \mathds{1}\{\F_i \notin \mcC_T(i): i \in \mcS\}\right]}_{\triangleq B_1(T)} \\
    &\qquad + \underbrace{\E_{\bnu}^{\rm U} \left[ \Bigg( U_h\left(\sum_{i\in[K]} a^{(1)}(i)\F_i\right)  - U_h\left(\sum_{i\in[K]} \frac{\tau_T^{\rm U}(i)}{T}\F_i\right) \Bigg) \mathds{1}\{\F_i \in \mcC_T(i) \;\forall i \in [K]\} \right]}_{\triangleq B_2(T)} \ .
\end{align}
Expanding $B_1(T)$, we have
\begin{align}
    B_1(T)&=\sum_{\mcS \subseteq [K]: \mcS \neq \emptyset}\P_{\bnu}^{\rm U}\Big ( \F_i \notin \mcC_T(i): i \in \mcS\Big )\nonumber\\
    &\times \quad
    \E_{\bnu}^{\rm U} \left[U_h\left(\sum_{i\in[K]} a^{(1)}(i)\F_i\right)  -  U_h\left(\sum_{i=1}^{K} \frac{\tau_T^{\rm U}(i)}{T}\F_i\right)\bigg\lvert\; \F_i \notin \mcC_T(i): i \in \mcS\right]\ .
\end{align}
Leveraging the fact that
\begin{align}
    \sum_{i=1}^{K} \binom{K}{i} x^i = (x+1)^K-1\ ,
\end{align}
along with the large deviation bound on the confidence sequences from Lemma~\ref{lemma: concentration}, i.e., $\P_{\bnu}^{\rm U}(\F_i\notin\mcC_T(i))\leq 1/T^2$ for every $i\in[K]$, we have
\begin{align}
\label{eq:probi}
    \sum\limits_{\mcS\subseteq[K] : \mcS\neq\emptyset}\P_{\bnu}^{\rm U} \Big(\F_i \notin \mcC_T(i):  i \in \mcS\Big) = \left(\frac{2}{T^2}+1\right)^K-1\ .
\end{align}
Furthermore, owing to the fact that $U_h(\sum_{i\in[K]}\alpha(i)\F_i)\leq B_h$ for any $\balpha\in\Delta^{K-1}$, we obtain
\begin{align}
\label{eq:A_1 bound refer}
    B_1(T)\;\leq\; B_h\left(\left(\frac{2}{T^2}+1\right)^K-1\right)\ .
\end{align}
Next, we will upper bound $B_2(T)$. For this purpose, we begin by defining the upper confidence bound of a parameter $\balpha\in\Delta^{K-1}$ at time $t\in\N$ as 
 \begin{align}
 \label{eq:UCB_alpha_def}
     {\rm UCB}_t(\balpha)\;\triangleq\; \max\limits_{\eta_1\in\mcC_t(1),\cdots,\mcC_t(K)}\; U\left ( \sum\limits_{i\in[K]}\alpha(i)\eta_i\right )\ .
\end{align}
Furthermore, let us define 
the optimistic CDF estimates which maximize the upper confidence bound for every arm $i\in[K]$ as 
\begin{align}  \widetilde\F_{i,t}\;\triangleq\;\argmax_{\eta_i\in\mcC_t(i)}\;\max\limits_{\eta_j\in\mcC_t(j) : j\neq i} {\rm UCB}_t(\ba^{\rm U}_t)\ .
\end{align}
Expanding $B_2(T)$, we have 
\begin{align}
    B_2(T) &= \left . \E_{\bnu}^{\rm U}\left[ U_h\left(\sum_{i\in[K]} a^{(1)}(i) \F_i \right) - U_h\left(\sum_{i\in[K]} \frac{\tau_T^{\rm U}(i)}{T} \F_i \right) \right\lvert \F_i \in \mcC_T(i) \;\forall i \in [K] \right]\\
    & \stackrel{\eqref{eq:UCB_alpha_def}}{\leq} \E_{\bnu}^{\rm U}\left[\left .{\rm UCB}_T(\ba^{(1)}) - U_h\left(\sum_{i\in[K]} \frac{\tau_T^{\rm U}(i)}{T}\F_i\right)\right\lvert \F_i \in \mcC_T(i) \;\forall i \in [K]\right] \\
    & \stackrel{\eqref{eq:UCB_alpha}}{\leq} \E_{\bnu}^{\rm U}\left[\left . {\rm UCB}_T(\ba^{\rm U}_T) - U_h\left(\sum_{i\in[K]} \frac{\tau_T^{\rm U}(i)}{T}\F_i\right)\right\lvert \F_i \in \mcC_T(i) \;\forall i \in [K]\right] \\
    & = \underbrace{\left .\E_{\bnu}^{\rm U}\left[{\rm UCB}_T(\ba^{\rm U}_T) - U_h\left(\sum_{i\in[K]} a^{\rm U}_T(i)\F_i\right)\right\lvert \F_i \in \mcC_T(i) \;\forall i \in [K]\right]}_{\triangleq B_{21}(T)}\nonumber\\
    &\quad +\underbrace{\left .\E_{\bnu}^{\rm U}\left[U_h\left(\sum_{i\in[K]} a^{\rm U}_T(i)\F_i\right) - U_h\left(\sum_{i\in[K]} \frac{\tau_T^{\rm U}(i)}{T}\F_i\right)\right\lvert \F_i \in \mcC_T(i)\; \forall i \in [K]\right]}_{B_{22}(T)}\ . 
\end{align}
Note that the term $B_{21}(T)$ captures the {\em CDF estimation error} incurred by the UCB algorithm, and the term $B_{22}(T)$ reflects the {\em sampling estimation error} incurred by the UCB algorithm. Subsequently, we analyze each of these components.

\subsection{Upper Bound on $B_{21}(T)$}
Expanding $B_{21}(T)$ we obtain
\begin{align}
    B_{21}(T) &\leq\left .\E_{\bnu}^{\rm U}\left[U_h\left(\sum_{i\in[K]} a_{T}^{\rm U}(i)\widetilde\F_{i,T}\right) - U_h\left(\sum_{i\in[K]} a_{T}^{\rm U}(i)\F_i\right)\right\lvert \F_i \in \mcC_T(i) \;\forall i \in [K]\right]\\
    &\stackrel{\eqref{eq:Holder}}{\leq} \E_{\bnu}^{\rm U} \left .\left[\mathcal{L} \sum_{i\in[K]} \left(a_{T}^{\rm U}(i)\right)^q \norm{ \widetilde{\F}_{i,T}-\F_i}_{\rm W}^q  \right\lvert \F_i \in \mcC_T(i) \;\forall i \in [K] \right] \\
&\stackrel{\eqref{eq:UCB_confidence_sets}}{\leq} \E_{\bnu}^{\rm U} \left . \left[\mcL (32)^q\Big(\sqrt{2\e\log T} + 32\big)^q \sum\limits_{i\in[K]} \big(a_T^{\rm U}(i)\big)^q\left( \frac{1}{\tau_T^{\rm U}(i)}\right)^{\frac{q}{2}} \right\lvert \F_i \in \mcC_T(i) \;\forall i \in [K] \right]\\
    \label{eq:UUB_1}
    &\leq  \E_{\bnu}^{\rm U} \left . \left[ \mcL\left( 32 \Big( \sqrt{2\e\log T}+32\Big)\right)^q\cdot \sum\limits_{i\in[K]} \left (\frac{1}{\tau_T^{\rm U}(i)} \right)^{\frac{q}{2}} \right\lvert \F_i \in \mcC_T(i) \;\forall i \in [K] \right]\\
    \label{eq:UUB_2}
    &\leq \mcL\left( 32 \Big( \sqrt{2\e\log T}+32\Big)\right)^q\cdot \sum\limits_{i\in[K]} \left (\frac{1}{\frac{\rho}{4} T \varepsilon}\right)^{\frac{q}{2}}\\
    &=\frac{\mcL K}{T} \left( \frac{32}{\sqrt{\frac{\rho}{4}}}\right)^q T^{1-\frac{q}{2}} \left(\frac{\Big(\sqrt{2\e\log T} + 32\Big)}{\sqrt{\varepsilon}}\right)^q\ ,
    \label{eq:A21}
\end{align}
where \eqref{eq:UUB_1} follows from the fact that $a_T^{\rm U}(i)\leq 1$ for every $i\in[K]$, and~\eqref{eq:UUB_2} follows from the explicit exploration of each arm $i\in[K]$ in Algorithm~\ref{algorithm:UCB-M}.

\subsection{Upper Bound on the Sampling Estimation Error}
\label{appendix: UCB sampling estimation error}

Finally, we will bound $B_{22}(T)$. The key idea in upper bounding the sampling estimation error is to show that after a certain instant, the PM-UCB-M algorithm outputs an estimate $\balpha_t^{\rm U}$ of the mixing coefficients that belongs to the set of discrete optimal mixtures with a very high probability. Subsequently, we show that the undersampling routine of the PM-UCB-M algorithm ensures that once it has identified a correct optimal mixture, it navigates its sampling proportions to match the estimated coefficient. Let us denote the vector of arm sampling fractions by $\btau_T^{\rm U}\triangleq [\tau_T^{\rm U}(1),\cdots,\tau_T^{\rm U}(K)]^\top$. For the first step, note that
\begin{align}
    B_{22}(T)\;&\stackrel{\eqref{eq:Holder}}{\leq}\;\mcL \E_{\bnu}^{\rm U}\left [ \norm{\sum\limits_{i\in[K]}\left(a^{\rm U}_T(i) - \frac{\tau_T^{\rm U}(i)}{T}\right)\F_i}^q_{\rm W} \;\bigg\lvert\; \F_i\in \mcC_T(i)\;\forall i \in [K]\right ]\\
    &\stackrel{\eqref{eq:W}}{\leq}\;\mcL\E_{\bnu}^{\rm U}\left[ W^q\cdot \norm{\ba_T^{\rm U} - \frac{1}{T}\btau_T^{\rm U}}_1^q\;\bigg\lvert\; \F_i\in \mcC_T(i)\;\forall i \in [K]\right]\ .
    \label{eq:UCB_UB2}
\end{align}
In order to bound $B_{22}(T)$, we will leverage a low probability event (which we will call $\mcE_{3,T}$), and expand~\eqref{eq:UCB_UB2} by conditioning on this event. To define the low probability event, let us first lay down a few definitions. First, let us define the event 
\begin{align}
    \mcE_{0,t} \triangleq \Big\{\F_i \in \mcC_t(i) \quad \forall i \in [K]\Big\}\ .
\end{align}
Let us denote the {\em set} of discrete optimal mixtures of the PM $U_h$ by
\begin{align}
    {\rm OPT}_{\varepsilon}\;\triangleq\; \argmax\limits_{\ba\in\Delta_{\varepsilon}^{K-1}}\; U_h\left (\sum\limits_{i\in[K]} a(i)\F_i\right)\ ,
\end{align}
and the {\em set} of optimistic mixtures at each instant $t$ by
\begin{align}
    \widetilde{\rm OPT}_{\varepsilon,t}\;\triangleq\; \argmax\limits_{\ba\in\Delta_{\varepsilon}^{K-1}}\;\max\limits_{\eta_i\in\mcC_t(i)\;\forall i\in[K]}\; U_h\left ( \sum\limits_{i\in[K]}a(i)\eta_i\right )\ .
\end{align}
Furthermore, for any $\ba^{(1)}\in{\rm OPT}_{\varepsilon}$ and $\widetilde\balpha_t\in\widetilde{\rm OPT}_{\varepsilon,t}$, let us define the events
\begin{align}
    \mcE_{1,t}(x) \triangleq \left\{\left| U_h\left(\sum_{i\in[K]} a^{(1)}(i)\widetilde{\F}_{i,t}\right) - U_h\left(\sum_{i\in[K]} a^{(1)}(i) \F_i\right) \right| < x \right\}\ ,
\end{align}
\begin{align}
    \mcE_{2,t}(x) \triangleq \left\{\left| U_h\left(\sum_{i\in[K]} a^{\rm U}_t(i)\widetilde{\F}_{i,t}\right) - U_h\left(\sum_{i\in[K]} a^{\rm U}_t(i)\F_i\right) \right| < x \right\}\ ,
\end{align}

\begin{align}
   \text{and,} \quad \mcE_t(x) \triangleq \mcE_{1,t}(x) \bigcap \mcE_{2,t}(x)\ .
\end{align}
Note that
\begin{align}
    \P_{\bnu}^{\rm U} \Big ( \widetilde{\rm OPT}_{\varepsilon,t} \neq {\rm OPT}_{\varepsilon}\Big )\;= \;\P_{\bnu}^{\rm U}\Big (\exists \ba\in\widetilde{\rm OPT}_{\varepsilon,t} : \ba\notin {\rm OPT}_{\varepsilon} \Big )\ .
\end{align}
Let us denote the mixing coefficients that is contained in $\widetilde{\rm OPT}_{\varepsilon,t}$ and yet not in ${\rm OPT}_{\varepsilon}$ by $\ba_t^{\rm U}$. Accordingly, we have
\begin{align}
    \P_{\bnu}^{\rm U}\Big(\ba^{\rm U}_t\notin{\rm OPT}_{\varepsilon}\Big) &= \P_{\bnu}^{\rm U}\Big (\ba^{\rm U}_t\notin{\rm OPT}_{\varepsilon}\mid \bar \mcE_{0,t}\Big )\P_{\bnu}^{\rm U}(\bar \mcE_{0,t}) + \P_{\bnu}^{\rm U}\Big(\ba^{\rm U}_t\notin{\rm OPT}_{\varepsilon} \mid \mcE_{0,t}\Big)\P(\mcE_{0,t})\\
    & \leq \P_{\bnu}^{\rm U}(\bar \mcE_{0,t}) + \P_{\bnu}^{\rm U}\Big(\ba^{\rm U}_t\notin{\rm OPT}_{\varepsilon} \mid \mcE_{0,t}\Big)\\
    &\stackrel{\eqref{eq:probi}}{\leq} \left(\left(\frac{2}{T^2}+1\right)^K-1\right) + \P_{\bnu}^{\rm U}\Big(\ba^{\rm U}_t\notin{\rm OPT}_{\varepsilon} \mid \mcE_{0,t}\Big)\ .
    \label{eq:UCB_UBh1}
\end{align}
Next, note that    
\begin{align}
    &\P_{\bnu}^{\rm U}\Big(\ba^{\rm U}_t\notin{\rm OPT}_{\varepsilon} \mid \mcE_{0,t}\Big)\nonumber\\
    &\;= \P_{\bnu}^{\rm U}\left(U_h\left(\sum_{i\in[K]} a^{(1)}(i) \F_i\right) > U_h\left(\sum_{i\in[K]} a^{\rm U}_t(i)\F_i\right) + \delta_{12}(\varepsilon) \bigg\lvert \mcE_{0,t}\right) \\
    &\; = \P_{\bnu}^{\rm U}\left(U_h\left(\sum_{i\in[K]} a^{(1)}(i) \F_i\right) > U_h\left(\sum_{i\in[K]}  a^{\rm U}_t(i)\F_i\right) + \delta_{12}(\varepsilon) \Bigg\lvert \mcE_{0,t},\; \mcE_t\left(\frac{1}{2}\delta_{12}(\varepsilon)\right)\right)\nonumber\\
    &\qquad\qquad\qquad\qquad\times\P_{\bnu}^{\rm U}\left(\mcE_t\left(\frac{1}{2}\delta_{12}(\varepsilon)\right)\bigg\lvert \mcE_{0,t}\right) \\
    &\quad + \P_{\bnu}^{\rm U}\left(U_h\left(\sum_{i\in[K]} a^{(1)}(i) \F_i\right) > U_h\left(\sum_{i\in[K]} a^{\rm U}_t(i)\F_i\right) + \delta_{12}(\varepsilon) \Bigg\lvert \mcE_{0,t},\; \bar\mcE_t\left(\frac{1}{2}\delta_{12}(\varepsilon)\right)\right)\nonumber\\
    &\qquad\qquad\qquad\qquad\times\P_{\bnu}^{\rm U}\left(\bar\mcE_t\left(\frac{1}{2}\delta_{12}(\varepsilon)\right)\bigg\lvert \mcE_{0,t}\right) \\
    & \;\leq \P_{\bnu}^{\rm U}\left(U_h\left(\sum_{i\in[K]} a^{(1)}(i) \F_i\right) > U_h\left(\sum_{i\in[K]} a^{\rm U}_t(i)\F_i\right) + \delta_{12}(\varepsilon) \Bigg\lvert \mcE_{0,t},\; \mcE_t\left(\frac{1}{2}\delta_{12}(\varepsilon)\right)\right)\nonumber\\
    &\qquad\qquad\qquad\qquad+\P_{\bnu}^{\rm U}\left(\bar\mcE_t\left(\frac{1}{2}\delta_{12}(\varepsilon)\right)\bigg\lvert \mcE_{0,t}\right)   \\
    \label{eq:UCB_UB3}
    &\;\leq \P_{\bnu}^{\rm U} \left( U_h\left(\sum_{i\in[K]} a^{(1)}(i)\widetilde{\F}_{i,t}\right)  >  U_h\left(\sum_{i\in[K]} a^{\rm U}_t(i)\widetilde{\F}_{i,t}\right) \Bigg\lvert \mcE_{0,t},\; \bar\mcE_t\left ( \frac{1}{2}\delta_{12}(\varepsilon)\right)\right) \nonumber\\
    &\qquad\qquad\qquad\qquad+\P_{\bnu}^{\rm U}\left(\bar\mcE_t\left(\frac{1}{2}\delta_{12}(\varepsilon)\right)\bigg\lvert \mcE_{0,t}\right)   \\
    \label{eq:UCB_UB4}
    &\;=\P_{\bnu}^{\rm U}\left(\bar\mcE_t\left(\frac{1}{2}\delta_{12}(\varepsilon)\right)\bigg\lvert \mcE_{0,t}\right)   \\
    \label{eq:UCB_UB5}
    & \;\leq \P_{\bnu}^{\rm U}\left(\bar\mcE_{1,t}\left(\frac{1}{2}\delta_{12}(\varepsilon)\right)\bigg\lvert \mcE_{0,t}\right) + \P_{\bnu}^{\rm U}\left(\bar\mcE_{2,t}\left(\frac{1}{2}\delta_{12}(\varepsilon)\right)\bigg\lvert \mcE_{0,t}\right)\ ,
\end{align}
where \eqref{eq:UCB_UB3} follows from the definition of the event $\mcE_t(\delta_{13}(\frac{1}{2}\varepsilon))$, \eqref{eq:UCB_UB4} follows from the definitions of $\ba^{(1)}$ (which is an optimizer in the set ${\rm OPT}_{\varepsilon}$) and $\balpha_t^{\rm U}$ (which is an optimizer in the set $\widetilde{\rm OPT}_{\varepsilon,t}$, and~\eqref{eq:UCB_UB5} follows from a union bound.

Next, we will find upper bounds on the two probability terms in~\eqref{eq:UCB_UB5}. Note that for any $t>K\lceil\rho T\varepsilon/4\rceil$ we have
\begin{align}
    &\P_{\bnu}^{\rm U}\left(\bar\mcE_{1,t}\left(\frac{1}{2}\delta_{12}(\varepsilon)\right)\bigg\lvert \mcE_{0,t}\right)\nonumber\\
    &\qquad= \P_{\bnu}^{\rm U}\left( \left| U_h\left(\sum_{i\in[K]} a^{(1)}(i)\widetilde{\F}_{i,t}\right) - U_h\left(\sum_{i\in[K]} a^{(1)}(i) \F_i\right) \right| \geq \frac{1}{2}\delta_{12}(\varepsilon) \bigg\lvert \mcE_{0,t} \right )\\
    &\qquad\stackrel{\eqref{eq:Holder}}{\leq} \P_{\bnu}^{\rm U} \left(\sum_{i\in[K]} (a^{(1)}(i))^q  \norm{ \widetilde{\F}_{i,t}-\F_i}_{\rm W}^q \geq \frac{1}{2\mathcal{L}}\delta_{12}(\varepsilon) \bigg\lvert \mcE_{0,t} \right )\\
    &\qquad\stackrel{\eqref{eq:UCB_confidence_sets}}{=} \P_{\bnu}^{\rm U}\left ( \sum\limits_{i\in[K]}\big(a^{(1)}(i))^q \norm{\widetilde\F_{i,t} - \F_i}_{\rm W}^q\bigg\lvert \mcE_{0,t}\right)\\
    &\qquad\stackrel{\eqref{eq:UCB_confidence_sets}}{\leq} \P_{\bnu}^{\rm U}\left( \sum\limits_{i\in[K]}\underbrace{\big( a^{(1)}(i)\big)^q}_{\leq 1}\left(16\sqrt{\frac{1}{\tau_T^{\rm U}}(i)}\cdot\Big( \sqrt{2\e\log T} + 32\Big)\right)^q>\frac{1}{2\mcL}\delta_{12}(\varepsilon)\bigg\lvert \mcE_{0,t}\right)\\
    \label{eq:E3_1}
    &\qquad\leq \sum\limits_{i\in[K]} \P_{\bnu}^{\rm U}\left(\left(\sqrt{\frac{1}{\tau_T^{\rm U}}(i)}\cdot\Big( \sqrt{2\e\log T} + 32\Big)\right)^q>\frac{1}{2K\mcL(16)^q}\delta_{12}(\varepsilon)\bigg\lvert \mcE_{0,t}\right)\\
    \label{eq:E3_2}
    &\qquad\leq \sum\limits_{i\in[K]} \P_{\bnu}^{\rm U}\left(\left(\frac{\Big( \sqrt{2\e\log T} + 32\Big)}{\sqrt{\frac{\rho}{4} T \varepsilon}}\right)^q>\frac{1}{2K\mcL(16)^q}\delta_{12}(\varepsilon)\bigg\lvert \mcE_{0,t}\right),
\end{align}
where the transition~\eqref{eq:E3_1}-\eqref{eq:E3_2} follows from the explicit exploration phase of the PM-UCB-M algorithm.
Now, let us define a time instant $T_0(\varepsilon)$ as follows.
\begin{align}
\label{eq:T_epsilon}
    T_0(\varepsilon)\; \triangleq \;\inf \left\{ t\in\N :\left(\frac{\Big( \sqrt{2\e\log s} + 32\Big)}{\sqrt{ \frac{\rho}{4} s \varepsilon }}\right) \leq \frac{1}{16} \left( \frac{\delta_{12}(\varepsilon)}{2K\mcL}\right)^{\frac{1}{q}} \quad \forall s \geq t\right\}\ .
\end{align}
Hence, $\forall t \geq T_0(\varepsilon)$ we have 
\begin{align}
\label{eq:UCB_UB8}
     \P_{\bnu}^{\rm U}\left(\bar\mcE_{1,t}\left(\frac{1}{2}\delta_{12}(\varepsilon)\right)\bigg\lvert \mcE_{0,t}\right)\;=\;0\ .
\end{align}
Using a similar line of arguments, we may readily show that for all $t \geq T_0(\varepsilon)$,
\begin{align}
\label{eq:UCB_UB9}
    \P_{\bnu}^{\rm U}\left(\bar\mcE_{2,t}\left(\frac{1}{2}\delta_{12}(\varepsilon)\right)\bigg\lvert \mcE_{0,t}\right)\;=\;0\ .
\end{align}
Combining~\eqref{eq:UCB_UB8} and~\eqref{eq:UCB_UB9} we infer that for all $t \geq T_0(\varepsilon)$,
\begin{align}
\label{eq:UCB_UB10}
    \P_{\bnu}^{\rm U}\Big(\ba^{\rm U}_t\notin{\rm OPT}_{\varepsilon} \mid \mcE_{0,t}\Big)\;=\;0\ ,
\end{align}
which implies that
\begin{align}
\label{eq:UCB_UB10-1}
    \P_{\bnu}^{\rm U}\Big(\widetilde{\rm OPT}_{\varepsilon,t}\neq{\rm OPT}_{\varepsilon} \mid \mcE_{0,t}\Big)\;=\;0\ .
\end{align}
We are now ready to define the low probability event $\mcE_{3,T}$. Let us define
\begin{align}
    \mcE_{3,T}\triangleq\Big\{\exists t \in[T_0(\varepsilon), T]: \widetilde{\rm OPT}_{\varepsilon,t}\neq{\rm OPT}_{\varepsilon}\Big\}\ .
\end{align}
Accordingly, we have the following lemma.

\begin{lemma} 
\label{lemma:probability of E3}
For event \(\mcE_{3,T}\), we have
\begin{align}
    \mathbb{P}_{\bnu}^ {\rm U}\left(\mcE_{3,T}\right) \leq T\left(\left(\frac{2}{T^2}+1\right)^K-1\right)\ .
\end{align}
\end{lemma}

\begin{proof}
Note that
    \begin{align}
\mathbb{P}_{\bnu}^{\rm U}\Big(\exists t>T_0(\varepsilon): \widetilde{\rm OPT}_{\varepsilon,t}\neq{\rm OPT}_{\varepsilon}\Big)  &\leq \sum_{t=T_0(\varepsilon)+1}^T \mathbb{P}\left(\widetilde{\rm OPT}_{\varepsilon,t}\neq{\rm OPT}_{\varepsilon}\right) \\
& \stackrel{\eqref{eq:UCB_UBh1}}{\leq} \sum_{t=T_0(\varepsilon)+1}^T \left(\left(\frac{2}{T^2}+1\right)^K-1\right) \nonumber \\ & \qquad + \P_{\bnu}^{\rm U}\Big(\ba^{\rm U}_t\notin{\rm OPT}_{\varepsilon}\med \mcE_{o,t}\Big)\\
& \stackrel{\eqref{eq:UCB_UB10}}{\leq} T\left(\left(\frac{2}{T^2}+1\right)^K-1\right)\ .
\end{align}
\end{proof}
First, thanks to the discretization, all arms are sampled at least once, according to under-sampling.
Let us assume that after the explicit exploration, an arm is never under sampled.
\begin{align}
\label{eq:samlingequiv}
    \tau_t(i) = \frac{\rho}{4} T \epsilon \ .
\end{align}
If an arm is not under sampled that would mean 
\begin{align}
\label{eq:oversampledproof}
    \tau_t(i) \geq t a_t^{U}(i) \Longrightarrow
    \frac{\rho}{4} T \epsilon t^{-1}  \geq a_t^{U}(i) 
     \geq \frac{\varepsilon}{2} \Longrightarrow \frac{\rho}{2} & \geq \frac{t}{T} \ ,
\end{align}
where right hand side follows from \eqref{eq:samlingequiv} and \eqref{eq:oversampledproof}, and from the definition of discretization.
For $t > \frac{\rho T}{2}$, this inequality does not hold. Therefore, no matter which mixing coefficient is chosen, if an arm is not sampled after the explicit exploration, it becomes under sampled at the latest at the time instant $\frac{\rho T}{2}$ which means after that time instant it will be sampled according to under-sampling rule. \newline
Next, we will show that under the event $\bar\mcE_{3,T}$, i.e., when the PM-UCB-M algorithm only selects a mixture distribution from the set of optimal mixtures, the PM-UCB-M arm selection routine, using the {\em under-sampling} procedure, eventually converges to an optimal mixture distribution. This is captured in the following lemma. Prior to stating the lemma, note that under the event $\bar\mcE_{3,T}$, the set of estimated mixing coefficients is the same in each iteration between $T_0(\varepsilon)$ and $T$. Hence, the PM-UCB-M algorithm aims to track {\em one} of these optimal mixing coefficients, which we denote by $\ba^{(1)}\in{\rm OPT}_{\varepsilon}$.

\begin{lemma}
\label{lemma:undersampling}
Under the event $\bar\mcE_{3,T}$, there exists a time instant $T(\varepsilon)<+\infty$, and $\ba^{(1)}\in{\rm OPT}_{\varepsilon}$ such that we have
\begin{align}
    \left| \frac{\tau_t^{\rm U}(i)}{t} - a^{(1)}(i)\right| < \frac{K}{t}  \quad \forall t \geq T(\varepsilon)\ .
\end{align}
\end{lemma}

\begin{proof}
We begin by defining a set of {\em over-sampled} arms as follows. We define the set $\mcO_t$ as
\begin{align}
    \mcO_t \;\triangleq\; \left\{i \in [K]: \frac{\tau_t^{\rm U}(i)}{t} > a^{(1)}(i) + \frac{1}{t}\right\}\ .
\end{align}
We will leverage Lemma~\ref{lemma:oversampling}, in which we show that there exists a finite time instant, which we call $T(\varepsilon)$, such that for all $t>T(\varepsilon)$ the set of over-sampled arms is {\em empty}. Specifically, leveraging Lemma~\ref{lemma:oversampling} we obtain
\begin{align}
    \frac{\tau_t^{\rm U}(i)}{t} - a^{(1)}(i) \;\leq\; \frac{1}{t}  \quad \forall t \geq T(\varepsilon)\ ,
\end{align}
which also implies that
\begin{align}
\label{eq:UK_oneside}
    \frac{\tau_t^{\rm U}(i)}{t} - a^{(1)}(i) \;<\; \frac{K}{t}  \quad \forall t \geq T(\varepsilon)\ .
\end{align}
Next, let us assume that there exists some $j\in[K]$ such that
\begin{align}
    \frac{\tau_t^{\rm U}(j)}{t} < a^{(1)}(j)  - \frac{K}{t} \ .
\end{align}
In this case, we have
\begin{align}
    \sum\limits_{i\in[K]} \frac{\tau_t^{\rm U}(i)}{t} &= \sum\limits_{i \neq j} \frac{\tau_t^{\rm U}(i)}{t} + \frac{\tau_t^{\rm U}(j)}{t} \\
    & \stackrel{\eqref{eq:UK_oneside}}{\leq}  \sum\limits_{i \neq j} a^{(1)}(i)  + \frac{K-1}{t} + a^{(1)}(j)  - \frac{K}{t}  \\
    & = 1 - \frac{1}{t}\ ,
\end{align}
which is a contradiction, since we should have 
\begin{align}
    \sum\limits_{i\in[K]} \frac{\tau_t^{\rm U}(i)}{t} \;=\; 1\ .
\end{align}
Hence, we can conclude that 
\begin{align}
  \left| \frac{\tau_t^{\rm U}(i)}{t} - a^{(1)}(i)\right| < \frac{K}{t}  \quad \forall t \geq T(\varepsilon) \ . 
\end{align}
\end{proof}

\begin{lemma}
\label{lemma:oversampling}
Under the event $\bar\mcE_{3,T}$, there exists a finite time $T(\varepsilon)<+\infty$ such that for every $t>T(\varepsilon)$, we have $\mcO_t = \emptyset$.
\end{lemma}
\begin{proof}
    For any time $t>T_0(\varepsilon)$, 
    and some arm $j\notin\mcO_t$, we have
    \begin{align}
        \frac{\tau_t^{\rm U}(j)}{t}\;\leq\; a^(1)(j) + \frac{1}{t}\ .
    \end{align}
    Based on this, we have the following two regimes of the arm selection fraction for arm $j$. \newline
    \textbf{Case (a): ($\tau_t^{\rm U}(j)/t \leq \bar\alpha_j^\star$).} Since the PM-UCB-M algorithm only samples under-sampled arms, it is probable that the arm $j$ is sampled at time $t+1$. In that case, we would have
    \begin{align}
    \frac{\tau_{t+1}^{\rm UCB}(j)}{t+1} \;&\leq\; \frac{\tau_t^{\rm U}(j)+1}{t+1}  \\
    & =\; \frac{\tau_t^{\rm U}(j)}{t+1} + \frac{1}{t+1} \\
    \label{eq:oversampling1}
    & \leq\; \frac{a^{(1)}(j) t}{t+1} + \frac{1}{t+1} \\
    \label{eq:UK_1}
    & \leq\; a^{(1)}(j) + \frac{1}{t+1}\ ,
    \end{align}
    where~\eqref{eq:oversampling1} follows from the fact that $\tau_t^{\rm U}(j)\leq t\bar\alpha_j^\star$. Clearly, we see that the arm $j$ does not enter the set of over-sampled arms $\mcO_t$ in the subsequent round $t+1$. \newline
    \textbf{Case (b): ($\bar\alpha_j^\star\leq\tau_t^{\rm U}(j)/t \leq \bar\alpha_j^\star + 1/t$).} In this case, there always exists at least one arm $k\in[K]$ which satisfies that $\tau_t^{\rm U}(k)/t\leq a^{(1)}_k$. This implies that the arm $j$ is not sampled by the PM-UCB-M arm selection rule, since it is not the most under-sampled arm. In this case, we have
    \begin{align}
    \frac{\tau_{t+1}^{\rm UCB}(j)}{t+1} \;&=\; \frac{\tau_t^{\rm U}(j)}{t+1}  \\
    & =\; \frac{\tau_t^{\rm U}(j)}{t}\frac{t}{t+1} \\
    \label{eq:oversampling2}
    & \leq\; \left(a^{(1)}(j) + \frac{1}{t}\right)\frac{t}{t+1} \\
    & =\; a^{(1)}(j)\frac{t}{t+1} + \frac{1}{t+1} \\ 
    \label{eq:UK_2}
    & \;\leq a^{(1)}(j) + \frac{1}{t+1}\ ,
\end{align}
where~\eqref{eq:oversampling2} from the definition of case (b). Again, we have concluded that $j\notin\mcO_{t+1}$. Combining~\eqref{eq:UK_1} and~\eqref{eq:UK_2}, we conclude that none of the arms which are not already contained in the set $\mcO_t$ ever enters this set. Hence, all we are left to show is the existence of the time instant after which the set $\mcO_t$ becomes an empty set. Evidently, since the PM-UCB-M algorithm never samples from the set of over-sampled arms $\mcO_t$, we will establish an upper bound $m$ such that after $t > T_0(\varepsilon) + m$, all arms leave the set $\mcO_t$. Notably, for any arm $i\in\mcO_t$, it holds that
\begin{align}
    \frac{\tau_t^{\rm U}(i)}{t}\; >\; a^{(1)}(i) + \frac{1}{t}\ .
\end{align}
Let $m$ be such that
\begin{align}
\label{eq:oversampling3}
    \frac{\tau_{T_0(\varepsilon)+m}^{\rm U}(i)}{T_0(\varepsilon)+m}\;\leq\;a^{(1)}(i) + \frac{1}{T_0(\varepsilon)+m}\ ,
\end{align}
which implies that the arm $i$ has left the set $\mcO_{T(\varepsilon)+m}$. Furthermore, since arm $i$ is never sampled between times instants $T(\varepsilon)$ and $T(\varepsilon)+m$ (as it belongs to the over-sampled set),~\eqref{eq:oversampling3} can be equivalently written as
\begin{align}
\label{eq:m_finite}
    \frac{\tau_{T_0(\varepsilon)}^{\rm U}(i)}{T_0(\varepsilon)+m}\;\leq\;a^{(1)}(i)+ \frac{1}{T_0(\varepsilon)+m}\ .
\end{align}
Next, noting that for any arm $i\in[K]$, we have the upper bound $\tau_t^{\rm U}(i)\leq t$ on the number of times $i$ is chosen up to time $t$, and that 
$a^{(1)}(i)\geq\varepsilon / 2$ for every $i\in[K]$, we have the following choice for $m$.  
\begin{align}
\label{eq:m}
    m \; = \; \left(\frac {2}{\varepsilon}-1 \right)T_0(\varepsilon) - \frac{2}{\varepsilon}\ . 
\end{align}
The proof is completed by defining 
\begin{align}
\label{eq:T_epsilon_actual}
    T(\varepsilon) \triangleq T_0(\varepsilon) + m\ .
\end{align}
\end{proof}
Next, from~\eqref{eq:UCB_UB2}, we have
\begin{align}
    B_{22}(T)\;&\leq\; \mcL W^q \E_{\bnu}^{\rm U} \left [ \sum_{i\in[K]} \left\lvert a_T^{\rm U}(i) - \frac{\tau_T^{\rm}(i)}{T}\right\rvert^q\;\Bigg\lvert\; \F_i\in\mcC_T(i)\;\forall\; i\in[K]\right ]\\
    \label{eq:UCB_f1}
    &\leq\; \frac{1}{\P_{\bnu}^{\rm U}\Big( \F_i\in\mcC_T(i)\;\forall\; i\in[K]\Big)}\mcL W^q \E_{\bnu}^{\rm U} \left [ \sum_{i\in[K]}\left \lvert a^{\rm U}_T(i) - \frac{\tau_T^{\rm U}(i)}{T}\right\rvert^q\right]\\
    \label{eq:UCB_f2}
    &\leq\; \frac{1}{2-\Big( \frac{1}{T^2} + 1\Big)^K}\mcL W^q \E_{\bnu}^{\rm U} \left [ \sum_{i\in[K]}\left \lvert a^{\rm U}_T(i) - \frac{\tau_T^{\rm U}(i)}{T}\right\rvert^q\right]\ , 
\end{align}
where~\eqref{eq:UCB_f1} follows from the fact that $\E[A|B] = (\E[A] - \E[A|\bar B]\cdot\P(\bar B))/\P(B) \leq \E[A]/\P(B)$, and, ~\eqref{eq:UCB_f2} follows from~\eqref{eq:probi}. \newline
Furthermore, for all $T\geq 3$ and $K\geq 2$, it can be readily verified that $2 - (2/T^2 + 1)^K > \frac{1}{2}$, which implies that
\begin{align}
    B_{22}(T) & \stackrel{\eqref{eq:UCB_f2}}{\leq} 2\mcL W^q \E_{\bnu}^{\rm U}\left [ \sum_{i\in[K]}\left\lvert a_T^{\rm U} - \frac{\tau_T^{\rm U}(i)}{T}\right\rvert^q\right]
    \\
    &= 2\mcL W^q \E_{\bnu}^{\rm U}\left [ \sum_{i\in[K]}\left\lvert a_T^{\rm U} - \frac{\tau_T^{\rm U}(i)}{T}\right\rvert^q\;\bigg\lvert \mcE_{3,T}\right]\P_{\bnu}^{\rm U}\Big( \mcE_{3,T}\Big) + 2\mcL W^q \E_{\bnu}^{\rm U}\left [ \sum_{i\in[K]}\left\lvert a_T^{\rm U} - \frac{\tau_T^{\rm U}(i)}{T}\right\rvert^q\;\bigg\lvert\;\bar\mcE_{3,T}\right]\P_{\bnu}^{\rm U}\Big( \bar\mcE_{3,T}\Big)\\
    \label{eq:UCB_UB11}
    &\leq 2\mcL K W^q \left(\P_{\bnu}^{\rm U}(\mcE_{3,T}) + K\left(\frac{K}{T}\right)^q\right)\\
    \label{eq:UCB_UB12}
    &\leq 2\mcL K W^q  \left(T\left(\left(\frac{2}{T^2}+1\right)^K-1\right) + \left(\frac{K}{T}\right)^q\right)\ ,
\end{align}
where~\eqref{eq:UCB_UB11} follows from Lemma~\ref{lemma:oversampling} and~\eqref{eq:UCB_UB12} follows from Lemma~\ref{lemma:probability of E3}.

\subsection{Proof of Theorem~\ref{theorem:UCB upper bound} for PM-UCB-M}
\label{Appendix:UCB_theorem_proof}
Finally, we add up all the terms to find an upper bound on the arm selection regret. We have
\begin{align}
    \label{eq:Discretized_regret}
    \Bar{\mathfrak{R}}_{\bnu,1}^{\rm U}(T)\;&=\; B_1(T) + B_{21}(T) + B_{22}(T) \\
    &\leq\; \frac{1}{T} \Bigg [ B_hT\left( \left( \frac{2}{T^2}+1\right)^K - 1\right)  +{\mcL K} \left( \frac{32}{\sqrt{\rho}}\right)^q T^{1-\frac{q}{2}} \Big(\sqrt{2\e\log T} + 32\Big)^q\nonumber\\
    &\qquad +2\mcL W^q  \left(KT^2\left(\left(\frac{2}{T^2}+1\right)^K-1\right) + K^{1+q}T^{1-q}\right) \Bigg]\ .
\end{align}
Leveraging the upper bounds on estimation regret and the arm selection regret, we state an upper bound on the regret of the PM-UCB-M algorithm.
\begin{align}
&\mathfrak{R}_{\bnu}^{\rm U}(T) = \delta_{01}(\varepsilon)  + \Bar{\mathfrak{R}}_{\bnu,1}^{\rm U}(T) \\ & \leq \delta_{01}(\varepsilon) + \frac{1}{T} \Bigg [ B_hT\left( \left( \frac{2}{T^2}+1\right)^K - 1\right)\nonumber + 2\mcL W^q \left(KT^2\left(\left(\frac{2}{T^2}+1\right)^K-1\right) + K^{1+q}T^{1-q}\right)\nonumber\\
    & +\underbrace{ {\mcL K} \left( \frac{64}{\sqrt{\rho {\varepsilon}}}\right)^q T^{1-\frac{q}{2}} \Big(\sqrt{2\e\log T} + 32\Big)^q\Bigg] }_{B_3(T)} \ . 
    \end{align}
\(B_3(T)\) is the dominating term for any $T>\e^K$. Hence, we can simplify the upper bound as follows.
    \begin{align}
    \label{eq:theorem_lastline}
 & \mathfrak{R}_{\bnu}^{\rm U}(T) \leq \delta_{01}(\varepsilon) + (B_h + 2\mcL K(W^q+1))\Big({64\varepsilon^{-\frac{1}{2}}} \rho^{-\frac{1}{2}} T^{-\frac{1}{2}}\Big( \sqrt{2\e\log T} + 32\Big)\Big)^q \ .
\end{align}

\subsection{Proof of Theorem~\ref{corollary:PM-UCB-M} for PM-UCB-M}
\label{Appendix:UCB_theorem_epsilon_proof}
Let us set 
\begin{align}
\varepsilon = \Theta \left(\left(K^{\frac{2}{q}}\log T / T\right)^{\kappa}\right)\ ,
\end{align}
 where $\kappa = \frac{1}{\frac{2\beta}{q}+2}$.  
Leveraging \(\delta_{12}(\varepsilon) = \Omega(\varepsilon^{\beta})\ \), it can be readily verified that this choice of the discretization level $\varepsilon$ satisfies the condition in~\eqref{eq:T_epsilon_actual}. Hence, for arm selection regret, from~\eqref{eq:Discretized_regret} we have
\begin{align}
\label{eq:Discrete_order_UCB_RS}
    \Bar{\mathfrak{R}}_{\bnu,1}^{\rm U}(T) &= O \left ((K^{\frac{2}{q}}\log T/T)^{\left(1-\kappa\right)\frac{q}{2}} \right) 
    \ .
\end{align}
Furthermore, from Lemma~\eqref{lemma:Delta_error}, we have
\begin{align}
    \delta_{01}(\varepsilon) &\leq \mcL (KW\varepsilon)^q  \ .
\end{align}
Hence, \(\delta_{02}(\varepsilon) = O(\varepsilon^q)\). Consequently, we have
\begin{align}
\label{eq:delta_order_UCB_RS}
    \delta_{02}(\varepsilon) = O\left(K^{q \left( 1 + \frac{2\kappa}{q} \right)}\left(\log T / T\right)^{q \kappa} \right).
\end{align}
From~\eqref{eq:Discrete_order_UCB_RS} and \eqref{eq:delta_order_UCB_RS} the regret of the PM-UCB-M is bounded from above by
    \begin{align}
    \label{eq:UCB upper bound appendix}
        \mathfrak{R}_{\bnu}^{\rm U}(T) &= O \left (\max\Big\{ K^{q \left( 1 + \frac{2\kappa}{q} \right)}\left(\log T / T\right)^{q \kappa}  \;,\;(K^{\frac{2}{q}}\log T/T)^{\left(1-\kappa\right)\frac{q}{2}} \Big\} \right) \\
        & = O \left (\max\Big\{K^{q \left( 1 + \frac{2\kappa}{q} \right)}, K^{1-\kappa} \Big\} \max\Big\{ \left(\log T / T\right)^{q \kappa}  \;,\;  (\log T/T)^{(1-\kappa)\frac{q}{2}}\Big\} \right)\ \ .
    \end{align}
When $\beta$ exists and $\beta \geq \frac{q}{2}$, this bound becomes
\begin{align}
    \mathfrak{R}_{\bnu}^{\rm U}(T) \leq O \left (\max\Big\{K^{q \left( 1 + \frac{2\kappa}{q} \right)}, K^{1-\kappa} \Big\}   (\log T/T)^{q\kappa} \right)\ .
\end{align}

\section{CE-UCB-M Algorithm}
In this section, we provide the proofs of Theorems \ref{theorem:UCB upper bound} and \ref{corollary:PM-UCB-M} for CE-UCB-M.
\label{Appendix:Upper_last_UCB}

\subsection{Proof of Theorem~\ref{theorem:UCB upper bound} for CE-UCB-M}
\label{Appendix:CE-UCB-M_algorithm}
The analysis of the CE-UCB-M algorithm closely follows that of the PM-UCB-M algorithm with some minute differences. We will briefly state the steps in the PM-UCB-M analysis, and in the process, highlight the key distinctions in the analysis of the CE-UCB-M algorithm. Henceforth, we will use $\pi={\rm C}$ to denote the policy CE-UCB-M. 

Similarly to the PM-UCB-M analysis and~\eqref{eq:regret_decomposition}, we decompose the regret into a estimation regret component $\delta_{01}(\varepsilon)$, and the arm selection regret $\bar{\mathfrak{R}}_{\bnu,1}^{\rm C}(T)$. Leveraging Lemma~\ref{lemma:Delta_error}, it may be readily verified that
\begin{align}
    \delta_{01}(\varepsilon) \leq \mcL (KW\varepsilon)^q\ .
\end{align}
Next, we digress to upper bounding the arm selection regret $\bar\R_{\bnu}^{\rm C}(T)$. We have
\begin{align}
    \Bar{\mathfrak{R}}_{\bnu,1}^{\rm C}(T) &=   U_h\left(\sum_{i\in[K]} a^{(1)}(i)\F_i\right)  - \E_{\bnu}^{\rm C} \left[U_h\left(\sum_{i\in[K]} \frac{\tau_t^{\rm C}(i)}{T}\F_i\right)\right] \\
    &\leq \underbrace{\sum_{\mcS \subseteq [K]: \mcS \neq \emptyset}
    \E_{\bnu}^{\rm C} \left[U_h\left(\sum_{i\in[K]} a^{(1)}(i)\F_i\right)  -  U_h\left(\sum_{i\in[K]} \frac{\tau_t^{\rm C}(i)}{T}\F_i\right)\mathds{1}\{\F_i \notin \mcC_T(i): i \in \mcS\}\right]}_{\triangleq C_1(T)} \\
    &\qquad + \underbrace{\E_{\bnu}^{\rm C} \left[ U_h\left(\sum_{i\in[K]} a^{(1)}(i)\F_i\right)  - U_h\left(\sum_{i\in[K]} \frac{\tau_t^{\rm C}(i)}{T}\F_i\right) \mathds{1}\{\F_i \in \mcC_T(i) \;\forall i \in [K]\} \right]}_{\triangleq C_2(T)}\ ,
\end{align}
where $\mcC_t(i)$ has the same definition as in~\eqref{eq:UCB_confidence_sets} for every $i\in[K]$. Furthermore, we have already upper bounded the term $A_1(T)$ in Appendix~\ref{proof:UCB upper bound} (equation ~\eqref{eq:A_1 bound refer}), which is given by
\begin{align}
    C_1(T)\;\leq\; B_h\left(\left(\frac{1}{T^2}+1\right)^K-1\right)\ .
\end{align}
We now resort to upper-bounding the term $C_2(T)$. Note that
\begin{align}
    C_2(T)\;&=\; \E_{\bnu}^{\rm C} \left[ \left(U_h\left(\sum_{i\in[K]} a^{(1)}(i)\F_i\right)  - U_h\left(\sum_{i\in[K]} \frac{\tau_t^{\rm C}(i)}{T}\F_i\right) \right)\mathds{1}\{\F_i \in \mcC_T(i) \;\forall i \in [K]\} \right]\\
    &\leq\; \E_{\bnu}^{\rm C} \left[ U_h\left(\sum_{i\in[K]} a^{(1)}(i)\F_i\right)  - U_h\left(\sum_{i\in[K]} \frac{\tau_t^{\rm C}(i)}{T}\F_i\right) \;\Big\lvert\; \F_i \in \mcC_T(i) \;\forall i \in [K]\right]\\
    \label{eq:CEUCB1}
    &\leq \E_{\bnu}^{\rm C} \Bigg[ U_h\left(\sum_{i\in[K]} a^{(1)}(i)\F_{i,T}^{\rm C}\right)  - U_h\left(\sum_{i\in[K]} \frac{\tau_t^{\rm C}(i)}{T}\F_i\right) \nonumber\\
    &\qquad\qquad\qquad +\mcL\sum\limits_{i\in[K]}a^{(1)}(i)^q \left(16\ \frac{\sqrt{2 {\rm e} \log T } + 32}{\sqrt{\tau^{\rm C}_t(i)}}  \right)^{q}\;\Big\lvert\; \F_i \in \mcC_T(i) \;\forall i \in [K]\Bigg]\\
    \label{eq:CEUCB2}
    &\leq \E_{\bnu}^{\rm C}\Bigg[ U_h\left(\sum_{i\in[K]} a^{\rm C}_T(i)\F_{i,T}^{\rm C}\right)  - U_h\left(\sum_{i\in[K]} \frac{\tau_t^{\rm C}(i)}{T}\F_i\right) \nonumber\\
    &\qquad\qquad\qquad +\mcL\sum\limits_{i\in[K]}a^{\rm C}_T(i)^q\left(16\ \frac{\sqrt{2 {\rm e} \log T } + 32}{\sqrt{\tau^{\rm C}_t(i)}}  \right)^{q}\;\Big\lvert\; \F_i \in \mcC_T(i) \;\forall i \in [K]\Bigg]\\
    &\leq \underbrace{\E_{\bnu}^{\rm C} \left[ U_h\left(\sum_{i\in[K]} a^{\rm C}_T(i)\F_{i,T}^{\rm C}\right)  - U_h\left(\sum_{i\in[K]} \frac{\tau_t^{\rm C}(i)}{T}\F_i\right) \;\Big\lvert\; \F_i \in \mcC_T(i) \;\forall i \in [K]\right]}_{\triangleq \;C_3(T)}\nonumber\\
    &\qquad\qquad\qquad +\mcL K\left(32\ \frac{\sqrt{2 {\rm e} \log T } + 32}{\sqrt{\rho T {\varepsilon}}}  \right)^{q}\ ,
    \label{eq:CEUCB3}
\end{align}
where~\eqref{eq:CEUCB1} follows from \holder~ defined in Definition~\ref{assumption:Holder} and the conditioning on the fact that $\F_i\in\mcC_T(i)$ for every $i\in[K]$, \eqref{eq:CEUCB2} follows from the upper confidence bound in~\eqref{eq:UCB_alpha2}, and~\eqref{eq:CEUCB3} follows from the explicit exploration phase of the PM-UCB-M algorithm. \newline
Furthermore, note that $C_3(T)$ can be expanded as
\begin{align}
    C_3(T)\;&=\; \E_{\bnu}^{\rm C} \Bigg[ U_h\left(\sum_{i\in[K]} a^{\rm C}_T(i)\F_{i,T}^{\rm C}\right)  - U_h\left(\sum\limits_{i\in[K]} \alpha_T(i)\F_i\right)\nonumber\\
    &\qquad\qquad +U_h\left(\sum\limits_{i\in[K]}a^{\rm C}_T(i)\F_i\right) - U_h\left(\sum_{i\in[K]} \frac{\tau_t^{\rm C}(i)}{T}\F_i\right) \;\Big\lvert\; \F_i \in \mcC_T(i) \;\forall i \in [K]\Bigg]\\
    &\leq \underbrace{\E_{\bnu}^{\rm C} \Bigg[U_h\left(\sum\limits_{i\in[K]} a^{\rm C}_T(i)\F_i\right) - U_h\left(\sum_{i\in[K]} \frac{\tau_t^{\rm C}(i)}{T}\F_i\right) \;\Big\lvert\; \F_i \in \mcC_T(i) \;\forall i \in [K]\Bigg]}_{C_4(T)}+\mcL K\left(32\ \frac{\sqrt{2 {\rm e} \log T } + 32}{\sqrt{\rho T {\varepsilon}}}  \right)^{q}\ .
\end{align}
Finally, note that the term $C_4(T)$ is similar to the term $B_{22}(T)$ in the PM-UCB-M analysis in Appendix~\ref{proof:UCB upper bound}, and can be handled in the exact same way. Recall the finite time instant $T(\varepsilon) = T_0(\varepsilon) + m$, where $m$ is defined in~\eqref{eq:m} and $T_0(\varepsilon)$ is defined in~\eqref{eq:T_epsilon}. For all $T>T(\varepsilon)$, we have the following bound on $C_4(T)$.
\begin{align}
    C_4(T)\;\leq\; 2\mcL K W^q  \left(T\left(\left(\frac{1}{T^2}+1\right)^K-1\right) + \left(\frac{K}{T}\right)^q\right)\ .
\end{align}
Aggregating $C_1(T)$ -- $C_4(T)$, we have that for all $T>T(\varepsilon)$,
\begin{align}
\label{eq:discrete_error_CE}
    \mathfrak{R}_{\bnu}^{\rm C}(T)\;&\leq\; \delta_{01}(\varepsilon) + \frac{1}{T}\Bigg[ \underbrace{2 T^{1-q/2} \mcL K\left({32}\ \frac{\sqrt{2 {\rm e} \log T } + 32}{\sqrt{\rho {\varepsilon}} }  \right)^{q}}_{\triangleq\; C_5(T)} \nonumber + 2\mcL K W^q  \left(T^2\left(\left(\frac{1}{T^2}+1\right)^K-1\right) + K^q T^{1-q} \right) \nonumber\\
    & + B_hT\bigg(\bigg( \frac{1}{T^2}+1\bigg)^K -1\bigg)\Bigg]\ .
\end{align}
Furthermore, similar to the PM-UCB-M analysis, it can be readily verified that $C_5(T)$ is the dominating term for any $T>\e^K$. Hence, we can simplify the upper bound as follows.
\begin{align}
\label{eq:theorem_lastline_CE}
    \mathfrak{R}_{\bnu}^{\rm C}(T)\; &\leq\; \delta_{01}(\varepsilon) + (B_h +  2\mcL K(W^q + 1))\Big(32 \varepsilon^{-\frac{1}{2}} \rho^{-\frac{1}{2}} T^{-\frac{1}{2}}\Big( \sqrt{2\e\log T} + 32\Big)\Big)^q  \\
    &\leq\; \delta_{01}(\varepsilon) + (B_h +  2\mcL K(W^q + 1))\Big(64 \varepsilon^{-\frac{1}{2}} \rho^{-\frac{1}{2}} T^{-\frac{1}{2}}\Big( \sqrt{2\e\log T} + 32\Big)\Big)^q  \ .
\end{align}

\subsection{Proof of Theorem~\ref{corollary:PM-UCB-M} for CE-UCB-M}

Let us set 
\begin{align}
\varepsilon = \Theta \left(\left(K^{\frac{2}{q}}\log T / T\right)^{\kappa}\right)\ ,    
\end{align}
where $\kappa = \frac{1}{\frac{2\beta}{q}+2}$. Leveraging \(\delta_{12}(\varepsilon) = \Omega(\varepsilon^{\beta})\ \), it can be readily verified that this choice of the discretization level $\varepsilon$ satisfies the condition in~\eqref{eq:T_epsilon}. Hence, for arm selection regret, from~\eqref{eq:discrete_error_CE} we have
\begin{align}
\label{eq:Discrete_order_CE}
    \Bar{\mathfrak{R}}_{\bnu,1}^{\rm C}(T) &= O \left ((K^{\frac{2}{q}}\log T/T)^{\left(1-\kappa\right)\frac{q}{2}} \right) 
    \ .
\end{align}
Furthermore, from Lemma~\ref{lemma:Delta_error}, we have
\begin{align}
    \delta_{01}(\varepsilon) &\leq \mcL (KW\varepsilon)^q \ .
\end{align}
Hence, \(\delta_{01}(\varepsilon) = O(\varepsilon^q)\). Consequently, we have
\begin{align}
\label{eq:delta_order_ce}
    \delta_{01}(\varepsilon) = O\left(K^{q \left( 1 + \frac{2\kappa}{q} \right)}\left(\log T / T\right)^{q \kappa} \right).
\end{align}
From~\eqref{eq:Discrete_order_CE} and \eqref{eq:delta_order_ce} the regret of the CE-UCB-M is bounded from above by
    \begin{align}
    \label{eq:UCB upper bound appendix CE}
        \mathfrak{R}_{\bnu}^{\rm C}(T) &= O \left (\max\Big\{ K^{q \left( 1 + \frac{2\kappa}{q} \right)}\left(\log T / T\right)^{q \kappa}  \;,\;(K^{\frac{2}{q}}\log T/T)^{\left(1-\kappa\right)\frac{q}{2}} \Big\} \right) \\
        & = O \left (\max\Big\{K^{q \left( 1 + \frac{2\kappa}{q} \right)}, K^{1-\kappa} \Big\} \max\Big\{ \left(\log T / T\right)^{q \kappa}  \;,\;  (\log T/T)^{(1-\kappa)\frac{q}{2}}\Big\} \right)\ \ .
    \end{align}
As $\beta \geq \frac{q}{2}$ this bound becomes
\begin{align}
    \mathfrak{R}_{\bnu}^{\rm C}(T) \leq O \left (\max\Big\{K^{q\left( 1 + \frac{2\kappa}{q} \right)}, K^{1-\kappa} \Big\}   (\log T/T)^{q\kappa} \right)\ .
\end{align}

\section{CIRT Algorithm}
\label{Appendix:Anytime-alg}
In this section, we provide the proofs for the performance guarantees corresponding to CIRT. This includes the proof of Lemma \ref{lemma:alpha_star_in} (which provides an upper bound on the probability of error), followed by the proof of Lemma~\ref{lemma:finite_stopping_time} (which shows the finiteness of the phase lengths), and finally, the proof of Theorem~\ref{theorem:any_regret}, which upper bounds the regret due to the CIRT algorithm.

\subsection{Proof of Lemma \ref{lemma:alpha_star_in}}
\label{Appendix:proof_of_lemma_alpha_star_in}
The proof of Lemma~\ref{lemma:alpha_star_in} contains two steps. The first step shows that an oracle strategy, which knows the arm distributions and performs adaptive discretization based on the true CDFs, will contain the optimal mixing coefficient $\balpha^{\star}$. The next step shows that the CIRT algorithm mimics the oracle strategy, which proves the claim of Lemma~\ref{lemma:alpha_star_in}. The first step holds for a $K$-arm Bernoulli bandit instance. To show the first step, we begin by specifying the oracle refinement strategy.
\begin{definition}[Refinement Oracle]
    A refinement oracle (${\rm RO}$) is defined as an algorithm that knows the arms' CDFs $\{\F_i : i\in[K]\}$. Starting from $\Delta_{{\rm RO}}^{(0)}=\Delta^{K-1}$, in each phase $\ell\in[L]$, ${\rm RO}$ chooses the best discrete mixing coefficient, which we denote by
    \begin{align}
    \ba_{\rm RO}^{(\ell)}\;\in\;\argmax\limits_{\ba\in\Delta^{(\ell)}_{\rm RO}} U\left(\sum\limits_{i\in[K]} \alpha(i)\F_i\right)\ ,
    \end{align}
    and performs the interval refinement for the next phase $\ell+1$ based on $\ba_{\rm RO}^{(\ell)}$, i.e.,
    \begin{align}
\label{eq:oracle_discrete_set}
\nonumber
    \Delta^{(\ell + 1)}_{\rm RO} \triangleq \bigg\{& \ba : 
     \;\;n \in \{\N \cup {0}\}, \;  a(i) \triangleq n \frac{2^{\ell}}{A^{\ell+1}}, \; a(i) \in \left[a_{\rm RO}^{(\ell)}(i) - \frac{2^{\ell-1}}{A^{\ell}}, a_{\rm RO}^{(\ell)}(i) + \frac{2^{\ell-1}}{A^{\ell}}\right]\;,\; \mathbf{1}^T\ba = 1\bigg\} \ .
\end{align}
\end{definition}
The following result shows that for a $K$-armed Bernoulli bandit, at the end of phase $L$, the optimal mixing coefficient is contained in the refined interval due to \rm{RO}.

\begin{lemma}
\label{lemma: best_alpha_Bernoulli_solitary}
    Consider a $K$-armed Bernoulli bandit instance $\bnu\triangleq\{{\rm Bern}(p(1)),\cdots,{\rm Bern}(p(K))\}$, and let $\bp \triangleq [p(1),\cdots,p(K)]^\top$ denote the vector of mean values. For a PM with a solitary arm as the optimal solution, when \rm{RO} performs interval refinement for $L$ phases on $\bnu$, we have $\balpha^{\star}_{\bnu}\in\Sigma^{(L)}_{\rm{RO}}$.
\end{lemma}

\begin{proof}
We will prove the claim by induction on the phase $\ell$. Let $i^\star$ denote the optimal arm, i.e.,
\begin{align}
    i^\star\;\in\; \argmax_{i\in[K]}\;U_h(\F_i)\ .
\end{align}
\textbf{Base case ($\ell = 0$):} As the base case, it can be readily verified that $\ba^{(1)}_{\rm{RO}}\triangleq \be_{i^\star}$, where we let $\be_i$ denote the unit vector with an $i^{\rm th}$ non-zero coordinate. Hence, we conclude that $\ba^{(1)}_{\rm{RO}}\in\Delta_{\rm{RO}}^{(1)}$.
\newline\textbf{Inductive hypothesis:} For any $\ell < L$, let us assume that $\be_{i^\star}\in\Delta_{\rm RO}^{(\ell)}$. 
\newline\textbf{Induction step:} Owing to the interval refinement step defined in~\eqref{eq:oracle_discrete_set}, it can be readily verified that $\be_{i^\star}\in \Delta_{\rm{RO}}^{(\ell+1)}$. Hence, we can conclude that $\be_{i^\star}\in\Delta_{\rm RO}^{(L)}$.

\end{proof}

\begin{lemma}
\label{lemma: best_alpha_Bernoulli_mixture}
    Consider a $K$-armed Bernoulli bandit instance with mean vector  $\bp \triangleq [p(1),\cdots,p(K)]^\top$. For a PM with concave distortion function which has a mixture for its optimal solution, when \rm{RO} performs interval refinement for $L$ phases on $\bnu$, we have $\balpha^{\star}_{\bnu}\in\Delta^{(L)}_{\rm{RO}}$.
\end{lemma}
We will prove that, given the optimal mixing coefficient is contained in the interval refined by the \rm{RO} in the current phase $\ell$ (denoted by $\Delta_{\rm{RO}}^{(\ell)}$), it would also be contained in the refined interval of the next phase, i.e., $\Delta_{\rm{RO}}^{(\ell+1)}$. Subsequently, following the same line of arguments as Lemma~\ref{lemma: best_alpha_Bernoulli_solitary}, using induction, the statement of the lemma readily follows. Let us begin by assuming that $\balpha^{\star}_{\bnu}\in\Delta^{(\ell)}_{\rm{RO}}$.  
For a $K$-arm Bernoulli bandit instance $\bnu$ with mean values $\bp = [p(1), \dots , p(K)]$, and for any mixing coefficient $\ba\in\Delta^{K-1}$ we have the following relationship:
\begin{align}
\label{eq:bernoulli_U_H}
    U_h\left(\sum_{i \in [K]} a(i)\F_i \right) = h\left(\sum_{i \in [K]} a(i)p(i) \right)\ .
\end{align}
Without loss of generality, let us assume that
\begin{align}
    \sum_{i \in [K]} \alpha^{\star}(i)p(i)  \leq \sum_{i \in [K]} a^{(\ell)}_{\rm{RO}}(i)p(i) \ .
\end{align}
Now, we will show that for any $\ba \in \Delta_{\rm{RO}}^{(\ell)}$ such that \begin{align}
    \sum_{i \in [K]} a(i)p(i) < \sum_{i \in [K]} a^{(\ell)}_{\rm{RO}}(i)p(i)\ ,
\end{align} due to the concavity of the distortion function, we have 
\begin{align}
\sum_{i \in [K]} a(i)p(i) < \sum_{i \in [K]} \alpha^{\star}(i)p(i) \ .\end{align}
We will prove this claim by contradiction. Let us assume that $\sum_{i \in [K]} a(i)p(i) \geq \sum_{i \in [K]} \alpha^{\star}(i)p(i)$. For any $\lambda \in (0,1)$ we have 
\begin{align}
\label{eq:concavity_alpha_a_ell}
h\left(\sum_{i \in [K]} a(i)p(i)\right) &\geq \lambda h\left(\sum_{i \in [K]} \alpha^\star(i)p(i)\right) + (1-\lambda)h\left(\sum_{i \in [K]} a^{(\ell)}_{\rm{RO}}(i)p(i)\right) \\
&> \min\left\{ h\left(\sum_{i \in [K]} \alpha^\star(i)p(i)\right) \;,\; h\left(\sum_{i \in [K]} a^{(\ell)}_{\rm{RO}}(i)p(i)\right) \right\} \\
\label{eq:alpha_great_a_ell}
&=h\left(\sum_{i \in [K]} a^{(\ell)}_{\rm{RO}}(i)p(i)\right) \ ,
\end{align}
where \eqref{eq:concavity_alpha_a_ell} follows from the concavity of $h$, and,
\eqref{eq:alpha_great_a_ell} follows from the definition of $\balpha^{\star}$. \newline
It can be readily verified that~\eqref{eq:alpha_great_a_ell} contradicts the definition of $\ba^{(\ell)}_{\rm{RO}}$. Therefore, we can conclude that $\sum_{i \in [K]}  a(i)p(i) \leq \sum_{i \in [K]}  \alpha^{\star}(i)p(i)$. Consequently, we we can express $\balpha^\star$ as a convex combination of any such discrete mixing coefficient $\ba$, and the best discrete mixing coefficient $\ba_{\rm{RO}}^{(\ell)}$ for $\lambda \in [0,1)$, i.e.,
\begin{align}
\label{eq:any_mix_alpha}
    \balpha^{\star} = \lambda  \ba^{(\ell)}_{\rm{RO}} + (1-\lambda) \ba\ .
\end{align}
We choose
\begin{align}
\label{eq:any_alternative_alpha_def}
\ba =
    \begin{cases}
     a^{(\ell)}_{\rm{RO}}(i) - 1/A^{\ell}, &i = m\ ,\\
     a^{(\ell)}_{\rm{RO}}(i) + 1/A^{\ell} , &i = n, p(m) > p(n)\ ,\\
     a^{(\ell)}_{\rm{RO}}(i), &i \neq {m, n} \ .
    \end{cases}
\end{align}
This choice ensures that 
\begin{align}
    \sum_{i \in [K]} a(i)p(i) < \sum_{i \in [K]} a^{(\ell)}_{\rm{RO}}(i)p(i)\ .
\end{align}
Hence, we have
\begin{align}
\balpha^{\star} =
    \begin{cases}
     a^{(\ell)}_{\rm{RO}}(i) - (1-\lambda)1/A^{\ell}, &i = m\ ,\\
     a^{(\ell)}_{\rm{RO}}(i) + (1-\lambda)1/A^{\ell} , &i = n, p(m) > p(n)\ ,\\
     a^{(\ell)}_{\rm{RO}}(i), &i \neq {m, n} \ .
    \end{cases}
\end{align}
This proves that $\balpha^\star\in\Delta^{(\ell+1)}_{\rm{RO}}$, based on the interval refinement stated in~\eqref{eq:oracle_discrete_set}. This completes our proof.

\begin{lemma}
\label{lemma:CIRT_error_probability_each_phase}
    For a $K$-arm Bernoulli bandit instance with mean values $\bp = [p(1), \cdots , p(K)]$, when $T > \tau^{(L)}$ the CIRT algorithm results in an optimal mixing coefficient being contained in $\Sigma^{K-1}_{\epsilon}$ with a high probability, i.e.,
    \begin{align}
        \P_{\bnu}^{\rm IR}\Big(\balpha^{\star} \notin \Sigma_{\epsilon}^{K-1}\Big) \;\leq\;  L \delta\ .
    \end{align}
\end{lemma}

\begin{proof}
From the Lemmas~\ref{lemma: best_alpha_Bernoulli_solitary} and~\ref{lemma: best_alpha_Bernoulli_mixture}, we have obtained that
\begin{align}
\label{eq: probability of detection}
    \P_{\bnu}^{\rm IR}\left(\balpha^{\star} \in \Sigma_{\epsilon}^{K-1}\right) \geq \P_{\bnu}^{\rm IR}\left(\bigcap_{\ell=1}^{L}\ba^{(\ell)} = \ba^{(\ell)}_{\rm{RO}} \right) \ .
\end{align}
Here, the inequality results from the fact that choosing $a^{(\ell)}_{\rm{RO}}$ to zoom into the simplex at each stage $\ell\in[L]$ results in {\em one} of the optimal solutions $\balpha^\star$ to be contained in the last simplex $\Delta_{\rm{RO}}^{(L)}$, and there could be possibly more optimal mixing coefficients. Let us define the event $S \triangleq \bigcap_{\ell=1}^{L}\{\ba^{(\ell)} = \ba^{(\ell)}_{\rm{RO}}\}$. It can be readily verified that
\begin{align}
\label{eq:S bar}
    \bar S \;=\; \bigcup_{\ell \in [L]} \left\{\Big\{\ba^{(\ell)} \neq \ba^{(\ell)}_{\rm{RO}}\Big\} \bigcap_{\ell^\prime \in [\ell]}\Big\{\ba^{(\ell^\prime)} = \ba^{(\ell^\prime)}_{\rm{RO}}\Big\} \right\}  \ .
\end{align}
Next, we will provide an upper bound on the probability of error, i.e., $\P_{\bnu}^{\rm IR}(\balpha^\star \notin \Sigma_{\epsilon}^{K-1})$. Note that leveraging~\eqref{eq: probability of detection} and~\eqref{eq:S bar}, it can be readily verified that
\begin{align}
    \P_{\bnu}^{\rm IR}\left(\balpha^{\star} \notin \Sigma_{L}^{K-1}\right) &\leq \P_{\bnu}^{\rm IR}\left(\bigcup_{\ell \in [L]} \left\{\ba^{(\ell)} \neq \ba^{(\ell^\prime)}_{\rm{RO}}\} \bigcap_{\ell^\prime \in [\ell]}\{\ba^{(\ell^\prime)} = \ba^{(\ell^\prime)}_{\rm{RO}}\} \right\}  \right) \\
    &\leq \sum_{\ell \in [L]} \P_{\bnu}^{\rm IR}\left(\ba^{(\ell)} \neq \ba^{(\ell)}_{\rm{RO}} \bigcap_{\ell^\prime \in [\ell]}\{\ba^{(\ell^\prime)} = \ba^{(\ell^\prime)}_{\rm{RO}}\}  \right) \ .
    %\\&= \sum_{\ell \in [L]} \P\left(\ba^{(\ell)} \neq \ba^{(1)}_\ell\right)
    \label{eq: error_prob_1}
\end{align}
For any time $t\in\N$, let us define the event
\begin{align}
    \mathcal{E}_{0,t} \triangleq \Big\{\forall i \in [K], \F_i \in C_t(i) \Big\}\ .
\end{align}
By the design of the CIRT confidence sets in~\eqref{eq:any_UCB_confidence_sets}, it can be readily verified that
\begin{align}
    \P_{\bnu}^{\rm IR}\left(\bar{\mcE}_{0,t} \right)\;\leq\; \delta\ .
    \label{eq: error_prob_2}
\end{align}
We will now bound each probability term under the summand in~\eqref{eq: error_prob_1}. We begin with the first phase, i.e., $\ell=1$. Note that
\begin{align}
    \P_{\bnu}^{\rm IR}\left(\ba^{(1)} \neq  \ba^{(1)}_{\rm{RO}}\right)
    &\leq \P_{\bnu}^{\rm IR}\left({\rm LCB}_{\tau^{(1)}}(\ba^{(1)}) > {\rm UCB}_{\tau^{(1)}}( \ba^{(1)}_{\rm{RO}}) \right)  \\
    &= \P_{\bnu}^{\rm IR}\left({\rm LCB}_{\tau^{(1)}}(\ba^{(1)}) > {\rm UCB}_{\tau^{(1)}}( \ba^{(1)}_{\rm{RO}}) \med \mathcal{E}_{0, \tau^{(1)}}  \right) \P_{\bnu}^{\rm IR}\Big(\mathcal{E}_{0, \tau^{(1)}}\Big)  \\
    & \quad + \P_{\bnu}^{\rm IR}\left({\rm LCB}_{\tau^{(1)}}(\ba^{(1)}) > {\rm UCB}_{\tau^{(1)}}( \ba^{(1)}_{\rm{RO}}) \med \bar{\mathcal{E}}_{0, \tau^{(1)}} \right) \P_{\bnu}^{\rm IR}\Big(\bar{\mathcal{E}}_{0, \tau^{(1)}}\Big) \\
    &\leq \P_{\bnu}^{\rm IR}\left({\rm LCB}_{\tau^{(1)}}(\ba^{(1)}) > {\rm UCB}_{\tau^{(1)}}( \ba^{(1)}_{\rm{RO}}) \med \mathcal{E}_{0, \tau^{(1)}}  \right)  + \P_{\bnu}^{\rm IR}(\bar{\mathcal{E}}_{0, \tau^{(1)}}) \\
    &\leq \underbrace{\P_{\bnu}^{\rm IR}\left( U_h \left(\sum_{i \in [K]} a^{(1)}(i) \F_i \right)> U_h \left(\sum_{i \in [K]} a^{(1)}_{\rm{RO}}(i) \F_i \med \mcE_{0,\tau^{(1)}}\right)\right)}_{=\;0}  + \P(\bar{\mathcal{E}}_{0, \tau^{(1)}}) \\
    &\stackrel{\eqref{eq: error_prob_2}}{\leq} \delta \ .
\end{align}
Subsequently, for any phase $\ell\geq 2$, we have
\begin{align}
    \P_{\bnu}^{\rm IR}&\left(\Big\{\ba^{(\ell)} \neq \ba^{(\ell)}_{\rm{RO}} \Big\}\bigcap_{\ell^\prime \in [\ell]} \Big\{\ba^{(\ell^\prime)} = \ba^{(\ell^\prime)}_{\rm{RO}}\Big\} \right) 
    \nonumber\\
    \nonumber
    &= \P_{\bnu}^{\rm IR}\left(\Big\{\ba^{(\ell)} \neq \ba^{(\ell)}_{\rm{RO}} \Big\}\bigcap_{\ell^\prime \in [\ell]} \Big\{\ba^{(\ell^\prime)} = \ba^{(\ell^\prime)}_{\rm{RO}}\Big\}  \med \mathcal{E}_{0, \tau^{(\ell)}}   \right) \P(\mathcal{E}_{0, \tau^{(\ell)} }  ) \\ & \quad \quad + \P\left(\Big\{\ba^{(\ell)} \neq \ba^{(\ell)}_{\rm{RO}}\Big\} \bigcap_{\ell^\prime \in [\ell]}\Big\{ \ba^{(\ell^\prime)} = \ba^{(\ell^\prime)}_{\rm{RO}}\Big\} \med \bar{\mathcal{E}}_{0, \tau^{(\ell)}}    \right) \P_{\bnu}^{\rm IR}( \bar{\mathcal{E}}_{0, \tau^{(\ell)}}   )\\ 
    &= \P_{\bnu}^{\rm IR}\left(\Big\{\ba^{(\ell)} \neq \ba^{(\ell)}_{\rm{RO}} \Big\}\bigcap_{\ell^\prime \in [\ell]} \Big\{\ba^{(\ell^\prime)} = \ba^{(\ell^\prime)}_{\rm{RO}}\Big\}  \med \mathcal{E}_{0, \tau^{(\ell)}}   \right) + \P_{\bnu}^{\rm IR}( \bar{\mathcal{E}}_{0, \tau^{(\ell)}}   )\\ 
    &\leq \P_{\bnu}^{\rm IR}\left(\Big\{\ba^{(\ell)} \neq \ba^{(\ell)}_{\rm{RO}} \Big\}\bigcap_{\ell^\prime \in [\ell]} \Big\{\ba^{(\ell^\prime)} = \ba^{(\ell^\prime)}_{\rm{RO}} \Big\} \med \mathcal{E}_{0, \tau^{(\ell)}}   \right) + \delta\\ 
    \nonumber
    &= \P_{\bnu}^{\rm IR}\left(\Big\{\ba^{(\ell)} \neq \ba^{(\ell)}_{\rm{RO}}  \Big\}\med  \bigcap_{\ell^\prime \in [\ell]}\Big\{ \ba^{(\ell^\prime)} = \ba^{(\ell^\prime)}_{\rm{RO}}\Big\}, \mathcal{E}_{0, \tau^{(\ell)}}   \right) \P_{\bnu}^{\rm IR}\left( \bigcap_{\ell^\prime \in [\ell]} \Big\{\ba^{(\ell^\prime)} = \ba^{(\ell^\prime)}_{\rm{RO}}\Big\} \med \mathcal{E}_{0, \tau^{(\ell)}}  \right) \\ & \quad + \delta\\ 
    \label{eq:any_ro_UP_1}
    &\leq \P_{\bnu}^{\rm IR}\left(\ba^{(\ell)} \neq \ba^{(\ell)}_{\rm{RO}}  \med\bigcap_{\ell^\prime \in [\ell]} \Big\{\ba^{(\ell^\prime)} = \ba^{(\ell^\prime)}_{\rm{RO}}\Big\}, \mathcal{E}_{0, \tau^{(\ell)}}   \right)  + \delta\\ 
    &\leq \P_{\bnu}^{\rm IR}\left(\text{LCB}_{\tau^{(\ell)}}(\ba^{(\ell)}) > \text{UCB}_{\tau^{(\ell)}}( \ba^{(\ell)}_{\rm{RO}}) \mid \bigcap_{\ell^\prime \in [\ell]} \Big\{\ba^{(\ell^\prime)} = \ba^{(\ell^\prime)}_{\rm{RO}}\Big\}, \mathcal{E}_{0, \tau^{(\ell)}}   \right) +\delta \\
    &\leq \P_{\bnu}^{\rm IR}\left( U_h \left(\sum_{i \in [K]} a^{(\ell)}(i) \F_i \right)> U_h \left(\sum_{i \in [K]} a^{(\ell)}_{\rm{RO}}(i) \F_i \right) \med \bigcap_{\ell^\prime \in [\ell]} \Big\{\ba^{(\ell^\prime)} = \ba^{(\ell^\prime)}_{\rm{RO}}\Big\}, \mathcal{E}_{0, \tau^{(\ell)}}\right) + \delta \\
    &= 0 + \delta = \delta
\end{align}
where \eqref{eq:any_ro_UP_1} follows from $\P_{\bnu}^{\rm IR}\left( \bigcap_{\ell^\prime \in [\ell]} \Big\{\ba^{(\ell^\prime)} = \ba^{(\ell^\prime)}_{\rm{RO}}\Big\} \med \mathcal{E}_{0, \tau^{(\ell)}}  \right) \leq 1$.
Revisiting~\eqref{eq: error_prob_1}, we may now conclude that
\begin{align}
    \P_{\bnu}^{\rm IR}\Big(\balpha^\star\notin\Sigma_{L+1}^{K-1} \Big)\;=\;\P_{\bnu}^{\rm IR}\Big(\balpha^\star\notin\Sigma_{\epsilon}^{K-1} \Big)\;\leq\;L\delta\ .
\end{align}

\end{proof}

\subsection{Proof of Lemma \ref{lemma:finite_stopping_time}}
\label{Appendix:proof_of_finite_stopping_time}
In this section, we show that on average, the stopping times $\tau^{(\ell)}$ for $\ell \in [L]$ are finite. Recalling our assumption that there exists a {\em unique} optimal discrete mixing coefficient, for any set of distributions $\{\mu_i\in\mcP(\Omega):i\in[K]\}$ and for any phase $\ell\in[L]$ let us define
\begin{align}
    \delta_{13} (\mu, A^{-\ell})\triangleq \min_{\bb \in \Delta^{K-1}_{\ell}: \bb \neq \ba} \left(\max_{\ba \in \Delta^{K-1}_\ell} U_h\left( \sum_{i \in [K]} a(i)\mu_i \right) - U_h\left( \sum_{i \in [K]} b(i)\mu_i \right) \right) \ ,
\end{align}
where $\delta_{13} (\mu, A^{-\ell})$ denotes the {\em instance-specific} minimum sub-optimality gap for bandit instance $\mu$ sampled from the space of $1$-sub-Gaussian distributions with finite first moment at phase $\ell\in[L]$. We assume that this gap is positive for every bandit instance in the class. Accordingly, let us define  
\begin{align}
    M(A^{-\ell})\; \triangleq\; \inf_{\mu} \delta_{13} (\mu, A^{-\ell})\ ,
\end{align}
where $M (A^{-\ell})$ quantifies the {\em maximum hardness} of separating the optimal mixing coefficient from the sub-optimal ones for {\em any} bandit instance in phase $\ell\in[L]$.
Our stopping rule is defined with in terms of the upper and lower confidence bounds ${\rm UCB}_t(\ba)$ and ${\rm LCB}_t(\ba)$, which have been defined in~\eqref{eq: CIRT_UCB_LCB} for any mixture $\ba\in\Delta^{K-1}$. Let us denote measures achieving the upper and lower confidence bounds by
\begin{align}
\label{eq: CIRT_UCB_LCB_CDF}
    \{\widetilde\F_{1,t},\cdots,\widetilde\F_{K,t} \}\;&\in\; \argmax_{\eta_i \in C_t(i), \forall i \in [K]} U_h\left( \sum\limits_{i\in[K]} a(i) \eta_i \right) ,\\
\text{and}\quad
    \{\F^{\prime}_{1,t},\cdots,\F^{\prime}_{K,t} \}\;&\in\; \argmin_{\eta_i \in C_t(i), \forall i \in [K]} U_h\left( \sum\limits_{i\in[K]} a(i) \eta_i \right)  \ ,
\end{align}
where we have dropped the dependence on the mixing coefficient $\ba$ for notational convenience. Furthermore, note that the utility gap between the estimates $\{\F^{\rm IR}_{i,t}:i\in[K]\}$ and the UCB and LCB estimates can be upper bounded as
\begin{align}
\label{eq:anytime_proof_1}
    U_h\left(\sum_{i \in [K]} a(i)\F^{\rm IR}_{i,t} \right) - \text{LCB}_t(a) &\leq \mcL  \left \|\sum_{i \in [K]} a(i) \left(\F^{\rm IR}_{i,t}-\F^\prime_{i,t}\right) \right \|^q \\
    \label{eq:anytime_proof_2}
    &\leq \mcL \sum_{i \in [K]} (a(i))^q \left \| \F^{\rm IR}_{i,t}-\F^\prime_{i,t} \right \|^q \\ 
    \label{eq:anytime_proof_3}
    &\leq \underbrace{\mcL \sum_{i \in [K]} (a(i))^q (C_t(i))^q}_{\triangleq\; r_t(\ba)}\ ,
\end{align}
where \eqref{eq:anytime_proof_1} follows from Assumption~\ref{assumption:Holder},~\eqref{eq:anytime_proof_2} follows from triangle inequality, and, 
\eqref{eq:anytime_proof_3} follows from the constraints in~\eqref{eq: CIRT_UCB_LCB}.
Similarly, one can obtain
\begin{align}
    \text{UCB}_t(a) - U_h\left(\sum_{i \in [K]} a(i)\F^{\rm IR}_{i,t} \right)  &\leq \mcL  \left \|\sum_{i \in [K]} a(i) \left(\F^{\rm IR}_{i,t}-\tilde\F_{i,t}\right) \right \|^q \\
    &\leq \mcL \sum_{i \in [K]} (a(i))^q \left \| \F^{\rm IR}_{i,t}-\tilde\F_{i,t} \right \|^q \\ 
    &\leq \mcL \sum_{i \in [K]} (a(i))^q (C_t(i))^q \\
    &= r_t(\ba) \ .
    \label{eq:anytime_proof_3b}
\end{align} 
At any time instant $t$, let us denote the second-best discrete coefficient with respect to the discretization in the $\ell^{\rm th}$ phase by
\begin{align}
    \bb_{t} \triangleq \argmax_{\ba \in \Delta^{K-1}_{\ell}: \bb_{t} \neq \ba_{t}} \text{UCB}_t(\ba)\ .
\end{align}
Based on the definitions of $\bb_t$ and the radii $r_t(\ba)$ defined in~\eqref{eq:anytime_proof_3} for any mixture $\ba\in\Delta^{K-1}$, we next specify a stochastic time instant which we show is finite almost surely. For any phase $\ell\in[L]$ let us define 
\begin{align}
\label{eq:t_1_A}
    T(\ell, A) \;\triangleq\; \sup\;\Big\{ s \in \N : M(A^{-\ell}) \leq  r_{s}(\ba^{\rm IR}_{s}) + r_{s}(\bb_{s})  \Big\}\ . 
\end{align}
The purpose of defining $T(\ell,A)$ is to have a convenient decomposition of the phase length, as will become clear in a few steps. First, note that we almost surely have
\begin{align}
    &r_{t}(\ba^{\rm IR}_{t}) + r_{t}(\bb_{t})\\
    &= \sum_{i \in [K]} (\ba^{\rm IR}_{t})^q (C_t(i))^q + \sum_{i \in [K]} (\bb_{t})^q (C_t(i))^q \\
    &\leq 2 \sum_{i \in [K]} (C_t(i))^q \\
    &\stackrel{\eqref{eq:any_UCB_confidence_sets}}{\leq} 2 \sum_{i \in [K]} \left( 16 \frac{\sqrt{{\rm e} \log{2/\delta_K} } + 32}{\sqrt{\tau^{\rm IR}_t(i)}} \right) \ \\
    \label{eq:MA_rhs}
    &\leq 2 \sum_{i \in [K]} \left( 16 \frac{\sqrt{{\rm e} \log{2/\delta_K} } + 32}{\sqrt{(\frac{t}{K})^{\frac{1}{1+\xi}}-1}} \right) \ ,
\end{align}
where~\eqref{eq:MA_rhs} follows from the explicit exploration phase of the CIRT algorithm combined with Lemma~\ref{lemma:any_explicit_exploration}. From \eqref{eq:MA_rhs}, we observe that the sum of radii $r_{t}(\ba^{\rm IR}_{t}) + r_{t}(\bb_{t})$ is a monotonically decreasing function in $t$ almost surely. This implies that 
\begin{align}
\label{eq:T_l_A_finiteness}
    T(\ell, A)\; < \;\infty\;\; \text{a.s.}\;\;\forall\ell\in[L]\ .
\end{align}
We will show that $\E_{\bnu}^{\rm IR}[\tau^{(1)}]$ is finite. The other phases follow the exact line of arguments. To this end, note the following decomposition of the average phase length.

\begin{align}
    \E_{\bnu}^{\rm IR} \left[\tau^{(1)} \right] &= \sum_{t=1}^{\infty} \P_{\bnu}^{\rm IR}(\tau^{(1)} > t)\\
    &= \sum_{t=1}^{T(1,A)} \P_{\bnu}^{\rm IR}(\tau^{(1)} > t) + \sum_{t=T(1, A)+1}^{\infty} \P_{\bnu}^{\rm IR}(\tau^{(1)} > t)\\
    &\leq T(1, A) + \sum_{t=T(1, A)+1}^{\infty} \P_{\bnu}^{\rm IR}(\tau^{(1)} > t)\ .
    \label{eq:anytime_proof_4}
\end{align}
We have established that $T(1,A)$ is finite almost surely in~\eqref{eq:T_l_A_finiteness}. Hence, we focus on the second term in~\eqref{eq:anytime_proof_4}. We have 
\begin{align}
    &\sum_{t=T(1, A)+1}^{\infty} \P_{\bnu}^{\rm IR}(\tau^{(1)} > t)\nonumber\\
    &= \sum_{t=T(1, A)+1}^{\infty} \sum_{s=t+1}^{\infty} \P_{\bnu}^{\rm IR}(\tau^{(1)}= s) \\ 
    &= \sum_{t=T(1,A)+1}^{\infty} \sum_{s=t+1}^{\infty} \P_{\bnu}^{\rm IR}\Big( {\rm LCB}_s(\ba^{\rm IR}_s) > \max_{\ba \in \Delta^{K-1}_1: \ba \neq \ba^{\rm IR}_s} {\rm UCB}_s(\ba), \\& \qquad \qquad \qquad
    \bigcap_ {u=1}^{s}{\rm LCB}_u(\ba_u) > \max_{\ba \in \Delta^{K-1}_1: \ba \neq \ba^{\rm IR}_u} {\rm UCB}_u(\ba) \Big) \\ 
    &\leq \sum_{t=T(1,A)+1}^{\infty} \sum_{s=t}^{\infty} \P_{\bnu}^{\rm IR}\Big( {\rm LCB}_s(\ba^{\rm IR}_s) \leq \max_{\ba \in \Delta^{K-1}_1: \ba \neq \ba^{\rm IR}_s} {\rm UCB}_s(\ba) \Big)
    \\ 
    \label{eq:def_b_anytime}&\stackrel{\eqref{eq:anytime_proof_3},\eqref{eq:anytime_proof_3b}}{\leq} \sum_{t=T(1,A)+1}^{\infty} \sum_{s=t}^{\infty} \P\left( U_h\left(\sum_{i \in [K]} a^{\rm IR}_s(i)\F_{i,s}^{\rm IR}\right) - r_{s}(\ba^{\rm IR}_s) \leq U_h \left(\sum_{i \in [K]} b_s(i)\F_{i,s}^{\rm IR}\right) + r_{s}(\bb_s) \right)\\ 
    &= \sum_{t=T(1,A)+1}^{\infty} \sum_{s=t}^{\infty} \P_{\bnu}^{\rm IR}\left( U_h\left(\sum_{i \in [K]} a^{\rm IR}_s(i)\F_{i,s}^{\rm IR}\right) - U_h\left(\sum_{i \in [K]} b_s(i)\F_{i,s}^{\rm IR}\right) \leq  r_{s}(\ba_s) + r_{s}(\bb_s) \right) \\
    &\leq \sum_{t=T(1,A)+1}^{\infty} \sum_{s=t}^{\infty} \P_{\bnu}^{\rm IR}\Big( M(A) \leq  r_{s}(a^{\rm IR}_s) + r_{s}(b_s) \Big) \\
    &\leq \sum_{t=T(1,A)+1}^{\infty} \sum_{s=t}^{\infty} 0 \\
    &= 0
\end{align}
where in \eqref{eq:def_b_anytime}, we denote $
\bb_s \in \Delta^{K-1}_1 \setminus \{\ba^{\mathrm{IR}}_s\}$.
Therefore, we conclude that
\begin{align}
    \E \left[\tau^{(1)} \right] \leq T(1, A) + 0  < \infty\ .
\end{align}
The proof concludes by noting that the proof for every other phase $\ell\in\{2,\cdots,L\}$ follows the exact line of arguments as $\ell=1$.

\subsection{Proof of Theorem \ref{theorem:any_regret}}
\label{sec:proof_any_regret}

We begin with the following regret decomposition.
\begin{align}
    R_{\bnu}^{\rm IR}(T) &= U_h\left(\sum_{i \in [K]} \alpha^{\star}(i) \F_i \right) - U_h\left(\sum_{i \in [K]} \frac{\tau_T^{\rm IR}(i)}{T}\F_i \right) \\
    &= \underbrace{U_h\left(\sum_{i \in [K]} \alpha^{\star}(i) \F_i \right) - U_h\left(\sum_{i \in [K]} a^{\rm IR}_T(i) \F_i \right)}_{\triangleq\;\delta_\tau(T)\text{ (estimation regret)}}  + \underbrace{U_h\left(\sum_{i \in [K]} a^{\rm IR}_T(i) \F_i \right)  - U_h\left(\sum_{i \in [K]} \frac{\tau_T^{\rm IR}(i)}{T}\F_i \right)}_{\triangleq\;H_2(T)\text{ (sampling error)}} \ .
\end{align}

\subsubsection{Upper bound on the CIRT estimation regret}
Note that the number of phases $L$ is determined based on the allowable precision $\epsilon$, i.e., $(2/A)^L < \epsilon$. Based on this, the estimation regret of the CIRT algorithm is upper-bounded as follows. Almost surely, we have
\begin{align}
\delta_\tau(T) 
    &=  U_h\left(\sum_{i\in[K]} \alpha^{\star} (i)\F_i\right) - U_h\left(\sum_{i\in[K]} a^{\rm IR}_T(i)\F_i\right) \\
    & \leq \mcL\left\|\sum_{i\in[K]}  \Big(\alpha^{\star}(i)-  a^{\rm IR}_T(i)\Big)  \F_i\right\|_{\rm W}^q \\
    & \leq \mcL  \norm{\balpha^{\star}- \ba^{\rm IR}_T}_1^q W^q \\
    & \leq \mcL K^q \left(\left(\frac{2}{A}\right)^{L}\right)^q  W^q  \\
    \label{eq:anytime_error_1}
    & \leq \mcL \Big(KW\epsilon)^q   \ .
\end{align}

\subsubsection{Upper bound on the Sampling Error}
Almost surely, we have,
\begin{align}
H_2(T) &= U_h\left(\sum_{i \in [K]} a^{\rm IR}_T(i) \F_i \right)  - U_h\left(\sum_{i \in [K]} \frac{\tau_T(i)}{T}\F_i \right) \\
&\stackrel{\eqref{eq:Holder}}{\leq}\;\mcL \norm{\sum\limits_{i\in[K]}\left(a^{\rm IR}_T(i) - \frac{\tau_T^{\rm U}(i)}{T}\right)\F_i}^q_{\rm W} \\
    &\stackrel{\eqref{eq:W}}{\leq}\;\mcL W^q\cdot \norm{\ba_T^{\rm IR} - \frac{1}{T}\btau_T^{\rm U}}_1^q\ .
\end{align}
Next, we will show that for the sampling rule defined in \eqref{eq:any_sampling_rule}, after a finite time instant, an arm is sampled {\em only if its under-sampled}. Note that for any arm $i\in[K]$, and at any time instant $t\in\N$, we may have the following scenarios.
% For an arm $i$ at any time instant, in terms of sampling there are four possible cases, these are
\begin{itemize}
    \item \textbf{Case 1:} arm $i$ may be under-explored and under-sampled, %in which case, based on~\eqref{eq:any_sampling_rule} $i$ should be sampled at $t+1$,
    \item \textbf{Case 2:} arm $i$ may be sufficiently explored but under-sampled,
    \item \textbf{Case 3:} arm $i$ may be under-explored but not under-sampled, and,
    \item \textbf{Case 4:} arm $i$ may be sufficiently explored and not under-sampled.
\end{itemize}
We will show that after a finite time instant, the third case cannot happen. Let us assume that there is an arm that is not under-sampled but under-explored, which implies that $i \in E_t$, or equivalently, 
\begin{align}
\label{eq:case_3_explore}
    \tau_t^{\rm IR} (i) \leq \left\lceil\left(\frac{t}{K} \right)^{\frac{1}{1+\xi}} \right\rceil\  = \left(\frac{t}{K} \right)^{\frac{1}{1+\xi}} + 1 \ .
\end{align}
Furthermore, we have 
\begin{align}
\label{eq:case_3_sample}
    a^{\rm IR}_t (i) < \frac{\tau_t^{\rm IR} (i)}{t} \ .
\end{align}
Combining \eqref{eq:case_3_explore} and \eqref{eq:case_3_sample}, we have 
\begin{align}
     \left(\frac{t}{K} \right)^{\frac{1}{1+\xi}} + 1 > t a^{\rm IR}_t (i)  \ .
\end{align}
Furthermore, each coordinate $i\in[K]$ of the chosen mixing coefficient $\ba_t$ is bounded away from $0$ by construction. After the final phase, we have $a_t(i) > 2^{L-1}/A^{L+1}$ for every $i\in[K]$, and we obtain
\begin{align}
     \left(\frac{t}{K} \right)^{\frac{1}{1+\xi}} + 1 > \frac{2^{L-1}}{A^{L+1}}\cdot t\ ,
\end{align}
which can be equivalently written as
\begin{align}
\label{eq:underexploresample}
     t^{\frac{-\xi}{1+\xi}} K^{\frac{-1}{1+\xi}} + t^{-1} > \frac{2^{L-1}}{A^{L+1}}\ .
\end{align}
Note that in \eqref{eq:underexploresample}, the left-hand-side (LHS) is a decreasing function in $t$ since $\xi > 0$, while the right-hand-side (RHS) is a constant. Therefore, there exists a finite instant $t$ after which~\eqref{eq:underexploresample} does not hold. More specifically, let us define
\begin{align}
    T_{\xi} = \sup\left\{ s\in\N :  s^{\frac{-\xi}{1+\xi}} K^{\frac{-1}{1+\xi}} + s^{-1} > \frac{2^{L-1}}{A^{L+1}}\right \} \ .
\end{align}
It can be readily verified that
$T_{\xi} < \infty $. For $t > T_{\xi}$, an arm is under-explored if and only if it is also under-sampled. 
Next, we will show that the CIRT arm selection routine, using the {\em under-sampling} procedure, eventually converges to an optimal mixture distribution. This is captured in the following lemma where we denote $\ba_t^{\rm IR} = \ba^{\rm IR}$ for $t > \tau^{(\ell)}$.

\begin{lemma}
\label{lemma:any_sampling_bound}
There exists a stochastic time instant $\tau^\epsilon$, $\E[\tau^\epsilon]<+\infty$, such that we have
\begin{align}
    \left| \frac{\tau_t^{\rm IR}(i)}{t} - a^{\rm IR}(i)\right| < \frac{K}{t}  \quad \forall t \geq \tau^\epsilon \ .
\end{align}
\end{lemma}

\begin{proof}
The proof follows the same line of arguments as Lemma \ref{lemma:undersampling} with minor adjustment accounting for the stochastic time $\tau^{(L)}$, the time instant $T_{\xi}$ dependent on the exploration parameter $\xi$, and the simplex $\Delta^{K-1}_{\epsilon}$ defined in \eqref{eq:any_discrete_set_tracking}. Instead of $T_0(\varepsilon)$, we define $\tau_{\max} \triangleq \max \{\tau^{(L)}, T_\xi \}$. Noting that for $t > \tau^{(L)}$, we have $a^{\rm IR}(i) \geq \frac{2^{L-1}}{A^{L+1}}$ $\forall i \in [K]$, we may set
\begin{align}
\label{eq:any_m}
    m^\prime \; \triangleq \; \left(\frac{A^{L+1}}{2^{L-1}}-1 \right)\tau_{\max} - \frac{A^{L+1}}{2^{L-1}}\ , 
\end{align}
where the quantity $m^\prime$ is the counterpart of $m$ in Lemma~\ref{lemma:undersampling}. Finally, we define
\begin{align}
\label{eq:any_T_epsilon_actual}
    \tau^{\epsilon} \;\triangleq\; \tau_{\max} + m^\prime \ .
\end{align}
Furthermore, the finiteness of $\tau^{\epsilon}$, on average, may be readily verified. Note that
\begin{align}
     \E_{\bnu}^{\rm IR}[\tau^{\epsilon}] &= \E_{\bnu}^{\rm IR}[ \tau_{\max} + m^\prime] \\
     &\stackrel{\eqref{eq:any_m}}{=}\E_{\bnu}^{\rm IR} \left[ \tau_{\max} + \left(\frac{A^{L+1}}{2^{L-1}}-1 \right)t_{\max} - \frac{A^{L+1}}{2^{L-1}} \right] \\
     &= \left(\frac{A^{L+1}}{2^{L-1}} \right) \E_{\bnu}^{\rm IR}[ \tau_{\max} ] - \frac{A^{L+1}}{2^{L-1}}\\
     &=\left(\frac{A^{L+1}}{2^{L-1}} \right) \E_{\bnu}^{\rm IR}\left[ \max\left\{\tau^{(L)}, T_\xi\right\}\right] - \frac{A^{L+1}}{2^{L-1}}\\
     \label{eq:t_max_finite}
     & < \infty\ ,
\end{align}
where \eqref{eq:t_max_finite} follows from Lemma \ref{lemma:finite_stopping_time} and $T_{\xi} < \infty$. All the remaining steps exactly follow from Lemma~\ref{lemma:undersampling}.

\end{proof}
From Lemma \ref{lemma:any_sampling_bound}, the following upper bound on the sampling error may be readily obtained.
\begin{align}
\label{eq:anytime_error_2}
H_2(T) \stackrel{\eqref{eq:W}}{\leq}\;\mcL W^q\cdot \norm{\ba_T^{\rm IR} - \frac{1}{T}\btau_T^{\rm IR}}_1^q \leq \;\mcL W^q K^{q+1} \frac{1}{T^q} \ .
\end{align}

\paragraph{Regret bound.}Finally, we have
\begin{align}
    \label{eq:anyregret}\mathfrak{R}_{\bnu}^{\rm IR}(T) &= \delta_\tau(T) + H_2(T) \leq \mcL W^q K^q \epsilon^q + \mcL W^q K^{1+q} \frac{1}{T^q} \leq \mcL W^q K^{q+1} \left(\epsilon^r + T^{-q}\right)
\end{align}
where \eqref{eq:anyregret} follows from \eqref{eq:anytime_error_1} and \eqref{eq:anytime_error_2}.

\end{document}